\documentclass[aap]{stylefile}
\RequirePackage{amsthm,amsmath,amsfonts,amssymb}
\RequirePackage[sort,numbers]{natbib}
\RequirePackage{mathtools}
\RequirePackage[shortlabels]{enumitem}
\RequirePackage{algorithm}
\RequirePackage{algorithmic}
\RequirePackage{tikz}
\RequirePackage[colorlinks,citecolor=blue,urlcolor=blue]{hyperref}
\RequirePackage{graphicx}
\usepackage[normalem]{ulem}
\usepackage{bbm}
\usepackage{newtxtext,newtxmath}

\usepackage[normalem]{ulem}

\usepackage{multirow}
\usepackage{makecell, cellspace}
\usepackage{booktabs}
\usepackage[flushleft]{threeparttable}
\usepackage{colortbl}
\usepackage{booktabs} 
\usepackage{ragged2e}
\usepackage{bm}
\usepackage{multicol}
\usepackage{multirow}
\usepackage[makeroom]{cancel}
\usepackage{pifont}

\DeclareMathOperator*{\argmax}{arg\,max}
\DeclareMathOperator*{\argmin}{arg\,min}
\usepackage[capitalize,noabbrev]{cleveref}

\startlocaldefs
\theoremstyle{plain}
\newtheorem{theorem}{Theorem}[section]

\newtheorem{lemma}{Lemma}[section]
\newtheorem{corollary}{Corollary}[theorem]
\newtheorem{proposition}{Proposition}[section]

\theoremstyle{definition}

\newtheorem{assumption}{Assumption}[section]

\theoremstyle{remark}
\newtheorem*{remark}{Remark}

\newcommand{\na}{N}
\newcommand{\sync}{K}
\newcommand{\ta}{\tilde{A}}
\newcommand{\tb}{\tilde{B}}
\newcommand{\zetaone}{\zeta_1}
\newcommand{\zetatwo}{\zeta_2}
\newcommand{\zetathree}{\zeta_3}
\newcommand{\zetafive}{\zeta_4}

\newcommand\norm[1]{\lnr#1\rnr}

\newcommand{\mco}{\mathcal O}

\newcommand{\mbf}{\mathbf}
\newcommand{\mbe}{\mathbb E}

\newcommand{\mbr}{\mathbb R}

\newcommand{\varphiz}{{\varphi}}
\newcommand{\A}{{A_1}}

\newcommand{\bx}{{\mathbf x}}

\newcommand{\bv}{{\mathbf v}}
\newcommand{\bvt}{{\mathbf v_t}}
\newcommand{\bvit}{{\mathbf{v}^i_{t}}}

\newcommand{\bV}{{\mathbf V}}

\newcommand{\btheta}{{\boldsymbol \theta}}
\newcommand{\bthetat}{{\boldsymbol \theta_t}}
\newcommand{\bthetait}{{\boldsymbol{\theta}^i_{t}}}
\newcommand{\bthetajt}{{\boldsymbol{\theta}^j_{t}}}

\newcommand{\bu}{{\mathbf u}}

\newcommand{\Ot}{{\Omega_t}}

\newcommand{\Dit}{{\Delta^i_t}}

\newcommand{\by}{{\mathbf y}}
\newcommand{\byt}{{\mathbf Y_t}}
\newcommand{\byi}{{\mathbf y^i}}

\newcommand{\bY}{{\mathbf Y}}

\newcommand{\byit}{{\mathbf{y}^i_{t}}}

\newcommand{\bw}{{\mathbf w}}

\newcommand{\Gi}{{\mathbf{G}^i}}

\newcommand{\M}{{M}}
\newcommand{\E}{{\mathbb{E}}}

\newcommand{\bmu}{{\boldsymbol\mu}}

\newcommand{\Q}{{Q}}

\newcommand{\Prob}{{\mathcal{P}}}

\newcommand{\G}{\nabla}

\newcommand{\bG}{{\bar{\mbf G}}}
\newcommand{\bGi}{{\bar{\mbf G}^i}}

\newcommand{\I}{\mathfrak{I}}

\newcommand{\lp}{\left(}
\newcommand{\rp}{\right)}

\newcommand{\lcb}{\left\{}
\newcommand{\rcb}{\right\}}
\newcommand{\lb}{\left[}
\newcommand{\rb}{\right]}
\newcommand{\lnr}{\left\|}
\newcommand{\rnr}{\right\|}
\newcommand{\lan}{\left\langle}
\newcommand{\ran}{\right\rangle}

\newcommand{\sumik}{\sum_{i=1}^\na}
\newcommand{\nn}{\nonumber}

\endlocaldefs

\begin{document}

\begin{frontmatter}
\title{Federated Stochastic Approximation under Markov Noise and Heterogeneity: Applications in Reinforcement Learning}
\begin{aug}
\author[A]{\fnms{Sajad} \snm{Khodadadian}\ead[label=e1,mark]{sajadk@vt.edu}}
,\author[B]{\fnms{Pranay} \snm{Sharma}\ead[label=e2,mark]{pranaysh@andrew.cmu.edu}} 
\author[B]{\fnms{Gauri} \snm{Joshi}\ead[label=e3,mark]{gaurij@andrew.cmu.edu}}
\and
\author[C]{\fnms{Siva Theja} \snm{Maguluri}\ead[label=e4,mark]{siva.theja@gatech.edu}}
\address[A]{Virginia Polytechnic Institute and State University,
\printead{e1}}
\address[B]{
Carnegie Mellon University,
\printead{e2,e3}}
\address[C]{
Georgia Institute of Technology,
\printead{e4}}
\end{aug}

\begin{abstract}
Since reinforcement learning algorithms are notoriously data-intensive, the task of sampling observations from the environment is usually split across multiple agents. However, transferring these observations from the agents to a central location can be prohibitively expensive in terms of communication cost, and it can also compromise the privacy of each agent's local behavior policy. Federated reinforcement learning is a framework in which $N$ \emph{ agents collaboratively learn a global model, without sharing their individual data and policies}. This global model is the unique fixed point of the average of $N$ local operators, corresponding to the $N$ agents. Each agent maintains a local copy of the global model and updates it using locally sampled data. In this paper, we show that by careful collaboration of the agents in solving this joint fixed point problem, we can find the global model $N$ times faster, also known as linear speedup. We first propose a general framework for federated stochastic approximation with Markovian noise and heterogeneity, showing linear speedup in convergence. We then apply this framework to federated reinforcement learning algorithms, examining the convergence of federated on-policy TD, off-policy TD, and $Q$-learning.
\end{abstract}
\end{frontmatter}

\section{Introduction}\label{sec:intro}

Stochastic Approximation (SA) \cite{robbinsmonroSA} is a fundamental concept in the field of optimization and computational mathematics. It refers to a family of iterative methods used to find the roots of functions when the function itself cannot be directly observed but can be estimated via noisy measurements \cite{NDP_book}. This methodology has been pivotal in the analysis \cite{moulines2011non} and modeling \cite{jumper2021highly} of various algorithms such as stochastic gradient descent \cite{chau2014overview}. The significance of stochastic approximation lies in its ability to handle noise and uncertainty in the data, making it a robust tool for applications in machine learning, signal processing, and econometrics, where exact function evaluations are often infeasible \cite{harold1997stochastic, benveniste2012adaptive}.

One prominent application of SA is in Reinforcement Learning (RL), a domain of machine learning where agents learn to make decisions by interacting with an environment \cite{suttonbartorlbook, szepesvari2022algorithms}. Core RL algorithms, such as Temporal Difference (TD)-learning and $Q$-learning, can be modeled and studied as special cases of SA \cite{sutton1988learning, NDP_book}. TD-learning updates value estimates based on the difference between consecutive estimates, while Q-learning improves the action-value function estimates. These methods rely on stochastic approximation techniques to converge to optimal policies, leveraging the iterative nature of these methods to handle inherent uncertainty and variability in the learning environment \cite{watkins1992q}.

Building on the foundations of SA in RL, there has been significant interest in federated learning (FL), a paradigm that allows multiple decentralized agents to collaboratively train a model while keeping their data locally stored \cite{kairouz2021advances}. This approach has been extensively studied in the context of supervised learning \cite{wang21coopSGD_jmlr, stich18localSGD_iclr, qu2020federated, li2019convergence}, offering solutions that preserve data privacy and reduce communication overhead \cite{fedavg17aistats, kairouz2021advances, qi2021federated, yang2019federated}. Recently, some work has been done on the application of FL to RL problems (also known as FRL) \cite{xu2021multi, zhang2022multi, pinto2023federated,zhou2023digital,shaik2024framu}. Studies on federated $Q$-learning \cite{woo2023blessing,fan2023fedhql}, and federated TD-learning \cite{wai20dist_markov_noLS_cdc, hong20asynch_a3c_arxiv, wang2023federated, wang2024momentum} investigate how decentralized agents can collectively solve RL problems without sharing their raw data, thus maintaining privacy and improving scalability \cite{zhuo19FedDeepRL_arxiv}. These investigations represent significant progress in applying federated principles to more complex and dynamic learning environments \cite{qi2021federated}.

The goal of this paper is to develop and study federated stochastic approximation with Markovian noise and heterogeneity. \emph{We first consider a general heterogeneous federated stochastic approximation and we establish its convergence bound. Next, we propose communication-efficient federated versions of on-policy TD, off-policy TD, and $Q$-learning algorithms. In addition, we characterize the convergence and sample complexity of these algorithms employing the convergence result of the general federated stochastic approximation. } We aim to demonstrate that our approach achieves linear speedup with respect to the number of agents involved. By addressing the challenges of heterogeneity among agents, our results can be applied to a wide range of off-policy and on-policy algorithms with function approximation. This contribution is expected to bridge the gap between stochastic approximation theory and practical federated RL applications, providing a robust framework for future research and development in the field \cite{smith20FL_SPmag, jin2022federated}.

The main contributions and organization of the paper are summarized in the following.
\begin{itemize}
    \item In Section \ref{sec:fed_sam}, we propose and study a general \underline{Fe}derated Hetero\underline{G}eneous \underline{S}tochastic \underline{A}pproximation framework with \underline{M}arkovian noise (FeGSAM). Considering Markovian sampling noise and heterogeneity poses a significant challenge in the analysis of this algorithm. The convergence result for FeGSAM serves as a workhorse that enables us to analyze both federated TD-learning and federated $Q$-learning. We characterize the convergence of FeGSAM with a refined analysis of general stochastic approximation algorithms, fundamentally improving on prior work.
    \item In the on-policy setting, in Section \ref{sec:OPFRL} we propose and analyze federated TD-learning with linear function approximation, where the agents' goal is to evaluate a common policy using on-policy samples collected in parallel from their environments. The agents only share the updated value function (not data) with the central server, thus saving communication cost. We prove a linear convergence speedup with the number of agents and also characterize the impact of communication frequency on the convergence.
    \item In the tabular off-policy setting, in Section \ref{sec:OfPFRLTR} we propose and analyze the federated off-policy TD-learning and federated $Q$-learning algorithms. Again, we establish a linear speedup in their convergence with respect to the number of agents, along with constant communication cost. Since every agent samples data using a private policy and only communicates the updated value or $Q$-function, the off-policy FRL helps to keep both the data and the policy private.
    \item In the off-policy setting with function approximation, in Section \ref{sec:OfPFRLFA} we propose and analyze the federated off-policy TD-learning with linear function approximation. Considering heterogeneity in the analysis of FeGSAM is essential to establish the sample complexity of this algorithm, and show a linear speedup with respect to the number of agents. However, we will observe that this setting requires more than a constant (error-dependent) number of communications.
    \item We finally provide a sketch of the proof of the convergence of FeGSAM in Section \ref{sec:pf_sketch}.
\end{itemize}
 
\section{Related Work}
  
In the related work section, we merge the results on stochastic approximation with their corresponding applications in RL.
  
\textbf{Single node TD-learning and $\mbf{\textit{Q}}$-learning.} Most existing RL literature is focused on designing and analyzing algorithms that run at a single computing node. In the on-policy setting, the asymptotic convergence of TD-learning was established in \citep{tsitsiklis1997analysis, tadic2001convergence, borkar2009stochastic}, and the finite-sample bounds were studied in \citep{dalal2018finite, csaba_iidnoise, bhandari18FiniteTD_LFA_colt, srikant2019finite, hu2019characterizing, chen2021Lyapunov_arxiv}. In the off-policy setting, \cite{maei2018convergent,zhang2019provably} study the asymptotic and \cite{chen2020finite,chen2021Lyapunov_arxiv} characterize the finite time bound of TD-learning. The $Q$-learning algorithm was first proposed in \cite{watkins1992q}. There has been a long line of work to establish the convergence properties of $Q$-learning. In particular, \cite{tsitsiklis1994asynchronous, jaakkola1994convergence, bertsekas1996neuro,borkar2000ode, borkar2009stochastic} characterize the asymptotic convergence of $Q$-learning, \cite{beck2012error, beck2013improved, wainwright2019stochastic, chen2020finite, chen2021Lyapunov_arxiv} study the finite-sample convergence bound in the mean-square sense, and \cite{Even_Dar_Qlearning, li2020sample, qu2020finite} study the high-probability convergence bounds of $Q$-learning.
  
\textbf{Federated Learning with i.i.d. Noise.} When multiple agents are used to expedite sample collection, transferring the samples to a central server for the purpose of training can be costly in applications with high-dimensional data \cite{shao2019survey} and it may also compromise the agents' privacy. Federated Learning is an emerging distributed optimization paradigm \cite{konevcny16federated, kairouz2019advances} that utilizes local computation at the agents to train models, such that only model updates, not data, is shared with the central server. In local Stochastic Gradient Descent (Local SGD or FedAvg) \cite{fedavg17aistats, stich18localSGD_iclr}, the core algorithm in FL, locally trained models are periodically averaged by the central server to achieve consensus among agents at a reduced communication cost. Although the convergence of local SGD has been extensively studied in prior work \cite{richtarik20localSGD_aistats, spiridonoff21comm_eff_SGD_neurips, qu2020federated, koloskova20unified_localSGD_icml}, these works assume i.i.d. noise in the gradients, which is acceptable for SGD but too restrictive for RL algorithms. 

\textbf{Distributed and Multi-agent RL.} Some previous works have analyzed distributed and multi-agent RL algorithms in the presence of Markovian noise in various settings such as decentralized stochastic approximation \cite{doan_dist_td_icml, sun20decenTD_LFA_aistats, wai20dist_markov_noLS_cdc, doan20decen_sa_noLS_arxiv}, 
TD learning with linear function approximation \cite{wang20decen_TD_neurips}, and
off-policy TD in actor-critic algorithms \cite{chen2021multi, chen21sample_A3C_arxiv}. However, all these works consider decentralized settings, where agents communicate with their neighbors after every local update. In \cite{hong20asynch_a3c_arxiv}, a parallel implementation of the asynchronous advantage actor-critic (A3C) algorithm (which does not have local updates) has been proposed under both i.i.d. and Markov sampling. However, the authors prove a linear speedup only for the i.i.d. case, and an almost linear speedup is observed experimentally for the Markovian case. 

More recently, \cite{lan2024asynchronous} introduced an asynchronous FRL framework with policy gradient updates, offering convergence guarantees. \cite{ganesh2024global} provided a global convergence guarantees for federated policy gradient methods in adversarial settings. \cite{wang2024momentum} emphasized the role of momentum in collaborative FRL, addressing heterogeneity across environments. The work by \cite{sun2024understanding} explored the generalization of federated learning by focusing on stability and heterogeneity's influence. \cite{fan2023fedhql} and \cite{woo2023blessing} investigated federated Q-learning, with \cite{fan2023fedhql} proposing a heterogeneous federated Q-learning framework (FedHQL) and \cite{woo2023blessing} demonstrating the benefits of heterogeneity. \cite{yang2023federated} focused on federated natural policy gradient methods for multi-task reinforcement learning, while \cite{jin2022federated} addressed the challenges of environment heterogeneity in federated reinforcement learning. 

\section{Heterogeneous Federated Stochastic Approximation with Markovian Noise}
In this section we will introduce our heterogeneous federated stochastic approximation with Markovian noise and establish the sample complexity of this algorithm. However, to lay the groundwork, we will first discuss SA in the single-node setting.

\subsection{Single node}
While generic stochastic approximation algorithms are studied under i.i.d. noise \cite{Even_Dar_Qlearning,shah2018q,wainwright2019stochastic,liu2015finite,dalal2018finite}, to apply them to studying RL we need to understand stochastic approximation under Markovian noise \cite{tsitsiklis1994asynchronous,qu2020finite,chen2021Lyapunov_arxiv,  chen2023lyapunov} which is significantly more challenging.

Consider a generic SA (i.i.d. or Markovian noise) with constant step size $\alpha$ and parameter vector $\bx_T$ as follows
\begin{align}\label{eq:SA_main}
    \bx_{t+1} = (1-\alpha)\bx_t + \alpha (G(\bx_t) + \bw_t).
\end{align}
Here, $G(\cdot)$ is the function, the fixed point of which we would like to discover, and $\bw_t$ represent the noise in observing the $G(\cdot)$ function. A wide range of algorithms, from stochastic gradient descent \cite{moulines2011non} to RL \cite{NDP_book}, can be studied as special cases of \eqref{eq:SA_main}. For example, stochastic gradient descent to minimize the function $f(\cdot)$, follows the following iteration: $\bx_{t+1} = (1-\alpha)\bx_t + \alpha (\bx_t-\nabla f(\bx_t)+ \bw_t)$, where $\bw_t$ is the noise in estimating the gradient function.

Assuming that $G(\cdot)$ has $\bx^*$ as its fixed point (that is, $G(\bx^*)=\bx^*$), it can be shown that the iterations in \ref{eq:SA_main} have the following convergence behavior.
\begin{align}\label{eq:general_stoch_conver}
    \E[\|\bx_T-\bx^*\|^2]\leq \mathcal{C}_1(1-\mathcal{C}_0\alpha)^{T}+\mathcal{C}_2\alpha,
\end{align}
where $\mathcal{C}_0, \mathcal{C}_1,$ and $\mathcal{C}_2$ are some problem dependent positive constants (Look at Appendix \ref{sec:lower_bound_stoch_app} for a discussion on a lower bound on the convergence of general SA). The first term is denoted as the bias and the second term is called the variance. According to this bound, $\bx_T$ geometrically converges to a ball around $\bx^*$ with radius proportional to $\mathcal{C}_2\alpha$. Notice that we can always reduce the variance term by reducing the step size $\alpha$, but this will lead to slower convergence in the bias term. In particular, to get $\E[\|\bx_T-\bx^*\|^2]\leq\epsilon$, it is easy to see that we need $T\geq \mathcal{O}\lp\frac{\mathcal{C}_2}{\epsilon}\log\frac{1}{\epsilon}\rp$ samples. Now, suppose that the constant $\mathcal{C}_2$ is large. In this case, the variance term in the bound in \eqref{eq:general_stoch_conver} is large, and the sample complexity, which is proportional to $\mathcal{C}_2$ will be poor. Notice that by the discussion in Appendix \ref{sec:lower_bound_stoch_app}, this bound is tight and cannot be improved.

This is where FL can be employed to control the variance term by generating more data. An example of SA is TD-learning, which will be discussed in detail in Sections \ref{sec:OPFRL}, \ref{sec:OfPFRLTR}, and \ref{sec:OfPFRLFA}. In federated TD-learning, multiple agents work together to evaluate the value function simultaneously. Due to this collaboration, agents can estimate the true value function with a lower variance. The same holds for estimating the optimal $Q$-function in federated $Q$-learning, which will be discussed in Section \ref{sec:Q_learning_main}.

\subsection{Generalized Heterogeneous Federated Stochastic Approximation} \label{sec:fed_sam}
In this section, we study the convergence of a general heterogeneous federated stochastic approximation for contractive operators, FeGSAM, which is presented in Algorithm \ref{alg:fed_stoch_app}. In this algorithm, there are $\na$ agents $i=1,2,\dots,\na$, which collectively aim to find the fixed point $\btheta^*$ of the function $\frac{1}{\na}\sumik G^i(\cdot)$ that satisfies $\frac{1}{\na}\sumik G^i(\btheta^*)=\btheta^*$. In this paper, we consider heterogeneity, which means that each function $G^i(\cdot)$ can have a different fixed point. In the special case where the fixed point of all agents coincides, we have a homogeneous federated SA, which was studied in \cite{khodadadian2022federated}. The heterogeneity assumption is essential when studying off-policy TD-learning with linear function approximation, which will be discussed in Section \ref{sec:OfPFRLFA}.

At each time step $t\geq 0$, each agent $i$ maintains the parameter $\bthetait\in\mathbb{R}^d$. 
At time $t=0$, all agents initialize their parameters with $\btheta_0^i=\btheta_0$. Next, at time $t\geq0$, each agent $i$ updates its parameter to $\btheta^i_{t+1} = \bthetait + \alpha \lp \Gi(\bthetait, \byit) - \bthetait + \mathbf{b}^i (\byit) \rp$. Here, $\alpha$ denotes the step size and $\byit$ is a Markovian noise throughout time $t$, but independent of all agents $j\neq i$. This notion is defined more concretely in Assumption \ref{ass:noise_independence}. 
We note that functions $\Gi(\cdot, \cdot)$ and $\mathbf{b}^i (\cdot)$ are allowed to be dependent on the agent $i$. This allows us to employ the convergence bound of FeGSAM in order to derive the convergence bound of off-policy TD-learning with different behavior policies across agents. To avoid divergence, for every $\sync$ time step, we synchronize the parameters of all agents as $\bthetait \leftarrow \bthetat \triangleq \frac{1}{\na} \sum_{j=1}^\na \btheta_t^j,$ for all $i \in [\na]$. Note that although smaller $\sync$ corresponds to more frequent synchronization and hence more ``accurate'' updates, at the same time it results in a higher communication cost, which is not desirable. 
Hence, to determine the optimal choice of the synchronization period, it is essential to characterize the dependence of the convergence on $\sync$. This is one of the results we will derive in Theorem \ref{thm:main}. Finally, the algorithm samples $\hat{T}\sim q_T^c$, where $q^c_T(t) = \frac{c^{-t}}{\sum_{t'=0}^{T-1}c^{-t'}}$ and outputs $\btheta_{\hat{T}}$. This sampling scheme is essential for the convergence of the overall algorithm. We also make some assumptions regarding the underlying process.

 \begin{algorithm}
\caption{Federated Heterogeneous Stochastic Approximation with Markovian Noise (FeGSAM)}
\label{alg1}
\begin{algorithmic}[1]\label{alg:fed_stoch_app}
	\STATE{\textbf{Input:} $c_{FSAM},T,\btheta_{0},K,\alpha$}
	\STATE $\btheta^i_{0}=\btheta_{0}$ for all $i=1,\dots, \na$.
	\FOR{$t=0$ to $T-1$}
    	\STATE{$\btheta^i_{t+1} = \bthetait + \alpha \lp \Gi(\bthetait, \byit) - \bthetait + \mathbf{b}^i (\byit) \rp, \forall \ i \in [\na]$}
	    \IF{$t+1$ mod $\sync = 0$}
	        \STATE{$\btheta^i_{t+1} = \btheta_{t+1} \triangleq \frac{1}{\na} \sum_{j=1}^\na \btheta^j_{t+1}, \forall \ i \in [\na]$}
	   \ENDIF
	\ENDFOR
	\STATE Sample $\hat{T}\sim q_T^{c_{FSAM}}(\cdot)$.
	\STATE{\textbf{Return:}} $\frac{1}{\na}\sumik\btheta^i_{\hat{T}}$
\end{algorithmic}
\end{algorithm}

First, we consider assumption \ref{ass:mixing} where we assume that the expectation of $\Gi (\btheta,\by_t^i)$ geometrically converges to some function $\bG^i(\btheta)$ and the expectation of $\mathbf{b}^i(\by^i_t)$ converges to $\bar{\mathbf{b}}^i$ as $t\rightarrow \infty$, and the fixed point of $\bG^i(\cdot) + \bar{\mathbf{b}}^i$ is denoted as $\btheta^{i,*}$ satisfies $\bG^i(\btheta^{i,*})+ \bar{\mathbf{b}}^i=\btheta^{i,*}$. Note that the fixed points $\btheta^{i,*}$ can vary by agent $i$. This is the source of heterogeneity in our model. In addition, the expectation of $\mathbf{b}^i  (\by_t^i)$ is bounded by a geometrically converging term plus $m_3\eta_i$, where $\eta_i=\|\btheta^{i,*} -\btheta^*\|_c$, and $\|\cdot\|_c$ is a fixed norm. Note that in the homogeneous setting studied in \cite{khodadadian2022federated}, we have $\eta_i=0$ for all $i$. 

\begin{assumption} \label{ass:mixing}
For every agent $i$, there exist a function $\bG^i(\btheta)$ such that we have
\begin{align}
    \lim_{t \to \infty} \E[ \Gi (\btheta,\by_t^i) ] &= \bG^i(\btheta)\nn.
\end{align}
Furthermore, there exists $m_1,m_2,m_3,m_4\geq0$ and $\rho \in [0,1)$, such that for every $i=1,2,\dots, \na$,
\begin{equation}
    \begin{aligned}
        \| \bG^i(\btheta) - \E[\Gi (\btheta, \by_t^i)] \|_c & \leq m_1  \| \btheta \|_c \rho^t \\
        \| \E[\mathbf{b}^i  (\by_t^i)] \|_c & \leq m_2 \rho^t + m_3\eta_i\\
        \left\| \E\left[\frac{1}{\na}\sumik\mathbf{b}^i  (\by_t^i)\right] \right\|_c & \leq m_4 \rho^t,
    \end{aligned}
    \label{eq:b_ass}
\end{equation}
where $\|\cdot\|_c$ is a given norm.
\end{assumption}
We further denote $\varepsilon = \frac{1}{N}\sumik\eta_i$ which characterizes the level of heterogeneity in the system. 

Next, we assume a contraction property on the expected operator $\bG^i(\btheta)$. 

\begin{assumption}
\label{ass:contraction}
We assume all expected operators $\bG^i(\btheta) $ are contraction mappings with respect to $\| \cdot \|_c$ norm with contraction factor $\gamma_c \in (0,1)$. That is, for all $i=1,2,\dots,\na$,
\begin{align*}
    \|\bG^i (\btheta_1) - \bG^i(\btheta_2) \|_c \leq \gamma_c \|\btheta_1 - \btheta_2 \|_c, \quad \text{for all } \btheta_1, \btheta_2 \in \mbr^d.
\end{align*}
\end{assumption}

Next, we consider some Lipschitz and boundedness properties on $\Gi (\cdot, \cdot)$ and $\mathbf{b}^i (\cdot)$.

\begin{assumption}
\label{ass:lipschitz}
For all $i=1,\dots,\na$, there exist constants $A_1$, $A_2$ and $B$ such that 
\begin{enumerate}
    \item $\| \Gi (\btheta_1, \by) - \Gi (\btheta_2, \by) \|_c\leq A_1 \| \btheta_1 - \btheta_2 \|_c$ for all $\btheta_1, \btheta_2, \by$.
    \item $\| \Gi (\btheta, \by) \|_c \leq A_2\|\btheta\|_c$ for all $\btheta,\by$.
    \item $\|\mathbf{b}^i (\by^i)\|_c\leq B$ for all $\byi$.
\end{enumerate}
\end{assumption}

Finally, we impose an assumption on the random noise $\byit$.
\begin{assumption}\label{ass:noise_independence}
We assume that the Markovian noise $\byit$ (Markovian with respect to time $t$) is independent across agents $i$. In other words, for all measurable functions $f(\cdot)$ and $g(\cdot)$, we assume the following
\begin{align*}
    \E_{t-r} [f(\by_t^i)\times g(\by_t^j)] = \E_{t-r} [f(\by_t^i)]\times \E_{t-r}[g(\by_t^j)]\quad \text{for all} \quad r \leq t, i\neq j.
\end{align*}
\end{assumption}

\subsection{Main result}
Theorem \ref{thm:main_body} states the convergence of Algorithm \ref{alg:fed_stoch_app}, which is the main result of this paper.

\begin{theorem}\label{thm:main_body}
Consider the federated heterogeneous stochastic approximation Algorithm \ref{alg:fed_stoch_app} with $c_{FSAM}= 1 - \frac{\alpha \varphiz_2}{2}\in (0,1)$ ($\varphiz_2$ is defined in Equation \eqref{eq:def:constants} in the appendix), and synchronization frequency $\sync$. Denote $\bthetat=\frac{1}{\na}\sumik\bthetait$, and consider $\btheta_{\hat{T}}$ as the output of this algorithm after $T$ iterations. Assume $\tau_\alpha=\lceil2\log_\rho\frac{1}{\alpha}\rceil$. For $T\geq \max\{K+\tau_\alpha,2\tau_\alpha\}$ and small enough step size $\alpha$, we have
\begin{align}\label{eq:fed_stoch_app_bound}
    \E[\|\btheta_{\hat{T}} - \btheta^*\|^2_c]\leq \mathcal{C}_1\frac{1}{\alpha}c_{FSAM}^{T-2\tau_\alpha+1}+\mathcal{C}_2\frac{\alpha\tau_\alpha^2}{\na}+\mathcal{C}_{3}(\sync-1)\alpha^2\tau_\alpha+\mathcal{C}_4(K-1)^2\varepsilon\alpha^2 \tau_\alpha + \mathcal{C}_5\alpha^3\tau_\alpha^2,
\end{align}
where
$\mathcal{C}_i$, $i=1,2,3,4, 5$ are some constants which are specified precisely in Appendix \ref{sec:fed_stoch_analysis_app}, and are independent of $\sync, \alpha, \na$. 
\end{theorem}

Theorem \ref{thm:main_body} establishes the convergence of $\btheta_{\hat{T}}$ to $\btheta^*$ in the expected mean-squared sense. The first term in \eqref{eq:fed_stoch_app_bound} converges geometrically to zero as $T$ grows. The second term is proportional to $\alpha$ similar to \eqref{eq:general_stoch_conver}. However, the number of agents $\na$ in the denominator ensures \textit{linear speedup}, meaning that for small enough $\alpha$ (such that $\alpha/\na$ is the dominant term), the sample complexity of each individual agent, relative to a centralized system, is reduced by a factor of $\na$. The third term has quadratic dependence on $\alpha$, and is zero when we have perfect synchronization, that is, $\sync = 1$. The fourth term is quadratic in $K-1$ and the step size $\alpha$. This term vanishes when we have complete homogeneity among the agents. 

The last term is proportional to $\alpha^3$, and has the weakest dependency on the step size $\alpha$. For $\sync > 1$, we can merge this term with the third and fourth terms by the upper bound $\alpha^3\leq \alpha^2$. The current upper bound, however, is tighter since with $K=1$ (i.e. perfect synchronization) we have no term of the order $\alpha^2$. Note that similar bounds have been established for the simpler federated setting with i.i.d. noise \cite{richtarik20localSGD_aistats, koloskova20unified_localSGD_icml} and in the Markovian noise \cite{khodadadian2022federated}. Here we achieved this result for the more general heterogeneous federated setting with Markov noise.

\begin{remark}
The bound in Theorem \ref{thm:main_body} is valid only after $T>\max\{K+\tau_\alpha,2\tau_\alpha\}$ and for all synchronization periods $K\geq1$. At $K=1$ the third and fourth terms in the bound disappear, and we will be left only with the first-order term, which is linearly decreasing with respect to the number of agents $\na$, and the third-order term $\tilde{\mathcal{O}}(\alpha^3)$. The last term, however, is not tight and can be further improved to be of the order $\mathcal{O}(\alpha^j), j>3$. But for that, we need to assume a larger $\tau_\alpha$, which means that the bound only holds after a longer waiting time. In particular, by choosing $\tau_\alpha=\lceil r\log_\rho\alpha\rceil$, we can get $\tilde{\mathcal{O}}(\alpha^{2r-1})$ for the last term (see the proof of Lemma \ref{lem:T_2}). 
\end{remark}

Next, in Corollary \ref{cor:sample_complexity_main} we establish the sample complexity and the communication complexity of Algorithm \ref{alg:fed_stoch_app} in homogeneous and heterogeneous settings.

\begin{corollary}\label{cor:sample_complexity_main}
    \begin{enumerate}
        \item In the homogeneous setting, where $\varepsilon = 0$, choosing $\alpha=\frac{8\log(\na T)}{T\varphiz_2}$ and $\sync=T/\na$, we get $T=\tilde{\mathcal{O}}\lp\frac{1}{\na \epsilon}\rp$ sample complexity for achieving $\E[\|\btheta_{\hat{T}}\|^2_c]\leq \epsilon$. Furthermore, we get $T/K=N$ communication complexity, which is independent of the error $\epsilon$.
        \item In the heterogeneous setting, where $\varepsilon > 0$, choosing $\alpha=\frac{8\log(\na T)}{T\varphiz_2}$ and $\sync=\sqrt{T}/\na$, we get $T=\tilde{\mathcal{O}}\lp\frac{1}{\na \epsilon}\rp$ sample complexity for achieving $\E[\|\btheta_{\hat{T}}\|^2_c]\leq \epsilon$. Furthermore, we get $T/K=\sqrt{N}/\sqrt{\epsilon}$ communication complexity. The communication complexity in this case is dependent on the error $\epsilon$.
    \end{enumerate}
\end{corollary}

The core idea that allows us to obtain a tight bound is as follows. Suppose the random variable $X$ has bounded variance $\sigma^2$, and suppose $X_i, i=1,\dots,\na$ are i.i.d. samples of this random variable. It is easy to show that the variance of $\bar{X}=\sumik X_i/\na$ is $\sigma^2/\na$, and this is how we can attain the linear speedup in the i.i.d. and homogeneous setting. However, with Markovian noise and heterogeneity among the agents, the problem is significantly more challenging, and the existing work fails to achieve linear speedup. In the homogeneous setting, where $\varepsilon = 0$, \cite{khodadadian2022federated} has established the sample and communication complexity of Corollary \ref{cor:sample_complexity_main}. However, the more general heterogeneous setting requires a more refined analysis (see Proposition \ref{prop:wtd_consensus_error}).

\section{Single agent Reinforcement Learning}

Next, we discuss and explain the underlying setting of RL. We model our RL setting with an MDP with 5 tuples $(\mathcal{S}, \mathcal{A}, \mathcal{P}, \mathcal{R}, \gamma)$, where $\mathcal{S}$ and $\mathcal{A}$ are finite sets of states and actions, $\mathcal{P}$ is the set of transition probabilities, $\mathcal{R}$ is the reward function and $\gamma\in(0,1)$ denotes the discount factor. At each time step $t$, the system is in some state $S_t$, and the agent takes some action $A_t\sim \pi(\cdot|S_t)$, where $\pi$ is the policy at hand. Given the state $S_t$ and the chosen action $A_t$, the agent receives the reward $\mathcal{R}(S_t,A_t)$. In the next time step, the system transitions to a new state $S_{t+1}$ according to the state transition probability $\mathcal{P} (\cdot|S_t, A_t)$. This series of states and actions $(S_t,A_t)_{t\geq 0}$ constructs a Markov chain, which is the source of Markovian noise in RL. Throughout this paper, we assume that this Markov chain is irreducible and aperiodic (also known as ergodic). It is known that this Markov chain asymptotically converges to a steady state, and we denote its stationary distribution by $\mu^\pi$.

To measure the long-term reward achieved by following policy $\pi$, we define the value function
\begin{align}\label{eq:value_func}
V^\pi(s) &= \E \left[\sum_{t=0}^{\infty} \gamma^t \mathcal{R} (S_t, A_t)|S_0=s, A_t \sim \pi (\cdot|S_t) \right].
\end{align}
Equation \eqref{eq:value_func} is the tabular representation of the value function. However, sometimes the size of the state space $|\mathcal{S}|$ is large and storing $V^\pi(s)$ for all $s \in \mathcal{S}$ is computationally infeasible. Hence, a low-dimensional vector $\bv^\pi \in \mathbb{R}^d$, where $d \ll |\mathcal{S}|$, can be used to approximate the value function as $V^\pi(s) \approx \phi(s)^\top \bv^\pi$ \cite{tsitsiklis1997analysis}. Here, $\phi(s) \in \mathbb{R}^d$ is a given feature vector corresponding to the state $s$. Using a low-dimensional vector $\bv^\pi$ to approximate a high-dimensional vector $(V^\pi(s))_{s\in\mathcal{S}}$ is referred to as the function approximation paradigm in RL. The task of evaluating this value function is denoted as the Temporal Difference (TD)-Learning. For each pair $(s,a)$, we also define the $Q$-function, $Q^\pi(s,a) = \E[\sum_{t=0}^{\infty} \gamma^t \mathcal{R} (S_t,A_t)|S_0 = s, A_0 = a, A_t\sim \pi(\cdot|S_t)]$, which will be used in $Q$-learning. Both TD-learning and $Q$-learning can be seen as variants of stochastic approximation \cite{tsitsiklis1994asynchronous, chen2020finite_nips, chen_offpolicyTD_arxiv, zaiwei_q_arxiv, chen2023lyapunov}. In sections \ref{sec:TD_learning_main} and \ref{sec:Q-learning_main} we will discuss TD-learning and $Q$-learning in more detail.

\subsection{Temporal Difference Learning}\label{sec:TD_learning_main}
An intermediate goal in RL is to estimate the value function (either $(V^\pi(s))_{s \in\mathcal{S}}$ or $\bv^\pi$) corresponding to a particular policy $\pi$ using data collected from the environment.
This task is denoted as \textit{policy evaluation} and one of the commonly used approaches to accomplish this is TD-learning \cite{sutton1988learning}. TD-learning is an iterative algorithm where the elements of a $d$ (or $|\mathcal{S}|$, in the tabular setting) dimensional vector is updated until it converges to  $\bv^\pi$ (or $V^\pi$). This evaluated value function can be employed in different RL algorithms such as actor-critic \cite{konda_tsitsiklis_actor_critic}. 

In online TD-learning with linear function approximation, we consider a full rank feature matrix $\Phi\in |\mathcal{S}|\times d$, and we denote the $s$-th row of this matrix with $\phi(s), s\in\mathcal{S}$. The goal is to find $\bv^\pi\in\mathbb{R}^d$ which solves the following fixed point equation:
\begin{align}
\label{eq:BE_FA}
    \Phi \bv^\pi = \Pi_\Phi((\mathcal{T}^\pi)^n \Phi \bv^\pi). 
\end{align}
In equation \eqref{eq:BE_FA}, $\Pi_\Phi(\cdot)$ is the projection with respect to the weighted 2-norm, that is, $\Pi_\Phi(\bV)=\argmin_{\mbf v\in\mathbb{R}^d}\|\Phi\bv - \bV\|_\pi$. Here, $\|\bV\|_\pi = \sqrt{\bV^\top \bmu^\pi\bV}$ and $\bmu^\pi$ is a diagonal matrix with diagonal entries corresponding to $\mu^\pi$. In equation \eqref{eq:BE_FA}, $(\mathcal{T}^\pi)^n$ denotes the $n$-step Bellman operator \cite{tsitsiklis1997analysis}. It is known \cite{tsitsiklis1997analysis} that equation \eqref{eq:BE_FA} has a unique solution $\bv^\pi$, and $(\Phi \bv^\pi(s))_{(s\in\mathcal{S})}$ is ``close'' to the true value function $(V^\pi(s))_{s\in \mathcal{S}}$. The update of the $n$-step TD-learning is as follows
\begin{align}\label{eq:TD_update}
   \begin{split}
    \text{Sample}&\quad A_{t+n}\sim\pi(\cdot|S_{t+n}), S_{t+n+1}\sim \Prob( \cdot|S_{t+n}, A_{t+n})\\
    \text{Update}& \quad\bv_{t+1} =\bvt+\alpha \phi(S_t) \sum_{l=t}^{t+n-1}\gamma^{l-t}(\mathcal{R}(S_l,A_l) +\!\gamma \phi (S_{l\!+\!1})^\top \bvt\!-\!\phi(S_l)^\top \bvt),
   \end{split}
\end{align}
where $\alpha$ is the step size. The $n$-step TD-learning is an iterative algorithm to obtain $\bv^\pi$ using samples from the environment. Note that in this algorithm, states and actions are sampled over a single trajectory, and hence the noise in updating $\bv_t$ is Markovian. Furthermore, since the policy which samples the actions and the evaluating policy are both $\pi$, this algorithm is on-policy. As described in \cite{tsitsiklis1997analysis,NDP_book}, the TD-learning algorithm can be studied under the umbrella of linear stochastic approximation with Markovian noise. More recently, the authors of \cite{bhandari18FiniteTD_LFA_colt, srikant2019finite} have shown that the update parameter of TD-learning $\bv_t$ converges to $\bv^\pi$ in the form $\E[\|\bv_t-\bv^\pi\|^2_2]\leq \mathcal{O}((1-\mathcal{C}_0\alpha)^{t}+\alpha)$. In Section \ref{sec:OPFRL}, we show how FRL can improve this result. 

As mentioned earlier, in this section, we studied the on-policy setting, where the evaluating policy and the sampling policy coincide. In contrast, in the off-policy setting, these two policies can in general differ, and we need to account for this difference while running the algorithm. We will further expand on off-policy TD-learning and its variants in Sections \ref{sec:OfPFRLTR} and \ref{sec:OfPFRLFA}.

\subsection{Control Problem and \texorpdfstring{$Q$}--learning} \label{sec:Q-learning_main}
Assuming some initial distribution $\xi$ on the state space, the average value function corresponding to policy $\pi$ is defined as $V^\pi (\xi) = \E_{s \sim \xi}[V^\pi(s)]$. 
This scalar quantity is a metric of the average long-term rewards achieved by the agent, when it starts from distribution $\xi$ and follows policy $\pi$. The ultimate goal of the agent is to obtain an optimal policy $\pi^*$ that results in the maximum long-term rewards, i.e. $\pi^* \in \argmax_\pi V^\pi(\xi)$. 
Throughout the paper, we denote the parameters corresponding to the optimal policy with $^*$, e.g., $V^{\pi^*}(\xi) \equiv V^*(\xi)$. The task of obtaining the optimal policy in RL is denoted as the \textit{control problem}. 

$Q$-learning \cite{watkins1992q, tsitsiklis1994asynchronous} is one of the most widely used algorithms in RL to solve the control problem. The goal of $Q$-learning is to evaluate $Q^*$. Knowing $Q^*$, one can obtain an optimal policy through greedy selection \cite{puterman2014Markov}, and hence resolve the control problem. Suppose $\{S_t,A_t\}_{t\geq 0}$ is generated by a fixed behavior policy $\pi_b$. At each time step $t$, $Q$-learning preserves a $|\mathcal{S}|.|\mathcal{A}|$ dimensional table $Q_t$ and updates the elements of this table as
\begin{align}\label{eq:Q_learning_update}
    \begin{split}
    \Q_{t+1}(s,a) = \begin{cases}
    	    \Q_t(S_t, A_t) + \alpha \lp \mathcal R(S_t, A_t) + \gamma \max_a \Q_t (S_{t+1}, a) - \Q_t (S_{t}, A_{t}) \rp & \text{ if } (s,a) = (S_t, A_t) \\
    	    \Q_t(s, a) & \text{ o.w.}
    	\end{cases}
    \end{split}
\end{align}
Taking $\pi_b$ as an ergodic policy, the asymptotic convergence of $Q_t$ to $Q^*$ has been established in \cite{bertsekas1996neuro}. Furthermore, it can be shown that $Q$-learning is a special case of stochastic approximation and enjoys a convergence bound similar to \eqref{eq:general_stoch_conver} \cite{Beck2012ErrorBF,li2020sample,qu2020finite,chen2021Lyapunov_arxiv}.

Two points worth mentioning about the $Q$-learning algorithm. Firstly, $Q$-learning is an off-policy algorithm in the sense that only samples from a fixed ergodic policy are needed to perform the algorithm. Secondly, unlike TD-learning, the update of $Q$ -learning is \textit{non-online}. This imposes a sharp contrast between the analysis of $Q$-learning and TD-learning \cite{zaiwei_q_arxiv}.

\section{On-policy Federated Reinforcement Learning} \label{sec:OPFRL}
The federated version of the on-policy $n$-step TD-learning with linear function approximation is shown in Algorithm \ref{alg:TD-learning}. In this algorithm, we consider $\na$ agents that collaborate to evaluate $\bv^\pi$. For each agent $i, i=1,2,\dots, \na$, we initialize their corresponding parameters $\bv^i_0=\mbf 0$. Furthermore, each agent $i$ samples its initial state $S_0^i$ from some given distribution $\xi$. In the next time steps, each agent follows a single Markovian trajectory generated by policy $\pi$, independently of other agents. At each time $t$, the parameter of each agent $i$ is updated using this independently generated trajectory as $\bv_{t+1}^i=\bvit+\alpha \phi(S_t^i)E_{t,n}^i$. Finally, in order to ensure convergence to a global optimum, in every $\sync$ time step all agents send their parameters to a central server. The central server evaluates the average of these parameters and returns this average to each of the agents. Each agent then continues their update procedure using this average.

Note that the averaging step is essential to ensure synchronization among agents. Smaller $\sync$ results in more frequent synchronization, and hence better convergence guarantees. However, setting a smaller $\sync$ is equivalent to more number of communications between individual agents and the central server, which incurs higher cost. Hence, an intermediate value for $\sync$ must be chosen to strike a balance between the cost of communication and the accuracy. At the end, the algorithm samples a time step $\hat{T}\sim q^c_T$, where 
\begin{align}\label{eq:q_dist}
    q^c_T(t) = \frac{c^{-t}}{\sum_{t'=0}^{T-1}c^{-t'}} \quad\text{for}\quad t=0,1,\dots,T-1
\end{align} 
and $c>1$ is some constant. Since we have $q^c_T(t)\geq 0$ and $\sum_{t=0}^{T-1}q^c_T(t) = 1$, it is clear that $q^c_T(\cdot)$ is a probability distribution over the time interval $[0,T-1]$. In Theorem \ref{thm:fed_TD_on} we characterize the convergence of this algorithm as a function of $\alpha$, $\na$, and $\sync$. Throughout the paper, $\tilde{O}(\cdot)$ ignores the logarithmic terms.

\begin{algorithm}
\caption{Federated $n$-step TD (On-policy, Function Approximation)}
\label{alg:TD-learning}
\begin{algorithmic}[1]
    \STATE{\textbf{Input:} Policy $\pi, \xi$ }
	\STATE{\textbf{Initialization:} $\bv^i_0=\mathbf{0}$ and $S_0^i\sim \xi$ and $\{A_l^i,S_{l+1}^i\}\sim \pi$ for $0\leq l\leq n-1$ and all $i$ }
	\FOR{$t=0$ to $T-1$}
	    \FOR{$i=1,\dots, \na$}
	    \STATE{Sample $A_{t+n}^i\sim\pi(\cdot|S_{t+n}^i), S_{t+n+1}^i\sim \Prob(\cdot|S_{t+n}^i,A_{t+n}^i)$}
    	\STATE $e_{t,l}^i\!=\!\mathcal{R}(S_l^i,A_l^i)\!+\!\gamma \phi(S_{l\!+\!1}^i)^\top \bvit\!-\!\phi(S_l^i)^\top \bvit$ for $l=t,\dots,t+n-1$ 
		\STATE $E_{t,n}^i=\sum_{l=t}^{t+n-1}\gamma^{l-t}e_{t,l}^i$
		\STATE $\bv_{t+1}^i=\bvit+\alpha \phi(S_t^i)E_{t,n}^i$
    	\ENDFOR
	    \IF{$t+1$ mod $\sync = 0$}
	        \STATE{$\bv^i_{t+1} \leftarrow  \frac{1}{\na} \sum_{j=1}^\na \bv^j_{t+1}, \forall \ i \in [\na]$}
	   \ENDIF
	\ENDFOR
	\STATE{\textbf{Sample $\hat{T}\sim q^{c_{TDL}}_T$}
	\STATE{\textbf{Return:} $\frac{1}{\na}\sumik \bv^i_{\hat{T}}$}}
\end{algorithmic}
\end{algorithm}    
    
\begin{theorem}\label{thm:fed_TD_on}
Let $\bv_{\hat{T}} = \frac{1}{\na}\sumik \bv^i_{\hat{T}}$ denote the average of the parameters between agents at random time $\hat{T}$. For small enough step size $\alpha$, and $T\geq\mathcal{O}( \log\frac{1}{\alpha})$, there exist constant $c_{TDL}\in(0,1)$ (see Section \ref{sec:fed_td_learning_on_app} for a precise statement), such that we have 
\begin{align*}
    \E[\|\bv_{\hat{T}}-\bv^\pi\|_2^2]\leq  \mathcal{C}_1^{TD_L}\frac{1}{\alpha}(1-\alpha\mathcal{C}_0^{TD_L})^{T} + \mathcal{C}_2^{TD_L}\frac{\alpha\tau_\alpha^2}{\na} +\mathcal{C}_3^{TD_L}(\sync-1) \alpha^2\tau_\alpha+\mathcal{C}_4^{TD_L} \alpha^3\tau_\alpha^2,
\end{align*}
where $\mathcal{C}_i^{TD_L}$, $i=0,1,2,3,4$ are problem dependent constants, and $\tau_\alpha=\mathcal{O}(\log(1/\alpha))$. Choosing $\alpha=\mathcal{O}(\frac{\log(\na T)}{T})$ and $\sync=T/\na$, we achieve $\E[\|\bv_{\hat{T}}-\bv^\pi\|_2^2]\leq\epsilon$ within $T=\tilde{\mathcal{O}}\lp\frac{1}{\na \epsilon}\rp$ iterations.
\end{theorem}
For brevity purposes, here we did not show the exact dependence of the constants $\mathcal{C}_i^{TD_L}$, $i=0,1,2,3$ on the problem-dependent constants. For a discussion of the detailed expression, see Section \ref{sec:fed_td_learning_on_app} in the appendix.

Theorem \ref{thm:fed_TD_on} shows that federated TD-learning with linear function approximation enjoys a linear speedup with respect to the number of agents. Compared to the convergence bound of general stochastic approximation in \eqref{eq:general_stoch_conver}, the bound in Theorem \ref{thm:fed_TD_on} has three differences. Firstly, the variance term that is proportional to the step size $\alpha$ is divided by the number of agents $\na$. This will allow us to control the variance (and hence improve sample complexity) by employing more number of agents. Secondly, we have an extra term that is zero with perfect synchronization $K=1$. Although this term is not divided with $\na$, it is proportional to $\alpha^2$, which is one order higher than the variance term in \eqref{eq:general_stoch_conver}. Finally, the last term is of the order $\tilde{\mathcal{O}}(\alpha^3)$, which can be handled by choosing a small enough step size. 

Furthermore, according to the choice of $\sync$ in Theorem \ref{thm:fed_TD_on}, after $T$ iterations, the communication cost of federated TD is $T/K={\na}$. However, by employing federated TD-learning in the naive setting where all agents communicate with the central server at every time step, the communication cost will be $\mathcal{O}(T)$. Hence, we observe that by carefully tuning the hyper parameters of federated TD-learning, we can significantly reduce the communication cost of the overall algorithm, while not losing performance in terms of the sample complexity.

Finally, the federated TD-learning Algorithm \ref{alg:TD-learning} preserves the privacy of agents. In particular, since single agents only need to share their parameters $\bv_{t+1}^i$, the central server will not be exposed to the state-action-reward trajectory generated by each agent. This can be essential in some applications where privacy is an issue \cite{mothukuri21privacyFL_elsevier, truex19privacyFL_acm}.
Examples of such applications include autonomous driving \cite{liang19fed_transfer_RL_arxiv,  zhao21IoV_FedRL_ICC},
Internet of Things (IoT) \cite{nguyen21FL_IoT_arxiv, ren19FL_IoT_ieee, wang2020fedRL_ieeeIoT}, and cloud robotics \cite{liu19lifelong_fedRL_ieee, xu2021multi}.

\begin{remark}
In algorithm \ref{alg:TD-learning}, the randomness in choosing $\hat{T}$ is independent of all the other randomness in the problem. Hence, in a practical setting, one can sample $\hat{T}$ in advance, before running the algorithm, and stop the algorithm at time step $\hat{T}$ and output $\bv_{\hat{T}}$. By this method, we require only one data point to be saved, which results in the memory complexity of $\mathcal{O}(1)$ for the algorithm.
\end{remark}

\section{Off-Policy Federated Reinforcement Learning with Tabular Representation} \label{sec:OfPFRLTR}
On-policy TD-learning requires online sampling from the environment, which might be costly (e.g. in robotics \citep{gu2017deep,levine2020offline}), high risk (e.g., in self-driving cars \citep{yurtsever2020survey, maddern20171}), or unethical (e.g., in clinical trials \citep{gottesman2019guidelines, liu2018representation, gottesman2020interpretable}). 
Off-policy training in RL refers to the paradigm in which we use data collected by a fixed behavior policy to run the algorithm. When used in a federated setting, off-policy RL also has privacy advantages \cite{foerster16multiagentRL_arxiv, qi2021federated, zhuo19FedDeepRL_arxiv}. In particular, suppose that each single agent attains a unique sampling policy and does not wish to reveal these policies to the central server. In off-policy FL, agents only transmit sampled data, and hence the sampling policies remain private to each agent.

In Section \ref{sec:off_tabular_TD}, we discuss off-policy TD learning, and in Section \ref{sec:Q_learning_main}, we cover $Q$-learning, an off-policy control algorithm. Additionally, we examine off-policy TD learning with linear function approximation in Section \ref{sec:OfPFRLFA}. It is important to note that federated off-policy TD learning with function approximation leads to heterogeneous federated learning, which presents additional challenges. Therefore, we dedicate a separate section to address these complexities."

\subsection{Federated Off-Policy TD-learning}\label{sec:off_tabular_TD}
In the following, we first discuss single-node off-policy TD-learning, and then we generalize it to the federated setting. 
\subsubsection{Off-policy TD-learning} \label{sec:off_tabular_TD_1}
In off-policy TD-learning the goal is to evaluate the value function $\bV^\pi=(V^\pi(s))_{s\in\mathcal{S}}$ corresponding to the policy $\pi$ using data sampled from some fixed behavior policy $\pi_b$. In this setting, the evaluating policy $\pi$ and the sampling policy $\pi_b$ can be arbitrarily different, and we must account for this difference while performing the evaluation. Although $\pi$ and $\pi_b$ can be different, notice that the value function $\bV^\pi$ does not depend on $\pi_b$. In order to account for this difference, we introduce the notion of \textit{importance sampling} as $\I^{b}(s,a)=\frac{\pi(a|s)}{\pi^b(a|s)}$ which is employed in the off-policy TD-learning. 

Several works studied the finite-time convergence of off-policy TD-learning. In particular, the authors in \cite{offpolicyNAC_ICML,chen2021Lyapunov_arxiv,chen2020finite_nips,chen_offpolicyTD_arxiv} show that, similar to on-policy TD, off-policy TD-learning can be studied under the umbrella of stochastic approximation. Hence, this algorithm has a similar convergence behavior as \eqref{eq:general_stoch_conver}.

\subsubsection{Federated off-policy tabular TD-learning}

The federated version of $n$-step off-policy tabular TD-learning is shown in Algorithm \ref{alg:TD-learning_off}. In this algorithm, each agent $i$ attains a unique (and private) sampling policy $\pi^i$ and follows an independent trajectory generated by this policy. Furthermore, at each time step $t$, each agent $i$ attains a $|\mathcal{S}|$-dimensional vector $\bV^i_t$ and updates this vector using the samples generated by $\pi^i$. In order to account for off-policy sampling, each agent utilizes $\I^{(i)}(S^i,A^i)=\frac{\pi(A^i|S^i)}{\pi^i(A^i|S^i)}$ in the update of their algorithm. We further define $\I_{\max} = \max_{s,a,i}\I^{(i)}(s,a)$, which is a measure of the discrepancy between the evaluation policy $\pi$ and the sampling policy $\pi^i$ of all agents.

To ensure synchronization, all agents transmit their parameter vectors to the central server at every $\sync$ time step. The central server returns the average of these vectors to each agent, and each agent follows this averaged vector afterwards. Notice that in federated off-policy TD-learning Algorithm \ref{alg:TD-learning_off}, each agent shares neither their sampled trajectory of state-action-rewards, nor their sampling policy with the central server. This provides two levels of privacy for the single agents. In the end, the algorithm samples a time step $\hat{T}\sim q_T^{c_{TD}}$, where the distribution $q_T^{c}$ is defined in \eqref{eq:q_dist} and $c_{TD}=1-\frac{\alpha\varphiz_{TD}}{2}$, where $\varphiz_{TD}=1-\frac{0.5 e^{1/4}(2-\mu_{\min}(1-\gamma^{n+1}))}{\sqrt{\sqrt{e}-1+\left(\frac{2-\mu_{\min}(1-\gamma^{n+1})}{2-2\mu_{\min}(1-\gamma^{n+1})}\right)^2}}$. Here, we denote $\mu_{\min} = \min_{s,i}\mu^{\pi^i}(s)$. The constant $c_{TD}$ is carefully chosen to ensure the convergence of Algorithm \ref{alg:TD-learning_off}. Furthermore, for a step size small enough $\alpha$, it can be shown that $0<c_{TD}<1$. 

\begin{algorithm}
\caption{Federated $n$-step TD (Off-policy Tabular Setting)}
\label{alg:TD-learning_off}
\begin{algorithmic}[1]
    \STATE{\textbf{Input:} Policy $\pi, \xi$}
	\STATE{\textbf{Initialization:} $\bV_{0}^i = \mathbf{0}$ and $S_0^i\sim \xi$ and $\{A_l^i,S_{l+1}^i\}\sim \pi^i$ for $0\leq l\leq n-1$ and all $i$ }
	\FOR{$t=0$ to $T-1$}
	    \FOR{$i=1,\dots, \na$}
	    \STATE{Sample $A_{t+n}^i\sim\pi^i(\cdot|S_{t+n}^i), S_{t+n+1}^i\sim \Prob(\cdot|S_{t+n}^i,A_{t+n}^i)$}
	    \STATE $e_{t,l}^i = \mathcal{R}(S_l^i,A_l^i)\!+\!\gamma \bV_t^i(S_{l+1}^i) -\bV_t^i(S_{l}^i)$ 
	    \STATE 
	    $\bV_{t+1}^i(s) = 
	    \begin{cases}
	    \bV_{t}^i(s) + \alpha \sum_{l=t}^{t+n-1} \gamma^{l-t} \lb \Pi_{j=t}^{l} \I^{(i)}(S^i_j, A^i_j) \rb e_{t,l}^i & \text{ if } s=S_{t}^i\\
	    \bV_{t}^i(s) & \text{ o.w }
	    \end{cases}$ 
    	\ENDFOR
	    \IF{$t+1$ mod $\sync = 0$}
	        \STATE{$\bV_{t+1}^i \leftarrow  \frac{1}{\na} \sum_{j=1}^\na \bV^j_{t+1}, \forall \ i \in [\na]$}
	   \ENDIF
	\ENDFOR
	\STATE{\textbf{Sample} $\hat{T}\sim q_T^{c_{TD}}$}
	\STATE{\textbf{Return: $\frac{1}{\na} \sumik \bV^i_{\hat{T}}$}}
\end{algorithmic}
\end{algorithm}

Theorem \ref{thm:fed_TD_off} states the convergence of this Algorithm.

\begin{theorem}\label{thm:fed_TD_off}
Consider the federated $n$-step off-policy tabular TD-learning Algorithm \ref{alg:TD-learning_off}. Denote $\bV_{\hat{T}} = \frac{1}{\na}\sumik \bV_{\hat{T}}^i$. For small enough step size $\alpha$ and large enough $T$, we have
\begin{align*}
    \E[\|\bV_{\hat{T}}-\bV^\pi\|_\infty^2]\leq \mathcal{C}_1^{TD_T}\frac{1}{\alpha} c_{TD}^{T} + \mathcal{C}_2^{TD_T} \frac{\alpha\tau_\alpha^2}{\na}+\mathcal{C}_3^{TD_T}(\sync-1)\alpha^2\tau_\alpha,
\end{align*}
where $\mathcal{C}_1^{TD_T} =  \bar{\mathcal{C}}_1^{TD_T}.\frac{\I_{\max}^{4n-2}}{\mu_{\min}(1-\gamma)^3}$, $\mathcal{C}_2^{TD_T}=\bar{\mathcal{C}}_2^{TD_T}.\frac{\I_{\max}^{3n-1}|\mathcal{S}|\log^2(|\mathcal{S}|)}{\mu_{\min}^2(1-\gamma)^4}$,  $\mathcal{C}_3^{TD_T}=\bar{\mathcal{C}}_3^{TD_T}.\frac{\I_{\max}^{7n-3}|\mathcal{S}|^2\log^2(|\mathcal{S}|)}{\mu_{\min}^4(1-\gamma)^8}$, and $\bar{\mathcal{C}}_i^{TD_T}$, $i=1,2,3$ are universal problem independent constants. In addition, choosing $\alpha=\frac{8\log(\na T)}{T\varphiz_{TD}}$ and $\sync=T/\na$, we have $\E[\|\bV_{\hat{T}}-\bV^\pi\|_\infty^2]\leq\epsilon$ after $T=\tilde{\mathcal{O}}\lp\frac{1}{\na \epsilon}.\frac{\I_{\max}^{7n-3}|\mathcal{S}|^2\log^2(|\mathcal{S}|)}{\mu_{\min}^5(1-\gamma)^9}\rp$ iterations.
\end{theorem}

The proof is given in Section \ref{sec:fed_td_learning_off_app} in the appendix. 
    
Note that similar to the on-policy TD-learning Algorithm \ref{alg:TD-learning}, off-policy TD-learning also enjoys a linear speedup while maintaining a low communication cost. In addition, this algorithm preserves the privacy of agents by keeping both the data and the sampling policy private.

\subsection{Federated \texorpdfstring{$Q$}--learning}\label{sec:Q_learning_main}
So far we have discussed federated policy evaluation with on and off-policy samples. Next, we will expand on federated $Q$-learning, which is presented in Algorithm \ref{alg:Q-learning}. We characterize the convergence of this algorithm in the following theorem.
   
\begin{algorithm}
\caption{Federated $Q$-learning}
\label{alg:Q-learning}
\begin{algorithmic}[1]
    \STATE{\textbf{Input:} Sampling policy $\pi^i_b$ for $i=1,2,\dots, \na$, initial distribution $\xi$}
	\STATE{\textbf{Initialization:} $\Q^i_0=\mathbf{0}$ and $S_0^i\sim \xi$ for all $i$}
	\FOR{$t=0$ to $T-1$}
	    \FOR{$i=1,\dots, \na$}
	    \STATE{Sample $A_{t}^i \sim \pi_b^i (\cdot|S_{t}^i), S_{t+1}^i\sim \Prob(\cdot|S_t^i, A_t^i)$}
    	\STATE{
        \begin{align*}
            \Q_{t+1}(s,a) = \begin{cases}
            	    \Q^i_t(S^i_t, A^i_t) + \alpha \lp \mathcal R(S^i_t, A^i_t) + \gamma \max_a \Q^i_t (S^i_{t+1}, a) - \Q^i_t (S^i_{t}, A^i_{t}) \rp & \text{ if } (s,a) = (S^i_t, A^i_t) \\
            	    \Q^i_t(s, a) & \text{ o.w.}
            	\end{cases}
        \end{align*}
    	}
    	\ENDFOR
	    \IF{$t+1$ mod $\sync = 0$}
	        \STATE{$Q^i_{t+1} \leftarrow  \frac{1}{\na} \sum_{j=1}^\na Q^j_{t+1}, \forall \ i \in [\na]$}
	   \ENDIF
	\ENDFOR
	\STATE{\textbf{Sample:} $\hat{T}\sim q_T^{c_{Q}}$}
	\STATE{\textbf{Return: $\frac{1}{\na} \sum_{i=1}^\na Q^i_{\hat{T}}$}}
\end{algorithmic}
\end{algorithm}

\begin{theorem}\label{thm:fed_Q}
Consider the federated $Q$-learning Algorithm \ref{alg:Q-learning} with $c_Q=1-\frac{\alpha\varphiz_Q}{2}\in(0,1)$, where $\varphiz_Q=1-\frac{0.5 e^{1/4}(2-\mu_{\min}(1-\gamma))}{\sqrt{\sqrt{e}-1+\left(\frac{2-\mu_{\min}(1-\gamma)}{2-2\mu_{\min}(1-\gamma)}\right)^2}}$ and we denote $\mu_{\min} = \min_{s,a,i}\mu^{\pi^i}(s)\pi^i(a|s)$. Denote $\Q_{\hat{T}} = \frac{1}{\na}\sumik \Q_{\hat{T}}^i$. For small enough step size $\alpha$ and large enough $T$, we have
\begin{align*}
    \E[\|Q_{\hat{T}}-Q^*\|_\infty^2]\leq \mathcal{C}_1^{Q}\frac{1}{\alpha} c_{Q}^{T} + \mathcal{C}_2^{Q} \frac{\alpha\tau_\alpha^2}{\na}+\mathcal{C}_3^{Q}(\sync-1)\alpha^2\tau_\alpha,
\end{align*}
where $\mathcal{C}_1^{Q} = \bar{\mathcal{C}}_1^{Q}. \frac{1}{\mu_{\min}(1-\gamma)^3}$, $\mathcal{C}_2^{Q}=\bar{\mathcal{C}}_2^{Q}.\frac{|\mathcal{S}|\log^2(|\mathcal{S}|)}{\mu_{\min}^2(1-\gamma)^4}$,  $\mathcal{C}_3^{Q}=\bar{\mathcal{C}}_3^{Q}. \frac{|\mathcal{S}|^2\log^2(|\mathcal{S}|)}{\mu_{\min}^4(1-\gamma)^8}$, and $\bar{\mathcal{C}}_i^{Q}$, $i=1,2,3$ are universal problem independent constants. In addition, choosing $\alpha=\frac{8\log(\na T)}{T\varphiz_Q }$ and $\sync=T/\na$, we have $\E[\|Q_{\hat{T}}-Q^*\|_\infty^2]\leq\epsilon$ within $T=\tilde{\mathcal{O}}\lp\frac{1}{\na \epsilon}.\frac{|\mathcal{S}|^2\log^2(|\mathcal{S}|)}{\mu_{\min}^5(1-\gamma)^9}\rp$ iterations.
\end{theorem}
According to Theorem \ref{thm:fed_Q}, the federated $Q$-learning Algorithm \ref{alg:Q-learning}, similar to the federated off-policy TD-learning, enjoys linear speedup, communication efficiency, and privacy guarantees. We would like to emphasize that the update of $Q$-learning is non-linear. Hence, the result of Theorem \ref{thm:fed_Q} cannot be derived from Theorems \ref{thm:fed_TD_on} and \ref{thm:fed_TD_off}.

\section{Off-policy Federated Reinforcement Learning with Function Approximation} \label{sec:OfPFRLFA}
In the off-policy TD-learning with the sampling policy $\pi_b$, the goal is to find the unique fixed point $\bv^\pi$ which solves the following linear equation
\begin{align}
\label{eq:BEO_FA}
    \Phi \bv^\pi = \Pi_\Phi^{\pi_b}((\mathcal{T}^\pi)^n \Phi \bv^\pi). 
\end{align}
Here, $\Pi_\Phi^{\pi_b}=\Phi (\Phi^\top\mathcal{K}^{\pi_b}\Phi)^{-1}\Phi^\top \mathcal{K}^{\pi_b}$ is a linear function that projects into the subspace spanned by the matrix $\Phi$, and $\mathcal{K}^{\pi_b}$ is a diagonal matrix, with diagonal entries equal to the stationary distribution of the behavior policy $\pi_b$.

In federated off-policy TD-learning with linear function approximation, the goal is to find the fixed point of the average operators of all the agents. In particular, suppose that each agent $i$ has the sampling policy $\pi_b^i$. Denote the fixed point of the average operator as the vector $\bv^\pi$ which solves the following equation
\begin{align}
\label{eq:BEO_FA_off}
    \Phi \bv^\pi = \left(\frac{1}{N}\sumik\Pi_\Phi^{\pi_b^i}(\mathcal{T}^\pi)^n \Phi \right)\bv^\pi. 
\end{align}
Here, $\pi_b^i$ is the sampling policy of agent $i$. The goal of the $\na$ agents is to find this fixed point $\bv^\pi$ jointly. Notice that the solution to the fixed point equation \eqref{eq:BEO_FA} depends on the behavior policy $\pi_b$. Hence, the solution to the joint fixed point equation \eqref{eq:BEO_FA_off} is different from the solution corresponding to each agent. This is the source of heterogeneity in the off-policy TD-learning with linear function approximation, which makes this problem more challenging than the prior results in this paper.

Algorithm \ref{alg:TD-learning_off_LFA} represents federated off-policy TD-learning with linear function approximation, and Theorem \ref{thm:fed_TD_off_lin} states the convergence of this algorithm.

\begin{theorem}\label{thm:fed_TD_off_lin}
Let $\bv_{\hat{T}} = \frac{1}{\na}\sumik \bv^i_{\hat{T}}$ denote the average of the parameters across agents at the random time $\hat{T}$. For a small enough step size $\alpha$, large enough $n$, and $T\geq\mathcal{O}( \log\frac{1}{\alpha})$, there exists a constant $c_{TDLO}\in(0,1)$, such that we have 
\begin{align*}
    \E[\|\bv_{\hat{T}}-\bv^\pi\|_2^2] \leq \mathcal{C}_1^{TDLO}\frac{1}{\alpha} c_{TDLO}^{T-2\tau_\alpha+1}+\mathcal{C}_2^{TDLO}\frac{\alpha\tau_\alpha^2}{\na}+\mathcal{C}_{3}^{TDLO}(\sync-1)\alpha^2\tau_\alpha+\mathcal{C}_4^{TDLO}(K-1)^2\varepsilon\alpha^2 \tau_\alpha + \mathcal{C}_5^{TDLO}\alpha^3\tau_\alpha^2,
\end{align*}
where $\mathcal{C}_i^{TDLO}$, $i=0,1,2,3,4,5$ are problem dependent constants, and $\tau_\alpha=\mathcal{O}(\log(1/\alpha))$. Choosing $\alpha=\mathcal{O}(\frac{\log(\na T)}{T})$ and $\sync=\sqrt{T}/\na$, we achieve $\E[\|\bv_{\hat{T}}-\bv^\pi\|_2^2]\leq\epsilon$ within $T=\tilde{\mathcal{O}}\lp\frac{1}{\na \epsilon}\rp$ iterations. Furthermore, the communication complexity to achieve this level of accuracy is $\mathcal{O}(\sqrt{N}/\sqrt{\epsilon})$.
\end{theorem}

Note that the communication complexity of Algorithm \ref{alg:TD-learning_off_LFA} depends on $\epsilon$. This is due to the heterogeneity of the agents, where to achieve a certain level of accuracy, requires the agents to communicate more frequently than the homogeneous setting.

\begin{algorithm}
\caption{Federated $n$-step TD (Off-policy, Function Approximation)}
\label{alg:TD-learning_off_LFA}
\begin{algorithmic}[1]
    \STATE{\textbf{Input:} Policy $\pi, \xi$ }\STATE{\textbf{Initialization:} $\bv^i_0=\mathbf{0}$ and $S_0^i\sim \xi$ and $\{A_l^i,S_{l+1}^i\}\sim \pi^i_b$ for $0\leq l\leq n-1$ and all $i$ }
	\FOR{$t=0$ to $T-1$}
	    \FOR{$i=1,\dots, \na$}
	    \STATE{Sample $A_{t+n}^i\sim\pi_b^i(\cdot|S_{t+n}^i), S_{t+n+1}^i\sim \Prob(\cdot|S_{t+n}^i,A_{t+n}^i)$}
    	\STATE $e_{t,l}^i\!=\!\mathcal{R}(S_l^i,A_l^i)\!+\!\gamma  \phi(S_{l\!+\!1}^i)^\top \bvit\!-\!\phi(S_l^i)^\top \bvit$ for $l=t,\dots,t+n-1$ 
		\STATE $E_{t,n}^i=\sum_{l=t}^{t+n-1}\gamma^{l-t}\lb \Pi_{j=t}^{l} \I^{(i)}(S^i_j, A^i_j) \rb e_{t,l}^i$
		\STATE $\bv_{t+1}^i=\bvit+\alpha \phi(S_t^i)E_{t,n}^i$
    	\ENDFOR
	    \IF{$t+1$ mod $\sync = 0$}
	        \STATE{$\bv^i_{t+1} \leftarrow  \frac{1}{\na} \sum_{j=1}^\na \bv^j_{t+1}, \forall \ i \in [\na]$}
	   \ENDIF
	\ENDFOR
	\STATE{\textbf{Sample $\hat{T}\sim q^{c_{TDLO}}_T$}
	\STATE{\textbf{Return:} $\frac{1}{\na}\sumik \bv^i_{\hat{T}}$}}
\end{algorithmic}
\end{algorithm}

\section{Proof Sketch}\label{sec:pf_sketch}

In this section, we provide an outline for the proof of Theorem \ref{thm:main_body}, which is the workhorse to establish the results in Theorems \ref{thm:fed_TD_on},\ref{thm:fed_TD_off}, \ref{thm:fed_Q}, and \ref{thm:fed_TD_off_lin}. We adopt the Lyapunov approach developed in \cite{chen2020finite_nips} with a Lyapunov function $M(\cdot)$. The key idea is to obtain a recursion of the form $\E[M(\btheta_{t+1})]\leq (1-\alpha)\E[M(\bthetat)]+\alpha^2$. This will result in a convergence bound of the form $\E[M(\btheta_{T})]\leq (1-\alpha)^T M(\btheta_{0})+\alpha$. Choosing the Lyapunov function to be a ``close enough'' approximation of $\|\btheta_T\|_c^2$, this implies a rate of convergence on $\|\btheta_T\|_c^2$. 

In order to show linear speedup in the number of agents, one ideally expects a bound of the form $\E[M(\btheta_{T})]\leq (1-\alpha)^T M(\btheta_{0})+\alpha/\na $. However, we obtain the slightly looser bound, 
\begin{align}\label{eq:lsu_pf_sk}
    \E[M(\btheta_{T})]\leq (1-\alpha)^T M(\btheta_{0})+\alpha/\na + \alpha^2.
\end{align}
In particular, according to \eqref{eq:lsu_pf_sk}, $\E[M(\btheta_{T})]$ converges geometrically fast to a ball around $0$ with radius $\alpha/\na+\alpha^2$. Hence, for small enough $\alpha$, the dominant term $\alpha/\na$ diminishes with increasing number of agents, and the higher-order term $\alpha^2$ turns out to be negligible. In order to construct such a bound, we first obtain a recursion of the form \begin{align}\label{eq:drift_with_speed_up}
    \E[M(\btheta_{t+1})]\leq (1-\alpha)\E[M(\bthetat)]+\alpha^2/\na+ o(\alpha^3), ~~ t\leq T. 
\end{align}

\subsection{Illustrative Example and the Challenge}
\label{sec:challenge}
The key challenge in obtaining linear speedup is handling Markovian noise. In order to illustrate the challenge, consider a special case of general stochastic approximation in the scalar case ($d=1$), with linear updates, $\Gi (\btheta,\by_t^i) \equiv -A\theta+\theta$ for all $i$, where the matrix $A$ is deterministic. In addition, the term $ b^i(\by^i_t)$ in the update of agent $i$ is driven by Markov chains $\{\by^i_t\}_{1\leq i\leq\na, t\geq 0}$ that are independent across agents. In addition, we assume full synchronization (that is, $\sync=1$). Hence, in this simplified setting, the update of the average parameter can be written as, 
\begin{align}\label{eq:simple_fed_SA}
    \theta_{t+1}=\theta_t+\alpha(-A\theta_t+b(\bY_t)),
\end{align}
where $b(\bY_t)=\frac{1}{\na}\sumik b^i(\by_t^i)$. 
In the following discussion, we show that \textit{even} in this special case, achieving a linear speedup with $\na$ is non-trivial. Consider the one-step drift of $\E[\theta_t^2]$. We have:
\begin{align*}
    \E[\theta_{t+1}^2]=\underbrace{(1-\alpha A)^2\E[\theta_t^2]}_{E_1} +\underbrace{2\alpha(1-\alpha A)\E[\theta_t b(\bY_t)]}_{E_2}+ \underbrace{\alpha^2\E[b^2(\bY_t)]}_{E_3}.
\end{align*}
Next, consider two special cases. 

\begin{enumerate}
    \item \textbf{$\mbf{\textit{b}(Y_{\textit{t}})}$ i.i.d. zero mean:} In this case $E_2$ is equal to zero. Also, for $E_3$ we have $\alpha^2\E[b^2(\bY_t)] = \frac{\alpha^2}{\na^2} \sumik \E[(b^i(\by_t^i))^2] + \frac{\alpha^2}{\na^2} \sum_{i\neq j} \E[b^i(\by_t^i)b^j(\by_t^j)] = \frac{\alpha^2}{\na^2}\sumik \E[(b^i(\by_t^i))^2] \leq B^2\frac{\alpha^2}{\na}$. Hence, we obtain the form of the recursion in \eqref{eq:drift_with_speed_up}, which results in linear speedup. This is essentially how we achieve linear speedup in the i.i.d. case \cite{koloskova20unified_localSGD_icml}.
    \item \textbf{$\mbf{\textit{b}(Y_{\textit{t}})}$ Markovian and zero stationary mean:}    First, linear speedup in $E_3$ is not as straightforward as in the i.i.d. case due to the Markovian nature of the underlying noise, and we no longer have $\E[b^i(Y_t)b^j(Y_t)]=0$ for $i\neq j$. This difficulty is easily overcome by exploiting geometric mixing of the underlying noise, which enables one to bound $E_3$ by $O(\frac{\alpha^2}{\na}+\alpha^3)$. However, the key difficulty is in bounding $E_2$ because $\theta_t$ and $b(Y_t)$ are no longer independent, and the stochastic approximation update in conjunction with the Markovian noise introduces complex dependencies between the two terms. A naive way to address this is to adopt the approach in \cite{srikant2019finite} to condition at time $t-\tau$ where $\tau$ is the mixing time of the underlying Markov chain. Thus, we have $\theta_t b(Y_t) = \theta_{t-\tau} b(Y_t)+(\theta_{t}-\theta_{t-\tau})b(Y_t)$. The first term can be handled by the fact that $\theta_{t-\tau}$ and $b(Y_t)$ are almost independent due to the Markov chain mixing. However, the second term is problematic. In particular, following the approach in \cite{srikant2019finite}, it can be shown that $|\theta_{t} - \theta_{t-\tau}|\simeq \alpha\tau$, which in turn gives a $\mathcal{O}(\alpha^2)$ bound on $E_2$, which results in not having any linear speedup. 
    
    Thus, establishing linear speedup even in this extremely simplified case is challenging. However, in this paper, we establish a linear speedup for the convergence of a far more general heterogeneous federated stochastic approximation setting.

\end{enumerate}
The main objective of Sections \ref{sec:recursion_M}, \ref{sec:bound_I1}, and \ref{sec:condition_t_2_tau} is to ensure a tight dependency on $\na$ for a general heterogeneous stochastic approximation. Due to the federated structure of the algorithm, we also need to deal with the synchronization error across the agents, which is done in Section \ref{sec:wtd_cons_error}. Throughout this section, ``$\lesssim$'' denotes inequalities in which we ignore the constants, except the step size $\alpha$ and $\tau$.

\subsection{Recursion on the Smooth Approximation \texorpdfstring{$M(\cdot)$}{M(.)}}\label{sec:recursion_M}
In Algorithm \ref{alg:fed_stoch_app}, each agent updates its own parameters $\bthetait$ at each iteration. These local parameters are averaged only once every $\sync$ steps. A key aspect of the analysis is to work with the \textit{virtual} sequence of average parameters $\{ \bthetat=\frac{1}{\na}\sumik \bthetait \}_{t\geq 0}$. Since the algorithm computes these averages only at the synchronization time instants, we shall account for the difference in parameter values across agents using synchronization error terms.
Using Step 4 in Algorithm \ref{alg:fed_stoch_app}, the update of this virtual parameter sequence is given by $\btheta_{t+1} = \bthetat + \alpha \lp \mbf G (\boldsymbol{\Theta}_t, \byt) - \bthetat + \mbf b(\byt) \rp$, where $\boldsymbol{\Theta}_t = \lcb \btheta^1_t, \hdots, \btheta^\na_t \rcb$, $\byt = \lcb \by_t^1, \hdots, \by_t^\na \rcb$, $\mbf G (\boldsymbol{\Theta}_t, \byt) = \frac{1}{\na} \sumik  \Gi(\bthetait, \byit)$, and $\mbf b(\byt) = \frac{1}{\na} \sumik  \mathbf{b}^i (\byit)$. 

In this paper, our aim is to show the convergence of $\E[\|\btheta_t-\btheta^*\|_c^2]$, but the function $\|\cdot\|_c^2$ might be non-smooth. Therefore, inspired by \cite{chen2020finite_nips}, we use $M(\btheta)=\min_{\mbf{x}\in\mathbb{R}^d} \lcb \frac{1}{2}\|\mbf{x}\|_c^2+\frac{1}{2\psi}\|\btheta-\mbf{x}\|_p^2 \rcb$, where $\|\cdot\|_p$ is the $p$-norm for some $p\geq 2$. $M(\cdot)$ is a smooth approximation of the (possibly non-smooth) function $\frac{1}{2}\|\cdot\|_c^2$, with smoothness constant $\frac{p-1}{\psi}$ with respect to $\|\cdot\|_p$ norm. Furthermore, $M(\cdot)$ is ``close'' to $\frac{1}{2}\|\cdot\|_c^2$ and $\frac{1}{2}\|\cdot\|_p^2$, meaning that they can be upper/lower bounded by each other by introducing some multiplicative constants.

Using $\frac{p-1}{\psi}$-smoothness of $\M(\cdot)$ gives
\begin{align}
	\M(\btheta_{t+1} ) 
	\leq \M (\bthetat ) & + \alpha  \underbrace{ \lan \G \M (\bthetat ), \bG (\bthetat) - \bthetat \ran}_{T_1: \text{ Expected update}} + \alpha \underbrace{ \langle \G \M (\bthetat ), \mbf b(\byt) \rangle}_{\substack{T_2: \text{ Due to Markovian } \\\text{  noise in $\mbf b(\byt)$ } }} + \alpha \underbrace{ \langle \G \M (\bthetat ), \mbf G (\boldsymbol{\Theta}_t, \byt) - \bG (\boldsymbol{\Theta}_t) \rangle}_{T_3: \text{ Due to Markovian noise in $\mbf G (\boldsymbol{\Theta}_t, \byt)$ } } \nn\\
	& + \alpha\underbrace{ \langle \G \M (\bthetat ), \bG (\boldsymbol{\Theta}_t) - \bG (\bthetat) \rangle}_{T_4: \text{ Due to local updates}} + \frac{(p-1) \alpha^2}{2 \psi} \underbrace{\lnr \mbf G (\boldsymbol{\Theta}_t, \byt) - \bthetat + \mbf b(\byt) \rnr_p^2}_{T_5: \text{ Due to noise and discretization}},
 	\label{eq:M_bound2}
\end{align}
where, $\bG (\btheta)=\frac{1}{\na}\sumik \bGi(\btheta)$. Next, we bound the individual terms in \eqref{eq:M_bound2}.
 
\begin{itemize}
    \item \textbf{Term $\mbf{\textit{T}_1}$:}  $T_1$ can be handled identically as in \cite[Lemma A.1]{chen2021Lyapunov_arxiv} yielding $T_1\lesssim -M(\bthetat)$, which is the source of the negative drift. This negative drift is achieved using the contraction property of the operators $\bGi(\cdot)$, $i=1,\dots,\na$ (Assumption \ref{ass:contraction}). 
\end{itemize}
Starting from the term $T_2$, the analysis diverts from the one in the previous work \cite{chen2021Lyapunov_arxiv}.  
\begin{itemize}
    \item \textbf{Term $\mbf{\textit{T}_2}$:} Suppose $\bY_t$ was i.i.d. noise. In that case $\E[\mbf b(\byt)]=\mbf 0$, which leads to $\E[T_2]=0$. Since $\byt$ is Markovian instead, by Assumption \ref{ass:mixing}, $\|\E[\mbf b(\byt)] \|_c$ approaches $0$ exponentially fast. Now, suppose that  $\mbf{\textit{T}_2}$ instead has the following term $\lan \G \M (\btheta_{t-\tau} ), \mbf b(\byt) \ran$. By taking the conditional expectation given the time $t-\tau$, denoted by $\E_{t-\tau}[\cdot]$, for large enough $\tau$ (such that $\rho^\tau = \alpha^2$), and using Assumption \ref{ass:mixing}, we have $\E_{t-\tau}[\lan \G \M (\btheta_{t-\tau} ), \mbf b(\byt) \ran] \lesssim \rho^\tau \lnr \G \M (\btheta_{t-\tau} ) \rnr_p^\star \lesssim \E_{t-\tau}[\alpha\lnr \btheta_{t-\tau} -\btheta_{t} \rnr_c+\alpha^4 + M(\bthetat)]$. Therefore, we add and subtract $\G \M (\btheta_{t-\tau})$ in $T_2$, to get {\small
    \begin{align*}
        T_2 \hspace{-1mm}&=\hspace{-1mm} \lan \G \M (\btheta_{t} ) - \G \M (\btheta_{t-\tau} ) + \G \M (\btheta_{t-\tau} ), \mbf b(\byt) \ran \hspace{-0.5mm}\leq \hspace{-0.5mm}\frac{1}{2\alpha} \|\bthetat - \btheta_{t-\tau} \|_c^2+\frac{\alpha}{2}\|b(\byt)\|_c^2 + \lan \G \M (\btheta_{t-\tau} ), \mbf b(\byt) \ran.
    \end{align*}
    }
    Substituting the bound on $\E_{t-\tau} \lan \G \M (\btheta_{t-\tau} ), \mbf b(\byt) \ran$ in the inequality above, we get $\E_{t-\tau} [T_2]$. Eventually, we shall use the tower property of expectation to get $\mbe \lb T_2 \rb$.

    \item \textbf{Term $\mbf{\textit{T}_3}$:} 
    Again, if the noise $\byt$ is i.i.d., $\E[T_3]=0$. However, in the presence of Markov noise, similar to $T_2$, we need to condition w.r.t. time $t-\tau$. Further, due to mismatch of parameters across agents $\boldsymbol{\Theta}_t$, we also get synchronization error terms. We have
    \begin{equation}
        \begin{aligned}
            T_3 =& \underbrace{\langle \G \M (\bthetat)- \G \M (\btheta_{t-\tau}), \mbf G (\boldsymbol{\Theta}_t, \byt) - \bG (\boldsymbol{\Theta}_t) \rangle}_{T_{31}} \\
            &+ \underbrace{\langle \G \M (\btheta_{t-\tau}), \mbf G (\boldsymbol{\Theta}_t, \byt) - \mbf G (\boldsymbol{\Theta}_{t-\tau}, \byt) + \bG (\boldsymbol{\Theta}_{t-\tau}) - \bG (\boldsymbol{\Theta}_t) \rangle}_{T_{32}} \\
            & + \underbrace{\langle \G \M (\btheta_{t-\tau}), \mbf G (\boldsymbol{\Theta}_{t-\tau}, \byt) - \bG (\boldsymbol{\Theta}_{t-\tau}) \rangle}_{T_{33}}.
        \end{aligned}\nn
    \end{equation}
    We can bound the first two terms as follows: $T_{31}\lesssim \|\bthetat-\btheta_{t-\tau}\|_c^2 + M(\bthetat) + \Omega_t $ and $T_{32}\lesssim \|\bthetat-\btheta_{t-\tau}\|_c^2 + M(\bthetat) + \Omega_t +\Omega_{t-\tau}$, where $\Omega_t, \Omega_{t-\tau}$ are the synchronization errors, defined as $\Omega_t \triangleq \frac{1}{\na} \sumik \|\bthetat - \bthetait\|_c^2$ and is bounded in Section \ref{sec:wtd_cons_error}. All the terms in $T_{33}$ other than $\byt$ depend on time $t-\tau$. Therefore, similar to $T_2$, $T_{33}$ can be bounded using the mixing property of $\byt$. In particular, $\E_{t-\tau}[T_{33}]\lesssim \alpha \E_{t-\tau}[M(\btheta_{t})+\lnr \btheta_{t}- \btheta_{t-\tau} \rnr_c^2+\Omega_{t-\tau}]$.
    Combining these three bounds, we get $\E_{t-\tau} [T_3] \lesssim \E_{t-\tau} \lb M(\bthetat) + \lnr \bthetat - \btheta_{t-\tau} \rnr_c^2 + \Omega_t + \Omega_{t-\tau} \rb$.
    \item \textbf{Term $\mbf{\textit{T}_4}$:} $T_4$ is due to the mismatch arising due to multiple local updates. With perfect synchronization ($K=1$), we have $T_4=0$. More generally, we show that $T_4 \lesssim M(\bthetat) + \Omega_t$.
    \item \textbf{Term $\mbf{\textit{T}_5}$:} Using Assumption \ref{ass:lipschitz}, we can bound $T_5\lesssim M(\bthetat) + \Omega_t + \lnr \mbf b(\byt) \rnr_c^2 $.
\end{itemize}

\vspace{0.13cm}
Substituting the above bounds on $T_1, T_2, T_3, T_4, T_5$ back into \eqref{eq:M_bound2}, we get
\begin{equation}
    \begin{aligned}
        \E_{t-\tau} \left[ M (\btheta_{t+1}) \right] \lesssim & (1-\alpha)\E_{t-\tau} \left[ M (\btheta_{t}) \right] +\alpha^4 +\underbrace{\E_{t-\tau}  \lb \| \btheta_{t} - \btheta_{t-\tau} \|_c^2 \rb  + \alpha^2 \E_{t-\tau}  \lb \| \btheta_{t} - \btheta_{t-\tau} \|_c \rb   }_{I_1} \\
        & + \alpha^2\E_{t-\tau}\|\mbf b(\byt)\|_c^2 + \alpha \E_{t-\tau}\lb \Omega_t+\Omega_{t-\tau}\rb.
    \end{aligned}
    \label{eq:M_drift_main22}
\end{equation}
First, we discuss a few special cases of \eqref{eq:M_drift_main22} under different assumptions.
\begin{enumerate}
    \item Perfect synchronization $(\sync=1)$ with i.i.d. noise: since $\Ot = 0$ for all $t$, $\tau = 0$, the bound reduces to the form we get for centralized systems with i.i.d. noise \cite{rakhlin12SGD_ICML}.
    \item Infrequent synchronization $(\sync > 1)$ with i.i.d. noise: $\tau = 0$, and the bound reduces to the form we get for federated stochastic optimization with i.i.d. noise \cite{richtarik20localSGD_aistats}.
    \item Perfect synchronization $(\sync=1)$ with Markov noise: Since $\Ot = 0$ for all $t$, the bound in \eqref{eq:M_drift_main22} generalizes the results in \cite{srikant2019finite}, which considers the linear stochastic approximation case. The bound also resembles the corresponding bound in \cite{chen2021Lyapunov_arxiv}.
\end{enumerate}
The last two terms in \eqref{eq:M_drift_main22} correspond to average of Markov noise in the $\na$ agents, and the synchronization error. These are bounded in Sections \ref{sec:condition_t_2_tau} and \ref{sec:wtd_cons_error}, respectively. But first, in the next section we bound $I_1$. We would like to emphasize that bounding $I_1$ carefully to eventually create the terms $o(\alpha^3)$ and $\alpha^2/\na$, is essential to obtain a recursion similar to \eqref{eq:drift_with_speed_up}. This is one of the important parts that separates our work from the previous results.

\subsection{Bounding \texorpdfstring{$I_1$}{I1}} \label{sec:bound_I1}
In this section, we bound $\| \btheta_{t} - \btheta_{t-\tau} \|_c^2$ in terms of $\alpha^2\sum_{l=t-\tau}^t\| \mbf b(\mbf Y_l)\|_c^2$ and $\alpha^2\sum_{l=t-\tau}^t\Omega_l$. 
The bound on $\| \btheta_{t} - \btheta_{t-\tau} \|_c$ follows analogously. First, using triangle inequality and the update of Algorithm \ref{alg:fed_stoch_app}, at each time $t$
\begin{align}
     \| \btheta_t - \btheta_{t-\tau} \|_c^2 & \leq  \tau\sum_{l=t-\tau}^{t-1} \| \btheta_{l+1} - \btheta_l \|_c^2  \lesssim \tau\alpha^2\sum_{l=t-\tau}^{t-1}  [\| \btheta_l \|_c^2 +  \| \mbf b (\bY_l) \|_c^2 +   \Omega_l ].\label{eq:theta_diff_square_sk}
\end{align}
Among the terms above, we need to further bound $\sum_{l=t-\tau}^{t-1} \|\btheta_l\|_c^2$. It can be shown that $\|\btheta_l\|_c^2\lesssim \|\btheta_{t-\tau}\|_c^2 +\alpha \sum_{k=t-\tau}^{l-1} \|\mbf b (\bY_k)\|_c^2 + \alpha \sum_{k=t-\tau}^{l-1} \Omega_k$. This bound can be explained as follows. Consider $\btheta_l$ for $l=t-\tau,t-\tau+1\dots,t-1$. Clearly, at $l=t-\tau$, $\|\btheta_l\|_c^2\leq \|\btheta_{t-\tau}\|_c^2$. Next, due to the update in Step 4 of Algorithm \ref{alg:fed_stoch_app}, starting at time $t-\tau$, $\btheta_l$ depends on $\{ \mbf b (\bY_k) \}_{k=t-\tau}^{l-1}$, which explains the $\alpha\sum_{k=t-\tau}^{l-1} \|\mbf b (\bY_k)\|_c^2$ term. In addition, the discrepancy between agents explains the term $\alpha\sum_{k=t-\tau}^{l-1} \Omega_k$.
Since $l \leq t-1$, we can bound \eqref{eq:theta_diff_square_sk} as
$$\| \btheta_t - \btheta_{t-\tau} \|_c^2  \lesssim \tau^2\alpha^2\|\btheta_{t-\tau}\|_c^2 + \tau\alpha^2  \sum_{l=t-\tau}^{t-1}  [  \| \mbf b (\bY_l) \|_c^2 +   \Omega_l ].$$ 

Substituting this bound in \eqref{eq:M_drift_main22}, we get
\begin{align}\label{eq:drift_R}
    \E_{t-\tau} \left[ M (\btheta_{t+1}) \right] \lesssim & (1-\alpha)\E_{t-\tau} \left[ M (\btheta_{t}) \right] +\alpha^3 +\underbrace{\alpha^2  \sum_{l=t-\tau}^{t-1}  \E_{t-\tau}[  \| \mbf b (\bY_l) \|_c^2  ]}_{I_2}+\alpha  \sum_{l=t-\tau}^{t}  \E_{t-\tau}[\Omega_l ].
\end{align}
Next, we bound the term $\sum_{l=t-\tau}^t\E_{t-\tau}[\| \mbf b(\mbf Y_l)\|_c^2]$ to create the form $\alpha^3+\alpha^2/\na$ promised in \eqref{eq:drift_with_speed_up}. 

\subsection{Bounding \texorpdfstring{$I_2$}{I2}}\label{sec:condition_t_2_tau}
If $\{ \by_l^i \}$ is i.i.d. noise across agents $i=1,2,\dots,\na$, it is easy to see that $\E[\|\mbf b(\bY_l)\|_c^2] \lesssim \frac{1}{\na}$. This leads to linear speedup in federated learning with i.i.d. noise \cite{fedavg17aistats}. However, in the presence of Markovian noise, the analysis is more involved.
A crude method to bound $I_2$ would be to use Assumption \ref{ass:lipschitz} to obtain $\| \mbf b (\bY_l) \|_c^2\leq B^2$, which results in $I_2 \lesssim \alpha^2$ \cite{srikant2019finite}. However, compared to the form in \eqref{eq:drift_with_speed_up}, this bound is not sufficient to ensure linear speedup in $\na$. 
Hence, we exploit conditional expectation to get a tighter dependence on $\alpha$ and $\na$. In particular, for $r\leq t$, 
\begin{align}
    \E_{t-r}[\|\mbf b(\bY_l)\|_c^2] \lesssim & \frac{1}{\na} +  \rho^{2(l-t+r)}, \qquad \forall  l \in \{ t-r, \hdots, t-1 \}.
    \label{eq:lemma_bound_byt_11111}
\end{align}
In \eqref{eq:lemma_bound_byt_11111}, the first term is similar to the i.i.d. noise case, while the second term appears due to the Markovian noise. Now consider $r=\tau$. Putting in \eqref{eq:lemma_bound_byt_11111}, we get $I_2\lesssim \frac{\alpha^2}{\na}+\alpha^2$. However, the term $\alpha^2$ is too loose to guarantee linear speedup (again, recall \eqref{eq:drift_with_speed_up}). Therefore, we need to go further back into the past. Next, consider $r=2\tau$. Then using the assumption $\rho^\tau=\alpha^2$, it is easy to show that $I_2\lesssim \frac{\alpha^2}{\na}+\alpha^3$, which is sufficient for linear speedup. Here we would like to emphasize that taking a conditional expectation conditioned on time $t-2\tau$ is essential to achieve this bound. 

Substituting the bound in \eqref{eq:drift_R} and using tower property of expectation $\E_{t-2\tau} \lb \E_{t-\tau} \lb \cdot \rb \rb = \E_{t-2\tau} \lb \cdot \rb$,
\begin{align}\label{eq:drift_R2}
    \E_{t-2\tau} \left[ M (\btheta_{t+1}) \right] \lesssim & (1-\alpha)\E_{t-2\tau} \left[ M (\btheta_{t}) \right]  +\alpha^2/\na+\alpha^3+\alpha  \sum_{l=t-\tau}^{t}  \E_{t-2\tau}[\Omega_l ].
\end{align}

Next, we discuss how to handle the synchronization error $\{ \Omega_l \}$ by taking a weighted sum of \eqref{eq:drift_R2}.

\subsection{Weighted Consensus Error}\label{sec:wtd_cons_error}

It can be shown that the expectation of the synchronization error at time $t$ satisfies
\begin{align}
    \mbe [\Ot] & \lesssim  \alpha^2 (t-s_t \sync)  +  \alpha^2 (t-s_t \sync)  \sum_{t' = s_t \sync}^{t-1} \mbe [ M(\btheta_{t'})]. \label{eq:lem:consensus_error22}
\end{align}
where $s_t \sync$ is the last time before $t$ when synchronization occurred. Intuitively, the synchronization error builds up over time, starting from zero at $s_t \sync$, and depends on all iterations encountered along the trajectory, i.e., $\{ \btheta_{t'} \}_{t'=s_t \sync}^{t-1}$.
Next, we define $w_t=1/(1-\alpha)^t$. We can show that a weighted sum of the inequality \eqref{eq:lem:consensus_error22} satisfies
\begin{align}
    & \frac{1}{W_T} \sum_{t=2\tau}^T w_t \lb \sum_{\ell=t-\tau}^{t} \mbe [\Omega_\ell] \rb \lesssim \alpha^2(1+\varepsilon(K-1)) (\sync-1) + \frac{1}{2 W_T} \sum_{t=0}^{T} w_t \mbe [M (\btheta_t)],\label{eq:w_weighted_sync_noise}
\end{align}
where $W_T=\sum_{t=0}^Tw_t$. Note that throughout the proof, the first place where the effect of heterogeneity appears is in the analysis of the consensus error in \eqref{eq:w_weighted_sync_noise}. 

Finally, by sampling $\hat{T}$ according to $P(\hat{T}=t)=\frac{w_t}{\sum_{t'=0}^Tw_t}$ independently of all other randomness in the problem, we have $\E[M(\btheta_{\hat{T}})]=\frac{1}{W_T} \sum_{t=0}^{T} w_t \mbe [M (\btheta_t)]$. Weighting both sides of \eqref{eq:drift_R2} with $w_t/W_T$ and summing up, we get the result in Theorem \ref{thm:main_body}.

Due to the dependence of the synchronization error $\Ot$ on the previous trajectory through $\{ M(\btheta_{t'}) \}_{t'=s_t \sync}^{t-1}$, the recursion in \eqref{eq:drift_R2} is not as ``clean'' as the recursion in \eqref{eq:drift_with_speed_up}. Therefore, we rely on the weighted sum technique described above and show convergence for the weighted sum of iterates $\btheta_{\hat{T}}$, rather than the last iterate $\btheta_{T}$. Similar bounds for weighted iterates have also been shown in the much simpler i.i.d. noise setting \cite{stich18localSGD_iclr, koloskova20unified_localSGD_icml}.

\section{Conclusion}

In this paper, we analyzed the convergence of general heterogeneous federated stochastic approximation algorithms under Markovian noise (FeGSAM), where $\na$ agents seek to find the fixed point of an operator. Our results demonstrate that the algorithm achieves linear speedup. However, in a heterogeneous setting, achieving an accuracy level of $\epsilon$ requires more communication compared to the homogeneous case. Specifically, the communication complexity is $\mathcal{O}(\na)$ in the homogeneous setting and $\mathcal{O}(\sqrt{N}/\sqrt{\epsilon})$ in the heterogeneous setting.

Based on these findings, we explored the convergence of federated reinforcement learning algorithms, where $\na$ agents collaborate to complete a common task. We introduced Federated TD-learning in both on-policy (with function approximation) and off-policy (in tabular and function approximation) settings, as well as federated $Q$-learning. In these algorithms, each agent independently follows a Markovian trajectory of state-action-reward and updates its parameters locally. To ensure convergence, agents synchronize their parameters after every $\sync$ iterations. Our results demonstrate linear speedup across all algorithms with respect to the number of agents. Notably, in off-policy TD-learning with linear function approximation, the communication complexity is $\mathcal{O}(\sqrt{\na}/\sqrt{\epsilon})$, while for the other algorithms it remains $\mathcal{O}(\na)$.
\bibliography{refs.bib}

\begin{thebibliography}{100}

\bibitem{banach1922operations}
Stefan Banach.
\newblock Sur les op{\'e}rations dans les ensembles abstraits et leur
  application aux {\'e}quations int{\'e}grales.
\newblock {\em Fund. math}, 3(1):133--181, 1922.

\bibitem{beck2017first}
Amir Beck.
\newblock {\em First-order methods in optimization}, volume~25.
\newblock SIAM, 2017.

\bibitem{Beck2012ErrorBF}
Carolyn~L. Beck and R.~Srikant.
\newblock Error bounds for constant step-size {$Q$}-learning.
\newblock {\em Syst. Control. Lett.}, 61:1203--1208, 2012.

\bibitem{beck2012error}
Carolyn~L Beck and Rayadurgam Srikant.
\newblock Error bounds for constant step-size {$Q$}-learning.
\newblock {\em Systems \& control letters}, 61(12):1203--1208, 2012.

\bibitem{beck2013improved}
Carolyn~L Beck and Rayadurgam Srikant.
\newblock Improved upper bounds on the expected error in constant step-size
  {$Q$}-learning.
\newblock In {\em 2013 American Control Conference}, pages 1926--1931. IEEE,
  2013.

\bibitem{benveniste2012adaptive}
Albert Benveniste, Michel M{\'e}tivier, and Pierre Priouret.
\newblock {\em Adaptive algorithms and stochastic approximations}, volume~22.
\newblock Springer Science \& Business Media, 2012.

\bibitem{bertsekas1995dynamic2}
Dimitri~P Bertsekas, Dimitri~P Bertsekas, Dimitri~P Bertsekas, and Dimitri~P
  Bertsekas.
\newblock {\em Dynamic programming and optimal control}, volume~2.
\newblock Athena scientific Belmont, MA, 1995.

\bibitem{NDP_book}
Dimitri~P Bertsekas and John~N Tsitsiklis.
\newblock {\em Neuro-dynamic programming}.
\newblock Athena Scientific, 1996.

\bibitem{bertsekas1996neuro}
Dimitri~P Bertsekas and John~N Tsitsiklis.
\newblock {\em Neuro-dynamic programming}.
\newblock Athena Scientific, 1996.

\bibitem{bhandari18FiniteTD_LFA_colt}
Jalaj Bhandari, Daniel Russo, and Raghav Singal.
\newblock A finite time analysis of temporal difference learning with linear
  function approximation.
\newblock In {\em Conference on learning theory}, pages 1691--1692. PMLR, 2018.

\bibitem{borkar2009stochastic}
Vivek~S Borkar.
\newblock {\em Stochastic approximation: a dynamical systems viewpoint},
  volume~48.
\newblock Springer, 2009.

\bibitem{borkar2000ode}
Vivek~S Borkar and Sean~P Meyn.
\newblock The {ODE} method for convergence of stochastic approximation and
  reinforcement learning.
\newblock {\em SIAM Journal on Control and Optimization}, 38(2):447--469, 2000.

\bibitem{chau2014overview}
Marie Chau and Michael~C Fu.
\newblock An overview of stochastic approximation.
\newblock {\em Handbook of simulation optimization}, pages 149--178, 2014.

\bibitem{NACLFA_arxiv}
Zaiwei Chen, Sajad Khodadadian, and Siva~Theja Maguluri.
\newblock {Finite-Sample Analysis of Off-Policy Natural Actor-Critic with
  Linear Function Approximation}.
\newblock {\em Preprint arXiv:2105.12540}, 2021.
\newblock Submitted to {NeurIPS} 2021.

\bibitem{chen2023lyapunov}
Zaiwei Chen, Siva~T Maguluri, Sanjay Shakkottai, and Karthikeyan Shanmugam.
\newblock A lyapunov theory for finite-sample guarantees of markovian
  stochastic approximation.
\newblock {\em Operations Research}, 2023.

\bibitem{chen2020finite}
Zaiwei Chen, Siva~Theja Maguluri, Sanjay Shakkottai, and Karthikeyan Shanmugam.
\newblock Finite-sample analysis of stochastic approximation using smooth
  convex envelopes.
\newblock {\em Under Review at Mathematics of Operations Research, Preprint
  arXiv:2002.00874}, 2020.

\bibitem{chen2020finite_nips}
Zaiwei Chen, Siva~Theja Maguluri, Sanjay Shakkottai, and Karthikeyan Shanmugam.
\newblock Finite-sample analysis of stochastic approximation using smooth
  convex envelopes.
\newblock In {\em Advances in Neural Information Processing Systems}, 2020.

\bibitem{chen2021Lyapunov_arxiv}
Zaiwei Chen, Siva~Theja Maguluri, Sanjay Shakkottai, and Karthikeyan Shanmugam.
\newblock {A Lyapunov Theory for Finite-Sample Guarantees of Asynchronous
  $Q$-Learning and TD-Learning Variants}.
\newblock {\em Under review by JMLR, Preprint arXiv:2102.01567}, 2021.

\bibitem{chen_offpolicyTD_arxiv}
Zaiwei Chen, Siva~Theja Maguluri, Sanjay Shakkottai, and Karthikeyan Shanmugam.
\newblock {Finite-Sample Analysis of Off-Policy TD-Learning via Generalized
  Bellman Operators}.
\newblock {\em Preprint arXiv:2106.12729}, 2021.

\bibitem{zaiwei_q_arxiv}
Zaiwei Chen, Sheng Zhang, Thinh~T Doan, John-Paul Clarke, and Siva~Theja
  Maguluri.
\newblock Finite-sample analysis of nonlinear stochastic approximation with
  applications in reinforcement learning.
\newblock {\em Automatica}, 146:110623, 2022.

\bibitem{chen2021multi}
Ziyi Chen, Yi~Zhou, and Rongrong Chen.
\newblock Multi-agent off-policy td learning: Finite-time analysis with
  near-optimal sample complexity and communication complexity.
\newblock {\em arXiv preprint arXiv:2103.13147}, 2021.

\bibitem{chen21sample_A3C_arxiv}
Ziyi Chen, Yi~Zhou, Rongrong Chen, and Shaofeng Zou.
\newblock Sample and communication-efficient decentralized actor-critic
  algorithms with finite-time analysis.
\newblock {\em arXiv preprint arXiv:2109.03699}, 2021.

\bibitem{dalal2018finite}
Gal Dalal, Bal{\'a}zs Sz{\"o}r{\'e}nyi, Gugan Thoppe, and Shie Mannor.
\newblock Finite sample analysis for {TD(0)} with function approximation.
\newblock In {\em Thirty-Second AAAI Conference on Artificial Intelligence},
  2018.

\bibitem{doan_dist_td_icml}
Thinh Doan, Siva Maguluri, and Justin Romberg.
\newblock {Finite-Time Analysis of Distributed TD$(0)$ with Linear Function
  Approximation on Multi-Agent Reinforcement Learning}.
\newblock In {\em International Conference on Machine Learning}, pages
  1626--1635, 2019.

\bibitem{Even_Dar_Qlearning}
Eyal Even-Dar and Yishay Mansour.
\newblock {Learning Rates for $Q$-Learning}.
\newblock {\em J. Mach. Learn. Res.}, 5:1–25, 2004.

\bibitem{fan2023fedhql}
Flint~Xiaofeng Fan, Yining Ma, Zhongxiang Dai, Cheston Tan, Bryan Kian~Hsiang
  Low, and Roger Wattenhofer.
\newblock Fedhql: Federated heterogeneous q-learning.
\newblock {\em arXiv preprint arXiv:2301.11135}, 2023.

\bibitem{foerster16multiagentRL_arxiv}
Jakob~N Foerster, Yannis~M Assael, Nando De~Freitas, and Shimon Whiteson.
\newblock Learning to communicate with deep multi-agent reinforcement learning.
\newblock {\em arXiv preprint arXiv:1605.06676}, 2016.

\bibitem{ganesh2024global}
Swetha Ganesh, Jiayu Chen, Gugan Thoppe, and Vaneet Aggarwal.
\newblock Global convergence guarantees for federated policy gradient methods
  with adversaries.
\newblock {\em arXiv preprint arXiv:2403.09940}, 2024.

\bibitem{gottesman2020interpretable}
Omer Gottesman, Joseph Futoma, Yao Liu, Sonali Parbhoo, Leo Celi, Emma
  Brunskill, and Finale Doshi-Velez.
\newblock Interpretable off-policy evaluation in reinforcement learning by
  highlighting influential transitions.
\newblock In {\em International Conference on Machine Learning}, pages
  3658--3667. PMLR, 2020.

\bibitem{gottesman2019guidelines}
Omer Gottesman, Fredrik Johansson, Matthieu Komorowski, Aldo Faisal, David
  Sontag, Finale Doshi-Velez, and Leo~Anthony Celi.
\newblock Guidelines for reinforcement learning in healthcare.
\newblock {\em Nature medicine}, 25(1):16--18, 2019.

\bibitem{gu2017deep}
Shixiang Gu, Ethan Holly, Timothy Lillicrap, and Sergey Levine.
\newblock Deep reinforcement learning for robotic manipulation with
  asynchronous off-policy updates.
\newblock In {\em 2017 IEEE international conference on robotics and automation
  (ICRA)}, pages 3389--3396. IEEE, 2017.

\bibitem{harold1997stochastic}
J~Harold, G~Kushner, and George Yin.
\newblock Stochastic approximation and recursive algorithm and applications.
\newblock {\em Application of Mathematics}, 35(10), 1997.

\bibitem{hu2019characterizing}
Bin Hu and Usman~Ahmed Syed.
\newblock Characterizing the exact behaviors of temporal difference learning
  algorithms using markov jump linear system theory.
\newblock {\em arXiv preprint arXiv:1906.06781}, 2019.

\bibitem{jaakkola1994convergence}
Tommi Jaakkola, Michael~I Jordan, and Satinder~P Singh.
\newblock Convergence of stochastic iterative dynamic programming algorithms.
\newblock In {\em Advances in neural information processing systems}, pages
  703--710, 1994.

\bibitem{jin2022federated}
Hao Jin, Yang Peng, Wenhao Yang, Shusen Wang, and Zhihua Zhang.
\newblock Federated reinforcement learning with environment heterogeneity.
\newblock In {\em International Conference on Artificial Intelligence and
  Statistics}, pages 18--37. PMLR, 2022.

\bibitem{jumper2021highly}
John Jumper, Richard Evans, Alexander Pritzel, Tim Green, Michael Figurnov,
  Olaf Ronneberger, Kathryn Tunyasuvunakool, Russ Bates, Augustin
  {\v{Z}}{\'\i}dek, Anna Potapenko, et~al.
\newblock Highly accurate protein structure prediction with alphafold.
\newblock {\em nature}, 596(7873):583--589, 2021.

\bibitem{kairouz2021advances}
Peter Kairouz, H~Brendan McMahan, Brendan Avent, Aur{\'e}lien Bellet, Mehdi
  Bennis, Arjun~Nitin Bhagoji, Kallista Bonawitz, Zachary Charles, Graham
  Cormode, Rachel Cummings, et~al.
\newblock Advances and open problems in federated learning.
\newblock {\em Foundations and trends{\textregistered} in machine learning},
  14(1--2):1--210, 2021.

\bibitem{kairouz2019advances}
Peter Kairouz, H~Brendan McMahan, Brendan Avent, Aur{\'e}lien Bellet, Mehdi
  Bennis, Arjun~Nitin Bhagoji, Keith Bonawitz, Zachary Charles, Graham Cormode,
  Rachel Cummings, et~al.
\newblock Advances and open problems in federated learning.
\newblock {\em Preprint arXiv:1912.04977}, 2019.

\bibitem{richtarik20localSGD_aistats}
Ahmed Khaled, Konstantin Mishchenko, and Peter Richt{\'a}rik.
\newblock Tighter theory for local sgd on identical and heterogeneous data.
\newblock In {\em International Conference on Artificial Intelligence and
  Statistics}, pages 4519--4529. PMLR, 2020.

\bibitem{offpolicyNAC_ICML}
Sajad Khodadadian, Zaiwei Chen, and Siva~Theja Maguluri.
\newblock {Finite-Sample Analysis of Off-Policy Natural Actor-Critic
  Algorithm}.
\newblock In {\em International Conference on Machine Learning}, 2021.

\bibitem{khodadadian2022federated}
Sajad Khodadadian, Pranay Sharma, Gauri Joshi, and Siva~Theja Maguluri.
\newblock Federated reinforcement learning: Linear speedup under markovian
  sampling.
\newblock In {\em International Conference on Machine Learning}, pages
  10997--11057. PMLR, 2022.

\bibitem{koloskova20unified_localSGD_icml}
Anastasia Koloskova, Nicolas Loizou, Sadra Boreiri, Martin Jaggi, and Sebastian
  Stich.
\newblock A unified theory of decentralized sgd with changing topology and
  local updates.
\newblock In {\em International Conference on Machine Learning}, pages
  5381--5393. PMLR, 2020.

\bibitem{konda_tsitsiklis_actor_critic}
Vijay~R Konda and John~N Tsitsiklis.
\newblock Actor-critic algorithms.
\newblock In {\em Advances in neural information processing systems}, pages
  1008--1014, 2000.

\bibitem{konevcny16federated}
Jakub Kone{\v{c}}n{\`y}, H~Brendan McMahan, Daniel Ramage, and Peter
  Richt{\'a}rik.
\newblock Federated optimization: Distributed machine learning for on-device
  intelligence.
\newblock {\em arXiv preprint arXiv:1610.02527}, 2016.

\bibitem{csaba_iidnoise}
Chandrashekar Lakshminarayanan and Csaba Szepesvari.
\newblock {Linear stochastic approximation: How far does constant step-size and
  iterate averaging go?}
\newblock In {\em International Conference on Artificial Intelligence and
  Statistics}, pages 1347--1355, 2018.

\bibitem{lan2024asynchronous}
Guangchen Lan, Dong-Jun Han, Abolfazl Hashemi, Vaneet Aggarwal, and
  Christopher~G Brinton.
\newblock Asynchronous federated reinforcement learning with policy gradient
  updates: Algorithm design and convergence analysis.
\newblock {\em arXiv preprint arXiv:2404.08003}, 2024.

\bibitem{levine2020offline}
Sergey Levine, Aviral Kumar, George Tucker, and Justin Fu.
\newblock Offline reinforcement learning: Tutorial, review, and perspectives on
  open problems.
\newblock {\em Preprint arXiv:2005.01643}, 2020.

\bibitem{li2020sample}
Gen Li, Yuting Wei, Yuejie Chi, Yuantao Gu, and Yuxin Chen.
\newblock {Sample Complexity of Asynchronous $Q$-Learning: Sharper Analysis and
  Variance Reduction}.
\newblock {\em Advances in neural information processing systems}, 2020.

\bibitem{smith20FL_SPmag}
Tian Li, Anit~Kumar Sahu, Ameet Talwalkar, and Virginia Smith.
\newblock Federated learning: Challenges, methods, and future directions.
\newblock {\em IEEE Signal Processing Magazine}, 37(3):50--60, 2020.

\bibitem{li2019convergence}
Xiang Li, Kaixuan Huang, Wenhao Yang, Shusen Wang, and Zhihua Zhang.
\newblock On the convergence of fedavg on non-iid data.
\newblock {\em arXiv preprint arXiv:1907.02189}, 2019.

\bibitem{liang19fed_transfer_RL_arxiv}
Xinle Liang, Yang Liu, Tianjian Chen, Ming Liu, and Qiang Yang.
\newblock Federated transfer reinforcement learning for autonomous driving.
\newblock {\em arXiv preprint arXiv:1910.06001}, 2019.

\bibitem{liu2015finite}
Bo~Liu, Ji~Liu, Mohammad Ghavamzadeh, Sridhar Mahadevan, and Marek Petrik.
\newblock Finite-sample analysis of proximal gradient {TD} algorithms.
\newblock In {\em Proceedings of the Thirty-First Conference on Uncertainty in
  Artificial Intelligence}, pages 504--513, 2015.

\bibitem{liu19lifelong_fedRL_ieee}
Boyi Liu, Lujia Wang, and Ming Liu.
\newblock Lifelong federated reinforcement learning: a learning architecture
  for navigation in cloud robotic systems.
\newblock {\em IEEE Robotics and Automation Letters}, 4(4):4555--4562, 2019.

\bibitem{liu2018representation}
Yao Liu, Omer Gottesman, Aniruddh Raghu, Matthieu Komorowski, Aldo~A Faisal,
  Finale Doshi-Velez, and Emma Brunskill.
\newblock {Representation Balancing MDPs for Off-policy Policy Evaluation}.
\newblock {\em Advances in Neural Information Processing Systems},
  31:2644--2653, 2018.

\bibitem{maddern20171}
Will Maddern, Geoffrey Pascoe, Chris Linegar, and Paul Newman.
\newblock 1 year, 1000 km: The oxford robotcar dataset.
\newblock {\em The International Journal of Robotics Research}, 36(1):3--15,
  2017.

\bibitem{maei2018convergent}
Hamid~Reza Maei.
\newblock Convergent actor-critic algorithms under off-policy training and
  function approximation.
\newblock {\em Preprint arXiv:1802.07842}, 2018.

\bibitem{fedavg17aistats}
Brendan McMahan, Eider Moore, Daniel Ramage, Seth Hampson, and Blaise~Aguera
  y~Arcas.
\newblock Communication-efficient learning of deep networks from decentralized
  data.
\newblock In {\em Artificial Intelligence and Statistics}, pages 1273--1282.
  PMLR, 2017.

\bibitem{mothukuri21privacyFL_elsevier}
Viraaji Mothukuri, Reza~M Parizi, Seyedamin Pouriyeh, Yan Huang, Ali
  Dehghantanha, and Gautam Srivastava.
\newblock A survey on security and privacy of federated learning.
\newblock {\em Future Generation Computer Systems}, 115:619--640, 2021.

\bibitem{moulines2011non}
Eric Moulines and Francis~R Bach.
\newblock Non-asymptotic analysis of stochastic approximation algorithms for
  machine learning.
\newblock In {\em Advances in Neural Information Processing Systems}, pages
  451--459, 2011.

\bibitem{nguyen21FL_IoT_arxiv}
Dinh~C Nguyen, Ming Ding, Pubudu~N Pathirana, Aruna Seneviratne, Jun Li, and
  H~Vincent Poor.
\newblock Federated learning for internet of things: A comprehensive survey.
\newblock {\em arXiv preprint arXiv:2104.07914}, 2021.

\bibitem{pinto2023federated}
Euclides~Carlos Pinto~Neto, Somayeh Sadeghi, Xichen Zhang, and Sajjad Dadkhah.
\newblock Federated reinforcement learning in iot: applications, opportunities
  and open challenges.
\newblock {\em Applied Sciences}, 13(11):6497, 2023.

\bibitem{puterman2014Markov}
Martin~L Puterman.
\newblock {\em Markov decision processes: discrete stochastic dynamic
  programming}.
\newblock John Wiley \& Sons, 2014.

\bibitem{qi2021federated}
Jiaju Qi, Qihao Zhou, Lei Lei, and Kan Zheng.
\newblock Federated reinforcement learning: Techniques, applications, and open
  challenges.
\newblock {\em arXiv preprint arXiv:2108.11887}, 2021.

\bibitem{qu2020finite}
Guannan Qu and Adam Wierman.
\newblock {Finite-Time Analysis of Asynchronous Stochastic Approximation and
  $Q$-Learning}.
\newblock In {\em Conference on Learning Theory}, pages 3185--3205. PMLR, 2020.

\bibitem{qu2020federated}
Zhaonan Qu, Kaixiang Lin, Jayant Kalagnanam, Zhaojian Li, Jiayu Zhou, and
  Zhengyuan Zhou.
\newblock Federated learning's blessing: Fedavg has linear speedup.
\newblock {\em arXiv preprint arXiv:2007.05690}, 2020.

\bibitem{rakhlin12SGD_ICML}
Alexander Rakhlin, Ohad Shamir, and Karthik Sridharan.
\newblock Making gradient descent optimal for strongly convex stochastic
  optimization.
\newblock In {\em Proceedings of the 29th International Coference on
  International Conference on Machine Learning}, pages 1571--1578, 2012.

\bibitem{ren19FL_IoT_ieee}
Jianji Ren, Haichao Wang, Tingting Hou, Shuai Zheng, and Chaosheng Tang.
\newblock Federated learning-based computation offloading optimization in edge
  computing-supported internet of things.
\newblock {\em IEEE Access}, 7:69194--69201, 2019.

\bibitem{robbinsmonroSA}
Herbert Robbins and Sutton Monro.
\newblock A stochastic approximation method.
\newblock {\em The annals of mathematical statistics}, pages 400--407, 1951.

\bibitem{shah2018q}
Devavrat Shah and Qiaomin Xie.
\newblock {$Q$}-learning with nearest neighbors.
\newblock In {\em Advances in Neural Information Processing Systems}, pages
  3111--3121, 2018.

\bibitem{shaik2024framu}
Thanveer Shaik, Xiaohui Tao, Lin Li, Haoran Xie, Taotao Cai, Xiaofeng Zhu, and
  Qing Li.
\newblock Framu: Attention-based machine unlearning using federated
  reinforcement learning.
\newblock {\em IEEE Transactions on Knowledge and Data Engineering}, 2024.

\bibitem{shalev2012online}
Shai Shalev-Shwartz et~al.
\newblock Online learning and online convex optimization.
\newblock {\em Foundations and Trends{\textregistered} in Machine Learning},
  4(2):107--194, 2012.

\bibitem{shao2019survey}
Kun Shao, Zhentao Tang, Yuanheng Zhu, Nannan Li, and Dongbin Zhao.
\newblock A survey of deep reinforcement learning in video games.
\newblock {\em arXiv preprint arXiv:1912.10944}, 2019.

\bibitem{hong20asynch_a3c_arxiv}
Han Shen, Kaiqing Zhang, Mingyi Hong, and Tianyi Chen.
\newblock Asynchronous advantage actor critic: Non-asymptotic analysis and
  linear speedup.
\newblock {\em arXiv preprint arXiv:2012.15511}, 2020.

\bibitem{spiridonoff21comm_eff_SGD_neurips}
Artin Spiridonoff, Alex Olshevsky, and Ioannis~Ch Paschalidis.
\newblock Communication-efficient sgd: From local sgd to one-shot averaging.
\newblock In {\em Advances in Neural Information Processing Systems},
  volume~34, 2021.

\bibitem{srikant2019finite}
Rayadurgam Srikant and Lei Ying.
\newblock Finite-time error bounds for linear stochastic approximation and {TD}
  learning.
\newblock In {\em Conference on Learning Theory}, pages 2803--2830. PMLR, 2019.

\bibitem{stich18localSGD_iclr}
Sebastian~U Stich.
\newblock Local sgd converges fast and communicates little.
\newblock In {\em International Conference on Learning Representations}, 2018.

\bibitem{sun20decenTD_LFA_aistats}
Jun Sun, Gang Wang, Georgios~B Giannakis, Qinmin Yang, and Zaiyue Yang.
\newblock Finite-time analysis of decentralized temporal-difference learning
  with linear function approximation.
\newblock In {\em International Conference on Artificial Intelligence and
  Statistics}, pages 4485--4495. PMLR, 2020.

\bibitem{sun2024understanding}
Zhenyu Sun, Xiaochun Niu, and Ermin Wei.
\newblock Understanding generalization of federated learning via stability:
  Heterogeneity matters.
\newblock In {\em International Conference on Artificial Intelligence and
  Statistics}, pages 676--684. PMLR, 2024.

\bibitem{sutton1988learning}
Richard~S Sutton.
\newblock Learning to predict by the methods of temporal differences.
\newblock {\em Machine learning}, 3(1):9--44, 1988.

\bibitem{suttonbartorlbook}
Richard~S Sutton and Andrew~G Barto.
\newblock {\em {Reinforcement learning: An introduction}}.
\newblock MIT press, 2018.

\bibitem{szepesvari2022algorithms}
Csaba Szepesv{\'a}ri.
\newblock {\em Algorithms for reinforcement learning}.
\newblock Springer nature, 2022.

\bibitem{tadic2001convergence}
Vladislav Tadi{\'c}.
\newblock On the convergence of temporal-difference learning with linear
  function approximation.
\newblock {\em Machine learning}, 42(3):241--267, 2001.

\bibitem{truex19privacyFL_acm}
Stacey Truex, Nathalie Baracaldo, Ali Anwar, Thomas Steinke, Heiko Ludwig, Rui
  Zhang, and Yi~Zhou.
\newblock A hybrid approach to privacy-preserving federated learning.
\newblock In {\em Proceedings of the 12th ACM Workshop on Artificial
  Intelligence and Security}, pages 1--11, 2019.

\bibitem{tsitsiklis1994asynchronous}
John~N Tsitsiklis.
\newblock {Asynchronous stochastic approximation and $Q$-learning}.
\newblock {\em Machine learning}, 16(3):185--202, 1994.

\bibitem{tsitsiklis1997analysis}
John~N Tsitsiklis and Benjamin Van~Roy.
\newblock Analysis of temporal-difference learning with function approximation.
\newblock In {\em Advances in neural information processing systems}, pages
  1075--1081, 1997.

\bibitem{wai20dist_markov_noLS_cdc}
Hoi-To Wai.
\newblock On the convergence of consensus algorithms with markovian noise and
  gradient bias.
\newblock In {\em 2020 59th IEEE Conference on Decision and Control (CDC)},
  pages 4897--4902. IEEE, 2020.

\bibitem{wainwright2019stochastic}
Martin~J Wainwright.
\newblock Stochastic approximation with cone-contractive operators: Sharp
  $\ell_\infty$-bounds for ${Q}$-learning.
\newblock {\em Preprint arXiv:1905.06265}, 2019.

\bibitem{wang20decen_TD_neurips}
Gang Wang, Songtao Lu, Georgios Giannakis, Gerald Tesauro, and Jian Sun.
\newblock {Decentralized TD Tracking with Linear Function Approximation and its
  Finite-Time Analysis}.
\newblock In H.~Larochelle, M.~Ranzato, R.~Hadsell, M.~F. Balcan, and H.~Lin,
  editors, {\em Advances in Neural Information Processing Systems}, volume~33,
  pages 13762--13772. Curran Associates, Inc., 2020.

\bibitem{wang2024momentum}
Han Wang, Sihong He, Zhili Zhang, Fei Miao, and James Anderson.
\newblock Momentum for the win: Collaborative federated reinforcement learning
  across heterogeneous environments.
\newblock {\em arXiv preprint arXiv:2405.19499}, 2024.

\bibitem{wang2023federated}
Han Wang, Aritra Mitra, Hamed Hassani, George~J Pappas, and James Anderson.
\newblock Federated temporal difference learning with linear function
  approximation under environmental heterogeneity.
\newblock {\em arXiv preprint arXiv:2302.02212}, 2023.

\bibitem{wang21coopSGD_jmlr}
Jianyu Wang and Gauri Joshi.
\newblock Cooperative sgd: A unified framework for the design and analysis of
  local-update sgd algorithms.
\newblock {\em Journal of Machine Learning Research}, 22(213):1--50, 2021.

\bibitem{wang2020fedRL_ieeeIoT}
Xiaofei Wang, Chenyang Wang, Xiuhua Li, Victor~CM Leung, and Tarik Taleb.
\newblock Federated deep reinforcement learning for internet of things with
  decentralized cooperative edge caching.
\newblock {\em IEEE Internet of Things Journal}, 7(10):9441--9455, 2020.

\bibitem{watkins1992q}
Christopher~JCH Watkins and Peter Dayan.
\newblock {$Q$}-learning.
\newblock {\em Machine learning}, 8(3-4):279--292, 1992.

\bibitem{woo2023blessing}
Jiin Woo, Gauri Joshi, and Yuejie Chi.
\newblock The blessing of heterogeneity in federated q-learning: Linear speedup
  and beyond.
\newblock In {\em International Conference on Machine Learning}, pages
  37157--37216. PMLR, 2023.

\bibitem{xu2021multi}
Minrui Xu, Jialiang Peng, BB~Gupta, Jiawen Kang, Zehui Xiong, Zhenni Li, and
  Ahmed~A Abd El-Latif.
\newblock Multi-agent federated reinforcement learning for secure incentive
  mechanism in intelligent cyber-physical systems.
\newblock {\em IEEE Internet of Things Journal}, 2021.

\bibitem{yang2019federated}
Qiang Yang, Yang Liu, Yong Cheng, Yan Kang, Tianjian Chen, and Han Yu.
\newblock Federated learning.
\newblock {\em Synthesis Lectures on Artificial Intelligence and Machine
  Learning}, 13(3):1--207, 2019.

\bibitem{yang2023federated}
Tong Yang, Shicong Cen, Yuting Wei, Yuxin Chen, and Yuejie Chi.
\newblock Federated natural policy gradient methods for multi-task
  reinforcement learning.
\newblock {\em arXiv preprint arXiv:2311.00201}, 2023.

\bibitem{yurtsever2020survey}
Ekim Yurtsever, Jacob Lambert, Alexander Carballo, and Kazuya Takeda.
\newblock A survey of autonomous driving: Common practices and emerging
  technologies.
\newblock {\em IEEE Access}, 8:58443--58469, 2020.

\bibitem{doan20decen_sa_noLS_arxiv}
Sihan Zeng, Thinh~T Doan, and Justin Romberg.
\newblock Finite-time analysis of decentralized stochastic approximation with
  applications in multi-agent and multi-task learning.
\newblock {\em arXiv preprint arXiv:2010.15088}, 2020.

\bibitem{zhang2022multi}
Sai~Qian Zhang, Jieyu Lin, and Qi~Zhang.
\newblock A multi-agent reinforcement learning approach for efficient client
  selection in federated learning.
\newblock {\em arXiv preprint arXiv:2201.02932}, 2022.

\bibitem{zhang2019provably}
Shangtong Zhang, Bo~Liu, Hengshuai Yao, and Shimon Whiteson.
\newblock Provably convergent two-timescale off-policy actor-critic with
  function approximation.
\newblock In {\em International Conference on Machine Learning}, pages
  11204--11213. PMLR, 2020.

\bibitem{zhao21IoV_FedRL_ICC}
Lei Zhao, Yongyi Ran, Hao Wang, Junxia Wang, and Jiangtao Luo.
\newblock Towards cooperative caching for vehicular networks with multi-level
  federated reinforcement learning.
\newblock In {\em ICC 2021-IEEE International Conference on Communications},
  pages 1--6. IEEE, 2021.

\bibitem{zhou2023digital}
Xiaokang Zhou, Xuzhe Zheng, Xuesong Cui, Jiashuai Shi, Wei Liang, Zheng Yan,
  Laurance~T Yang, Shohei Shimizu, I~Kevin, and Kai Wang.
\newblock Digital twin enhanced federated reinforcement learning with
  lightweight knowledge distillation in mobile networks.
\newblock {\em IEEE Journal on Selected Areas in Communications}, 2023.

\bibitem{zhuo19FedDeepRL_arxiv}
Hankz~Hankui Zhuo, Wenfeng Feng, Yufeng Lin, Qian Xu, and Qiang Yang.
\newblock Federated deep reinforcement learning.
\newblock {\em arXiv preprint arXiv:1901.08277}, 2019.

\end{thebibliography}
\bibliographystyle{plain}
\begin{appendix}
\pagebreak

\begin{center}
    {\LARGE \textbf{Appendices}}
\end{center}
The appendices are organized as follows. In Section \ref{sec:lower_bound_stoch_app}, we discuss the lower bound on the convergence of a general stochastic approximation. 
In Section \ref{sec:fed_stoch_analysis_app} we derive the convergence bound of the FeGSAM algorithm. Next, we employ the results in Section \ref{sec:fed_stoch_analysis_app} to derive the convergence bounds of federated TD-learning in Section \ref{sec:fed_td_learning_app} and federated $Q$-learning in Section \ref{sec:fed_Q_learning_app}.

\section{Lower Bound on the Convergence of General Stochastic Approximation}\label{sec:lower_bound_stoch_app}

In this section, we discuss the convergence of the general stochastic approximation. In this discussion, we provide a simple stochastic approximation with i.i.d. noise, which can give insight into the general convergence bound in \eqref{eq:general_stoch_conver}. In particular, we show that the convergence bound in \eqref{eq:general_stoch_conver} is tight and cannot be improved.

Consider a one-dimensional random variable $X$ with zero mean $\E[X]=0$ and bounded variance $\E[X^2]=\sigma^2$. Consider the following update
\begin{align}\label{eq:simple_iid_update}
    x_{t+1} = x_t + \alpha (X_t-x_t), \quad t\geq 0,
\end{align}
where we start with some fixed deterministic $x_0$ and $X_t$ is a an i.i.d. sample of the random variable $X$. It is easy to see that the update \eqref{eq:simple_iid_update} is a special case of the update of the general stochastic approximation with the fixed point $x^*=0$.

By expanding the update \eqref{eq:simple_iid_update}, we have
\begin{align*}
    x_t = (1-\alpha)^t x_0 + \alpha\sum_{k=0}^{t-1}(1-\alpha)^{t-k-1} X_k.
\end{align*}
Hence, we have
\begin{align*}
    x_t^2 = (1-\alpha)^{2t} x_0^2 + \left(\alpha\sum_{k=0}^{t-1}(1-\alpha)^{t-k-1} X_k\right)^2 + 2\alpha(1-\alpha)^t x_0\sum_{k=0}^{t-1}(1-\alpha)^{t-k-1} X_k.
\end{align*}
Taking expectation on both sides, and using the zero mean property of $X_k$, we have
\begin{align}
    \E[x_t^2] = &(1-\alpha)^{2t} x_0^2 + \alpha^2\E\left(\sum_{k=0}^{t-1}(1-\alpha)^{t-k-1} X_k\right)^2\nn\\
    = & (1-\alpha)^{2t} x_0^2 + \alpha^2\E\left(\sum_{k=0}^{t-1} (1-\alpha)^{2(t-k-1)}X_k^2 + \sum_{k,k'=0, k\neq k'}^{t-1}(1-\alpha)^{2t-k-k'-2} X_kX_k'\right)\nn\\
    = & (1-\alpha)^{2t} x_0^2 + \alpha^2\left(\sum_{k=0}^{t-1} (1-\alpha)^{2(t-k-1)}\sigma^2 + \sum_{k,k'=0, k\neq k'}^{t-1}(1-\alpha)^{2t-k-k'-2} \E[X_k]\E[X_k']\right)\tag{i.i.d. property}\\
    = & (1-\alpha)^{2t} x_0^2 + \alpha^2\left(\sum_{k=0}^{t-1} (1-\alpha)^{2(t-k-1)}\sigma^2 \right)\tag{zero mean}\\
    = & (1-\alpha)^{2t} x_0^2 + \alpha^2\sigma^2\frac{1-(1-\alpha)^{2t}}{1-(1-\alpha)^2}\nn\\
    = & x_0^2(1-\alpha)^{2t}  + \sigma^2\frac{1-(1-\alpha)^{2t}}{2-\alpha}\alpha\nn\\
    = & \underbrace{(x_0^2-\frac{\alpha\sigma^2}{2-\alpha})(1-\alpha)^{2t}}_{T_1: \text{bias}}  + \underbrace{\frac{\sigma^2}{2-\alpha}\alpha}_{T_2: \text{variance}}.\label{eq:tight_bound}
\end{align}
It is clear that \eqref{eq:tight_bound} has the same form as the bound in \eqref{eq:general_stoch_conver} with $T_1$ as the geometric term which converges to zero as $t\rightarrow \infty$, and $T_2$ term proportional to the step size $\alpha$. In addition, note that in the above derivation, we did not use any inequality and therefore the bounds in \eqref{eq:tight_bound} and \eqref{eq:general_stoch_conver} are tight. 

\section{Analysis of Federated Stochastic Approximation} \label{sec:fed_stoch_analysis_app}
First, we restate Theorem
\ref{thm:main_body} with explicit expressions of the different constants.
\begin{theorem}\label{thm:main}
Consider the FeGSAM Algorithm \ref{alg:fed_stoch_app} with $c_{FSAM}= 1 - \frac{\alpha \varphiz_2}{2}$ ($\varphiz_2$ is defined in \eqref{eq:def:constants}) and suppose that Assumptions \ref{ass:mixing}, \ref{ass:contraction}, \ref{ass:lipschitz}, \ref{ass:noise_independence} are satisfied. Consider a sufficiently small step size $\alpha$ that satisfies the assumptions in \eqref{eq:alpha_const_1}, \eqref{eq:alpha_const_2}, \eqref{eq:alpha_const_3}, 
\eqref{eq:alpha_const_4}. Furthermore, denote $\tau_\alpha=\lceil 2\log_\rho\alpha\rceil$, and take large enough $T$ such that $T>\max\{K+\tau_\alpha,2\tau_\alpha\}$. Then, the output of the FeGSAM Algorithm \ref{alg:fed_stoch_app}, $\btheta_{\hat{T}} \triangleq \frac{1}{\na} \sumik \btheta_{\hat{T}}^i$, satisfies
\begin{align}
    \E[\|\btheta_{\hat{T}}\|_c^2]\leq \mathcal{C}_1 \frac{1}{\alpha} \lp 1 - \frac{\alpha \varphiz_2}{2} \rp^{T-2\tau_\alpha+1} +\mathcal{C}_2  \frac{\alpha\tau_\alpha^2}{\na}+ \mathcal{C}_3  (\sync-1) \alpha^2\tau_\alpha + \mathcal{C}_4(K-1)^2\varepsilon\alpha^2 \tau_\alpha +\mathcal{C}_5\alpha^3\tau_\alpha^2,\label{eq:thm_main_eq}
\end{align}
where $\mathcal{C}_1=16u_{cm}^2 M_0(\log_\rho \frac{1}{e}+\frac{1}{\varphiz_2})$, $M_0 = \frac{1}{l_{cm}^2}\left(\frac{1}{C_1^2} \lp B + (A_2 + 1) \lp \|\btheta_0\|_c + \frac{B}{2 C_1} \rp \rp^2+\|\btheta_0\|_c^2\right)$, $\mathcal{C}_2 = 8u_{cm}^2\lp C_8\frac{u_{cD}^2B^2}{l_{cD}^2} + \frac{1}{2} + C_{12}\rp/\varphiz_2$, and $\mathcal{C}_3 = 80\tb^2 C_{17}u_{cm}^2\lp 1 + \frac{4 (m_2+ m_4) \rho}{B (1 - \rho)} \rp/\varphiz_2$, $\mathcal{C}_4= \frac{320u_{cm}^2\tb^2 C_{17}m_3}{\varphiz_2B}$ and $\mathcal{C}_5 = 8u_{cm}^2  \lp C_7 + C_{11}   + 0.5 C_3^2 C_9^2   + C_3 C_{10}  + 2 C_1 C_3 C_{10} + 3 A_1 C_3  + C_{13} + C_8 \frac{u_{cD}^2}{l_{cD}^2} 2 m^2_2 \alpha^2\rp/\varphiz_2 $. Here $u_{cm},l_{cm}, \varphiz_2,C_1, C_3,C_7,C_8, C_9,C_{10},C_{11},C_{12},C_{13}, C_{17},\tb$ are problem-dependent constants which are defined in the following proposition and lemmas.
\end{theorem}

Next, we characterize the sample complexity of the FeGSAM Algorithm \ref{alg:fed_stoch_app}, where we establish a linear speedup in the convergence of the algorithm.
\begin{corollary}\label{lem:sample_complexity} Consider FeGSAM Algorithm \ref{alg:fed_stoch_app} with fixed number of iterations $T$ and step size $\alpha=\frac{8\log(\na T)}{\varphiz_2 T}$. Suppose $T$ is large enough, such that $\alpha$ satisfies the requirements of the step size in Theorem \ref{thm:main} and $T>\max\{4\tau_\alpha,N\}$. Furthermore, for the choice of $K$, assume
\begin{itemize}
    \item if $\varepsilon>0$, take $\sync=\sqrt{T}/\na$.
    \item if $\varepsilon=0$, take $\sync=T/\na$.
\end{itemize}

Then we have
$\E[\|\btheta_{\hat{T}}\|_c^2] \leq\epsilon$ after $T=\mathcal{O}\lp\frac{1}{\na \epsilon}\rp$ iterations.
\end{corollary}

Corollary \ref{lem:sample_complexity} establishes the sample and communication complexity of the FeGSAM Algorithm \ref{alg:fed_stoch_app}. The $\mco (1/(\na \epsilon))$ sample complexity shows the linear speedup with increasing $\na$. Another aspect of the cost is the number of communications required between the agents and the central server. According to Corollary \ref{lem:sample_complexity}, when $\varepsilon>0$, we need $T/K = \tilde{\mathcal{O}}(\sqrt{N}/\sqrt{\epsilon})$ number of communications. Furthermore, when $\varepsilon=0$, we need $T/K = \tilde{\mathcal{O}}(\na)$ rounds of communication in order to reach an $\epsilon$-optimal solution. Hence, even in the presence of Markov noise, in the homogeneous setting, the required number of communications is independent of the desired final accuracy $\epsilon$, and grows linearly with the number of agents. Our result generalizes the existing result achieved for the simpler i.i.d. noise case in \cite{richtarik20localSGD_aistats, spiridonoff21comm_eff_SGD_neurips}.

In the following sections, we discuss the proof of Theorem \ref{thm:main}. In Section \ref{sec:FedSAM_prelim}, we introduce some notation and preliminary results to facilitate our analysis based on the Lyapunov function. Next, in Section \ref{sec:FedSAM_thm_proof}, we state some primary propositions, which are then used to prove Theorem \ref{thm:main}. In Sections \ref{sec:FedSAM_main_prop_proofs_1}, \ref{sec:FedSAM_main_prop_proofs_2},\ref{sec:FedSAM_main_prop_proofs_3}, we prove the aforementioned propositions. Along the way, we state several intermediate results, which are stated and proved in Sections \ref{sec:FedSAM_aux_lem} and \ref{sec:FedSAM_aux_lem_proofs}, respectively. 

Throughout the appendix we have several sets of constants. The constants $C_i, i=1,\dots, 17$ are problem-dependent constants which we define recursively. The final constants that appear in the resulting bound in Theorem \ref{thm:main} are shown as $\mathcal{C}_i, i=1,\dots,4$. Finally, the constant $c$ is used in the sampling of the time step $\hat{T}$. 

\subsection{Preliminaries}
\label{sec:FedSAM_prelim}
We define the following notation:
\begin{itemize}
    \item $\bthetat \triangleq \frac{1}{\na} \sumik \bthetait$ : virtual sequence of average (across agents) parameter.
    \item $\boldsymbol{\Theta}_t = \lcb \btheta^1_t, \hdots, \btheta^\na_t \rcb$: set of local parameters at individual nodes.
    \item $\byt = \lcb \by_t^1, \hdots, \by_t^\na \rcb$: Markov chains at individual nodes.
    \item $\mu^i$: the stationary distribution of $\byit$ as $t\rightarrow\infty$.
    \item $\mbf G (\boldsymbol{\Theta}_t, \byt) \triangleq \frac{1}{\na} \sumik  \Gi(\bthetait, \byit)$: average of the noisy local operators at the individual local parameters.
    \item $\mbf G ( \bthetat, \byt) \triangleq \frac{1}{\na} \sumik  \Gi(\bthetat, \byit)$ : average of the noisy local operators at the average parameter. 
    \item $\mbf b(\byt) \triangleq \frac{1}{\na} \sumik  \mathbf{b}^i (\byit)$ : average of Markovian noise.
    \item $\bar{\mbf G} (\boldsymbol{\Theta}_t) \triangleq \frac{1}{\na} \sumik  \bar{\mbf G}^i(\bthetait)$ : average of the expected operators evaluated at the local parameter.
    \item $\bar{\mbf G} (\bthetat) \triangleq \frac{1}{\na} \sumik  \bar{\mbf G}^i(\bthetat)$ : average of the expected operator evaluated at the average parameter.
    \item $\btheta^{i,*}$ is the unique fixed point that satisfies $\btheta^{i,*} = \bar{\mbf G}^i(\btheta^{i,*})$ for all $i=1,\dots,\na$. Note that this follows directly from the assumptions. In particular, by Assumption \ref{ass:contraction}, $\bar{\mbf G}^i(\cdot)$ is a contraction, and hence by the Banach fixed point theorem \cite{banach1922operations}, there exists a unique fixed point of this operator. 
    \item The vector $\mbf0$ is the unique fixed point that satisfies $\bar{\mbf G}(\mbf0)=\mbf0$. Note that the existence and uniqueness of this fixed point is due to the contractive property of $\bar{\mbf G}(\cdot)$ and the Banach fixed point theorem.
    \item $\Delta_t^i = \|\bthetat - \bthetait\|_c, \quad \Delta_t = \frac{1}{\na} \sumik \Delta_t^i, \quad \Omega_t = \frac{1}{\na} \sumik (\Delta_t^i)^2$: measures of the synchronization error.
\end{itemize}
Throughout this proof, we assume $\|\cdot\|_c$ as some given norm. $\E_t[\cdot] \triangleq \E[\cdot|\mathcal{F}_t]$, where $\mathcal{F}_t$ is the sigma algebra generated by $\{\btheta_r^i\}_{r=1,\dots,t}^{i=1,\dots, \na}$. Unless otherwise specified, $\|\cdot\|$ denotes the Euclidean norm.

\textbf{Generalized Moreau Envelope:}
Consider the norm $\|\cdot\|_c$ which appears in Assumptions \ref{ass:mixing}-\ref{ass:lipschitz}. The square of this norm does not need to be smooth. Inspired by \cite{chen2020finite}, we use the Generalized Moreau Envelope as a Lyapunov function for the analysis of the convergence of Algorithm \ref{alg1}. The Generalized Moreau Envelope of $f(\cdot)$ with respect to $g(\cdot)$, for $\psi > 0$, is defined as 
\begin{align}
    \M_f^{\psi, g}(\bx) = \min_{u \in \mathbb{R}^d} \left\{ f(\bu) + \frac{1}{\psi} g(\bx - \bu) \right\}. \label{eq:Moreau_envelope}
\end{align}
Let $f(\bx) = \frac{1}{2}\|\bx\|_c^2$ and $g(\bx) = \frac{1}{2}\|\bx\|_s^2$, which is $L$-smooth with respect to $\|\cdot\|_s$ norm. For this choice of $f,g$, $\M_f^{\psi, g}(\cdot)$ is essentially a smooth approximation to $f$, which is henceforth denoted with the simpler notation $\M(\cdot)$. Also, due to the equivalence of norms, there exist $l_{cs}, u_{cs} > 0$ such that
\begin{align}\label{eq:c_to_s}
    l_{cs}\|\cdot\|_s\leq\|\cdot\|_c\leq u_{cs}\|\cdot\|_{s}. 
\end{align}
We next summarize the properties of $M(\cdot)$ in the following proposition, which were established in \cite{chen2020finite}.

\begin{proposition}[\cite{chen2020finite_nips}]\label{prop:Moreau}
The function $\M(\cdot)$ satisfies the following properties.
\begin{enumerate}
	\item $\M(\cdot)$ is convex, and $\frac{L}{\psi}$-smooth with respect to $\|\cdot\|_s$. That is, $\M(\by)\leq \M(\bx)+\langle \nabla \M(\bx),\by-\bx\rangle+\frac{L}{2\psi}\|\bx-\by\|_s^2$ for all $\bx,\by\in\mathbb{R}^d$.
	\item There exists a norm, denoted by $\|\cdot\|_m$, such that $\M(\bx)=\frac{1}{2}\|\bx\|_m^2$.
	\item Let $\ell_{cm}=(1+\psi \ell_{cs}^2)^{1/2}$ and $u_{cm}=(1+\psi u_{cs}^2)^{1/2}$. Then we have $\ell_{cm}\|\cdot\|_m\leq \|\cdot\|_c\leq u_{cm}\|\cdot\|_m$. 
\end{enumerate}
\end{proposition}
By Proposition \ref{prop:Moreau}, we can use $M(\cdot)$ as a smooth surrogate for $\frac{1}{2}\|\cdot\|_c^2$. Furthermore, we denote
\begin{align}\label{eq:def:constants}
	\varphiz_1 = \frac{1+\psi u_{cs}^2}{1+\psi \ell_{cs}^2},\quad  \varphiz_2=1-\gamma_c\varphiz_1^{1/2}, \quad \text{and} \quad \varphiz_3=\frac{114L (1+\psi u_{cs}^2)}{\psi\ell_{cs}^2}.
\end{align}
Note that by choosing $\psi>0$ small enough, we can ensure $\varphiz_2 \in (0,1)$. 
\subsection{Proof of Theorem \ref{thm:main}}
\label{sec:FedSAM_thm_proof}
In this section, first we state three key results (Propositions \ref{prop:intermediate_lem}, \ref{prop:intermediate_lem2}, \ref{prop:wtd_consensus_error}). These are then used to prove Theorem \ref{thm:main}.
The first step of the proof is to characterize the one-step drift of the Lyapunov function $M (\cdot)$, with the parameters generated by the FeGSAM Algorithm \ref{alg:fed_stoch_app}, which is formally stated in the following proposition.

\begin{proposition}[One-step drift - I]\label{prop:intermediate_lem}
Consider the update of the FeGSAM Algorithm \ref{alg:fed_stoch_app}. Suppose that Assumptions \ref{ass:mixing}, \ref{ass:contraction}, \ref{ass:lipschitz},  and \ref{ass:noise_independence} are satisfied. Consider $\tau=\lceil 2\log_\rho \alpha\rceil$ and $t\geq 2\tau$, we have
\begin{align}
    & \E_{t-2\tau} \left[ M (\btheta_{t+1}) \right] \nn\\
    & \leq \bigg( 1 - \alpha \lb 2  \varphiz_2 - \zetaone^2 - \zetatwo^2 \frac{2 L A_2}{\psi l_{cs}^2} - \zetathree^2 \frac{L (A_1 + 1)}{\psi l_{cs}^2} - \zetafive^2 \rb + \alpha^2 \frac{6 L u_{cm}^2}{l_{cs}^2 \psi} \lb m_1 + \frac{(A_2+1)^2}{2} \rb \bigg) \E_{t-2\tau} \left[ M (\btheta_{t}) \right] \nn \\
    & + \alpha^2 \underbrace{\left( \frac{ m_4 L}{\psi l_{cs}^2}  \right)}_{C_{3}} \E_{t-2\tau} \lb \| \btheta_{t} - \btheta_{t-\tau} \|_c \rb \label{eq:C_3} \\
    & + \underbrace{\frac{L}{\psi l_{cs}^2} \left( \frac{L}{2 \psi l_{cs}^2} + \alpha A_2 \lp 1 + \tfrac{u_{cm}^2}{\zetatwo^2} \rp + \alpha (A_1+1) \lp \tfrac{3 u_{cm}^2}{2 \zetathree^2} + 2 \rp + 3 m_1 \alpha^2 \right)}_{C_{4}} \E_{t-2\tau} [\| \bthetat - \btheta_{t-\tau} \|_c^2] \label{eq:C_4} \\
    & + \alpha \underbrace{\frac{L}{\psi l_{cs}^2} \left( \lp \frac{3 u_{cm}^2}{2 \zetathree^2} + \frac{3}{2} \rp (A_1+1) + A_2 + \frac{L u_{cm}^2 }{2 l_{cs}^2 \zetafive^2  \psi} + \frac{3 A_1^2 \alpha^2}{2} \right)}_{C_{5}} \E_{t-2\tau} [\Omega_t] \label{eq:C_5} \\
    & + \alpha \underbrace{\frac{L}{2 l_{cs}^2 \psi} \left( \lp \frac{3 u_{cm}^2}{\zetathree^2} + 3 \rp (A_1+1) + m_1 \alpha \right)}_{C_6} \E_{t-2\tau} [\Omega_{t-\tau}] \label{eq:C_6} \\
    & + \underbrace{ \frac{m^2_4 u_{cm}^2 L^2}{2 \zetaone^2 l_{cs}^4 \psi^2}\alpha }_{C_7} \alpha^4 \label{eq:C_7} \\
    & + \underbrace{\frac{1}{2} \left( 1 + \frac{3 L}{\psi l_{cs}^2} \right)}_{C_8} \alpha^2 \E_{t-2\tau} \lb \| \mbf b(\byt)\|_c^2 \rb, \label{eq:C_8}
\end{align}
where $\zetaone$, $\zetatwo$, $\zetathree$, and $\zetafive$ are arbitrary positive constants.
\end{proposition}

\begin{proof}
The proof of Proposition \ref{prop:intermediate_lem} is presented in full detail in Section \ref{sec:FedSAM_main_prop_proofs_1}.
\end{proof}

Before discussing the bound in its full generality, we discuss a few special cases.
\begin{itemize}
    \item Perfect synchronization $(\sync=1)$ with i.i.d. noise: since $\Ot = 0$ for all $t$, $\tau = 0$ (independence across time) and $C_7=0$ (see Lemmas \ref{lem:T_2} and \ref{lem:T_5} in Section \ref{sec:FedSAM_aux_lem}), the terms \eqref{eq:C_3}, \eqref{eq:C_4} \eqref{eq:C_5},\eqref{eq:C_6}, \eqref{eq:C_7} will not appear in the bound, which is the form we get for centralized systems with i.i.d. noise \cite{rakhlin12SGD_ICML}.
    \item Infrequent synchronization $(\sync > 1)$ with i.i.d. noise: $\tau = 0$ , and $C_7=0$, the terms \eqref{eq:C_3}, \eqref{eq:C_4}, \eqref{eq:C_6}, \eqref{eq:C_7} will not appear in the bound, which is the form we get for federated stochastic optimization with i.i.d. noise \cite{richtarik20localSGD_aistats}.
    \item Perfect synchronization $(\sync=1)$ with Markov noise: since $\Ot = 0$ for all $t$, the bound in Proposition \ref{prop:intermediate_lem} generalizes the results in \cite{srikant2019finite, chen2021Lyapunov_arxiv}.
\end{itemize}
Next, we substitute the bound on $\|\btheta_t-\btheta_{t-\tau}\|_c$ (Lemma \ref{lem:theta_diff}) and $\|\btheta_t-\btheta_{t-\tau}\|_c^2$ (Lemma \ref{lem:theta_diff_2}) to further bound the one-step drift. Establishing a tight bound for these two quantities is essential to ensure linear speedup. 

\begin{proposition}[One-step drift - II]\label{prop:intermediate_lem2}
Consider the update of the FeGSAM Algorithm \ref{alg:fed_stoch_app}. Suppose Assumptions \ref{ass:mixing}, \ref{ass:contraction}, \ref{ass:lipschitz},  and \ref{ass:noise_independence} are satisfied. Define $C_{15}(\tau)=\lp\frac{6m_1Lu_{cm}^2}{l_{cs}^2\psi}+\frac{6L(A_2+1)^2u_{cm}^2}{2 \psi l_{cs}^2}+144 (A_2^2+1) C_4u_{cm}^2\rp\tau^2$. For
\begin{align}\label{eq:alpha_const_1}
    \alpha\leq \min \lcb \frac{1}{360\tau (A_2^2+1)}, \frac{\varphiz_2}{2C_{15}(\tau)} \rcb,
\end{align}
$\tau=\lceil 2\log_{\rho}\alpha\rceil$, and $t\geq 2\tau$, we have
\begin{align}
    & \E_{t-2\tau} \lb M (\btheta_{t+1}) \rb \leq \left( 1 - \alpha \varphiz_2 \right) \E_{t-2\tau} \left[ M(\btheta_{t}) \right] + C_{14} (\tau) \alpha^4 + C_{16} (\tau) \frac{\alpha^2}{\na}  + \alpha C_{17} \sum_{k=t-\tau}^{t} \E_{t-2\tau} [\Omega_k], \nn
\end{align}
where $C_{14}(\tau) = \lp C_7 + C_{11}   + 0.5 C_3^2 C_9^2   + C_3 C_{10}  + 2 C_1 C_3 C_{10} + 3 A_1 C_3  + C_{13} + C_8 \frac{u_{cD}^2}{l_{cD}^2} 2 m^2_2 \alpha^2 \rp \tau^2$, $C_{16} (\tau) = \lp C_8 \frac{u_{cD}^2 B^2}{l_{cD}^2} + \frac{1}{2} + C_{12}\rp  \tau^2$ and $C_{17} = (3 A_1 C_3 + 8  A_1^2 C_4  + C_5 + C_6)$. Here we define $C_9=\frac{8 u_{cD} B}{l_{cD}}$, $C_{10}=\frac{8 m_2 u_{cD}}{l_{cD} (1-\rho)}$, $C_{11}=\frac{8 C_1^2 C_3^2 u_{cm}^2}{\zeta_{6}^2}$, $C_{12}=\frac{8 C_4 u_{cD}^2 B^2}{l_{cD}^2}$, $C_{13}= \frac{14 C_4 u_{cD}^2 m_2^2}{l_{cD}^2 (1-\rho^2)}$.
\end{proposition}

\begin{proof}
The proof of Proposition \ref{prop:intermediate_lem2} is presented in full detail in Section \ref{sec:FedSAM_main_prop_proofs_2}.
\end{proof}

\paragraph*{\textbf{Conditional expectation} $\E_{t-2\tau}[\cdot]$.}
Conditional expectation $\E_{t-2\tau}[\cdot]$ used in Proposition \ref{prop:intermediate_lem2} is essential when dealing with Markovian noise. The idea of using conditional expectation to deal with Markovian noise is not novel per se. In the previous work \cite{bhandari18FiniteTD_LFA_colt, srikant2019finite, chen2021Lyapunov_arxiv}, conditioning on $t-\tau$ is sufficient to establish the convergence results. Due to the mixing property (Assumption \ref{ass:mixing}), the Markov chain geometrically converges to its stationary distribution. Therefore, choosing ``large enough'' $\tau$, and conditioning on $t-\tau$, one can ensure that the Markov chain at time $t$ is ``almost in steady state.'' However, in federated setting, conditioning on $t-\tau$ results in bounds that are too loose. In particular, consider the differences $\|\btheta_t-\btheta_{t-\tau}\|_c$ and $\|\btheta_t-\btheta_{t-\tau}\|_c^2$ in \eqref{eq:C_3} and \eqref{eq:C_4}, respectively. In the centralized setting, as in \cite{bhandari18FiniteTD_LFA_colt, srikant2019finite, chen2021Lyapunov_arxiv}, these terms can be deterministically bounded to yield the bound $\asymp \alpha^2$. However, in the federated setting, this crude bound does not result in linear speedup in $\na$.
In this work, to achieve a finer bound on $\|\btheta_t-\btheta_{t-\tau}\|_c$, we go $\tau$ steps further back in time. This ensures that the difference behaves almost like the difference of the average of i.i.d. random variables, resulting in a tighter bound (see Lemma \ref{lem:theta_diff}). Using the conditional expectation $\E_{t-2\tau}[\cdot]$, we derive a refined analysis to bound this term as $\mathcal{O}(\alpha^2/\na + \alpha^4)$, which guarantees a linear speedup (see Lemmas \ref{lem:theta_diff} and \ref{lem:theta_diff_2}). 

Taking total expectation in Proposition \ref{prop:intermediate_lem2} (using tower property), we get
\begin{align}
    & \E \lb M (\btheta_{t+1}) \rb \leq \left( 1 - \alpha \varphiz_2 \right) \E \left[ M(\btheta_{t}) \right] + C_{14} (\tau) \alpha^4 + C_{16} (\tau) \frac{\alpha^2}{\na}  + \alpha C_{17} \sum_{k=t-\tau}^{t} \E [\Omega_k] \label{eq_prop:intermediate_lem2}
\end{align}
To understand the bound in Proposition \ref{prop:intermediate_lem2}, consider the case of $K=1$ (i.e. full synchronization). In this case, we have $\Omega_i=0$ for all $i$, and the bound in Proposition \ref{prop:intermediate_lem2} simplifies to $\E \lb M (\btheta_{t+1}) \rb\leq \left( 1 - \alpha \varphiz_2 \right) \E \left[ M(\btheta_{t}) \right] + C_{14} (\tau) \alpha^3 + C_{16} (\tau) \frac{\alpha^2}{\na}$. This recursion is sufficient to achieve linear speedup. However, the bound in Proposition \ref{prop:intermediate_lem2} also includes terms that are proportional to the error due to synchronization. In order to ensure convergence along with linear speedup, we need to further upper bound this term with terms which are of the order $\mathcal{O}(\alpha^3 (K-1))$ and $M(\btheta_i)$. The following lemma is the next important contribution of the paper, where we establish such a bound for weighted sum of the synchronization error. Notice that the weights $\{ w_t \}$ are carefully chosen to ensure the best rate of convergence for the overall algorithm.

\begin{proposition}[Synchronization Error] \label{prop:wtd_consensus_error} Suppose $T>K+\tau$. For $\alpha$ such that
\begin{equation}
    \begin{aligned}
        \alpha^2 & \leq \min \lcb \frac{1}{\varphiz_2^2}, \frac{\ln(5/4)}{2 (1 + \ta_1) (\sync-1)^2} \rcb, \\
        \alpha^2 (\log_\rho (\alpha)+1) & \leq \frac{v}{2 \sync^2}, \\
        \alpha^2 (\log_\rho (\alpha)+1)^3 \exp \lp \alpha \varphiz_2 \lp 2 \log_\rho \alpha + 1 \rp \rp & \leq \frac{v}{4 \sync} \exp \lp - \varphiz_2 \sqrt{\frac{\ln (5/4)}{2 (1+\ta_1)}} \rp,
    \end{aligned}
    \label{eq:alpha_const_2}
\end{equation}
where $v = \frac{\varphiz_2}{80 \ta_2 C_{17} u_{cm}^2}$, the weighted consensus error satisfies
\begin{align}
    & \frac{2 C_{17}}{\varphiz_2 W_T} \sum_{t=2\tau}^T w_t \lb \sum_{\ell=t-\tau}^{t} \mbe \Omega_\ell \rb \nn \\
    & \leq \alpha^2 \tb^2 \frac{10 C_{17}}{\varphiz_2} \lp 1 + \frac{4 (m_2 + m_4) \rho}{B (1 - \rho)} + \frac{4 m_3 (K-1)}{BN} \sum_{i=1}^N \eta_i \rp (\tau+1) (\sync-1) + \frac{1}{2 W_T} \sum_{t=0}^{T} w_t \mbe M (\btheta_t). 
\end{align}
\end{proposition}

\begin{proof}
The proof is presented in Section \ref{sec:FedSAM_main_prop_proofs_3}.
\end{proof}

Finally, incorporating the results of Propositions \ref{prop:intermediate_lem}, \ref{prop:intermediate_lem2}, and \ref{prop:wtd_consensus_error}, we can establish the convergence of FeGSAM in Theorem \ref{thm:main}.

\begin{proof}[Proof of Theorem \ref{thm:main}]

Assume $w_0=1$, and  consider the weights $w_t$ generated by the recursion $w_t=w_{t-1}\lp 1 - \frac{\alpha \varphiz_2}{2} \rp^{-1}$. Multiplying both sides of \eqref{eq_prop:intermediate_lem2} by $\frac{2 w_t}{\alpha \varphiz_2}$, and rearranging the terms, we get
\begin{align}
    & w_t \E M(\bthetat) \leq \frac{2 w_t}{\alpha \varphiz_2} \lp 1 - \frac{\alpha \varphiz_2}{2} \rp \E M(\btheta_{t}) - \frac{2 w_t}{\alpha \varphiz_2} \E M (\btheta_{t+1}) \nn \\
    & \qquad \qquad \qquad + \frac{2 w_t}{\alpha \varphiz_2} \lp C_{14} (\tau) \alpha^4 + C_{16} (\tau) \frac{\alpha^2}{\na} \rp + \frac{2 w_t \alpha C_{17}}{\alpha \varphiz_2} \sum_{\ell=t-\tau}^{t} \E \lb \Omega_\ell \rb \nn \\
    & = \frac{2}{\alpha \varphiz_2} \lb w_{t-1} \E M(\btheta_{t}) - w_t \E M (\btheta_{t+1}) \rb + \frac{2 w_t}{\varphiz_2} \lp C_{14} (\tau) \alpha^3 + C_{16} (\tau) \frac{\alpha}{\na} \rp + \frac{2 w_t C_{17}}{\varphiz_2} \sum_{\ell=t-\tau}^{t} \E \lb \Omega_\ell \rb, \label{eq:wtd_M_t_everything_1}
\end{align}
where we use $w_{t-1} = w_t \lp 1 - \frac{\alpha \varphiz_2}{2} \rp$.
Summing \eqref{eq:wtd_M_t_everything_1} over $t = 2\tau$ to $T$ (define $W_T = \sum_{t=2\tau}^T w_t$), we get
\begin{align}
    & \frac{1}{W_T} \sum_{t=2\tau}^T w_t \E M(\bthetat) \nn \\
    & \leq \frac{2}{\alpha \varphiz_2 W_T} \lb w_{2\tau-1} \E  M(\btheta_{2 \tau}) - w_T \E M (\btheta_{T+1}) \rb + \frac{2}{\varphiz_2} \lp C_{14} (\tau) \alpha^3 + C_{16} (\tau) \frac{\alpha}{\na} \rp \frac{1}{W_T} \sum_{t=2\tau}^T w_t \nn \\
    & \quad + \frac{1}{W_T} \sum_{t=2\tau}^T \frac{2 w_t C_{17}}{\varphiz_2} \lb \sum_{\ell=t-\tau}^{t} \Omega_\ell \rb \nn \\
    & \leq \frac{2 w_{2\tau-1}}{\alpha \varphiz_2 w_T} \E M(\btheta_{2 \tau}) + \frac{2}{\varphiz_2} \lp C_{14} (\tau) \alpha^3 + C_{16} (\tau) \frac{\alpha}{\na} \rp + \frac{2 C_{17}}{\varphiz_2 W_T} \sum_{t=2\tau}^T w_t \sum_{\ell=t-\tau}^{t} \Omega_\ell \tag{$\because M(\btheta) \geq 0, W_T \geq w_T$} \\
    &= \frac{2}{\alpha \varphiz_2} \lp 1 - \frac{\alpha \varphiz_2}{2} \rp^{T-2\tau+1} \E M(\btheta_{2 \tau}) + \frac{2}{\varphiz_2} \lp C_{14} (\tau) \alpha^3 + C_{16} (\tau) \frac{\alpha}{\na} \rp + \frac{2 C_{17}}{\varphiz_2 W_T} \sum_{t=2\tau}^T w_t \lb \sum_{\ell=t-\tau}^{t} \Omega_\ell \rb. \label{eq:wtd_M_t_everything_2}
\end{align}

Substituting the bound on $\frac{1}{W_T} \sum_{t=2\tau}^T w_t \lb \sum_{\ell=t-\tau}^{t} \Omega_\ell \rb$ from Proposition \ref{prop:wtd_consensus_error} into \eqref{eq:wtd_M_t_everything_2}, we get
\begin{align}
    \frac{1}{W_T} \sum_{t=2\tau}^T w_t \E M(\bthetat) \leq& \frac{2}{\alpha \varphiz_2} \lp 1 - \frac{\alpha \varphiz_2}{2} \rp^{T-2\tau+1} \E M(\btheta_{2 \tau}) + \frac{2}{\varphiz_2} \lp C_{14} (\tau) \alpha^3 + C_{16} (\tau) \frac{\alpha}{\na} \rp \nn \\
    & \ + \alpha^2 \tb^2 \frac{10 C_{17}}{\varphiz_2} \lp 1 + \frac{4 (m_2 + m_4) \rho}{B (1 - \rho)} + \frac{4 m_3 (K-1)}{BN} \sum_{i=1}^N \eta_i \rp (\tau+1) (\sync-1)\nn\\
    &\ + \frac{1}{2 W_T} \sum_{t=0}^{T} w_t \mbe M (\btheta_t) \nn \\
    \Rightarrow  \frac{1}{W_T} \sum_{t=2\tau}^T w_t \E \lb M(\bthetat) \rb \leq& \frac{4}{\alpha \varphiz_2} \lp 1 - \frac{\alpha \varphiz_2}{2} \rp^{T-2\tau+1} M_0 + \frac{4}{\varphiz_2} \lp C_{14} (\tau) \alpha^3 + C_{16} (\tau) \frac{\alpha}{\na} \rp \nn \\
    &  + \alpha^2 \tb^2 \frac{20 C_{17}}{\varphiz_2} \lp 1 + \frac{4 (m_2 + m_4) \rho}{B (1 - \rho)} + \frac{4 m_3 (K-1)}{BN} \sum_{i=1}^N \eta_i \rp (\tau+1) (\sync-1) \nn\\
    &\ + 2 \tau \lp 1 - \frac{\alpha \varphiz_2}{2} \rp^{T-2\tau+1} M_0,    \label{eq:wtd_M_t_everything_3}
\end{align}

where $M_0$ is a problem dependent constant and is defined in Lemma \ref{lem:theta_diff}.
To simplify \eqref{eq:wtd_M_t_everything_3}, we define $\bar{C}_{18}(\tau) =   \frac{40\tb^2 C_{17}}{\varphiz_2}  \lp 1 + \frac{4 (m_2+m_4) \rho}{B (1 - \rho)} \rp \tau $, $\bar{C}_{18}'(\tau) =  \frac{160\tb^2 C_{17}m_3}{\varphiz_2B}  \tau $, $\bar{C}_{19} (\tau) = \frac{4}{\varphiz_2} C_{16} (\tau)$. We have
\begin{align}
     \frac{1}{W_T} \sum_{t=2\tau}^T w_t \E M(\bthetat) \leq &\lp \frac{4}{\alpha \varphiz_2} + 2 \tau \rp \lp 1 - \frac{\alpha \varphiz_2}{2} \rp^{T-2\tau+1} M_0+\frac{4C_{14} (\tau)}{\varphiz_2} \alpha^3 \nonumber\\
    &+ \bar{C}_{18}(\tau) (\sync-1) \alpha^2+\bar{C}_{18}'(\tau)(K-1)^2\varepsilon\alpha^2 + \bar{C}_{19} (\tau) \frac{\alpha}{\na}. \label{eq:wtd_M_t_everything_4}
\end{align}
Furthermore, define $\tilde{W}_T = \sum_{t=0}^T w_t\geq W_T$. By definition of $\hat{T}$, we have $\E[M(\btheta_{\hat{T}})] = \frac{1}{\tilde{W}_T} \sum_{t=0}^T w_t \E M(\bthetat)$, and hence
\begin{align*}
    \E[M(\btheta_{\hat{T}})] =& \frac{1}{\tilde{W}_T} \sum_{t=0}^{2\tau-1} w_t \E M(\bthetat) + \frac{1}{\tilde{W}_T} \sum_{t=2\tau}^{T} w_t \E M(\bthetat)\\
    \leq & \frac{M_0}{\tilde{W}_T} \sum_{t=0}^{2\tau-1} w_t  + \frac{1}{\tilde{W}_T} \sum_{t=2\tau}^{T} w_t \E M(\bthetat)\tag{Lemma \ref{lem:theta_diff}}\\
    \leq & \frac{2\tau M_0 w_{2\tau-1}}{\tilde{W}_T}+\frac{1}{\tilde{W}_T} \sum_{t=2\tau}^{T} w_t \E M(\bthetat)\tag{$w_t\leq w_{t+1}$ for all $t\geq 0$}\\
    \leq & \frac{2\tau M_0 w_{2\tau-1}}{w_T}+\frac{1}{\tilde{W}_T} \sum_{t=2\tau}^{T} w_t \E M(\bthetat)\\
    = & 2\tau M_0 \lp 1 - \frac{\alpha \varphiz_2}{2} \rp^{T-2\tau+1}+\frac{1}{\tilde{W}_T} \sum_{t=2\tau}^{T} w_t \E M(\bthetat) \\
    \leq &2\tau M_0 \lp 1 - \frac{\alpha \varphiz_2}{2} \rp^{T-2\tau+1}+\frac{1}{W_T} \sum_{t=2\tau}^{T} w_t \E M(\bthetat) \\
    \leq &\left(4 \tau  + \frac{4}{\alpha \varphiz_2} \right) M_0 \lp 1 - \frac{\alpha \varphiz_2}{2} \rp^{T-2\tau+1}+\frac{4C_{14} (\tau)}{\varphiz_2} \alpha^3 \nonumber\\
    &+ \bar{C}_{18}(\tau) (\sync-1) \alpha^2+\bar{C}_{18}'(\tau)(K-1)^2\varepsilon\alpha^2 + \bar{C}_{19} (\tau) \frac{\alpha}{\na} \tag{by \eqref{eq:wtd_M_t_everything_4}}\\
    = &\bar{C}_{20}(\alpha,\tau) \lp 1 - \frac{\alpha \varphiz_2}{2} \rp^{T-2\tau+1}+\frac{4C_{14} (\tau)}{\varphiz_2} \alpha^3 \nonumber\\
    &+ \bar{C}_{18}(\tau) (\sync-1) \alpha^2+\bar{C}_{18}'(\tau)(K-1)^2\varepsilon\alpha^2 + \bar{C}_{19} (\tau) \frac{\alpha}{\na},
\end{align*}
where $\bar{C}_{20}(\alpha,\tau)=\left(4 \tau M_0 + \frac{4 M_0}{\alpha \varphiz_2} \right)$. Furthermore, by Proposition \ref{prop:Moreau}, we have $M(\btheta_{\hat{T}}) = \frac{1}{2}\|\btheta_{\hat{T}}\|_m^2\geq \frac{1}{2u_{cm}^2}\|\btheta_{\hat{T}}\|_c^2$, and hence
\begin{align}\label{eq:main_thm_final}
    \E[\| \btheta_{\hat{T}}\|_c^2]  \leq & \bar{\mathcal{C}}_1(\alpha,\tau) \lp 1 - \frac{\alpha \varphiz_2}{2} \rp^{T-2\tau+1} +\mathcal{C}_2(\tau) \frac{\alpha}{\na} + \mathcal{C}_3(\tau) (\sync-1) \alpha^2+ \mathcal{C}_4 (\tau)(K-1)^2 \varepsilon\alpha^2+ \mathcal{C}_5(\tau)\alpha^3
\end{align}
where $\bar{\mathcal{C}}_1(\alpha,\tau) = 2u_{cm}^2\bar{C}_{20}(\alpha,\tau)$, $\mathcal{C}_2(\tau) = 2u_{cm}^2\bar{C}_{19} (\tau)$, $\mathcal{C}_3(\tau) = 2u_{cm}^2\bar{C}_{18}(\tau)$, $\mathcal{C}_4(\tau)=2u_{cm}^2\bar{C}_{18}'(\tau)$ and $\mathcal{C}_5(\tau) =\frac{8u_{cm}^2C_{14} (\tau)}{\varphiz_2} $. 

Finally, note that by definition of $\tau$, we have $\tau=\lceil 2\log_\rho\alpha\rceil\leq 1+2\log_\rho\alpha=1+2(\ln\alpha)(\log_\rho e) =1+(2\log_\rho \frac{1}{e})\ln\frac{1}{\alpha}\leq 1+(2\log_\rho \frac{1}{e})\frac{1}{\alpha}$. Hence, we have $\bar{\mathcal{C}}_1(\alpha,\tau) \leq 2u_{cm}^2 M_0 (4+\frac{8}{\alpha}\log_\rho \frac{1}{e}+\frac{4}{\alpha\varphiz_2}) \leq 16u_{cm}^2 M_0(\log_\rho \frac{1}{e}+\frac{1}{\varphiz_2})\frac{1}{\alpha}=\mathcal{C}_1.\frac{1}{\alpha}$, where $\mathcal{C}_1=16u_{cm}^2 M_0(\log_\rho \frac{1}{e}+\frac{1}{\varphiz_2})$.

Furthermore, we have $\mathcal{C}_2(\tau) = 2u_{cm}^2\bar{C}_{19} (\tau)=2u_{cm}^2\frac{4}{\varphiz_2} C_{16} (\tau) =\frac{8u_{cm}^2\lp C_8\frac{u_{cD}^2B^2}{l_{cD}^2} + \frac{1}{2} + C_{12}\rp}{\varphiz_2}   \tau^2\equiv \frac{8u_{cm}^2\lp C_8\frac{u_{cD}^2B^2}{l_{cD}^2} + \frac{1}{2} + C_{12}\rp}{\varphiz_2}\tau_{\alpha}^2\equiv\mathcal{C}_2\tau_{\alpha}^2$, where we denote $\tau_{\alpha}\equiv\tau$ to emphasize the dependence of $\tau$ on $\alpha$, and $\mathcal{C}_2 = \frac{8u_{cm}^2\lp C_8\frac{u_{cD}^2B^2}{l_{cD}^2} + \frac{1}{2} + C_{12}\rp}{\varphiz_2} $. Note that we have $\tau_\alpha = \mathcal{O} (\log(1/\alpha))$. 

In addition,  $\mathcal{C}_3(\tau) =2u_{cm}^2\bar{C}_{18}(\tau)= \frac{80\tb^2 C_{17}u_{cm}^2\lp 1 + \frac{4 (m_2+m_4) \rho}{B (1 - \rho)} \rp}{\varphiz_2}  .\tau\equiv \mathcal{C}_3\tau_\alpha$, where $\mathcal{C}_3 = \frac{80\tb^2 C_{17}u_{cm}^2\lp 1 + \frac{4 (m_2+m_4) \rho}{B (1 - \rho)} \rp}{\varphiz_2}$. 

Additionaly, $\mathcal{C}_4(\tau)= 2u_{cm}^2\bar{C}_{18}'(\tau) =  2u_{cm}^2 \frac{160\tb^2 C_{17}m_3}{\varphiz_2B}  \tau\equiv  \mathcal{C}_4\tau$, where $\mathcal{C}_4=\frac{320u_{cm}^2\tb^2 C_{17}m_3}{\varphiz_2B}$.

And lastly, $\mathcal{C}_5(\tau) = \frac{8u_{cm}^2  \lp C_7 + C_{11}   + 0.5 C_3^2 C_9^2   + C_3 C_{10}  + 2 C_1 C_3 C_{10} + 3 A_1 C_3  + C_{13} + C_8 \frac{u_{cD}^2}{l_{cD}^2} 2 m^2_2 \alpha^2\rp }{\varphiz_2}\tau^2\equiv\mathcal{C}_5\tau_\alpha^2$, where $\mathcal{C}_5 = 8u_{cm}^2  \lp C_7 + C_{11}   + 0.5 C_3^2 C_9^2   + C_3 C_{10}  + 2 C_1 C_3 C_{10} + 3 A_1 C_3  + C_{13} + C_8 \frac{u_{cD}^2}{l_{cD}^2} 2 m^2_2 \alpha^2\rp/\varphiz_2$.
\end{proof}

Next, we will state the proof of Corollary \ref{lem:sample_complexity}.

\begin{proof}[Proof of Corollary \ref{lem:sample_complexity}]
By this choice of step size, for large enough $T$, $\alpha$ will be small enough and can satisfy the requirements of step size in \eqref{eq:alpha_const_1}, \eqref{eq:alpha_const_2}, \eqref{eq:alpha_const_3}, 
\eqref{eq:alpha_const_4}. Furthermore, the first term in \eqref{eq:main_thm_final} will be
\begin{align*}
    \mathcal{C}_{1}\frac{1}{\alpha}\lp 1 - \frac{\alpha \varphiz_2}{2} \rp^{T-2\tau+1} \leq& \mathcal{C}_{1}\frac{1}{\alpha}e^{-\frac{4\log(\na T)}{T} (T-2\tau+1)}=\mathcal{C}_{1}\frac{\varphiz_2 T}{8\log(NT)}e^{-\log((\na T)^4)(1+\frac{1-2\tau}{T})}\\
    =& \frac{\mathcal{C}_{1} \varphiz_2 T}{8 \log(NT)}\left(\frac{1}{N^4T^4}\right)^{(1+\frac{1-2\tau}{T})}\\
    \leq& \frac{\mathcal{C}_{1} \varphiz_2 T}{8 \log(NT)}\left(\frac{1}{N^4T^4}\right)^{0.5}\tag{Assumption on $T$}\\
    =&\frac{ \mathcal{C}_{1} \varphiz_2 T}{8 \log(NT)}\frac{1}{N^2T^2}\\
    =&  \tilde{\mathcal{O}}\lp\frac{\mathcal{C}_{1} \varphiz_2}{N T}\rp.
\end{align*}
Furthermore, for the second term we have
\begin{align*}
    \mathcal{C}_{2}  \frac{\alpha\tau_\alpha^2}{\na} = \tilde{\mathcal{O}}\lp \frac{\mathcal{C}_{2}/\varphiz_2 }{\na T}\rp.
\end{align*}
For the third, and the fourth terms we have
\begin{itemize}
    \item For $\varepsilon >0$, 
    \begin{align*}
    \mathcal{C}_{3} (\sync-1) \alpha^2\tau_\alpha+\mathcal{C}_4(K-1)^2\varepsilon\alpha^2 \tau_\alpha &= \tilde{\mathcal{O}} \left(\frac{\sqrt{T}}{N}.\frac{1}{T^2} + \frac{T}{N^2}\varepsilon\frac{1}{T^2} \right)\\
    &= \tilde{\mathcal{O}}\left(\frac{1}{NT}\right)
    \end{align*}
    \item For $\varepsilon =0$, the fourth term is zero. Hence, for the third term we have
    \begin{align*}
    \mathcal{C}_{3} (\sync-1) \alpha^2\tau_\alpha &= \tilde{\mathcal{O}} \left(\frac{T}{N}.\frac{1}{T^2} \right)\\
    &= \tilde{\mathcal{O}}(\frac{1}{NT})
    \end{align*}
\end{itemize}
For the fifth term we have
\begin{align*}
    \mathcal{C}_5\alpha^3\tau_\alpha^2 &= \tilde{\mathcal{O}} \left(\frac{1}{T^3}\right)\\
    &= \tilde{\mathcal{O}}(\frac{1}{NT})\tag{$T>N$}
\end{align*}

Upper bounding  \eqref{eq:main_thm_final} with $\epsilon$, we get $\tilde{\mathcal{O}}\lp \frac{1}{\na T}\rp\leq \epsilon$. Hence, we need to have $T=\tilde{\mathcal{O}}\lp\frac{1}{\na \epsilon}\rp$ number of iterations to get to a ball around the optimum with radius $\epsilon$.  
\end{proof}

\subsection{Proof of Proposition \ref{prop:intermediate_lem}}\label{sec:FedSAM_main_prop_proofs_1}
The update of the virtual parameter sequence $\{ \bthetat \}$ can be written as follows 
\begin{align}
    \btheta_{t+1} = \bthetat + \alpha \lp \mbf G (\boldsymbol{\Theta}_t, \byt) - \bthetat + \mbf b(\byt) \rp. \label{eq:theta_update}
\end{align}
Using $\frac{p-1}{\psi}$-smoothness of $\M(\cdot)$ (Proposition \ref{prop:Moreau}), we get
\begin{align}
    \label{eq:composition1}
	\M(\btheta_{t+1} ) 
	& \leq \M (\bthetat ) + \lan \G \M(\bthetat ), \btheta_{t+1} - \bthetat \ran + \frac{L}{2 \psi} \lnr \btheta_{t+1} - \bthetat \rnr_s^2 \nn \tag{Smoothness of $\M(\cdot)$} \\
	&= \M (\bthetat ) + \alpha \lan \G \M (\bthetat ), \mbf G (\boldsymbol{\Theta}_t, \byt) - \bthetat + \mbf b(\byt) \ran + \frac{L \alpha^2}{2 \psi} \lnr \mbf G (\boldsymbol{\Theta}_t, \byt) - \bthetat + \mbf b(\byt) \rnr_s^2 \nn\\
	&= \M (\bthetat ) + \alpha\underbrace{ \lan \G \M (\bthetat ), \bG (\bthetat) - \bthetat \ran}_{T_1: \text{ Expected update}} + \alpha \underbrace{ \langle \G \M (\bthetat ), \mbf b(\byt) \rangle}_{\substack{T_2: \text{ Error due to Markovian } \\\text{  noise } \mbf b(\byt)}} \nn\\
	& \quad +\alpha\underbrace{ \langle \G \M (\bthetat ), \mbf G (\boldsymbol{\Theta}_t, \byt) - \bG (\boldsymbol{\Theta}_t) \rangle}_{T_3: \text{ Error due to Markovian noise } Y_k}  + \alpha\underbrace{ \langle \G \M (\bthetat ), \bG (\boldsymbol{\Theta}_t) - \bG (\bthetat) \rangle}_{T_4: \text{ Error due local updates}}. \nn\\
	& \quad + \frac{L \alpha^2}{2 \psi} \underbrace{\lnr \mbf G (\boldsymbol{\Theta}_t, \byt) - \bthetat + \mbf b(\byt) \rnr_s^2}_{T_5: \text{ Error due to noise and discretization}}. \label{eq:M_bound}
\end{align}

The inequality in \eqref{eq:M_bound} characterizes the one-step drift of the Lyapunov function $M(\bthetat)$. The term $T_1$ is responsible for the negative drift of the overall recursion. $T_2$ and $T_3$ appear due to the presence of Markovian noises $\mbf b^i(\byit)$ and $\mbf G^i (\bthetait, \byit)$ in the update of Algorithm \ref{alg:fed_stoch_app}. $T_4$ appears due to the mismatch between the parameters of the agents $\bthetait, i=1,\dots,\na$. Finally, $T_5$ appears due to the discretization error in the smoothness upper bound. Next, we state bounds on $T_1, T_2, T_3, T_4, T_5$ in the following intermediate lemmas.

\begin{lemma}\label{lem:T_1}
For all $\btheta\in\mathbb{R}^d$, the operator $\bG (\btheta)$ satisfies the following
\begin{align*}
    T_1 = \lan \G \M (\btheta ), \bG (\btheta) - \btheta \ran\leq -2 \varphiz_2 M(\btheta).
\end{align*}
\end{lemma}
Lemma \ref{lem:T_1} guarantees the negative drift in the one-step recursion analysis of Proposition \ref{prop:intermediate_lem}. This follows from the Moreau envelope construction \cite{chen2021Lyapunov_arxiv} and the contraction property of the operators $\bGi (\cdot), i=1,\dots,\na$ (Assumption \ref{ass:contraction}). 

\begin{lemma}\label{lem:T_2}
Consider the iteration $t$ of the Algorithm \ref{alg:fed_stoch_app}, and consider $\tau = \lceil2\log_\rho\alpha\rceil$. We have
\begin{align*}
    \E_{t-\tau} [T_2] &= \E_{t-\tau} \langle \G \M (\bthetat ), \mbf b(\byt) \rangle \\
    & \leq \frac{L^2}{2\alpha\psi^2 l_{cs}^4} \E_{t-\tau} \| \bthetat - \btheta_{t-\tau} \|_c^2 + \frac{\alpha}{2} \E_{t-\tau} \lb \| \mbf b(\byt)\|_c^2 \rb \\
    & \quad + \frac{\alpha m_4 L}{\psi l_{cs}^2} \E_{t-\tau} \| \btheta_{t-\tau} -\btheta_{t} \|_c   + \frac{1}{2} \lp \frac{Lm_4 \alpha^2 u_{cm}}{\zetaone l_{cs}^2\psi}  \rp^2 + \zetaone^2 \E_{t-\tau} [M(\bthetat)],
\end{align*}
where $\zetaone$ is an arbitrary positive constant.
\end{lemma}
In the i.i.d. noise setting, $\mbe [T_2] = 0$. In Markov noise setting, going back $\tau$ steps (which introduce $\btheta_{t-\tau}$) enables us to use Markov chain mixing property (Assumption \ref{ass:mixing}).

\begin{lemma}\label{lem:T_3}
For any $t\geq 0$, denote $T_3=\langle \G \M (\bthetat ), \mbf G (\boldsymbol{\Theta}_t, \byt) - \bG (\boldsymbol{\Theta}_t) \rangle$. For any $\tau<t$, we have
\begin{align}
     \E_{t-\tau}[T_3]&\leq \left( \zetatwo^2 \frac{2LA_2}{\psi l_{cs}^2} + \zetathree^2 \frac{L (A_1 + 1)}{\psi l_{cs}^2} + \frac{6\alpha m_1L u_{cm}^2}{l_{cs}^2\psi} \right) \E_{t-\tau} [M(\bthetat)] \nn \\
    & \quad + \left(\frac{2LA_2}{\psi l_{cs}^2} \lp \frac{1}{2} + \frac{u_{cm}^2}{2\zetatwo^2} \rp + \frac{L (A_1 + 1)}{\psi l_{cs}^2} \lp \frac{3u_{cm}^2}{2\zetathree^2} + 2 \rp + \frac{3m_1L\alpha}{l_{cs}^2\psi} \right) \E_{t-\tau} \lb \lnr \bthetat - \btheta_{t-\tau} \rnr_c^2 \rb \nn \\
    & \quad + \left( \lp \frac{3u_{cm}^2}{2\zetathree^2} + \frac{3}{2} \rp \cdot \frac{L (A_1 + 1)}{\psi l_{cs}^2} + \frac{LA_2}{\psi l_{cs}^2} \right) \E_{t-\tau} [\Omega_t] \nn\\
    &\quad+ \left( \lp \frac{3 u_{cm}^2}{2\zetathree^2} + \frac{3}{2} \rp \frac{L(A_1+1)}{\psi l_{cs}^2} + \frac{m_1L\alpha}{2l_{cs}^2\psi} \right) \E_{t-\tau} [\Omega_{t-\tau}],\nn
\end{align}
where $\zetatwo$ and $\zetathree$ are arbitrary positive constants.
\end{lemma}

\begin{lemma}\label{lem:T_4}
For any $t\geq 0$, we have
\begin{align}
    T_4 = \langle \G \M (\bthetat ), \bG (\boldsymbol{\Theta}_t) - \bG (\bthetat) \rangle \leq 
     \zetafive^2 M(\bthetat) +  \frac{L^2 u_{cm}^2 }{2 l_{cs}^4 \zetafive^2  \psi^2} \Omega_t,\nn 
\end{align}
where $\zetafive$ is an arbitrary positive constant.
\end{lemma}

\begin{lemma}\label{lem:T_5}
For any $0\leq\tau<t$, we have
\begin{align}
    T_5 &= \lnr \mbf G (\boldsymbol{\Theta}_t, \byt) - \bthetat + \mbf b(\byt) \rnr_s^2 \leq \frac{6 (A_2+1)^2 u_{cm}^2}{l_{cs}^2} M (\bthetat) + \frac{3 A_1^2}{l_{cs}^{2}} \Omega_t + \frac{3}{l_{cs}^2} \| \mbf b(\byt)\|_c^2. \nn
\end{align}
\end{lemma}

Substituting the bounds in Lemmas \ref{lem:T_1}, \ref{lem:T_2}, \ref{lem:T_3}, \ref{lem:T_4}, \ref{lem:T_5}, and taking expectation, we get the final bound in Proposition \ref{prop:intermediate_lem}.

\subsection{Proof of Proposition \ref{prop:intermediate_lem2}}\label{sec:FedSAM_main_prop_proofs_2}

First, we state the following two intermediate lemmas, which are proved in Section \ref{sec:FedSAM_aux_lem_proofs}. 
\begin{lemma}
\label{lem:theta_diff}
Suppose $\tau = \lceil 2\log_{\rho}\alpha\rceil$ and
\begin{align}\label{eq:alpha_const_3}
    \alpha\tau \leq \min \lcb \frac{1}{24\sqrt{ A^2_2 + 1}},\frac{1}{8(A_2+1)} \rcb.
\end{align}
For any $0\leq t\leq2\tau$ we have the following
\begin{align}\label{eq:lem_theta_diff_1}
    M(\btheta_t) \leq\frac{1}{l_{cm}^2}\left(\frac{1}{C_1^2} \lp B + (A_2 + 1) \lp \|\btheta_0\|_c + \frac{B}{2 C_1} \rp \rp^2+\|\btheta_0\|_c^2\right)\equiv M_0.
\end{align}
Furthermore, for any $t\geq 2\tau$, we have the following
\begin{equation}
    \begin{aligned}
        \E_{t-2\tau} [\| \btheta_t - \btheta_{t-\tau} \|_c] & \leq 4 \alpha \tau C_1 \E_{t-2\tau} [\| \btheta_{t} \|_c] + 8 \alpha \tau \frac{u_{cD}}{l_{cD}} \frac{B}{\sqrt{\na}} + \frac{u_{cD}}{l_{cD}} \frac{8\sqrt{Bm_4}}{1-\sqrt{\rho}} \alpha^2 \\
        & \qquad + 6 A_1 \alpha \sum_{i=t-\tau}^{t} \E_{t-2\tau} [\Delta_i].
    \end{aligned}
    \label{eq:lem_theta_diff_2}
\end{equation}
\end{lemma}

\begin{lemma}\label{lem:theta_diff_2}
Suppose $\tau = \lceil 2\log_{\rho}\alpha\rceil$ and 
\begin{align}\label{eq:alpha_const_4}
    \alpha \leq \min \lcb \frac{1}{C_1},\frac{1}{8\tau C_2},\frac{1}{40\tau C_1^2} \rcb,   
\end{align} 
where $C_1 = 3\sqrt{ A^2_2 + 1}$ and $C_2=3C_1+8$.
We have the following
\begin{align*}
    \E_{t-2\tau}[\|\btheta_t-\btheta_{t-\tau}\|_c^2] & \leq 8\tau^2\alpha^2 C_1^2\E_{t-2\tau}\|\btheta_{t}\|_c^2 + 8 \frac{u_{cD}^2}{l_{cD}^2} \frac{B^2}{\na} \alpha^2 \tau^2 \\ 
    & \quad + 8 \frac{u_{cD}^2}{l_{cD}^2} \frac{Bm_4 \alpha^4\tau}{1-\rho}  + 8 \alpha^2 A_1^2 \tau \sum_{i=0}^{\tau} \E_{t-2\tau} \lb \Delta_{t-i}^2 \rb.
\end{align*}
\end{lemma}

In Lemma \ref{lem:theta_diff}, we define $C_9 \triangleq \frac{8 u_{cD} B}{l_{cD}}, C_{10} \triangleq \frac{8\sqrt{Bm_4}  u_{cD}}{l_{cD} (1-\sqrt{\rho})}$. Hence, we can bound the term in \eqref{eq:C_3} as
\begin{align}
    & \alpha^2 C_3 \E_{t-2\tau}[\| \btheta_t - \btheta_{t-\tau} \|_c] \nn \\
    & \leq \alpha^2 C_3 \left( 4 \alpha \tau C_1 \E_{t-2\tau} [\| \btheta_{t} \|_c] + C_9 \frac{\alpha \tau}{\sqrt{\na}} + C_{10} \alpha^2 + 6 A_1 \alpha \sum_{k=t-\tau}^{t} \E_{t-2\tau} [\Delta_k] \right) \nn \\
    & \leq \alpha^2 C_3 \left( 4 C_1 u_{cm} \alpha \tau \E_{t-2\tau} \lb \sqrt{2 M(\bthetat)} \rb +  C_9 \frac{\alpha \tau}{\sqrt{\na}} + C_{10} \alpha^2  + 6 A_1 \alpha \sum_{k=t-\tau}^{t} \E_{t-2\tau}[\Delta_k] \right) \tag{Proposition \ref{prop:Moreau}} \\
    &= \frac{1}{\zeta_{6}} 4 C_1 C_3 u_{cm} \alpha^{5/2} \tau \cdot \sqrt{\alpha} \zeta_6 \sqrt{2 \E_{t-2\tau} [M(\btheta_{t})]} \nn \\
    & \qquad \qquad + \alpha^2 C_3 \left( C_9 \frac{\alpha \tau}{\sqrt{\na}} + C_{10} \alpha^2  + 6 A_1 \alpha \sum_{k=t-\tau}^{t} \E_{t-2\tau}[\Delta_k] \right) \tag{$\sqrt{.}$ is concave, $\zeta_6 > 0$} \\
    & \leq \frac{8 C_1^2 C_3^2 u_{cm}^2}{\zeta_{6}^2} \alpha^{5} \tau^2 + \alpha \zeta_6^2 \E_{t-2\tau} [M(\btheta_{t})] + C_3  C_9 \frac{\alpha^3 \tau}{\sqrt{\na}} + C_3 C_{10} \alpha^4  \nn \\
    & \qquad \qquad + 6 A_1 C_3 \sum_{k=t-\tau}^{t} \E_{t-2\tau} \lb \frac{1}{2} \alpha^{5} + \frac{1}{2} \alpha \Delta_k^2 \rb \tag{Young's inequality} \\
    & \leq \underbrace{\frac{8 C_1^2 C_3^2 u_{cm}^2}{\zeta_{6}^2}}_{C_{11}} \alpha^{5} \tau^2 + \alpha \zeta_6^2 \E_{t-2\tau} [M(\btheta_{t})] + C_3  C_9 \frac{\alpha^3 \tau}{\sqrt{\na}} + \alpha^4 [C_3 C_{10}  + 3 A_1 C_3 \alpha (\tau+1)] \nn\\
    & \qquad \qquad + 3 \alpha A_1 C_3 \sum_{k=t-\tau}^{t} \E_{t-2\tau} [\Omega_k] \tag{By \eqref{eq:delta_2_equal_Omega}} \\
    & \leq C_{11} \alpha^{5} \tau^2 + \alpha \zeta_6^2 \E_{t-2\tau} [M(\btheta_{t})] +  \frac{1}{2} C_3^2  C_9^2 \alpha^4 \tau + \frac{1}{2} \frac{\alpha^2\tau}{\na} + \alpha^4 [C_3 C_{10}  + 3 A_1 C_3 \alpha (\tau+1)] \nn \\
    & \qquad \qquad + 3 \alpha A_1 C_3 \sum_{k=t-\tau}^{t} \E_{t-2\tau} [\Omega_k]. \label{eq:theta_diff_term_in_M}
\end{align}
Furthermore, using Lemma \ref{lem:theta_diff_2}, \eqref{eq:delta_2_equal_Omega} and Proposition \ref{prop:Moreau}, the term in \eqref{eq:C_4} can be bounded as follows:
\begin{align}
    C_4 \E_{t-2\tau} \lb \| \btheta_t - \btheta_{t-\tau} \|_c^2 \rb & \leq 16 \tau^2 \alpha^2 C_1^2 C_4 u_{cm}^2 \E_{t-2\tau} \lb M(\btheta_{t}) \rb + 8 \alpha^2 A_1^2 C_4 \tau \sum_{k=0}^{\tau} \E_{t-2\tau} \lb \Omega_{t-k} \rb \nn \\
    & \qquad + \underbrace{\frac{8 C_4 u_{cD}^2 B^2}{l_{cD}^2}}_{C_{12}} \frac{\alpha^2}{\na} \tau^2 + \underbrace{\frac{8 C_4 u_{cD}^2 Bm_4}{l_{cD}^2 (1-\rho)}}_{C_{13}} \alpha^4 \tau. \label{eq:theta_diff_sq_term_in_M_2}
\end{align}
Inserting the upper bounds in \eqref{eq:theta_diff_term_in_M} and \eqref{eq:theta_diff_sq_term_in_M_2} in the upper bound in Proposition \ref{prop:intermediate_lem}, we have
\begin{align}
    & \E_{t-2\tau} \left[ M (\btheta_{t+1}) \right] \nn\\
    & \leq \bigg( 1 - \alpha \lb 2  \varphiz_2 - \zetaone^2 - \zetatwo^2 \frac{2 L A_2}{\psi l_{cs}^2} - \zetathree^2 \frac{L (A_1 + 1)}{\psi l_{cs}^2} - \zetafive^2 - \zeta_6^2 \rb \nn \\
    & \qquad \qquad + \alpha^2 \frac{6 L u_{cm}^2}{l_{cs}^2 \psi} \lb m_1 + \frac{(A_2+1)^2}{2} \rb + \alpha^2 16 \tau^2 C_1^2 C_4 u_{cm}^2 \bigg) \E_{t-2\tau} \left[ M (\btheta_{t}) \right] \nn \\
    & \quad + \alpha C_5 \E_{t-2\tau} [\Omega_t] + \alpha C_6 \E_{t-2\tau} [\Omega_{t-\tau}] \nn \\
    & \quad + \underbrace{\lp C_7 + C_{11} \alpha \tau^2 + 0.5 C_3^2 C_9^2  \tau + C_3 C_{10}   + 3 A_1 C_3 \alpha (\tau + 1) + C_{13}  \tau + C_8 \frac{u_{cD}^2}{l_{cD}^2} Bm_4 \alpha^{2} \rp}_{C_{14}'(\tau)} \alpha^4 \nn \\
    & \quad + \left( C_8 \frac{u_{cD}^2 B^2}{l_{cD}^2} + \frac{\tau}{2} + C_{12}  \tau^2  \right) \frac{\alpha^2}{\na} + \lp 3 \alpha A_1 C_3 + 8 \alpha^2 A_1^2 C_4 \tau \rp \sum_{k=t-\tau}^{t} \E_{t-2\tau} [\Omega_k]. \label{eq:M_t_everything_2}
\end{align}
We define {\small$C_{14}'(\tau)\leq C_{14}(\tau) \triangleq \lp C_7 + C_{11}   + 0.5 C_3^2 C_9^2   + C_3 C_{10}   + 3 A_1 C_3  + C_{13} + C_8 \frac{u_{cD}^2}{l_{cD}^2} B m_4 \alpha^2 \rp \tau^2$}. Also, we choose $\zetaone = \zetafive = \zeta_6 = \sqrt{\varphiz_2/10}$, $\zetatwo = \sqrt{\frac{\varphiz_2}{10} \cdot \frac{\psi l_{cs}^2}{2 L A_2}}$, $\zetathree = \sqrt{\frac{\varphiz_2}{10} \cdot \frac{\psi l_{cs}^2}{L (A_1 + 1)}}$, and denote 
\begin{align}
    C_{15}(\tau)=\lp\frac{6m_1Lu_{cm}^2}{l_{cs}^2\psi}+\frac{6L(A_2+1)^2u_{cm}^2}{2 \psi l_{cs}^2}+144 (A_2^2+1) C_4u_{cm}^2\rp\tau^2. \label{eq:C_15}
\end{align}
This yields
\begin{align}
    & \E_{t-2\tau} \left[ M (\btheta_{t+1}) \right] \leq \left( 1 - \frac{3}{2} \alpha \varphiz_2 + \alpha^2 C_{15} (\tau) \right) \E_{t-2\tau} \left[ M(\btheta_{t}) \right] + \left( C_8 \frac{u_{cD}^2 B^2}{l_{cD}^2} + \frac{\tau}{2} + C_{12} \tau^2 \right) \frac{\alpha^2}{\na} \nn \\
    & \quad + C_{14} (\tau) \alpha^4 + (3 \alpha A_1 C_3 + 8 \alpha^2 A_1^2 C_4 \tau + \alpha C_5 + \alpha C_6) \sum_{k=t-\tau}^{t} \E_{t-2\tau} [\Omega_k] \nn \\
    & \leq \left( 1 - \frac{3}{2} \alpha \varphiz_2 + \alpha^2 C_{15} (\tau) \right) \E_{t-2\tau} \left[ M (\btheta_{t}) \right] + C_{14} (\tau) \alpha^4 + C_{16} (\tau) \frac{\alpha^2}{\na} + \alpha C_{17} \sum_{k=t-\tau}^{t} \E_{t-2\tau}[\Omega_k], \label{eq:M_t_everything_3}
\end{align}
where $C_{16} (\tau) = \lp C_8 \frac{u_{cD}^2 B^2}{l_{cD}^2} + \frac{1}{2} + C_{12}\rp  \tau^2 $ and $C_{17} = (3 A_1 C_3 + 8  A_1^2 C_4  + C_5 + C_6)$. Due to $\alpha\leq \frac{\varphiz_2}{2C_{15}(\tau)}$, we have $1-\frac{3}{2}\alpha\varphiz_2+\alpha^2C_{15}(\tau) \leq 1-\alpha\varphiz_2$. This completes the proof.

\subsection{Proof of Proposition \ref{prop:wtd_consensus_error}}\label{sec:FedSAM_main_prop_proofs_3}

First, we state the following lemma, which characterizes a bound on the expectation of the synchronization error $\Ot$. 

\begin{lemma}
\label{lem:consensus_error}
Suppose Assumptions \ref{ass:mixing}, \ref{ass:lipschitz} holds and the step size $\alpha$ satisfies $\alpha\leq \sqrt{\frac{\ln(5/4)}{2(1 + \ta_1)(\sync-1)^2}}$.
Then, for $s \sync \leq t \leq (s+1) \sync - 1$, where $s = \lfloor t/\sync \rfloor$, the network consensus error $\Ot \triangleq \frac{1}{\na} \sumik \| \btheta_{t} - \btheta_{t}^i \|_c^2$ satisfies
\begin{align}
    \mbe \Ot & \leq 5 \alpha^2 (t-s \sync) \tb^2 \lp 1 + \frac{4 (m_2 + m_4) \rho}{B (1 - \rho)} + \frac{4 m_3 (t-sK-1)}{BN} \sum_{i=1}^N \eta_i \rp + 5 \alpha^2 (t-s \sync) \ta_2 \sum_{t' = s \sync}^{t-1} \mbe \lnr \btheta_{t'} \rnr_c^2, 
\end{align}
where $\ta_1 = \frac{2 A_1^2 u_{c2}^2}{l_{c2}^2}, \ta_2 = \frac{2 A_2^2 u_{c2}^2}{l_{c2}^2}, \tb = \frac{u_{c2}}{l_{c2}} B$. 
Here, $A_1, A_2, B$ are the constants defined in Assumption \ref{ass:lipschitz}, and
$l_{c2}, u_{c2}$ are constants involved in the equivalence of the norms: $l_{c2} \norm{\cdot}_2 \leq \norm{\cdot}_c \leq u_{c2} \norm{\cdot}_2$.
\end{lemma}

Due to the assumption on step size, the bound in Lemma \ref{lem:consensus_error} holds.
Substituting the bound on $\Omega_\ell$ from Lemma \ref{lem:consensus_error}, we get
\begin{align}
    & \frac{2 C_{17}}{\varphiz_2 W_T} \sum_{t=2\tau}^T w_t \lb \sum_{\ell=t-\tau}^{t} \mbe \Omega_\ell \rb \nn \\
    \leq & \frac{2 C_{17}}{\varphiz_2 W_T} \sum_{t=2\tau}^T w_t \sum_{\ell=t-\tau}^{t} \bigg[ 5 \alpha^2 (\ell-s_\ell \sync) \tb^2 \lp 1 + \frac{4 (m_2 + m_4) \rho}{B (1 - \rho)} + \frac{4 m_3 (l-1-sK)}{BN} \sum_{i=1}^N \eta_i \rp \nonumber\\
    &\qquad\qquad \qquad\qquad + 5 \alpha^2 (\ell-s_\ell \sync) \ta_2 \sum_{t' = s_\ell \sync}^{\ell-1} \mbe \lnr \btheta_{t'} \rnr_c^2 \bigg]. \label{eq:wtd_consensus_error_1}
 \end{align}
where $s_\ell = \lfloor \ell/\sync \rfloor$.
Hence, $s_\ell \sync$ denotes the last time instant before $\ell$ when synchronization occurred. The first term in \eqref{eq:wtd_consensus_error_1} can be upper bounded as follows.
\begin{align}
    & \frac{2 C_{17}}{\varphiz_2 W_T} \sum_{t=2\tau}^T w_t \sum_{\ell=t-\tau}^{t} 5 \alpha^2 (\ell-s_\ell \sync) \tb^2 \lp 1 + \frac{4 (m_2 + m_4) \rho}{B (1 - \rho)} + \frac{4 m_3 (\ell - 1 - s_{\ell} K)}{BN} \sum_{i=1}^N \eta_i \rp \nn \\
    &= \alpha^2 \tb^2 \lp 1 + \frac{4 (m_2 + m_4) \rho}{B (1 - \rho)} \rp \frac{10 C_{17}}{\varphiz_2} \frac{1}{W_T} \sum_{t=2\tau}^T w_t \sum_{\ell=t-\tau}^{t} (\ell-s_\ell \sync) \nn \\
    & \quad + \alpha^2 \tb^2 \lp \frac{4 m_3}{BN} \sum_{i=1}^N \eta_i \rp \frac{10 C_{17}}{\varphiz_2} \frac{1}{W_T} \sum_{t=2\tau}^T w_t \sum_{\ell=t-\tau}^{t} (\ell-s_\ell \sync) (\ell - 1 - s_{\ell} K) \nn \\
    & \leq \alpha^2 \tb^2 \lp 1 + \frac{4 (m_2 + m_4) \rho}{B (1 - \rho)} \rp \frac{10 C_{17}}{\varphiz_2} \frac{1}{W_T} \sum_{t=2\tau}^T w_t \sum_{\ell=t-\tau}^{t} (\sync-1) \nn \\
    & \quad + \alpha^2 \tb^2 \lp \frac{m_3}{BN} \sum_{i=1}^N \eta_i \rp \frac{40 C_{17}}{\varphiz_2} \frac{1}{W_T} \sum_{t=2\tau}^T w_t \sum_{\ell=t-\tau}^{t} (\sync-1)^2 \tag{since $ \ell - s_\ell \sync \leq \sync-1$} \\
    &= \alpha^2 \tb^2 \lp 1 + \frac{4 (m_2 + m_4) \rho}{B (1 - \rho)} \rp \frac{10 C_{17}}{\varphiz_2} \lb \frac{1}{W_T} \sum_{t=2\tau}^T w_t \rb (\tau+1) (\sync-1) \nn \\
    & \quad + \alpha^2 \tb^2 \lp \frac{m_3}{BN} \sum_{i=1}^N \eta_i \rp \frac{40 C_{17}}{\varphiz_2} \lb \frac{1}{W_T} \sum_{t=2\tau}^T w_t \rb (\tau+1) (\sync-1)^2 \nn \\
    & \leq \alpha^2 \tb^2 \lp 1 + \frac{4 (m_2 + m_4) \rho}{B (1 - \rho)} + \frac{4 m_3 (K-1)}{BN} \sum_{i=1}^N \eta_i \rp\lb \frac{1}{W_T} \sum_{t=2\tau}^T w_t \rb \frac{10 C_{17}}{\varphiz_2} (\tau+1) (\sync-1). \label{eq:wtd_consensus_error_1b}
\end{align}

Next, we compute the second term in \eqref{eq:wtd_consensus_error_1}.
\begin{align}
    & \frac{2 C_{17}}{\varphiz_2 W_T} \sum_{t=2\tau}^T w_t \sum_{\ell=t-\tau}^{t} 5 \alpha^2 (\ell-s_\ell \sync) \ta_2 \sum_{\ell' = s_\ell \sync}^{\ell-1} \lnr \btheta_{\ell'} \rnr_c^2 \nn \\
    & \leq \alpha^2 \ta_2 \frac{20 C_{17} u_{cm}^2}{\varphiz_2 W_T} \sum_{t=2\tau}^T w_t \sum_{\ell=t-\tau}^{t} (\ell-s_\ell \sync) \sum_{\ell' = s_\ell \sync}^{\ell-1} M(\btheta_{\ell'}) \tag{Proposition \ref{prop:Moreau}} \\
    & \leq \alpha^2 \ta_2 \frac{20 C_{17} u_{cm}^2}{\varphiz_2 W_T} \lb \underbrace{\sum_{t=2\tau}^\sync w_t \sum_{\ell=t-\tau}^{t} (\ell-s_\ell \sync) \sum_{\ell' = s_\ell \sync}^{\ell-1} M(\btheta_{\ell'})}_{\mathcal I_1} + \underbrace{\sum_{t=\sync+1}^T w_t \sum_{\ell=t-\tau}^{t} (\ell-s_\ell \sync) \sum_{\ell' = s_\ell \sync}^{\ell-1} M(\btheta_{\ell'})}_{\mathcal I_2} \rb, \label{eq:wtd_consensus_error_2}
\end{align}
where if $\sync < 2\tau$, $\mathcal I_1 = 0$.
Next, we bound $\mathcal I_1, \mathcal I_2$ separately.
\begin{align}
    \mathcal I_1 &= \sum_{t=2\tau}^\sync w_t \sum_{\ell=t-\tau}^{t} (\ell-s_\ell \sync) \sum_{\ell' = s_\ell \sync}^{\ell-1} M(\btheta_{\ell'}) \nn \\
    & \leq (\sync-1) \sum_{t=2\tau}^{\sync} w_t \sum_{\ell=t-\tau}^{t} \sum_{\ell' = 0}^{\ell-1} M(\btheta_{\ell'}) \tag{for $\ell < \sync$, $s_\ell = 0$; for $\ell = \sync$, $\ell -s_\ell \sync = 0$} \\
    & \leq (\sync-1) (\tau+1) \sum_{t=2\tau}^\sync w_t \sum_{\ell' = 0}^{t-1} M(\btheta_{\ell'}) \nn \\
    & \leq (\sync-1) (\sync-2 \tau+1) (\tau+1) w_\sync \sum_{t = 0}^{\sync-1} M(\btheta_{t}), \label{eq:wtd_consensus_error_3}
\end{align}
where, \eqref{eq:wtd_consensus_error_3} follows since $w_{t-1} \leq w_t, \forall \ t$.
Next, to bound $\mathcal I_2$ in \eqref{eq:wtd_consensus_error_2}, we again split it into two terms.
\begin{align}
    & \mathcal I_2 = \underbrace{\sum_{t=\sync+1}^{\sync+\tau} w_t \sum_{\ell=t-\tau}^{t} (\ell-s_\ell \sync) \sum_{\ell' = s_\ell \sync}^{\ell-1} M(\btheta_{\ell'})}_{\mathcal I_3} + \underbrace{\sum_{t=\sync+\tau+1}^T w_t \sum_{\ell=t-\tau}^{t} (\ell-s_\ell \sync) \sum_{\ell' = s_\ell \sync}^{\ell-1} M(\btheta_{\ell'})}_{\mathcal I_4}. \nn
\end{align}
First, we bound $\mathcal I_3$.
\begin{align}
    & \mathcal I_3 \leq \sum_{t=\sync+1}^{\sync+\tau} w_t \lb \sum_{\ell=t-\tau}^{\sync} (\ell-s_\ell \sync) \sum_{\ell' = s_\ell \sync}^{\ell-1} M(\btheta_{\ell'}) + \sum_{\ell=\sync+1}^{t} (\ell-s_\ell \sync) \sum_{\ell' = s_\ell \sync}^{\ell-1} M(\btheta_{\ell'}) \rb \nn \\
    & \leq \sum_{t=\sync+1}^{\sync+\tau} w_t \lb \sum_{\ell=t-\tau}^{\sync} \ell \sum_{\ell' = 0}^{\ell-1} M(\btheta_{\ell'}) + \sum_{\ell=\sync+1}^{t} (\ell-\sync) \sum_{\ell' = \sync}^{\ell-1} M(\btheta_{\ell'}) \rb \tag{$ s_\ell = 0$ for $\ell < \sync$, $s_\ell \geq 1$ for $\sync \leq \ell \leq \sync+\tau$} \\
    & \leq \sum_{t=\sync+1}^{\sync+\tau} w_t \lb \sync (\sync+1-t+\tau) \sum_{\ell' = 0}^{\sync-1} M(\btheta_{\ell'}) + (t-\sync)^2 \sum_{\ell' = \sync}^{t-1} M(\btheta_{\ell'}) \rb \nn \\
    & \leq w_{\sync+\tau} \sync \tau^2 \sum_{t = 0}^{\sync-1} M(\btheta_{t}) + \tau^3 w_{\sync+\tau} \sum_{t=\sync}^{\sync+\tau-1} M(\btheta_{t}). \label{eq:wtd_consensus_error_4a}
\end{align}
Next, we bound $\mathcal I_4$, assuming $t_0 \sync + \tau \leq T < (t_0+1) \sync + \tau$, where $t_0$ is a non-negative integer.
\begin{align}
    & \mathcal I_4 = \sum_{t=\sync+\tau+1}^T w_t \sum_{\ell=t-\tau}^{t} \underbrace{(\ell-s_\ell \sync)}_{\leq \sync - 1} \sum_{\ell' = s_\ell \sync}^{\ell-1} M(\btheta_{\ell'}) \nn \\
    & \leq (\sync - 1) \sum_{t=\sync+\tau+1}^T w_t \sum_{\ell=t-\tau}^{t} \sum_{\ell' = s_\ell \sync}^{\ell-1} M(\btheta_{\ell'}) \nn \\
    &= (\sync - 1) \lb w_{\sync+\tau+1} \sum_{\ell=\sync+1}^{\sync+\tau+1} \sum_{\ell' = s_\ell \sync}^{\ell-1} M(\btheta_{\ell'}) + \cdots + w_{\sync+\tau+\sync-1} \sum_{\ell=2\sync-1}^{2 \sync + \tau-1} \sum_{\ell' = s_\ell \sync}^{\ell-1} M(\btheta_{\ell'}) \rb \nn \\
    & \quad + (\sync - 1) \lb w_{2\sync+\tau} \sum_{\ell=2\sync}^{2\sync+\tau} \sum_{\ell' = s_\ell \sync}^{\ell-1} M(\btheta_{\ell'}) + \cdots + w_{3\sync+\tau-1} \sum_{\ell=3\sync-1}^{3 \sync + \tau-1} \sum_{\ell' = s_\ell \sync}^{\ell-1} M(\btheta_{\ell'}) \rb \nn \\
    & \quad + \cdots + (\sync - 1) \lb w_{t_0\sync+\tau} \sum_{\ell=t_0\sync}^{t_0\sync+\tau} \sum_{\ell' = s_\ell \sync}^{\ell-1} M(\btheta_{\ell'}) + \cdots + w_{T} \sum_{\ell=T-\tau}^{T} \sum_{\ell' = s_\ell \sync}^{\ell-1} M(\btheta_{\ell'}) \rb \nn \\
    & \leq (\sync - 1) (\tau + 1) \lb w_{\sync+\tau+1} \sum_{\ell=\sync}^{\sync+\tau} M(\btheta_{\ell}) + \cdots + w_{2\sync+\tau-1} \sum_{\ell=\sync}^{2 \sync + \tau-2} M(\btheta_{\ell}) \rb \nn \\
    & \quad + (\sync - 1) (\tau + 1) \lb w_{2\sync+\tau} \sum_{\ell = 2 \sync}^{2\sync+\tau-1} M(\btheta_{\ell}) + \cdots + w_{3\sync+\tau-1} \sum_{\ell = 2 \sync}^{3 \sync + \tau-2} M(\btheta_{\ell}) \rb \nn \\
    & \quad + \cdots + (\sync - 1) (\tau + 1) \lb w_{t_0 \sync+\tau} \sum_{\ell = t_0 \sync}^{t_0\sync+\tau-1} M(\btheta_{\ell}) + \cdots + w_{T} \sum_{\ell = t_0 \sync}^{T-1} M(\btheta_{\ell}) \rb \nn \\
    & \leq (\sync - 1) \sync (\tau + 1) \lb w_{2\sync+\tau-1} \sum_{\ell=\sync}^{2 \sync + \tau-2} M(\btheta_{\ell}) + w_{3\sync+\tau-1} \sum_{\ell' = 2 \sync}^{3 \sync + \tau-2} M(\btheta_{\ell'}) + \cdots + w_{T} \sum_{\ell = t_0 \sync}^{T-1} M(\btheta_{\ell}) \rb \nn \\
    & \leq (\sync - 1) \sync (\tau + 1) \Bigg[ w_{2\sync+\tau-1} \lp \sum_{\ell=\sync}^{2 \sync - 1} M(\btheta_{\ell}) + \sum_{\ell=2 \sync}^{2 \sync + \tau-2} M(\btheta_{\ell}) \rp \nn \\
    & \qquad \qquad \qquad + w_{3\sync+\tau-1} \lp \sum_{\ell' = 2 \sync}^{3 \sync -1} M(\btheta_{\ell'}) + \sum_{\ell' = 3 \sync}^{3 \sync + \tau-2} M(\btheta_{\ell'}) \rp \nn \\
    & \qquad \qquad \qquad  + \cdots + w_{t_0\sync+\tau-1} \lp \sum_{\ell' = (t_0 - 1) \sync}^{t_0 \sync -1} M(\btheta_{\ell'}) + \sum_{\ell' = t_0 \sync}^{t_0 \sync + \tau -1} M(\btheta_{\ell'}) \rp + w_{T} \sum_{\ell = t_0 \sync}^{T-1} M(\btheta_{\ell}) \Bigg] \nn \\
    & \leq (\sync - 1) \sync (\tau + 1) \left \lceil \frac{(\tau-1)}{\sync} \right \rceil \Bigg[ \sum_{\ell=\sync}^{2 \sync - 1} w_{\ell + \sync+\tau-1} M(\btheta_{\ell}) + \sum_{\ell = 2 \sync}^{3 \sync -1} w_{\ell + \sync+\tau-1}  M(\btheta_{\ell}) + \cdots + w_{T} \sum_{\ell = t_0 \sync}^{T-1} M(\btheta_{\ell}) \Bigg] \nn \\
    & \leq (\sync - 1) \tau (\tau + 1) \Bigg[ \sum_{\ell=\sync}^{T - \sync - \tau} w_{\ell + \sync+\tau-1} M(\btheta_{\ell}) + w_{T} \sum_{\ell = T - \sync - \tau + 1}^{T-1} M(\btheta_{\ell}) \Bigg] \nn \\
    &= (\sync - 1) \tau (\tau + 1) \Bigg[ \sum_{\ell=\sync}^{T - \sync - \tau} \frac{w_{\ell}}{\lp 1 - \frac{\alpha \varphiz_2}{2} \rp^{\sync+\tau-1}} M(\btheta_{\ell}) + \sum_{\ell = T - \sync - \tau + 1}^{T-1} \frac{w_{\ell}}{\lp 1 - \frac{\alpha \varphiz_2}{2} \rp^{T-\ell}} M(\btheta_{\ell}) \Bigg] \tag{$ w_t = w_{t+1} \lp 1 - \frac{\alpha \varphiz_2}{2} \rp$} \\
    & \leq \frac{(\sync - 1) \tau (\tau + 1)}{\lp 1 - \frac{\alpha \varphiz_2}{2} \rp^{\sync+\tau-1}} \sum_{t=\sync}^{T - 1} w_{t} M(\btheta_{t}). \label{eq:wtd_consensus_error_4b}
\end{align}
Using the bounds on $\mathcal I_3, \mathcal I_4$ from \eqref{eq:wtd_consensus_error_4a} and \eqref{eq:wtd_consensus_error_4b} respectively, we can bound $\mathcal I_2$.
\begin{align}
    \mathcal I_2 & \leq w_{\sync+\tau} \sync \tau^2 \sum_{t = 0}^{\sync-1} M(\btheta_{t}) + \tau^3 w_{\sync+\tau} \sum_{t=\sync}^{\sync+\tau-1} M(\btheta_{t}) + \frac{(\sync - 1) \tau (\tau + 1)}{\lp 1 - \frac{\alpha \varphiz_2}{2} \rp^{\sync+\tau-1}} \sum_{t=\sync}^{T - 1} w_{t} M(\btheta_{t}). \label{eq:wtd_consensus_error_4}
\end{align}
Substituting the bounds on $\mathcal I_1, \mathcal I_2$ from \eqref{eq:wtd_consensus_error_3}, \eqref{eq:wtd_consensus_error_4} respectively, into \eqref{eq:wtd_consensus_error_2}, we get
\begin{align}
    & \frac{2 C_{17}}{\varphiz_2 W_T} \sum_{t=2\tau}^T w_t \sum_{\ell=t-\tau}^{t} 5 \alpha^2 (\ell-s_\ell \sync) \ta_2 \sum_{\ell' = s_\ell \sync}^{\ell-1} \lnr \btheta_{\ell'} \rnr_c^2 \nn \\
    & \leq \alpha^2 \ta_2 \frac{20 C_{17} u_{cm}^2}{\varphiz_2 W_T} \lb (\sync-1) (\sync-2 \tau+1) (\tau+1) w_\sync \sum_{t = 0}^{\sync-1} M(\btheta_{t}) + \sync \tau^2 w_{\sync+\tau} \sum_{t = 0}^{\sync-1} M(\btheta_{t}) \rb \nn \\
    & \quad + \alpha^2 \ta_2 \frac{20 C_{17} u_{cm}^2}{\varphiz_2 W_T} \Bigg[ \tau^3 w_{\sync+\tau} \sum_{t=\sync}^{\sync+\tau-1} M(\btheta_{t}) + \frac{(\sync - 1) \tau (\tau + 1)}{\lp 1 - \frac{\alpha \varphiz_2}{2} \rp^{\sync+\tau-1}} \sum_{t=\sync}^{T - 1} w_{t} M(\btheta_{t}) \Bigg]
    \label{eq:wtd_consensus_error_5}
\end{align}
We analyze the terms in \eqref{eq:wtd_consensus_error_5} separately. First, for the terms with $\sum_{t = 0}^{\sync-1} M(\btheta_{t})$,
\begin{align}
    & \frac{1}{W_T} \alpha^2 \ta_2 \frac{20 C_{17} u_{cm}^2}{\varphiz_2} \sync (\tau+1) \lb (\sync-2 \tau+1) w_\sync + \tau w_{\sync+\tau} \rb \sum_{t = 0}^{\sync-1} M(\btheta_{t}) \nn \\
    &= \frac{1}{W_T} \alpha^2 \ta_2 \frac{20 C_{17} u_{cm}^2}{\varphiz_2} \sync (\tau+1) \sum_{t = 0}^{\sync-1} \lb (\sync-2 \tau+1) \frac{w_{t}}{\lp 1 - \frac{\alpha \varphiz_2}{2} \rp^{\sync-t}} + \tau \frac{w_{t}}{\lp 1 - \frac{\alpha \varphiz_2}{2} \rp^{\sync+\tau-t}} \rb M(\btheta_{t}) \nn \\
    &= \frac{1}{W_T} \alpha^2 \ta_2 \frac{20 C_{17} u_{cm}^2}{\varphiz_2} \sync (\tau+1) \sum_{t = 0}^{\sync-1} \lb (\sync-2 \tau+1) + \frac{\tau}{\lp 1 - \frac{\alpha \varphiz_2}{2} \rp^{\tau}} \rb \frac{w_{t}}{\lp 1 - \frac{\alpha \varphiz_2}{2} \rp^{\sync-t}} M(\btheta_{t}) \nn \\
    & \leq \alpha^2 \ta_2 \frac{20 C_{17} u_{cm}^2}{\varphiz_2} \sync (\tau+1) \lb (\sync-2 \tau+1) + \frac{\tau}{\lp 1 - \frac{\alpha \varphiz_2}{2} \rp^{\tau}} \rb \frac{1}{W_T} \sum_{t = 0}^{\sync-1} w_{t} M(\btheta_{t}) \nn\\
    & \leq \frac{1}{2 W_T} \sum_{t = 0}^{\sync-1} w_t M(\btheta_{t}), \qquad \qquad \qquad 
    \label{eq:wtd_consensus_error_7a}
\end{align}
where \eqref{eq:wtd_consensus_error_7a} holds since we choose $\alpha$ small enough such that
\begin{align}
    \alpha^2 \ta_2 \frac{20 C_{17} u_{cm}^2}{\varphiz_2} \lb (\sync-2 \tau+1) + \frac{\tau}{\lp 1 - \frac{\alpha \varphiz_2}{2} \rp^{\tau}} \rb \sync (\tau+1) \leq \frac{1}{2}. \nn
\end{align}
To get this, we use the inequality $1-x \geq \exp{\lp-\frac{x}{1-x}\rp}$ for $x < 1$, $\frac{\alpha \varphiz_2}{2} \leq \frac{1}{2}$ and $\tau <  2 \log_\rho \alpha + 1$, we get $\frac{1}{\lp 1 - \frac{\alpha \varphiz_2}{2} \rp^{\tau}} \leq \exp \lp \alpha \varphiz_2 \lp 2 \log_\rho \alpha + 1 \rp \rp$. For \eqref{eq:wtd_consensus_error_7a} to hold, it is  sufficient that
\begin{align}
    & \alpha^2 \leq \frac{\varphiz_2}{80 C_{17} u_{cm}^2 \ta_2 \sync 2 (\log_{\rho}(\alpha)+1)} \min \lcb \frac{1}{\sync}, \frac{1}{2 (\log_{\rho}(\alpha)+1) \exp \lp \alpha \varphiz_2 \lp 2 \log_\rho \alpha + 1 \rp \rp} \rcb
    \label{eq:wtd_consensus_error_7b}
\end{align}
Next, for the remaining terms in the third line of \eqref{eq:wtd_consensus_error_5}, by the assumption on the step size, we have
\begin{equation}
    \begin{aligned}
        \alpha^2 \ta_2 \frac{20 C_{17} u_{cm}^2}{\varphiz_2} \frac{\tau^3}{\lp 1 - \frac{\alpha \varphiz_2}{2} \rp^{\tau}} & \leq \frac{1}{4}, \\
        \alpha^2 \ta_2 \frac{20 C_{17} u_{cm}^2}{\varphiz_2} \frac{(\tau+1) \tau (\sync-1)}{\lp 1 - \frac{\alpha \varphiz_2}{2} \rp^{\sync + \tau - 1}} & \leq \frac{1}{4},
    \end{aligned}
    \label{eq:wtd_consensus_error_8b}
\end{equation}
and hence we get
\begin{align}
    & \frac{1}{W_T} \alpha^2 \ta_2 \frac{20 C_{17} u_{cm}^2}{\varphiz_2} \Bigg[ \frac{\tau^3 w_{\sync}}{\lp 1 - \frac{\alpha \varphiz_2}{2} \rp^{\tau}} \sum_{t=\sync}^{\sync+\tau-1} M(\btheta_{t}) + \frac{(\tau+1) \tau (\sync-1)}{\lp 1 - \frac{\alpha \varphiz_2}{2} \rp^{\sync + \tau - 1}} \sum_{t=\sync}^{T} w_t M (\btheta_t) \Bigg] \nn \\
    & \leq \frac{1}{2 W_T} \sum_{t=\sync}^{T} w_t M (\btheta_t). \label{eq:wtd_consensus_error_8a}
\end{align}
Substituting \eqref{eq:wtd_consensus_error_7a}, \eqref{eq:wtd_consensus_error_8a} in \eqref{eq:wtd_consensus_error_5}, we get
\begin{align}
    & \frac{2 C_{17}}{\varphiz_2 W_T} \sum_{t=2\tau}^T w_t \sum_{\ell=t-\tau}^{t} 5 \alpha^2 (\ell-s_\ell \sync) \ta_2 \sum_{t = s_\ell \sync}^{\ell-1} \lnr \btheta_{t} \rnr_c^2 \leq \frac{1}{2 W_T} \sum_{t=0}^{T} w_t M (\btheta_t).
    \label{eq:wtd_consensus_error_9}
\end{align}
Finally, substituting the bounds in \eqref{eq:wtd_consensus_error_1b}, \eqref{eq:wtd_consensus_error_9} into \eqref{eq:wtd_consensus_error_1}, we get
\begin{align}
    & \frac{2 C_{17}}{\varphiz_2 W_T} \sum_{t=2\tau}^T w_t \lb \sum_{\ell=t-\tau}^{t} \mbe \Omega_\ell \rb \nn \\
    & \leq \alpha^2 \tb^2 \lp 1 + \frac{4 (m_2 + m_4) \rho}{B (1 - \rho)} + \frac{4 m_3 (K-1)}{BN} \sum_{i=1}^N \eta_i \rp \frac{10 C_{17}}{\varphiz_2} (\tau+1) (\sync-1) + \frac{1}{2 W_T} \sum_{t=0}^{T} w_t \mbe M (\btheta_t). 
\end{align}

\subsection{Auxiliary Lemmas}
\label{sec:FedSAM_aux_lem}

The following lemma is of central importance in proving the linear speedup of FeGSAM.
\begin{lemma} \label{lem:after_expectation}
Let $l_{cD}$ and $u_{cD}$ be constants that satisfy $l_{cD}\|\cdot\|_D\leq\|\cdot\|_c\leq u_{cD}\|\cdot\|_D$, where $\|\cdot\|_D = \sqrt{x^\top D x}$ for some positive definite matrix $D \succ 0$. Note that for any $D \succ 0$, these constants always exist due to norm equivalence. Furthermore, in case the norm $\|x\|_c$ is defined in the form $\sqrt{x^\top D x}$ for some $D \succ 0$, we take $l_{cD}= u_{cD} = 1$. We have
\begin{align}
    \E_{t-r}[\|\mbf b(\byt)\|_c]\leq& \frac{u_{cD}}{l_{cD}}\left[\frac{B}{\sqrt{N}}+\sqrt{Bm_4}\rho^{r/2}\right] \label{eq:lemma_bound_byt_1} \\
    \E_{t-r} [\| \mbf b(\byt)\|_c^2] \leq & \frac{u_{cD}^2}{l_{cD}^2} \left[ \frac{B^2}{\na} +  Bm_4\rho^r \right]. \label{eq:lemma_bound_byt_2}
\end{align}
\end{lemma}
Lemma \ref{lem:after_expectation} is essential in characterizing the linear speedup in Theorem \ref{thm:main}. This lemma characterizes the bound on the conditional expectation of $\| \mbf b(\byt)\|_c$ and $\| \mbf b(\byt)\|_c^2$, conditioned on the $r$ time steps before. In order to understand this lemma, consider the bound in \eqref{eq:lemma_bound_byt_2}. For the sake of understanding, suppose that the noise $\byt$ is i.i.d. In this case, we will end up with the first term which is proportional to $1/\na$. This is precisely the linear reduction of the variance of sum of $\na$ i.i.d. random variables. Furthermore, in order to extend the i.i.d. noise setting to the more general Markovian noise, we need to pay an extra price by adding the exponentially decreasing term to the first variance term. 

\begin{lemma}\label{lem:G_diff}
The following hold
\begin{align*}
    \|\mbf G (\boldsymbol{\Theta}_t, \byt)-\mbf G (\boldsymbol{\btheta}_t, \byt)\|^2_c\leq A_1^2 \Delta_t^2 \leq & A_1^2\Omega_t
\end{align*}
\begin{align}
    \Delta_t^2 \leq \Omega_t.\label{eq:delta_2_equal_Omega}
\end{align}
\end{lemma}

\begin{lemma}
\label{lem:Moreau_grad}
For the generalized Moreau Envelope defined in \eqref{eq:Moreau_envelope}, it holds that
\begin{align*}
    \norm{\G \M_f^{\psi, g}(x)}_m^{\star} &= \norm{x}_m, \\
    \lan \G \M_f^{\psi, g}(x), x \ran & \geq 2 \M_f^{\psi, g}(x).
\end{align*}
\end{lemma}

\subsection{Proof of Lemmas}
\label{sec:FedSAM_aux_lem_proofs}

\begin{proof}[Proof of Lemma \ref{lem:T_1}]
Using Cauchy-Schwarz inequality, we have
\begin{align}
    \lan \G \M (\btheta ), \bG (\btheta) - \btheta \ran 
    \leq & \underbrace{\lnr \G \M (\btheta )\rnr_m^\star \cdot \lnr \bG(\btheta)   \rnr_m}_{T_{11}} - \underbrace{\lan \G \M (\btheta ), \btheta \ran}_{T_{12}}, \label{eq:T_1_bound_1}
\end{align}
where $\|\cdot\|_m^\star$ denotes the dual of the norm $\|\cdot\|_m$. Furthermore, by Proposition \ref{prop:Moreau} we have
\begin{align}
    T_{11} \overset{(a)}{\leq} & \|\btheta\|_m.\| \bG(\btheta) \|_m \tag{Lemma \ref{lem:Moreau_grad}} \nn \\
    \leq & \|\btheta\|_m.l_{cm}^{-1}\left\| \frac{1}{\na}\sumik\bG^i(\btheta)\right\|_c \tag{By Proposition \eqref{prop:Moreau}} \\
    \leq &\|\btheta\|_m.l_{cm}^{-1} \frac{1}{\na}\sumik\left\|\bG^i(\btheta)\right\|_c \nn \tag{triangle inequality}\\
    \leq & \|\btheta\|_m.\frac{u_{cm}\gamma_c}{l_{cm}}\| \btheta\|_m \tag{Assumption \ref{ass:contraction}, Proposition \ref{prop:Moreau}} \\
    =&\frac{2u_{cm}\gamma_c}{l_{cm}}M(\btheta). \label{eq:T_11}
\end{align}
Furthermore, by the convexity of the $\|\cdot\|_m$ norm (Lemma \ref{lem:Moreau_grad}), we have
\begin{align}
    T_{12} \geq \| \btheta \|^2_m = 2 \M(\btheta). \label{eq:T_12}
\end{align}
Hence, using \eqref{eq:T_11}, \eqref{eq:T_12} in \eqref{eq:T_1_bound_1} we get
\begin{align}
    \lan \G \M (\btheta ), \bG (\btheta) - \btheta \ran \leq -2\left(1-\frac{u_{cm}\gamma_c}{l_{cm}}\right) M(\btheta) = -2 \varphiz_2 M(\btheta), \label{eq:T_1_bound_2}
\end{align}
where $\varphiz_2$ is defined in \eqref{eq:def:constants}.
\end{proof}

\begin{proof}[Proof of Lemma \ref{lem:T_2}]
Given some $\tau<t$, $T_2 = \langle \G \M (\bthetat ), \mbf b(\byt) \rangle$ can be written as follows:
\begin{align}
    T_2 =& \lan \G \M (\bthetat )-\G \M (\btheta_{t-\tau} ), \mbf b(\byt) \ran +\lan \G \M (\btheta_{t-\tau} ), \mbf b(\byt) \ran \nn \\
    \leq&\|\G \M (\bthetat )-\G \M (\btheta_{t-\tau} )\|_s^\star.\|b(\byt)\|_s + \lan \G \M (\btheta_{t-\tau} ), \mbf b(\byt) \ran\tag{Cauchy–Schwarz} \nn \\
    \leq &\frac{L}{\psi l_{cs}}\|\bthetat - \btheta_{t-\tau} \|_c.\frac{1}{l_{cs}}\|b(\byt)\|_c + \lan \G \M (\btheta_{t-\tau} ), \mbf b(\byt) \ran\tag{Proposition \ref{prop:Moreau}} \nn \\
    \leq &\frac{1}{2\alpha} \lp \frac{L}{\psi l_{cs}^2} \rp^2 \|\bthetat - \btheta_{t-\tau} \|_c^2 + \frac{\alpha}{2}\|b(\byt)\|_c^2 + \lan \G \M (\btheta_{t-\tau} ), \mbf b(\byt) \ran. \label{eq:T_2_bound_1}
\end{align}
Taking expectation on both sides, we have
\begin{align}
    \E_{t-\tau}[T_2] =& \frac{L^2}{2\alpha\psi^2 l_{cs}^4}\E_{t-\tau}\|\bthetat - \btheta_{t-\tau} \|_c^2 + \frac{\alpha}{2} \E_{t-\tau} \|b(\byt)\|_c^2 + \lan \G \M (\btheta_{t-\tau} ), \E_{t-\tau}\left[\mbf b(\byt) \right] \ran \nn \\
    \leq & \frac{L^2}{2 \alpha \psi^2 l_{cs}^4} \E_{t-\tau} \|\bthetat - \btheta_{t-\tau} \|_c^2 + \frac{\alpha}{2} \E_{t-\tau} \|b(\byt)\|_c^2 + \underbrace{\| \G \M (\btheta_{t-\tau} )\|_s^\star\| \E_{t-\tau}\left[\mbf b(\byt) \right]\|_s}_{T_{21}}. \label{eq:T_2_bound_2}
\end{align}
For $T_{21}$, we have
\begin{align}
    T_{21} \leq & l_{cs}^{-1} \lb \lnr \G \M (\btheta_{t-\tau} ) \rnr_s^\star \lnr \E_{t-\tau} \left[ \mbf b(\byt) \right] \rnr_c \rb \nn \\
    \leq & l_{cs}^{-1}[\| \G \M (\btheta_{t-\tau} )\|_s^\star m_4\rho^\tau]\tag{Assumption \ref{ass:mixing}}\\
    \leq & \frac{m_4\alpha^2}{l_{cs}} \lb \| \G \M (\btheta_{t-\tau} )-\G \M (\btheta_{t} )\|_s^\star + \| \G \M (\btheta_{t} )\|_s^\star \rb\tag{assumption on $\tau$}\\
    \leq & \frac{m_4\alpha^2}{l_{cs}} \lb \frac{L}{\psi}\| \btheta_{t-\tau} -\btheta_{t} \|_s + \| \G \M (\btheta_{t} )-\G \M (\mbf 0 )\|_s^\star \rb \tag{Proposition \ref{prop:Moreau}}\\
    \leq & \frac{m_4\alpha^2}{l_{cs}^2}\frac{L}{\psi}[\| \btheta_{t-\tau} -\btheta_{t} \|_c + \| \btheta_{t} \|_c ]\tag{By \eqref{eq:c_to_s}}\\
    \leq & \frac{m_4\alpha}{l_{cs}^2}\frac{L}{\psi} \| \btheta_{t-\tau} -\btheta_{t} \|_c + \frac{m_4 \alpha^2 u_{cm}}{\zetaone l_{cs}^2} \frac{L}{\psi} \cdot \zetaone \| \btheta_{t} \|_m \tag{Proposition \ref{prop:Moreau}; $\zetaone > 0$}\\
    \leq & \frac{m_4\alpha }{l_{cs}^2} \frac{L}{\psi} \lnr \btheta_{t-\tau} -\btheta_{t} \rnr_c+ \frac{1}{2} \lp \frac{m_4 \alpha^2 u_{cm}}{\zetaone l_{cs}^2} \frac{L}{\psi} \rp^2 + \zetaone^2 M(\bthetat). \label{eq:T_21}
\end{align}
Substituting \eqref{eq:T_21} in \eqref{eq:T_2_bound_2}, we get
\begin{align}
    \E_{t-\tau}[T_2] \leq & \frac{L^2}{2\alpha\psi^2 l_{cs}^4} \E_{t-\tau} \| \bthetat - \btheta_{t-\tau} \|_c^2 + \frac{\alpha}{2} \E_{t-\tau} \|b(\byt)\|_c^2 + \frac{\alpha m_4 L}{\psi l_{cs}^2} \E_{t-\tau} \| \btheta_{t-\tau} -\btheta_{t} \|_c \nn \\
    & + \frac{1}{2} \lp \frac{m_4 \alpha^2 u_{cm}}{\zetaone l_{cs}^2} \frac{L}{\psi} \rp^2 + \zetaone^2 \E_{t-\tau} [M(\bthetat)].\nn 
\end{align}
\end{proof}

\begin{proof}[Proof of Lemma \ref{lem:T_3}]

\begin{equation}
    \begin{aligned}
        T_3 =& \langle \G \M (\bthetat), \mbf G (\boldsymbol{\Theta}_t, \byt) - \bG (\boldsymbol{\Theta}_t) \rangle \\
        =& \underbrace{\langle \G \M (\bthetat)- \G \M (\btheta_{t-\tau}), \mbf G (\boldsymbol{\Theta}_t, \byt) - \bG (\boldsymbol{\Theta}_t) \rangle}_{T_{31}} \\
        &+ \underbrace{\langle \G \M (\btheta_{t-\tau}), \mbf G (\boldsymbol{\Theta}_t, \byt) - \mbf G (\boldsymbol{\Theta}_{t-\tau}, \byt) + \bG (\boldsymbol{\Theta}_{t-\tau}) - \bG (\boldsymbol{\Theta}_t) \rangle}_{T_{32}} \\
        & + \underbrace{\langle \G \M (\btheta_{t-\tau}), \mbf G (\boldsymbol{\Theta}_{t-\tau}, \byt) - \bG (\boldsymbol{\Theta}_{t-\tau}) \rangle}_{T_{33}}.
    \end{aligned}
    \label{eq:T_3_bound_1}
\end{equation}
Next, we bound all three terms individually.

\textbf{I. Bound on $T_{31}$:}
\begin{align}
    T_{31} =& \frac{1}{\na} \sumik \lan \G \M (\bthetat)- \G \M (\btheta_{t-\tau}), \Gi (\bthetait, \byit) - \bar{\mbf G}^i (\bthetait) \ran \nn \\
    \leq & \frac{1}{\na} \sumik \underbrace{\lnr \G \M (\bthetat)- \G \M (\btheta_{t-\tau}) \rnr_s^\star}_{T_{311}} \cdot \underbrace{\lnr \Gi (\bthetait, \byit) - \bar{\mbf G}^i (\bthetait) \rnr_s}_{T_{312}}. \label{eq:T_31_bound_1}
\end{align}
For $T_{311}$, we have
\begin{align}
    T_{311} =& \lnr \G \M (\bthetat)- \G \M (\btheta_{t-\tau}) \rnr_s^\star \leq \frac{L}{\psi l_{cs}} \| \bthetat - \btheta_{t-\tau} \|_c, \label{eq:T_311}
\end{align}
where the inequality follows from Proposition \ref{prop:Moreau} and \eqref{eq:c_to_s}. For $T_{312}$, we have,
\begin{align}
    T_{312} =& \| \Gi (\bthetait, \byit) - \E_{\by\sim\mu^i}\mbf G^i (\bthetait,\by) \|_s \nn \\
    \leq & l_{cs}^{-1} \lb \lnr \Gi (\bthetait, \byit) \rnr_c + \E_{\by \sim \mu^i} \lnr \mbf G^i (\bthetait,\by) \rnr_c \rb \tag{Triangle and Jensen's inequality}\\
    \leq & l_{cs}^{-1} [A_2 \| \bthetait \|_c + A_2 \| \bthetait \|_c] \tag{Assumption \ref{ass:lipschitz}}\\
    \leq & \frac{2 A_2}{l_{cs}} \lb \| \bthetait - \bthetat \|_c + \| \bthetat \|_c \rb \tag{triangle inequality} \\
    =& \frac{2 A_2}{l_{cs}} \lb \Delta^i_t + \| \bthetat \|_c \rb. \label{eq:T_312}
\end{align}
Substituting \eqref{eq:T_311} and \eqref{eq:T_312} in \eqref{eq:T_31_bound_1}, we get
\begin{align}
    T_{31} \leq & \frac{2LA_2}{\psi l_{cs}^2} \left[ \frac{1}{\na} \sumik \| \bthetat - \btheta_{t-\tau} \|_c \cdot \lb \Delta^i_t + \|\bthetat\|_c \rb \right] \nn \\
    \leq & \frac{2LA_2}{\psi l_{cs}^2}\left[\frac{1}{\na}\sumik  \left[\| \bthetat - \btheta_{t-\tau} \|_c \cdot \Delta^i_t \right] + \frac{u_{cm}}{\zetatwo} \| \bthetat - \btheta_{t-\tau} \|_c \cdot \zetatwo \| \bthetat \|_m \right] \tag{Proposition \ref{prop:Moreau}} \\
    \leq & \frac{2 L A_2}{\psi l_{cs}^2} \left[ \frac{1}{\na} \sumik \left[ \frac{1}{2} \| \bthetat - \btheta_{t-\tau} \|_c^2 + \frac{1}{2} (\Delta^i_t)^2 \right] +\frac{u_{cm}^2}{2\zetatwo^2}\| \bthetat - \btheta_{t-\tau} \|_c^2+ \frac{\zetatwo^2}{2}\|\bthetat\|_m^2\right]\tag{Young's inequality}\\
    =& \frac{2 L A_2}{\psi l_{cs}^2} \lb \lp \frac{1}{2} + \frac{u_{cm}^2}{2 \zetatwo^2} \rp \| \bthetat - \btheta_{t-\tau} \|_c^2 + \zetatwo^2 M(\bthetat) + \frac{1}{2} \Omega_t \rb. \label{eq:T_31_bound_2}
\end{align}

\textbf{II. Bound on $T_{32}$:}
\begin{align}
    T_{32} =& \frac{1}{\na} \sumik \lan \G \M (\btheta_{t-\tau}), \Gi (\bthetait, \byit) - \Gi (\btheta^i_{t-\tau}, \byit) + \bG^i (\btheta^i_{t-\tau}) - \bG^i (\bthetait) \ran \nn \\
    \leq & \frac{1}{\na} \sumik \underbrace{\| \G \M (\btheta_{t-\tau}) \|_s^\star}_{T_{321}} \cdot \underbrace{\| \Gi (\bthetait, \byit) - \Gi (\btheta^i_{t-\tau}, \byit) + \bG^i (\btheta^i_{t-\tau}) - \bG^i (\bthetait) \|_s}_{T_{322}}, \label{eq:T_32_bound_1}
\end{align}
where the last inequality is by H\"older's inequlaity. For $T_{321}$, we have
\begin{align}
    T_{321} = & \| \G \M (\btheta_{t-\tau})-\G \M (\mbf 0)\|_s^\star \leq \frac{L}{\psi} \| \btheta_{t-\tau} - \mbf 0\|_s \tag{Proposition \ref{prop:Moreau}} \nn \\
    \leq& \frac{L}{\psi l_{cs}} \lb \lnr \bthetat \rnr_c + \lnr \bthetat - \btheta_{t-\tau} \rnr_c \rb. \label{eq:T_321}
\end{align}
For $T_{322}$, we have
\begin{align}
    T_{322} =& \| \Gi (\bthetait, \byit) - \Gi (\btheta^i_{t-\tau}, \byit) + \bG^i (\btheta^i_{t-\tau}) - \bG^i (\bthetait) \|_s \nn \\
    \leq & \frac{1}{l_{cs}} \lb \lnr \Gi (\bthetait, \byit) - \Gi (\btheta^i_{t-\tau}, \byit) \rnr_c + \lnr \bG^i (\btheta^i_{t-\tau}) - \bG^i (\bthetait) \rnr_c \rb \tag{triangle inequality} \\
    \leq & \frac{1}{l_{cs}} \lb \A \| \bthetait - \btheta^i_{t-\tau} \|_c + \gamma_c \| \bthetait - \btheta^i_{t-\tau} \|_c \rb \tag{Assumptions \ref{ass:contraction} and \ref{ass:lipschitz}} \nn \\
    \leq & \frac{A_1+1}{l_{cs}} \lb \lnr \bthetait - \bthetat \rnr_c + \lnr \bthetat - \btheta_{t-\tau} \rnr_c + \lnr \btheta_{t-\tau} - \btheta^i_{t-\tau} \rnr_c \rb \tag{$\gamma_c\leq 1$; triangle inequality} \nn \\
    =& \frac{A_1+1}{l_{cs}} \lb \Delta_t^i + \| \bthetat - \btheta_{t-\tau} \|_c + \Delta_{t-\tau}^i \rb. \label{eq:T_322}
\end{align}
Substituting \eqref{eq:T_321} and \eqref{eq:T_322} in \eqref{eq:T_32_bound_1}, we get
\begin{align}
    T_{32} \leq & \frac{L(A_1+1)}{\psi l_{cs}^2} \Bigg[ \frac{1}{\na} \sumik \zetathree \| \bthetat \|_c \cdot \frac{1}{\zetathree} \lb \Delta_t^i + \| \bthetat - \btheta_{t-\tau} \|_c + \Delta_{t-\tau}^i \rb \nn \\
    & \qquad \qquad \qquad \qquad \qquad + \frac{1}{\na} \sumik \|  \bthetat - \btheta_{t-\tau}  \|_c \cdot  \lb \Delta_t^i + \| \bthetat - \btheta_{t-\tau} \|_c + \Delta_{t-\tau}^i \rb \Bigg] \nn \\
    \leq & \frac{L(A_1+1)}{\psi l_{cs}^2} \bigg[ \frac{1}{2}\zetathree^2 \| \bthetat \|_m^2 + \frac{1}{2\na} \sumik \frac{u_{cm}^2}{\zetathree^2} \lb \Delta_t^i + \| \bthetat - \btheta_{t-\tau} \|_c + \Delta_{t-\tau}^i \rb^2 \nn \\
    &\quad\quad\quad\quad\quad +  \frac{1}{2} \| \bthetat - \btheta_{t-\tau}  \|_c^2 + \frac{1}{2 \na} \sumik \lb \Delta_t^i + \| \bthetat - \btheta_{t-\tau} \|_c + \Delta_{t-\tau}^i \rb^2\bigg]\tag{Young's inequality} \\
    \leq & \frac{L(A_1+1)}{\psi l_{cs}^2} \bigg[ \zetathree^2 M(\bthetat) + \lp \frac{3u_{cm}^2}{2\zetathree^2} + 2\rp \| \bthetat - \btheta_{t-\tau} \|_c^2 + \lp \frac{3u_{cm}^2}{2 \zetathree^2} + \frac{3}{2} \rp \frac{1}{\na} \sumik  \left[ (\Delta_t^i)^2 + (\Delta_{t-\tau}^i)^2 \right] \bigg] \nn \\
    \leq & \frac{L(A_1+1)}{\psi l_{cs}^2} \bigg[ \zetathree^2 M(\bthetat) + \lp \frac{3u_{cm}^2}{2\zetathree^2} + 2 \rp \| \bthetat - \btheta_{t-\tau} \|_c^2 + \lp \frac{3u_{cm}^2}{2\zetathree^2} + \frac{3}{2} \rp \left[\Omega_t+\Omega_{t-\tau}\right]\bigg]. \label{eq:T_32_bound_2}
\end{align}

\textbf{III. Bound on $T_{33}$:}
Taking expectation on both sides of $T_{33}$, we have
\begin{align}
    \E_{t-\tau}[ T_{33} ] 
    =& \lan \G \M (\btheta_{t-\tau}),\E_{t-\tau}[ \mbf G (\boldsymbol{\Theta}_{t-\tau}, \byt)] - \bG (\boldsymbol{\Theta}_{t-\tau}) \ran \nn \\
    =& \frac{1}{\na} \sumik \lan \G \M (\btheta_{t-\tau}), \E_{t-\tau} [ \mbf G^i (\btheta^i_{t-\tau}, \by^i_t) ] - \bG^i ( \btheta^i_{t-\tau}) \ran \nn \\
    \leq & \frac{1}{l_{cs}} \frac{1}{\na} \sumik \lnr \G \M (\btheta_{t-\tau}) \rnr_s^\star \cdot \lnr \E_{t-\tau} [ \mbf G^i (\btheta^i_{t-\tau}, \by^i_t) ] - \bG^i ( \btheta^i_{t-\tau}) \rnr_c \tag{By \eqref{eq:c_to_s}} \\
    \leq & \frac{1}{l_{cs}} \frac{1}{\na} \sumik \lnr \G \M (\btheta_{t-\tau}) \rnr_s^\star \cdot m_1\|\btheta_{t-\tau}^i\|_c \rho^\tau \tag{Assumption \ref{ass:mixing}} \\
    \leq & \frac{m_1\alpha}{l_{cs}} \frac{1}{\na} \sumik\lnr \G \M (\btheta_{t-\tau}) \rnr_s^\star \|\btheta_{t-\tau}^i\|_c \tag{assumption on $\tau$} \\
    \leq & \frac{m_1L\alpha}{l_{cs}^2\psi} \frac{1}{\na} \sumik\lnr  \btheta_{t-\tau} \rnr_c \|\btheta_{t-\tau}^i\|_c \tag{By Proposition \ref{prop:Moreau} and \eqref{eq:c_to_s}}\\
    \leq & \frac{m_1L\alpha}{l_{cs}^2\psi} \frac{1}{\na} \sumik\lnr  \btheta_{t-\tau} \rnr_c \left[\|\btheta_{t-\tau}\|_c+\|\btheta_{t-\tau}-\btheta_{t-\tau}^i\|_c\right] \tag{tringle inequality}\\
    \leq & \frac{m_1L\alpha}{l_{cs}^2\psi} \lb \lnr \btheta_{t-\tau} \rnr_c^2+\frac{1}{\na} \sumik\left[\frac{1}{2} \lnr \btheta_{t-\tau} \rnr_c^2 +  \frac{1}{2} \left\| \btheta_{t-\tau} - \btheta_{t-\tau}^i \right\|_c^2 \right] \rb \tag{Young's inequality}\\
    \leq & \frac{m_1L\alpha}{l_{cs}^2\psi} \E_{t-\tau} \left[3\left(\lnr  \btheta_{t} \rnr_c^2+\lnr \btheta_{t}- \btheta_{t-\tau} \rnr_c^2\right)+\frac{1}{2}\Omega_{t-\tau}\right]\tag{$(a+b)^2\leq 2a^2+2b^2$}\\
    = & \frac{m_1L\alpha}{l_{cs}^2\psi} \E_{t-\tau} \left[6u_{cm}^2M(\btheta_{t})+3\lnr \btheta_{t}- \btheta_{t-\tau} \rnr_c^2+\frac{1}{2}\Omega_{t-\tau}\right]\label{eq:T_33_bound_1}.
\end{align}
Substituting the bounds on $T_{31}, T_{32}, T_{33}$ from \eqref{eq:T_31_bound_2}, \eqref{eq:T_32_bound_2}, \eqref{eq:T_33_bound_1}, in \eqref{eq:T_3_bound_1}, we get the result.
\end{proof}

\begin{proof}[Proof of Lemma \ref{lem:T_4}]
Denote $T_4=\langle \G \M (\bthetat ), \bG (\boldsymbol{\Theta}_t) - \bG (\bthetat) \rangle$.
By the Cauchy–Schwarz inequality, we have
\begin{align}
    T_4 \leq \underbrace{\| \G \M (\bthetat) \|_s^\star}_{T_{41}} \cdot \underbrace{\| \bG (\boldsymbol{\Theta}_t) - \bG (\bthetat) \|_s}_{T_{42}}. \label{eq:T_4_bound_1}
\end{align}
For $T_{41}$ we have
\begin{align}
     \| \G \M (\bthetat) \|_s^\star =&  \| \G \M (\bthetat)-\G \M (\mbf 0) \|_s^\star \nn\\
     \leq& \frac{L}{\psi} \| \bthetat \|_s \tag{Proposition \ref{prop:Moreau}} \nn \\
     \leq & \frac{L}{l_{cs}\psi} \| \bthetat \|_c\tag{By \eqref{eq:c_to_s}}\\
     \leq &\frac{Lu_{cm}}{l_{cs} \psi} \| \bthetat \|_m \tag{Proposition \ref{prop:Moreau}} \\
     = & \frac{Lu_{cm}}{l_{cs} \psi} \sqrt{2M (\bthetat)}. \label{eq:T_41}
\end{align}
For $T_{42}$ we have
\begin{align}
    \| \bG (\boldsymbol{\Theta}_t) - \bG (\bthetat) \|_s =& \lnr \frac{1}{\na} \sum_{i=1}^\na \lp \bG^i (\bthetait) - \bG^i (\bthetat) \rp \rnr_s \leq \frac{1}{\na} \sum_{i=1}^\na \lnr \bG^i (\bthetait) - \bG^i (\bthetat) \rnr_s \tag{Jensen's inequality} \\
    \leq & \frac{1}{l_{cs} \na} \sum_{i=1}^\na \lnr \bG^i (\bthetait) - \bG^i (\bthetat) \rnr_c \leq \frac{\gamma_c}{l_{cs} \na} \sum_{i=1}^\na \lnr \bthetait - \bthetat \rnr_c, \label{eq:T_42}
\end{align}
where \eqref{eq:T_42} follows from Assumption \ref{ass:contraction}.
Combining \eqref{eq:T_41} and \eqref{eq:T_42}, for an arbitrary $\zetafive>0$, we get
\begin{align}
    T_4 \leq & \frac{Lu_{cm}}{l_{cs} \psi} \sqrt{2M (\bthetat)} \cdot \frac{\gamma_c}{l_{cs} \na} \sum_{i=1}^\na \lnr \bthetait - \bthetat \rnr_c \nn \\
    &\leq \frac{1}{\na} \sum_{i=1}^\na \zetafive\sqrt{ 2 M(\bthetat)} \cdot \frac{1}{\zetafive } \lnr \bthetait - \bthetat \rnr_c \cdot \frac{L u_{cm} }{l_{cs}^2 \psi} \tag{$\gamma_c\leq 1$}\\
    &\leq \frac{1}{\na} \sum_{i=1}^\na \lb \zetafive^2  M (\bthetat) + \frac{1}{2 \zetafive^2 } \lnr \bthetait - \bthetat \rnr_c^2 \cdot \frac{L^2 u_{cm}^2 }{l_{cs}^4 \psi^2} \rb \tag{Young's inequality} \\
    &= \zetafive^2 M(\bthetat) +  \frac{L^2 u_{cm}^2 }{2 l_{cs}^4 \zetafive^2  \psi^2} \Omega_t. \label{eq:T_4_bound_2}
\end{align}
\end{proof}

\begin{proof}[Proof of Lemma \ref{lem:T_5}]
Denote $T_5=\lnr \mbf G (\boldsymbol{\Theta}_t, \byt) - \bthetat + \mbf b(\byt)  \rnr_s^2$. We have
\begin{align}
    T_5 =& \lnr \mbf G (\bthetat, \byt)+\mbf G (\boldsymbol{\Theta}_t, \byt) - \mbf G (\bthetat, \byt) - \bthetat + \mbf b(\byt) \rnr_s^2 \nn \\
    \leq & 3 \underbrace{\lnr \mbf G (\bthetat, \byt) - \bthetat \rnr_s^2}_{T_{51}} + 3 \underbrace{\lnr \mbf G (\boldsymbol{\Theta}_t, \byt) - \mbf G (\bthetat, \byt) \rnr_s^2}_{T_{52}} + 3 \lnr \mbf b(\byt) \rnr_s^2. \label{eq:T_5_bound_1}
\end{align}
For $T_{51}$ we have
\begin{align}
    T_{51}
    \leq & l_{cs}^{-2} \lp \lnr \mbf G (\bthetat, \byt) \rnr_c + \lnr \bthetat \rnr_c \rp^2 \tag{By \eqref{eq:c_to_s}} \\
    \leq & l_{cs}^{-2} \lp A_2 \lnr \bthetat \rnr_c + \lnr \bthetat \rnr_c \rp^2 \tag{Assumption \ref{ass:lipschitz}} \\
    = & \frac{2 (A_2+1)^2 u_{cm}^2}{l_{cs}^2} M (\bthetat). \label{eq:T_51}
\end{align}
For $T_{52}$ we have
\begin{align}
    T_{52} =& \lnr \frac{1}{\na} \sumik (\mbf G^i(\bthetait, \byit) - \mbf G^i(\bthetat, \byit)) \rnr_s^2 \leq \frac{1}{\na} \sumik \lnr \mbf G^i(\bthetait, \byit) - \mbf G^i(\bthetat, \byit) \rnr_s^2 \tag{$\|\cdot\|_s^2$ is convex} \\
    \leq & \frac{1}{\na l_{cs}^2} \sumik \lnr \mbf G^i(\bthetait, \byit) - \mbf G^i(\bthetat, \byit) \rnr_c^2\tag{By \eqref{eq:c_to_s}} \\
    \leq& \frac{A_1^2}{l_{cs}^{2}} \frac{1}{\na} \sumik \lnr \bthetait - \bthetat \rnr_c^2 \tag{Assumption \ref{ass:lipschitz}}\\
    = & \frac{A_1^2}{l_{cs}^{2}} \Omega_t. \label{eq:T_52}
\end{align}
Using the bounds in \eqref{eq:T_51}, \eqref{eq:T_52}, we get
\begin{align}
    T_5 \leq & \frac{6 (A_2+1)^2 u_{cm}^2}{l_{cs}^2} M (\bthetat) + \frac{3 A_1^2}{l_{cs}^{2}} \Omega_t + 3 \lnr \mbf b(\byt) \rnr_s^2, \label{eq:T_5_bound_2}
\end{align}
where the last inequality is by the assumption on $\tau$.
\end{proof}

\begin{proof}[Proof of Lemma \ref{lem:theta_diff}]
Define $\theta_l = \frac{1}{\na} \sumik \left\|  \btheta_l^i \right\|_c$. 
By the update rule of Algorithm \ref{alg:fed_stoch_app}, if $l+1~ \text{mod}~ K\neq 0$,
we have
\begin{align}
    \theta_{l+1} = &\frac{1}{\na}\sumik\left\|  \btheta_{l+1}^i\right\|_c = \frac{1}{\na}\sumik\left\|  \btheta_{l}^i + \alpha(\mbf G^i (\btheta_l^i, \by_l^i) - \btheta_l^i+\mbf b^i(\by_l^i)) \right\|_c \nn \\
    \leq & \theta_l + \frac{\alpha}{\na} \sumik \left[ \|\mbf G^i (\btheta_l^i, \by_l^i) \|_c + \| \btheta_l^i \|_c + \| \mbf b^i(\by_l^i) \|_c \right] \tag{triangle inequality}\\
    \leq & \theta_l + \alpha\frac{1}{\na}\sumik \left[(A_2+1) \| \btheta_l^i \|_c + B\right]\tag{Assumption \ref{ass:lipschitz}} \\
    = & (1+\alpha(A_2+1))\theta_l  + \alpha B. \label{eq:theta_norm_avg_1}
\end{align}
Furthermore, if $l+1~ \text{mod}~ K= 0$, we have $\btheta_{l+1}^i = \frac{1}{\na}\sum_{j=1}^\na (\btheta_{l}^j + \alpha(\mbf G^j (\btheta_l^j, \by_l^j) - \btheta_l^j+\mbf b^j(\by_l^j)))$, and hence 
\begin{align}
    \theta_{l+1} = &\frac{1}{\na}\sumik\left\|  \btheta_{l+1}^i\right\|_c = \frac{1}{\na}\sumik\left\|  \frac{1}{\na}\sum_{j=1}^\na (\btheta_{l}^j + \alpha(\mbf G^j (\btheta_l^j, \by_l^j) - \btheta_l^j+\mbf b^j(\by_l^j))) \right\|_c \nn \\
    \leq & \theta_l + \frac{\alpha}{\na} \sumik \left[ \|\mbf G^i (\btheta_l^i, \by_l^i) \|_c + \| \btheta_l^i \|_c + \| \mbf b^i(\by_l^i) \|_c \right], \tag{triangle inequality}
\end{align}
and the same bound as in \eqref{eq:theta_norm_avg_1} holds. By recursive application of \eqref{eq:theta_norm_avg_1}, we get
\begin{align}
    \theta_{l+1} \leq & (1 + \alpha(A_2 + 1))^{l+1} \theta_0 + \alpha B \sum_{\ell=0}^{l} (1 + \alpha (A_2+1))^{\ell} \nn \\
    = & (1+\alpha(A_2+1))^{l+1} \theta_0 + \alpha B\frac{(1 + \alpha(A_2 + 1))^{l+1} - 1}{\alpha (A_2 + 1)}. \label{eq:theta_norm_avg_2}
\end{align}
Notice that for $x\leq \frac{\log 2}{\tau }$, we have $(1+x)^{\tau+1}\leq 1+2x(\tau+1)$. If $0\leq l\leq 2\tau-1$, by the assumption on $\alpha$, we have $(1+\alpha (A_2+1))^{l+1}\leq(1+\alpha (A_2+1))^{2\tau} \leq 1 + 4\alpha (A_2+1) \tau \leq 2$. Hence, we have
\begin{align}
    \theta_l \leq & 2 \theta_0 + \alpha B\frac{4\alpha(A_2+1)\tau}{\alpha(A_2+1)} = 2 \theta_0 + 4 \alpha B \tau. \label{eq:theta_norm_avg_3}
\end{align}
Furthermore, we have
\begin{align}
    \| \btheta_{l+1} - \btheta_l \|_c =& \alpha \| \mbf G (\boldsymbol{\Theta}_l, \by_l) - \btheta_l + \mbf b(\by_l) \|_c \nn \\
    \leq & \alpha \| \mbf G (\boldsymbol{\Theta}_l, \by_l) - \btheta_l\|_c + \alpha\|\mbf b(\by_l) \|_c\tag{triangle inequality}\\
    =&\alpha \left\| \frac{1}{\na}\sumik(\mbf G^i (\btheta_l^i, \by_l^i) - \btheta_l^i)\right\|_c + \alpha\|\mbf b(\by_l) \|_c\nn\\
    \leq & \alpha\frac{1}{\na}\sumik\left\| \mbf G^i (\btheta_l^i, \by_l^i) - \btheta_l^i\right\|_c + \alpha B\tag{convexity of norm}\\
    \leq & \alpha\frac{1}{\na}\sumik(A_2+1)\left\|  \btheta_l^i\right\|_c + \alpha B\tag{Assumption \ref{ass:lipschitz}}\\
    = & \alpha (A_2+1)\theta_l + \alpha B. \label{eq:lem:theta_diff_consec}
\end{align}

Suppose $0\leq t \leq 2\tau$. We have
\begin{align}
     \| \btheta_t - \btheta_{0} \|_c \leq & \left[ \sum_{k=0}^{t-1} \| \btheta_{k+1} - \btheta_k \|_c \right] \tag{triangle inequality} \\
    \leq & \left[\sum_{k=0}^{t-1} \alpha (A_2 + 1) \theta_k + \alpha B  \right] \tag{By \eqref{eq:lem:theta_diff_consec}} \\
    \leq &B\alpha t +  \alpha (A_2+1) \sum_{k=0}^{t-1} (2 \theta_0 + 4 B \alpha \tau) \tag{By \eqref{eq:theta_norm_avg_3}} \\
    \leq & 2 \alpha \tau (B + (A_2 + 1) (2 \theta_0 + 4 B \alpha \tau)) \tag{$t\leq 2\tau$}\\
    \leq & \frac{1}{C_1} \lp B + (A_2 + 1) \lp \theta_0 + \frac{B}{2 C_1} \rp \rp. \label{eq:lem:theta_diff_t_to_0}
\end{align}
Furthermore, by Proposition \ref{prop:Moreau}, we have $M(\btheta_t)=\frac{1}{2}\|\btheta_t\|_m^2$, and hence for any $0\leq t\leq 2\tau$ we have
\begin{align*}
    M(\btheta_t) \leq & \frac{1}{2l_{cm}^2}\|\btheta_t\|_c^2 \\
    =& \frac{1}{2l_{cm}^2}(\|\btheta_t-\btheta_0+\btheta_0\|_c)^2\\
    \leq& \frac{1}{2l_{cm}^2}(\|\btheta_t-\btheta_0\|_c+\|\btheta_0\|_c)^2\tag{triangle inequality}\\
    \leq& \frac{1}{l_{cm}^2}(\|\btheta_t-\btheta_0\|_c^2+\|\btheta_0\|_c^2)\\
    \leq & \frac{1}{l_{cm}^2}\left(\frac{1}{C_1^2} \lp B + (A_2 + 1) \lp \|\btheta_0\|_c + \frac{B}{2 C_1} \rp \rp^2+\|\btheta_0\|_c^2\right),
\end{align*}
which proves the first claim. 

Next we prove the second claim. By the update rule in \ref{eq:theta_update}, we have
\begin{align}
    \| \btheta_{l+1} - \btheta_l \|^2_c =& \alpha^2 \| \mbf G (\boldsymbol{\Theta}_l, \bY_l) - \btheta_l + \mbf b(\bY_l) \|^2_c \nn \\
    \leq & 3 \alpha^2 \| \mbf G (\boldsymbol{\btheta}_l, \bY_l) - \btheta_l \|_c^2 + 3 \alpha^2 \| \mbf b(\bY_l) \|^2_c + 3 \alpha^2 \|\mbf G (\boldsymbol{\Theta}_l, \bY_l) - \mbf G (\boldsymbol{\btheta}_l, \bY_l) \|^2_c \nn \\
    \leq & 6 \alpha^2 (\| \mbf G (\boldsymbol{\btheta}_l, \bY_l)\|_c^2 + \| \btheta_l \|_c^2) + 3 \alpha^2 \| \mbf b(\bY_l) \|^2_c + 3 \alpha^2 \|\mbf G (\boldsymbol{\Theta}_l, \bY_l) - \mbf G (\boldsymbol{\btheta}_l, \bY_l)\|^2_c \nn \\
    = & 6 \alpha^2 \left(\left\|\frac{1}{\na}\sumik \mbf G^i (\boldsymbol{\btheta}_l, \by_l^i)\right\|_c^2 + \| \btheta_l \|_c^2\right) + 3 \alpha^2 \| \mbf b(\bY_l) \|^2_c + 3 \alpha^2 \|\mbf G (\boldsymbol{\Theta}_l, \bY_l) - \mbf G (\boldsymbol{\btheta}_l, \bY_l)\|^2_c \nn \\
    \leq & 6 \alpha^2 \left(\left(\frac{1}{\na}\sumik\left\| \mbf G^i (\boldsymbol{\btheta}_l, \by_l^i)\right\|_c\right)^2 + \| \btheta_l \|_c^2\right) + 3 \alpha^2 \| \mbf b(\bY_l) \|^2_c + 3 \alpha^2 \|\mbf G (\boldsymbol{\Theta}_l, \bY_l) - \mbf G (\boldsymbol{\btheta}_l, \bY_l)\|^2_c \tag{convexity of norm} \\
    \leq & 6 \alpha^2 \left(\left(\frac{1}{\na}\sumik A_2\|\boldsymbol{\btheta}_l\|_c\right)^2 + \| \btheta_l \|_c^2\right) + 3 \alpha^2 \| \mbf b(\bY_l) \|^2_c + 3 \alpha^2 \|\mbf G (\boldsymbol{\Theta}_l, \bY_l) - \mbf G (\boldsymbol{\btheta}_l, \bY_l)\|^2_c\tag{Assumption \ref{ass:lipschitz}}\\
    =& 6\alpha^2 ( A^2_2 + 1) \| \btheta_l \|_c^2 + 3\alpha^2\|\mbf b(\bY_l) \|^2_c+3\alpha^2 \|\mbf G (\boldsymbol{\Theta}_l, \bY_l)-\mbf G (\boldsymbol{\btheta}_l, \bY_l)\|^2_c\nn\\
    \leq& 6\alpha^2 ( A^2_2 + 1)\|\btheta_l\|_c^2 + 3\alpha^2\|\mbf b(\bY_l) \|^2_c+3\alpha^2 A_1^2\Delta_l^2, \label{eq:lem:theta_diff_sq_consec_0}
\end{align}
where \eqref{eq:lem:theta_diff_sq_consec_0} follows from Lemma \ref{lem:G_diff}.
Taking square root on both sides, we get
\begin{align}
    \| \btheta_{l+1} - \btheta_l \|_c\leq 3\alpha \sqrt{ A^2_2 + 1}\|\btheta_l\|_c + 2\alpha\|\mbf b(\bY_l) \|_c+2\alpha A_1\Delta_l.\label{eq:theta_diff_consec_1}
\end{align}
Combining the above inequality with the fact that $\| \btheta_{l+1} \|_c - \| \btheta_l \|_c \leq \| \btheta_{l+1} - \btheta_l \|_c$, we get
\begin{align}
    \|\btheta_{l+1}\|_c\leq & (1 + 3 \alpha \sqrt{ A^2_2 + 1}) \| \btheta_l \|_c + 2 \alpha \|\mbf b(\bY_l) \|_c + 2 \alpha A_1 \Delta_l \nonumber \\
    =& (1 + \alpha C_1) \| \btheta_l \|_c + 2 \alpha \| \mbf b(\bY_l) \|_c + 2 \alpha A_1 \Delta_l, \label{eq:theta_l_norm_1}
\end{align}
where we denote $C_1 = 3\sqrt{ A^2_2 + 1}$. Assuming $t-\tau\leq l\leq t$, and taking expectation on both sides, we have
\begin{align}
    & \E_{t-2\tau}[\| \btheta_{l+1} \|_c] \leq (1 + \alpha C_1) \E_{t-2\tau} [\| \btheta_l \|_c] + 2 \alpha \E_{t-2\tau}[\| \mbf b(\bY_l) \|_c] + 2 \alpha A_1 \E_{t-2\tau}[\Delta_l] \nn \\
    & \leq (1+\alpha C_1)\E_{t-2\tau}[\|\btheta_l\|_c] + 2\alpha\frac{u_{cD}}{l_{cD}}\left[\frac{B}{\sqrt{N}}+\sqrt{Bm_4}\rho^{(l-t+2\tau)/2}\right] +2\alpha A_1\E_{t-2\tau}[\Delta_l] \tag{Lemma \ref{lem:after_expectation}}\\
    & \leq (1 + \alpha C_1) \E_{t-2\tau} [\|\btheta_l\|_c] + 2 \alpha \frac{u_{cD}}{l_{cD}} \left[\frac{B}{\sqrt{\na}} +\sqrt{Bm_4}  \alpha \sqrt{\rho}^{l - t + \tau} \right] + 2 \alpha A_1 \E_{t-2\tau}[\Delta_l] \tag{Assumption on $\tau$} \\
    &= (1 + \alpha C_1) \E_{t-2\tau} [\| \btheta_l \|_c] + \alpha c_t (l) + 2 \alpha A_1 \E_{t-2\tau} [\Delta_l], \label{eq:theta_l_norm_2}
\end{align}
where $c_t (l) = 2 \frac{u_{cD}}{l_{cD}} \left[ \frac{B}{\sqrt{\na}} + \sqrt{Bm_4} \alpha \sqrt{\rho}^{l-t+\tau} \right]$.
By applying this inequality recursively, we have
\begin{align}
    & \E_{t-2\tau} [\| \btheta_{l+1} \|_c] \leq (1 + \alpha C_1) \E_{t-2\tau} [\| \btheta_{l} \|_c] + \alpha c_t (l) + 2\alpha A_1 \E_{t-2\tau} [\Delta_l] \nn \\
    & \leq (1+\alpha C_1)\left[ (1+\alpha C_1) \E_{t-2\tau} [\| \btheta_{l-1} \|_c] + \alpha c_t(l-1) + 2 \alpha A_1 \E_{t-2\tau} [\Delta_{l-1}] \right] + \alpha c_t (l) + 2 \alpha A_1 \E_{t-2\tau} [\Delta_l]\nonumber\\
    & \leq (1 + \alpha C_1)^{l+1-t+\tau} \E_{t-2\tau} [\| \btheta_{t-\tau} \|_{c}] + \alpha \sum_{k=t-\tau}^l (1 + \alpha C_1)^{l-k} c_t (k) + 2 \alpha A_1 \E_{t-2\tau} \left[ \sum_{k=t-\tau}^l (1 + \alpha C_1)^{l-k} \Delta_{k} \right] \nn \\
    & \leq (1 + \alpha C_1)^{\tau+1} \E_{t-2\tau} \| \btheta_{t-\tau} \|_{c} + \alpha \underbrace{\sum_{k=t-\tau}^t (1 + \alpha C_1)^{t-k} c_t (k)}_{T_1} + 2 \alpha A_1 \E_{t-2\tau} \underbrace{ \left[ \sum_{k=t-\tau}^t (1 + \alpha C_1)^{t-k} \Delta_{k} \right]}_{T_2}. \label{eq:theta_l_norm_3}
\end{align}

We study $T_1$ and $T_2$ in \eqref{eq:theta_l_norm_3} separately. For $T_1$ we have
\begin{align}
    T_1 &= 2\frac{u_{cD}}{l_{cD}}\sum_{k=0}^\tau (1+\alpha C_1)^{\tau-k} \left[\frac{B}{\sqrt{\na}} + \sqrt{Bm_4}\alpha\sqrt{\rho}^{k}\right] \nn \\
    &= 2\frac{u_{cD}}{l_{cD}} \left[\frac{B}{\sqrt{\na}} \frac{(1+\alpha C_1)^{\tau+1}-1}{\alpha C_1}+ \sqrt{Bm_4}\alpha(1+\alpha C_1)^\tau\sum_{k=0}^\tau \left(\frac{\sqrt{\rho}}{1+\alpha C_1}\right)^k\right] \nn \\ 
    &\leq 2\frac{u_{cD}}{l_{cD}}\left[\frac{B}{\sqrt{\na}} \frac{(1+\alpha C_1)^{\tau+1}-1}{\alpha C_1}+ \sqrt{Bm_4}\alpha(1+\alpha C_1)^\tau\sum_{k=0}^\tau \sqrt{\rho}^{k}\right]\tag{$\alpha>0$}\\
    &\leq 2\frac{u_{cD}}{l_{cD}}\left[\frac{B}{\sqrt{\na}} \frac{(1+\alpha C_1)^{\tau+1}-1}{\alpha C_1}+ \sqrt{Bm_4}\alpha(1+\alpha C_1)^\tau\frac{1}{1-\sqrt{\rho}}\right]. \label{eq:theta_l_norm_T1_1}
\end{align}
Notice that for $x \leq \frac{\log 2}{\tau }$, we have $(1+x)^{\tau+1}\leq 1+2x(\tau+1)$. By the assumption on $\alpha$, we have $(1+\alpha C_1)^{\tau+1} \leq 1 + 2\alpha C_1 (\tau+1)\leq 1+4\alpha \tau C_1 \leq 2$ and $(1+\alpha C_1)^{\tau} \leq 1 + 2\alpha C_1 \tau \leq 1 + 1/2 \leq 2$. Hence, we have
\begin{align}
    T_1 \leq 2 \frac{u_{cD}}{l_{cD}} \left[ \frac{B}{\sqrt{\na}} 2 (\tau + 1) + \frac{2\sqrt{Bm_4}  \alpha}{1 - \sqrt{\rho}} \right]. \label{eq:theta_l_norm_T1_2}
\end{align}
Furthermore, for the term $T_2$ we have 
\begin{align}
    T_2 =& \sum_{k=0}^\tau (1 + \alpha  C_1)^{\tau-k} \Delta_{t-\tau + k} \leq \sum_{k=0}^\tau (1 + \alpha C_1)^{\tau} \Delta_{t-\tau + k} \tag{due to $\alpha>0$} \\
    \leq & \sum_{k=0}^\tau (1 + 2 \alpha C_1 \tau) \Delta_{t-\tau+k} \leq 2 \sum_{k=0}^\tau \Delta_{t-k}. \label{eq:theta_l_norm_T2_1}
\end{align}
Subtituting \eqref{eq:theta_l_norm_T1_2}, \eqref{eq:theta_l_norm_T2_1} in \eqref{eq:theta_l_norm_3}, for every $t-\tau\leq l\leq t$, we get
\begin{align}
    \E_{t-2\tau} [\| \btheta_{l}\|_c] \leq 2 \E_{t-2\tau}[\|\btheta_{t-\tau}\|_c] + \alpha \frac{2u_{cD}}{l_{cD}}\left[\frac{4B}{\sqrt{\na}} \tau+ \frac{2\sqrt{Bm_4}\alpha}{1-\sqrt{\rho}}\right]  + 4 A_1 \alpha \sum_{k=0}^\tau  \Delta_{t-k}. \label{eq:theta_l_norm_4} 
\end{align}
But we have
\begin{align}
    & \E_{t-2\tau} \| \btheta_t - \btheta_{t-\tau} \|_c \leq \E_{t-2\tau} \left[ \sum_{i=t-\tau}^{t-1} \| \btheta_{i+1} - \btheta_i \|_c \right] \tag{triangle inequality} \\
    & \leq \sum_{i=t-\tau}^{t-1} \E_{t-2\tau} \left[ \alpha C_1 \| \btheta_i \|_c + 2 \alpha \| \mbf b (\bY_i) \|_c + 2 \alpha A_1 \Delta_i \right] \tag{by \eqref{eq:theta_diff_consec_1}} \\
    & \leq \sum_{i=t-\tau}^{t-1} \alpha C_1 \left[ 2 \E_{t-2\tau} [\| \btheta_{t-\tau} \|_c] + \alpha \frac{2 u_{cD}}{l_{cD}} \left[ \frac{4 B}{\sqrt{\na}} \tau + \frac{2\sqrt{Bm_4} \alpha}{1-\sqrt{\rho}} \right]  + 4 A_1 \alpha \sum_{j=0}^\tau  \E_{t-2\tau}[ \Delta_{t-j} ] \right] \tag{by \eqref{eq:theta_l_norm_4}} \\
    & \quad + 2 \alpha  \sum_{i=t-\tau}^{t-1} \frac{u_{cD}}{l_{cD}} \left[ \frac{B}{\sqrt{\na}} + \sqrt{Bm_4} \sqrt{\rho}^{i-t+2\tau} \right] \tag{Lemma \ref{lem:after_expectation}} \\
    & \quad + 2\alpha A_1 \sum_{i=t-\tau}^{t-1} \E_{t-2\tau} [\Delta_i] \tag{assumption on $\alpha$} \\
    & \leq \alpha\tau C_1 \left[ 2 \E_{t-2\tau} [\| \btheta_{t-\tau} \|_c] + \alpha \frac{2 u_{cD}}{l_{cD}} \left[ \frac{4B}{\sqrt{\na}} \tau + \frac{2\sqrt{Bm_4} \alpha }{1-\sqrt{\rho}} \right] + 4 A_1 \alpha \sum_{j=0}^\tau  \E_{t-2\tau}[\Delta_{t-j}] \right] \nn \\
    & \quad + 2 \alpha \tau \frac{u_{cD}}{l_{cD}} \frac{B}{\sqrt{\na}} + 2 \alpha^2 \frac{u_{cD}}{l_{cD}}\sqrt{Bm_4}  \frac{1}{1-\sqrt{\rho}} +2 \alpha A_1 \sum_{i=t-\tau}^{t-1} \E_{t-2\tau} [\Delta_i] \tag{assumption on $\tau$} \\
    & \leq 2 \alpha \tau C_1 \E_{t-2\tau}[\| \btheta_{t-\tau} \|_c] + \frac{u_{cD}}{l_{cD}} \frac{B}{\sqrt{\na}} \left[8 C_1 \alpha^2 \tau^2 + 2 \alpha \tau \right] + \frac{1}{1-\sqrt{\rho}}\frac{u_{cD}}{l_{cD}}\sqrt{Bm_4}\left[4C_1\alpha^3\tau+2\alpha^2\right] \nonumber\\
    & \quad + (4 A_1 \alpha^2 \tau C_1 + 2 \alpha A_1) \sum_{i=t-\tau}^{t} \E_{t-2\tau} [\Delta_i] \tag{$\Delta_i \geq 0$} \\
    & \leq 2 \alpha \tau C_1 \E_{t-2\tau}[\| \btheta_{t-\tau} \|_c] + 4 \alpha \tau \frac{u_{cD}}{l_{cD}} \frac{B}{\sqrt{\na}} + \frac{u_{cD}}{l_{cD}} \frac{4\sqrt{Bm_4}}{1-\sqrt{\rho}} \alpha^2 + 3 A_1 \alpha \sum_{i=t-\tau}^{t} \E_{t-2\tau} [\Delta_i]. \label{eq:lem_theta_diff_3}
\end{align}
Furthermore, by triangle inequality, we have $\|\btheta_{t-\tau}\|_c\leq \|\btheta_{t}-\btheta_{t-\tau}\|_c + \|\btheta_{t}\|_c$. By assumption on $\alpha$, we have $2\alpha\tau C_1\|\btheta_{t-\tau}\|_c\leq 2\alpha\tau C_1\|\btheta_{t}-\btheta_{t-\tau}\|_c + 2\alpha\tau C_1\|\btheta_{t}\|_c\leq 0.5 \|\btheta_{t}-\btheta_{t-\tau}\|_c + 2\alpha\tau C_1\|\btheta_{t}\|_c$. By taking expectation on both sides, and substituting it in \eqref{eq:lem_theta_diff_3}, we get \eqref{eq:lem_theta_diff_2}.
\end{proof}

\begin{proof}[Proof of Lemma \ref{lem:theta_diff_2}]
From \eqref{eq:theta_l_norm_1} we have
\begin{align}
    \| \btheta_{l+1} \|_c^2 \leq & (1 + \alpha C_1)^2 \| \btheta_l \|_c^2 + 4 \alpha^2 \| \mbf b(\bY_l) \|_c^2+4\alpha^2 A_1^2\Delta_l^2\nonumber\\
    &+\underbrace{4\alpha(1+\alpha C_1)\|\btheta_l\|_c\|\mbf b(\bY_l) \|_c}_{T_1} + \underbrace{4\alpha A_1(1+\alpha C_1)\|\btheta_l\|_c\Delta_l}_{T_2} + \underbrace{8 \alpha ^2 A_1\Delta_l\|\mbf b(\bY_l) \|_c}_{T_3}. \label{eq:norm_theta_sq_1}
\end{align}
For $T_1$ we have
\begin{align}
    T_1 =&  2 \sqrt{\alpha (1 + \alpha C_1)} \| \btheta_l \|_c \cdot 2 \sqrt{\alpha (1 + \alpha C_1)} \| \mbf b(\bY_l) \|_c \nn \\
    \leq & 2 \alpha (1 + \alpha C_1) \| \btheta_l \|_c^2 + 2 \alpha (1 + \alpha C_1) \| \mbf b(\bY_l) \|_c^2 \tag{$ab\leq \frac{1}{2}a^2+\frac{1}{2}b^2$}\\
    \leq & 4 \alpha \| \btheta_l \|_c^2 + 4 \alpha \|\mbf b(\bY_l) \|_c^2, \label{eq:norm_theta_sq_T1}
\end{align}
where the last inequality is by the assumption on $\alpha$. Analogously for $T_2$ we have
\begin{align}
    T_2=& 2 \sqrt{\alpha (1 + \alpha C_1)} \| \btheta_l \|_c \cdot 2 A_1 \sqrt{\alpha (1 + \alpha C_1)} \Delta_l \nn \\
    \leq & 4 \alpha \| \btheta_l \|_c^2 + 4 \alpha A_1^2 \Delta_l^2. \label{eq:norm_theta_sq_T2}
\end{align}
For $T_3$ we have
\begin{align}
    T_3 = &2\alpha \|\mbf b(\bY_l) \|_c \cdot 4\alpha A_1 \Delta_l \nn \\
    \leq & 2\alpha ^2 \|\mbf b(\bY_l) \|_c^2 + 8 \alpha^2A_1^2\Delta_l^2. \label{eq:norm_theta_sq_T3}
\end{align}
Combining the bounds in \eqref{eq:norm_theta_sq_T1}, \eqref{eq:norm_theta_sq_T2}, and \eqref{eq:norm_theta_sq_T3}, and noting that $(1 + \alpha C_1)^2 \leq 1 + 3\alpha C_1$,  from \eqref{eq:norm_theta_sq_1} we have 
\begin{align}
    \|\btheta_{l+1}\|_c^2\leq& (1+3\alpha C_1)\|\btheta_l\|_c^2 + 4\alpha^2\|\mbf b(\bY_l) \|_c^2+4\alpha^2 A_1^2\Delta_l^2 \nonumber \\
    & +4\alpha\|\btheta_l\|_c^2 + 4\alpha\|\mbf b(\bY_l) \|_c^2 + 4\alpha\|\btheta_l\|_c^2 + 4\alpha A_1^2\Delta_l^2 + 2\alpha ^2 \|\mbf b(\bY_l) \|_c^2 + 8 \alpha^2A_1^2\Delta_l^2 \nonumber \\
    = & (1+\alpha (3C_1+8))\|\btheta_l\|_c^2+ (6\alpha^2+4\alpha)\|\mbf b(\bY_l) \|_c^2+ A_1^2(12\alpha^2+4\alpha)\Delta_l^2 \nonumber \\
    \leq & (1+\alpha C_2)\|\btheta_l\|_c^2+10\alpha \|\mbf b(\bY_l) \|_c^2 + 16\alpha A_1^2\Delta_l^2. \label{eq:norm_theta_sq_2}
\end{align}
where $C_2 = 3C_1 + 8$. Taking expectation on both sides, we have
\begin{align}
    & \E_{t-2\tau}\|\btheta_{l+1}\|_c^2 \leq (1+\alpha C_2)\E_{t-2\tau}\|\btheta_l\|_c^2+10\alpha \E_{t-2\tau}\|\mbf b(\bY_l) \|_c^2 + 16\alpha A_1^2\E_{t-2\tau}[\Delta_l^2] \nn \\
    & \leq (1+\alpha C_2)\E_{t-2\tau}\|\btheta_l\|_c^2+10\alpha \frac{u_{cD}^2}{l_{cD}^2}\left[\frac{B^2}{\na}+Bm_4\rho^{l-t+2\tau}\right] + 16\alpha A_1^2\E_{t-2\tau}[\Delta_l^2] \tag{Lemma \ref{lem:after_expectation}}\\
    &= (1+\alpha C_2)\E_{t-2\tau}\|\btheta_l\|_c^2+\alpha \bar{c}_t (l) + 16\alpha A_1^2\E_{t-2\tau}[\Delta_l^2], \label{eq:norm_theta_sq_3}
\end{align}
where $\bar{c}_t (l)=10 \frac{u_{cD}^2}{l_{cD}^2}\left[\frac{B^2}{\na}+Bm_4\alpha^2\rho^{l-t+\tau}\right]$. Hence, for any $t-\tau\leq l\leq t$, we have
\begin{align}
    & \E_{t-2\tau} \| \btheta_{l+1} \|_c^2 \leq (1+\alpha C_2)\left[(1+\alpha C_2)\E_{t-2\tau}\|\btheta_{l-1}\|_c^2+\alpha \bar{c}_t (l-1) + 16\alpha A_1^2\E_{t-2\tau}[\Delta_{l-1}^2]\right] \nonumber\\
    & \qquad + \alpha \bar{c}_t (l) + 16 \alpha A_1^2\E_{t-2\tau}[\Delta_l^2]\nonumber \\
    & \leq  (1+\alpha C_2)^{l+1-t+\tau}\E_{t-2\tau}\|\btheta_{t-\tau}\|_c^2 + \alpha \sum_{i=t-\tau}^l (1+\alpha C_2)^{l-i} \bar{c}_t (i) + 16 \alpha A_1^2\E_{t-2\tau}\left[\sum_{i=t-\tau}^l (1 + \alpha C_2)^{l-i} \Delta_i^2 \right] \nn \\
    & \leq (1+\alpha C_2)^{\tau+1}\E_{t-2\tau}\|\btheta_{t-\tau}\|_c^2 + \alpha \underbrace{\sum_{i=t-\tau}^t (1+\alpha C_2)^{t-i} \bar{c}_t (i)}_{T_4} + 16 \alpha A_1^2\E_{t-2\tau}\underbrace{\left[\sum_{i=t-\tau}^t (1+\alpha C_2)^{t-i}\Delta_i^2\right]}_{T_5}. \label{eq:norm_theta_sq_4}
\end{align} 
For the term $T_4$ we have
\begin{align}
    T_4 =& 10\frac{u_{cD}^2}{l_{cD}^2}\sum_{i=0}^\tau(1+\alpha C_2)^{\tau-i}\left[\frac{B^2}{\na}+Bm_4\alpha^2\rho^{i}\right] \nn \\
    =& 10\frac{u_{cD}^2}{l_{cD}^2}\frac{B^2}{\na}\frac{(1+\alpha C_2)^{\tau+1}-1}{\alpha C_2} + 10\frac{u_{cD}^2}{l_{cD}^2}Bm_4\alpha^2(1+\alpha C_2)^\tau \sum_{i=0}^\tau \left(\frac{\rho}{1+\alpha C_2}\right)^i \nn \\
    \leq & 10 \frac{u_{cD}^2}{l_{cD}^2} \frac{B^2}{\na}\frac{(1+\alpha C_2)^{\tau+1}-1}{\alpha C_2} + 10\frac{u_{cD}^2}{l_{cD}^2}Bm_4\alpha^2(1+\alpha C_2)^\tau \sum_{i=0}^\tau \rho^{i}\tag{$\alpha C_2\geq 0$}\\
    \leq & 10\frac{u_{cD}^2}{l_{cD}^2}\frac{B^2}{\na}\frac{(1+\alpha C_2)^{\tau+1}-1}{\alpha C_2} + 10\frac{u_{cD}^2}{l_{cD}^2}Bm_4\alpha^2(1+\alpha C_2)^\tau\frac{1}{1-\rho}. \label{eq:norm_theta_sq_T4_1}
\end{align}
By the same argument as in Lemma \ref{lem:theta_diff}, and by the assumption on $\alpha$, we have $(1+\alpha C_2)^{\tau+1} \leq 1 + 2\alpha C_2 (\tau+1)\leq 1+4\alpha \tau C_2 \leq 2$ and $(1+\alpha C_2)^{\tau} \leq 1 + 2\alpha C_2 \tau \leq 1 + 1/2 \leq 2$. Hence, we have
\begin{align}
    T_4 \leq 40 \frac{u_{cD}^2}{l_{cD}^2} \lb \frac{B^2}{\na}\tau + \frac{Bm_4 \alpha^2}{1-\rho} \rb. \label{eq:norm_theta_sq_T4_2}
\end{align}
Furthermore, for $T_5$, we have
\begin{align}
    T_5=&\sum_{i=0}^\tau (1+\alpha C_2)^{\tau-i}\E_{t-2\tau}[\Delta_{t-\tau + i}^2] \nn \\
    \leq & \sum_{i=0}^\tau (1+\alpha C_2)^{\tau}\E_{t-2\tau}[\Delta_{t-\tau + i}^2]\tag{$\alpha C_2>0$} \nn \\
    \leq& \sum_{i=0}^\tau (1+2\alpha C_2\tau) \E_{t-2\tau}[\Delta_{t-\tau+i}^2] \tag{assumption on $\alpha$} \\
    \leq& 2\sum_{i=0}^\tau  \E_{t-2\tau}[\Delta_{t-i}^2]. \label{eq:norm_theta_sq_T5_1}
\end{align}
Combining the bounds on $T_4$ \eqref{eq:norm_theta_sq_T4_2} and $T_5$ \eqref{eq:norm_theta_sq_T5_1} in \eqref{eq:norm_theta_sq_4} we have for any $t-\tau\leq l\leq t$,
\begin{align}
    \E_{t-2\tau}\|\btheta_{l}\|_c^2\leq 2\E_{t-2\tau}\|\btheta_{t-\tau}\|_c^2 + 40 \frac{u_{cD}^2}{l_{cD}^2} \lb \frac{B^2}{\na} \alpha \tau + Bm_4 \alpha^3 \frac{1}{1-\rho} \rb + 32\alpha A_1^2\sum_{i=0}^\tau  \E_{t-2\tau}[\Delta_{t-i}^2]. \label{eq:norm_theta_sq_5}
\end{align}
Furthermore, we have
\begin{align}
    \| \btheta_t - \btheta_{t-\tau} \|_c^2 \leq & \left( \sum_{i=t-\tau}^{t-1} \| \btheta_{i+1} - \btheta_i \|_c \right)^2 \tag{triangle inequality} \\
    \leq & \tau \sum_{i=t-\tau}^{t-1} \|\btheta_{i+1}-\btheta_i\|_c^2 \nn \\
    \leq & \tau \sum_{i=t-\tau}^{t-1} \lb \alpha^2 C_1^2 \| \btheta_i \|_c^2 + 3 \alpha^2 \| \mbf b(\by_i) \|^2_c + 3 \alpha^2 A_1^2 \Delta_i^2 \rb. \tag{from \eqref{eq:lem:theta_diff_sq_consec_0}}
\end{align}
Taking expectation on both sides, and using the bounds in Lemma \ref{lem:after_expectation} and \eqref{eq:norm_theta_sq_5}, we get
\begin{align}
    & \E_{t-2\tau} \| \btheta_t - \btheta_{t-\tau} \|_c^2 \nn \\
    & \leq \tau \alpha^2 C_1^2 \sum_{i=t-\tau}^{t-1} \lb 2 \E_{t-2\tau} \| \btheta_{t-\tau} \|_c^2 + 40 \frac{u_{cD}^2}{l_{cD}^2} \lp \frac{B^2}{\na} \alpha \tau + \frac{Bm_4 \alpha^3}{1-\rho} \rp + 32 \alpha A_1^2 \sum_{i=0}^\tau  \E_{t-2\tau} [\Delta_{t-i}^2] \rb \nn \\
    & \quad + 3 \alpha^2 \tau \sum_{i=t-\tau}^{t-1} \frac{u_{cD}^2}{l_{cD}^2} \left[\frac{B^2}{\na} + B m_4 \rho^{i-t+2\tau} \right] + 3\alpha^2 A_1^2 \tau \sum_{i=t-\tau}^{t-1} \E_{t-2\tau} \lb \Delta_i^2 \rb \nonumber\\
    & \leq \tau^2 \alpha^2 C_1^2 \left[ 2 \E_{t-2\tau} \| \btheta_{t-\tau} \|_c^2 + 40 \frac{u_{cD}^2}{l_{cD}^2} \lp \frac{B^2}{\na}\alpha\tau + \frac{Bm_4 \alpha^3}{1-\rho} \rp \right] + 3 \alpha^2 \tau \frac{u_{cD}^2}{l_{cD}^2} \left[\frac{\tau B^2}{\na} + \frac{Bm_4 \alpha^2}{1 - \rho} \right] \tag{$ \rho^{\tau} \leq \alpha^2$} \\
    & \quad + \alpha^2 A_1^2 \tau (3 + 32 \alpha \tau C_1^2) \sum_{i=0}^{\tau} \E_{t-2\tau} \Delta_{t-i}^2. \label{eq:first_claim_diff_lemma_2}
\end{align}
Furthermore, by triangle inequality, we have $\|\btheta_{t-\tau}\|_c\leq \|\btheta_{t}-\btheta_{t-\tau}\|_c + \|\btheta_{t}\|_c$. Squaring both sides, we have $\|\btheta_{t-\tau}\|_c^2\leq (\|\btheta_{t}-\btheta_{t-\tau}\|_c + \|\btheta_{t}\|_c)^2\leq 2\|\btheta_{t}-\btheta_{t-\tau}\|_c^2 + 2\|\btheta_{t}\|_c^2$. By assumption on $\alpha$, we have  $2\tau^2\alpha^2C_1^2\|\btheta_{t-\tau}\|_c^2\leq 4\tau^2\alpha^2C_1^2\|\btheta_{t}-\btheta_{t-\tau}\|_c^2 + 4\tau^2\alpha^2C_1^2\|\btheta_{t}\|_c^2\leq 0.5\|\btheta_{t}-\btheta_{t-\tau}\|_c^2 + 4\tau^2\alpha^2C_1^2\|\btheta_{t}\|_c^2$. Substituting this in \eqref{eq:first_claim_diff_lemma_2} we get the result. 
\end{proof}

\begin{proof}[Proof of Lemma \ref{lem:consensus_error}]
For $s \sync + 1 \leq t \leq (s+1) \sync - 1$, where $s = \lfloor t/\sync \rfloor$,
\begin{align}
    & \Ot \triangleq \frac{1}{\na} \sumik \lp \Dit \rp^2 = \frac{1}{\na} \sumik \| \btheta_{t} - \btheta_{t}^i \|_c^2 \nn \\
    &= \frac{1}{\na} \sumik \lnr \lp \btheta^i_{s \sync} - \btheta_{s \sync} \rp + \alpha \sum_{t' = s \sync}^{t-1} \lb \lp \mbf G^i (\btheta^i_{t'}, \by^i_{t'}) - \btheta^i_{t'} + \mathbf{b}^i (\by^i_{t'}) \rp - \lp \mbf G (\boldsymbol{\Theta}_{t'}, \by_{t'}) - \btheta_{t'} + \mbf b(\by_{t'}) \rp \rb \rnr_c^2 \nn \\
    & \leq \frac{2 \alpha^2}{\na} \sumik \lb \lnr \sum_{t' = s \sync}^{t-1} \lb \mathbf{b}^i (\by^i_{t'}) - \mbf b(\by_{t'}) \rb \rnr_c^2 + \lnr \sum_{t' = s \sync}^{t-1} \lb \lp \mbf G^i (\btheta^i_{t'}, \by^i_{t'}) - \btheta^i_{t'} \rp - \lp \mbf G (\boldsymbol{\Theta}_{t'}, \by_{t'}) - \btheta_{t'} \rp \rb \rnr_c^2 \rb \tag{$\because \btheta^i_{s \sync} = \btheta_{s \sync}$ and Young's inequality} \\
    & \leq \frac{2 \alpha^2 u_{c2}^2}{\na} \sumik \lb \sum_{t' = s \sync}^{t-1} \lnr \mathbf{b}^i (\by^i_{t'}) - \mbf b(\by_{t'}) \rnr_2^2 + \sum_{\substack{t', t'' = s \sync \\ t' \neq t''}}^{t-1} \lan \mathbf{b}^i (\by^i_{t'}) - \mbf b(\by_{t'}), \mathbf{b}^i (\by^i_{t''}) - \mbf b(\by_{t''}) \ran \rb \nn \\
    & \quad + \frac{4 \alpha^2}{\na} (t - s \sync) \sum_{t' = s \sync}^{t-1} \sumik \lb u_{c2}^2 \lnr \mbf G^i (\btheta^i_{t'}, \by^i_{t'}) - \mbf G (\boldsymbol{\Theta}_{t'}, \by_{t'}) \rnr_2^2 + \lnr \btheta^i_{t'} - \btheta_{t'} \rnr_c^2 \rb \tag{Triangle inequality, and $ \norm{\cdot}_c \leq u_{c2} \norm{\cdot}_2$} \\
    & \leq \frac{2 \alpha^2 u_{c2}^2}{\na} \sumik \lb \sum_{t' = s \sync}^{t-1} \frac{1}{l_{c2}^2} \lnr \mathbf{b}^i (\by^i_{t'}) \rnr_c^2 + \sum_{\substack{t', t'' = s \sync \\ t' \neq t''}}^{t-1} \lan \mathbf{b}^i (\by^i_{t'}) - \mbf b(\by_{t'}), \mathbf{b}^i (\by^i_{t''}) - \mbf b(\by_{t''}) \ran \rb \nn \\
    & \quad + \frac{4 \alpha^2}{\na} (t - s \sync) \sum_{t' = s \sync}^{t-1} \sumik \lb \frac{u_{c2}^2}{l_{c2}^2} \lnr \mbf G^i (\btheta^i_{t'}, \by^i_{t'}) \rnr_c^2 + \lnr \btheta^i_{t'} - \btheta_{t'} \rnr_c^2 \rb \tag{$ l_{c2} \norm{\cdot}_2 \leq \norm{\cdot}_c$ and $Var(X)\leq E[X^2]$}.
\end{align}
Taking expectation
\begin{align}
    & \mbe \lb \Ot \rb \leq \frac{2 \alpha^2 u_{c2}^2}{l_{c2}^2} (t - s \sync) B^2 + \frac{4 \alpha^2 u_{c2}^2}{\na} \sumik \sum_{\substack{t', t'' = s \sync \\ t' < t''}}^{t-1} \mbe \lan \mathbf{b}^i (\by^i_{t'}) - \mbf b(\by_{t'}), \mbe_{t'} \lb \mathbf{b}^i (\by^i_{t''}) - \mbf b(\by_{t''}) \rb \ran \nn \\
    & \quad + \frac{4 \alpha^2}{\na} (t - s \sync) \sum_{t' = s \sync}^{t-1} \sumik \mbe \lb \frac{2 u_{c2}^2}{l_{c2}^2} \lcb \lnr \mbf G^i (\btheta^i_{t'}, \by^i_{t'}) - \mbf G^i (\btheta_{t'}, \by^i_{t'}) \rnr_c^2 + \lnr \mbf G^i (\btheta_{t'}, \by^i_{t'}) \rnr_c^2 \rcb + \lnr \btheta^i_{t'} - \btheta_{t'} \rnr_c^2 \rb \tag{using Assumption \ref{ass:lipschitz}} \\
    & \leq \frac{2 \alpha^2 u_{c2}^2}{l_{c2}^2} (t - s \sync) B^2 + \frac{4 \alpha^2 u_{c2}^2}{l_{c2} \na} \sumik \sum_{\substack{t', t'' = s \sync \\ t' < t''}}^{t-1} \mbe \lb \norm{\mathbf{b}^i (\by^i_{t'}) - \mbf b(\by_{t'})}_2 \norm{\mbe_{t'} \lb \mathbf{b}^i (\by^i_{t''}) - \mbf b(\by_{t''}) \rb}_c \rb \tag{using $ l_{c2} \norm{\cdot}_2 \leq \norm{\cdot}_c$} \\
    & \quad + \frac{4 \alpha^2 (t-s \sync)}{\na} \sum_{t' = s \sync}^{t-1} \sumik \mbe \left[ \frac{2 u_{c2}^2}{l_{c2}^2} \lcb A_1^2 \lnr \btheta^i_{t'} - \btheta_{t'} \rnr_c^2 + A_2^2 \lnr \btheta_{t'} \rnr_c^2 \rcb + \lnr \btheta^i_{t'} - \btheta_{t'} \rnr_c^2 \right] \tag{Assumption \ref{ass:lipschitz}} \\
    & \overset{(a)}{\leq} \frac{2 \alpha^2 u_{c2}^2}{l_{c2}^2} (t - s \sync) B^2 + \frac{4 \alpha^2 u_{c2}^2}{l_{c2} \na} \frac{2 B}{l_{c2}} \sumik \sum_{\substack{t' = s \sync}}^{t-1} \sum_{\substack{t'' = s \sync \\ t' < t''}}^{t-1} \left[(m_2+m_4) \lb \rho^{t'' - t'} \rb+m_3\eta_i\right] \nn \\
    & \quad + 4 \alpha^2 (t-s \sync) \sum_{t' = s \sync}^{t-1} \lb \lp 1 + \frac{2 A_1^2 u_{c2}^2}{l_{c2}^2} \rp \Omega_{t'} + \frac{2 A_2^2 u_{c2}^2}{l_{c2}^2} \lnr \btheta_{t'} \rnr_c^2 \rb \nn \\
    & \overset{(b)}{\leq} \frac{2 \alpha^2 u_{c2}^2}{l_{c2}^2} (t - s \sync) \lb B^2 + \frac{4 m_2 B \rho}{1 - \rho} \rb + 4 \alpha^2 (t-s \sync) \sum_{t' = s \sync}^{t-1} \lb \lp 1 + \frac{2 A_1^2 u_{c2}^2}{l_{c2}^2} \rp \Omega_{t'} + \frac{2 A_2^2 u_{c2}^2}{l_{c2}^2} \lnr \btheta_{t'} \rnr_c^2 \rb, \label{eq:lem:consensus_error_1}
\end{align}
where $(a)$ follows since
\begin{align*}
    & \mbe \lb \norm{\mathbf{b}^i (\by^i_{t'}) - \mbf b(\by_{t'})}_2 \norm{\mbe_{t'} \lb \mathbf{b}^i (\by^i_{t''}) - \mbf b(\by_{t''}) \rb}_c \rb \nn \\
    & \leq \mbe \lb \lp \norm{\mathbf{b}^i (\by^i_{t'})}_2 + \norm{\mbf b(\by_{t'})}_2 \rp \cdot \lcb \norm{\mbe_{t'} \lb \mathbf{b}^i (\by^i_{t''}) \rb}_c + \norm{\mbe_{t'} \lb \frac{1}{\na} \sum_{j=1}^\na \mbf b(\by^j_{t''}) \rb}_c \rcb \rb \tag{Triangle inequality} \\
    & \leq \frac{2B}{l_{c2}} \cdot \lb (m_2 + m_4) \mbe \lb \rho^{t'' - t'} \rb + m_3 \eta_i \rb. \tag{Assumption \ref{ass:mixing}, \ref{ass:lipschitz}}
\end{align*}
In addition, the inequality in $(b)$ follows since
\begin{align*}
    \sum_{\substack{t' = s \sync}}^{t-1} \sum_{\substack{t'' = s \sync \\ t' < t''}}^{t-1} \left[(m_2+m_4) \lb \rho^{t'' - t'} \rb+m_3\eta_i\right] &\leq (m_2+m_4)\sum_{\substack{t' = s \sync}}^{t-1} \sum_{\substack{t'' = t' + 1}}^{t-1} \rho^{t'' - t'} + m_3(t-sK)(t-1-sK)\eta_i \\
    &= (m_2+m_4)\sum_{\substack{t' = s \sync}}^{t-1} \frac{\rho - \rho^{t-t'}}{1 - \rho} + m_3(t-sK)(t-1-sK)\eta_i\\
    &= (m_2+m_4)\frac{\rho (t - s \sync)-\rho \lp 1 - \rho^{t - s \sync}\rp}{1 - \rho} + m_3 (t-sK)(t-1-sK)\eta_i\\
    & \leq (m_2+m_4)\frac{\rho (t - s \sync)}{1 - \rho}+ m_3(t-sK)(t-1-sK)\eta_i.
\end{align*}

For simplicity, we define $\ta_1 = \frac{2 A_1^2 u_{c2}^2}{l_{c2}^2}, \ta_2 = \frac{2 A_2^2 u_{c2}^2}{l_{c2}^2}, \tb = \frac{u_{c2}}{l_{c2}} B$. 
Hence, \eqref{eq:lem:consensus_error_1} simplifies to
\begin{align}
     \mbe \Ot \leq& 2 \alpha^2 (t - s \sync) \bigg[ \underbrace{\tb^2 \lp 1 + \frac{4 (m_2+m_4) \rho}{B (1 - \rho)} + \frac{4m_3(t-sK-1)}{B}\frac{1}{N}\sumik \eta_i\rp}_{\check{B}} \nonumber\\
    &\qquad\qquad\qquad + 2 \sum_{t' = s \sync}^{t-1} \lp (1 + \ta_1) \mbe \Omega_{t'} + \ta_2 \mbe \lnr \btheta_{t'} \rnr_c^2 \rp \bigg]. \label{eq:lem:consensus_error_2}
\end{align}
Recursively applying \eqref{eq:lem:consensus_error_2}, going back $2$ steps, we see
\begin{align}
    \mbe \Ot & \leq 2 \alpha^2 \check{B} (t - s \sync)  + 4 \alpha^2 (t-s \sync) \ta_2 \sum_{t' = s \sync}^{t-1} \mbe \lnr \btheta_{t'} \rnr_c^2 + 4 \alpha^2 (t-s \sync) (1 + \ta_1) \sum_{t' = s \sync}^{t-2} \mbe \Omega_{t'} \nn \\
    & \quad + 4 \alpha^2 (t-s \sync) (1 + \ta_1) \underbrace{2 \alpha^2 (t - s \sync - 1) \lb \check{B} + 2 \sum_{t' = s \sync}^{t-2} \lp (1 + \ta_1) \mbe \Omega_{t'} + \ta_2 \mbe \lnr \btheta_{t'} \rnr_c^2 \rp \rb}_{\geq \Omega_{t-1}} \nn \\
    &= 2 \alpha^2 \check{B} (t - s \sync)  \lb 1 + 4 \alpha^2 (1 + \ta_1) (t-1-s \sync) \rb \nn \\
    & \quad + 4 \alpha^2 (t-s \sync) \ta_2 \lb \sum_{t' = s \sync}^{t-1} \lnr \btheta_{t'} \rnr_c^2 + 4 \alpha^2 (1 + \ta_1) (t-1-s \sync) \sum_{t' = s \sync}^{t-2} \mbe \lnr \btheta_{t'} \rnr_c^2 \rb \nn \\
    & \quad + 4 \alpha^2 (t-s \sync) (1 + \ta_1) \lb 1 + 4 \alpha^2 (1 + \ta_1) (t-1-s \sync) \rb \sum_{t' = s \sync}^{t-2} \mbe \Omega_{t'} \nn \\
    & \leq 4 \alpha^2 (t-s \sync) \lb 1 + 4 \alpha^2 (1 + \ta_1) (t-1-s \sync) \rb \Bigg[ \check{B} + \ta_2 \sum_{t' = s \sync}^{t-1} \mbe \lnr \btheta_{t'} \rnr_c^2 \nn \\
    & \qquad \qquad \qquad \qquad \qquad \qquad \qquad \qquad \qquad \qquad \qquad + (1 + \ta_1) \sum_{t' = s \sync}^{t-2} \mbe \Omega_{t'} \Bigg]. \label{eq:lem:consensus_error_3}
\end{align}
To derive the bound for going back, in general, $j$ steps (such that $t-j \geq s \sync$), we use an induction argument.
Suppose for going back $k (< j)$ steps, the bound is
\begin{align}
    \mbe \Ot & \leq 4 \alpha^2 (t-s \sync) \check{B} \prod_{\ell=1}^{k-1} \lb 1 + 4 \alpha^2 (1 + \ta_1) (t-\ell-s \sync) \rb \nn \\
    & \quad + 4 \alpha^2 (t-s \sync) \ta_2 \prod_{\ell=1}^{k-1} \lb 1 + 4 \alpha^2 (1 + \ta_1) (t-\ell-s \sync) \rb \sum_{t' = s \sync}^{t-1} \mbe \lnr \btheta_{t'} \rnr_c^2 \nn \\
    & \quad + 4 \alpha^2 (t-s \sync) (1 + \ta_1) \prod_{\ell=1}^{k-1} \lb 1 + 4 \alpha^2 (1 + \ta_1) (t-\ell-s \sync) \rb \sum_{t' = s \sync}^{t-k} \mbe \Omega_{t'}. \label{eq:lem:consensus_error_3a}
\end{align}
We derive the bound for $k+1$ steps. For this, we further bound the last term in \eqref{eq:lem:consensus_error_3a}.
\begin{align}
    & 4 \alpha^2 (t-s \sync) (1 + \ta_1) \prod_{\ell=1}^{k-1} \lb 1 + 4 \alpha^2 (1 + \ta_1) (t-\ell-s \sync) \rb \lb \sum_{t' = s \sync}^{t-k-1} \mbe \Omega_{t'} + \mbe \Omega_{t-k} \rb \nn \\
    & \leq 4 \alpha^2 (t-s \sync) (1 + \ta_1) \prod_{\ell=1}^{k-1} \lb 1 + 4 \alpha^2 (1 + \ta_1) (t-\ell-s \sync) \rb \sum_{t' = s \sync}^{t-k-1} \mbe \Omega_{t'} \nn \\
    & \quad + 4 \alpha^2 (t-s \sync) (1 + \ta_1) \prod_{\ell=1}^{k-1} \lb 1 + 4 \alpha^2 (1 + \ta_1) (t-\ell-s \sync) \rb \nn \\
    & \qquad \qquad \qquad \times 4 \alpha^2 (t-k-s \sync) \lb \tb^2 \lp 1 + \frac{4 m_2 \rho}{B (1 - \rho)} \rp + \sum_{t' = s \sync}^{t-k-1} \lp (1 + \ta_1) \mbe \Omega_{t'} + \ta_2 \mbe \lnr \btheta_{t'} \rnr_c^2 \rp \rb \tag{using \eqref{eq:lem:consensus_error_2}} \\
    & \leq 4 \alpha^2 (t-s \sync) (1 + \ta_1) \prod_{\ell=1}^{k-1} \lb 1 + 4 \alpha^2 (1 + \ta_1) (t-\ell-s \sync) \rb \sum_{t' = s \sync}^{t-k-1} \mbe \Omega_{t'} \nn \\
    & \quad + 4 \alpha^2 (t-s \sync) \prod_{\ell=1}^{k-1} \lb 1 + 4 \alpha^2 (1 + \ta_1) (t-\ell-s \sync) \rb \nn \\
    & \qquad \qquad \qquad \times 4 \alpha^2 (1 + \ta_1) (t-k-s \sync) \lb \tb^2 \lp 1 + \frac{4 m_2 \rho}{B (1 - \rho)} \rp + \sum_{t' = s \sync}^{t-k-1} \lp (1 + \ta_1) \mbe \Omega_{t'} + \ta_2 \mbe \lnr \btheta_{t'} \rnr_c^2 \rp \rb.
    \label{eq:lem:consensus_error_3b}
\end{align}
Substituting \eqref{eq:lem:consensus_error_3b} into \eqref{eq:lem:consensus_error_3a}, we see that the induction hypothesis in \eqref{eq:lem:consensus_error_3a} holds.
We can go back as far as the last instant of synchronization, $j \leq t - s \sync$. 
For $j = t - s \sync$, we get
\begin{align}
    \mbe \Ot 
    & \leq 4 \alpha^2 (t-s \sync) \prod_{\ell=1}^{t-1 -s \sync} \lb 1 + 4 \alpha^2 (1 + \ta_1) (t-\ell-s \sync) \rb \nn \\
    & \qquad \qquad \qquad \times \lb \check{B} + \ta_2 \sum_{t' = s \sync}^{t-1} \mbe \lnr \btheta_{t'} \rnr_c^2 + (1+\ta) \mbe \Omega_{s \sync} \rb. \label{eq:lem:consensus_error_5}
\end{align}
Next, using $1 + x \leq e^{x}$ for $x \geq 0$, we get
\begin{align}
    \prod_{\ell=1}^{t-1 -s \sync} \lb 1 + 4 \alpha^2 (1 + \ta_1) (t-\ell-s \sync) \rb & \leq \exp{\lp \sum_{\ell=1}^{t-1 -s \sync} 4 \alpha^2 (1 + \ta_1) (t-\ell-s \sync) \rp} \nn \\
    & \leq \exp{\lp 2 \alpha^2 (1 + \ta_1) (t-s \sync)^2 \rp} \nn \\
    & \leq \frac{5}{4}, \label{eq:lem:consensus_error_5a}
\end{align}
if $\alpha$ is small enough such that $2 \alpha^2 (1 + \ta_1) (\sync-1)^2 \leq \ln{\frac{5}{4}}$, which holds true by the assumption on the step size.
Using \eqref{eq:lem:consensus_error_5a} in \eqref{eq:lem:consensus_error_5}, we get
\begin{align}
    \mbe \Ot & \leq 5 \alpha^2 (t-s \sync) \lb \check{B} + \ta_2 \sum_{t' = s \sync}^{t-1} \mbe \lnr \btheta_{t'} \rnr_c^2 + (1+\ta) \underbrace{\mbe \Omega_{s \sync}}_{=0} \rb \nn \\
    & = 5 \alpha^2 (t-s \sync) \check{B} + 5 \alpha^2 (t-s \sync) \ta_2 \sum_{t' = s \sync}^{t-1} \mbe \lnr \btheta_{t'} \rnr_c^2, \label{eq:lem:consensus_error_6}
\end{align}
which concludes the proof.
\end{proof}

\begin{proof}[Proof of Lemma \ref{lem:after_expectation}]
We have
\begin{align}
    & \E_{t-r} [\| \mbf b(\byt) \|_c] \leq u_{cD} \E_{t-r} [\| \mbf b(\byt) \|_D] \nn \\
    &= u_{cD} \E_{t-r} \left[ \sqrt{\mbf b(\byt)^\top D \mbf b(\byt)} \right] \nn \\
    \leq & u_{cD} \sqrt{\E_{t-r} \lb \mbf b (\byt)^\top D \mbf b (\byt) \rb} \tag{concavity of square root} \\
    &=\frac{u_{cD}}{N}  \sqrt{\E_{t-r} \left[ \left(  \sumik  \mathbf{b}^i  (\byit) \right)^\top D \left(  \sumik \mathbf{b}^i  (\byit) \right) \right]} \nn \\
    &=\frac{u_{cD}}{N} \sqrt{\sum_{i=1}^N\E_{t-r}\left[{\mathbf{b}^i(\byit)}^\top D\sum_{j=1}^N \mathbf{b}^j(\by^j_t)\right] } \nn\\
    &=\frac{u_{cD}}{N} \sqrt{\sum_{i=1}^N\E_{t-r}\left[{\mathbf{b}^i(\byit)}^\top D\left(\mathbf{b}^i(\byit) + \sum_{\substack{j=1 \\ j\neq i}}^N \mathbf{b}^j(\by^j_t) \right)\right] } \nn\\
    &\leq\frac{u_{cD}}{N} \sqrt{\sum_{i=1}^N\E_{t-r}\left[{\mathbf{b}^i(\byit)}^\top D \mathbf{b}^i(\byit)\right]} +\sqrt{\sum_{i=1}^N \E_{t-r}\left[{\mathbf{b}^i(\byit)}^\top D\left( \sum_{\substack{j=1 \\ j\neq i}}^N \mathbf{b}^j(\by^j_t) \right)\right] } \nn\\
    &=\frac{u_{cD}}{N} \sqrt{\underbrace{\sum_{i=1}^N\E_{t-r}\left[\|\mathbf{b}^i(\byit)\|_D^2\right]}_{T_1}} +\frac{u_{cD}}{N}\sqrt{\underbrace{\sum_{i=1}^N \E_{t-r}\left[{\mathbf{b}^i(\byit)}\right]^\top D\E_{t-r}\left[ \sum_{\substack{j=1 \\ j\neq i}}^N \mathbf{b}^j(\by^j_t) \right]}_{T_2} } \nn\tag{Assumption \ref{ass:noise_independence}}
\end{align}
For the term $T_1$ we have
\begin{align}
    T_1 \leq &  \sumik \E_{t-r} \lb l_{cD}^{-2} \lnr \mathbf{b}^i (\byit) \rnr_c^2 \rb \leq \frac{1}{l_{cD}^2}  \sumik \E_{t-r}[B^2] \tag{Assumption \ref{ass:lipschitz}}\\
    = & \frac{B^2N}{l_{cD}^2}. \label{eq:bound_byt_T_1}
\end{align}
For the term $T_2$ we have
\begin{align}
    T_2 = & \sum_{i=1}^N \E_{t-r}\left[{\mathbf{b}^i(\byit)}\right]^\top D\E_{t-r}\left[-\mathbf{b}^i(\byit) +  \sum_{\substack{j=1 }}^N \mathbf{b}^j(\by^j_t) \right]\nn\\
    = & -\sum_{i=1}^N \E_{t-r}\left[{\mathbf{b}^i(\byit)}\right]^\top D\E_{t-r}\left[\mathbf{b}^i(\byit)\right]\nn \\
    &+ \sum_{i=1}^N \E_{t-r}\left[{\mathbf{b}^i(\byit)}\right]^\top D\E_{t-r}\left[ \sum_{\substack{j=1 }}^N \mathbf{b}^j(\by^j_t) \right]\nn\\
    \leq & -\sum_{i=1}^N \|\E_{t-r}\left[{\mathbf{b}^i(\byit)}\right]\|_D^2\nn\\
    &+ \sum_{i=1}^N\left\| \E_{t-r}\left[{\mathbf{b}^i(\byit)}\right]\right\|_D. \left\|\E_{t-r}\left[ \sum_{\substack{j=1 }}^N \mathbf{b}^j(\by^j_t) \right]\right\|_D\nn\\
    \leq& l_{cD}^{-2}\sum_{i=1}^N\left\| \E_{t-r}\left[{\mathbf{b}^i(\byit)}\right]\right\|_c. \left\|\E_{t-r}\left[ \sum_{\substack{j=1 }}^N \mathbf{b}^j(\by^j_t) \right]\right\|_c\nn\\
    \leq& \frac{B}{l_{cD}^{2}}\sum_{i=1}^N \left\|\E_{t-r}\left[ \sum_{\substack{j=1 }}^N \mathbf{b}^j(\by^j_t) \right]\right\|_c\tag{Assumption \ref{ass:lipschitz}}\\
    \leq& \frac{N^2B}{l_{cD}^{2}}m_4\rho^r.\tag{Assumption \ref{ass:mixing}}
\end{align}
Substituting the above term and \eqref{eq:bound_byt_T_1} for $T_1$ and $T_2$, we get the result in \eqref{eq:lemma_bound_byt_1}.

The proof of \eqref{eq:lemma_bound_byt_2} follows analogously.
\end{proof}

\begin{proof}[Proof of Lemma \ref{lem:G_diff}]
By definition, we have 
\begin{align*}
    \| \mbf G (\boldsymbol{\Theta}_t, \byt)-\mbf G (\boldsymbol{\btheta}_t, \byt)\|^2_c =& \left\|\frac{1}{\na}\sumik(\Gi(\bthetait, \byit) - \Gi(\bthetat, \byit)) \right\|_c^2 \\
    \leq & \left( \frac{1}{\na} \sumik \lnr \Gi(\bthetait, \byit) - \Gi(\bthetat, \byit) \rnr_c \right)^2 \tag{convexity of norm} \\
    \leq & \left( \frac{1}{\na} \sumik A_1 \| \bthetait - \bthetat \|_c\right)^2 \tag{Assumption \ref{ass:lipschitz}}\\
    =& \left(A_1\Delta_t\right)^2\\
    \leq & \frac{1}{\na} \sumik A_1^2 \|\bthetait - \bthetat\|_c^2 \tag{convexity of square}\\
    =&A_1^2 \Omega_t.\tag{By definition of $\Omega_t$}
\end{align*}

Furthermore, by the convexity of $(\cdot)^2$, we have
\begin{align}
    \Delta_t^2 = \left(\frac{1}{\na} \sumik \Delta_t^i\right)^2\leq \frac{1}{\na} \sumik (\Delta_t^i)^2 = \Omega_t.\nn
\end{align}
\end{proof}

\begin{proof}[Proof of Lemma \ref{lem:Moreau_grad}]
Since $\M_f^{\psi, g}(\cdot)$ is convex, and there exists a norm, $\| \cdot \|_m$, such that $\M_f^{\psi, g}(x) = \frac{1}{2} \norm{x}_m^2$ (see Proposition \ref{prop:Moreau}), using the chain rule of subdifferential calculus,
\begin{align*}
    \G \M_f^{\psi, g}(x) = \norm{x}_m u_x,
\end{align*}
where $u_x \in \partial \norm{x}_m$ is a subgradient of $\norm{x}_m$ at $x$.
Hence,
\begin{align*}
    \norm{\G \M_f^{\psi, g}(x)}_m^{\star} = \norm{x}_m \norm{u_x}_m^{\star},
\end{align*}
where $\norm{\cdot}_m^{\star}$ is the dual norm of $\norm{\cdot}_m$. 
Since $\norm{\cdot}_m$ is convex and, as a function of $x$, is $1$-Lipschitz w.r.t. $\norm{\cdot}_m$, we have $\norm{u_x}_m^{\star} \leq 1$ (see Lemma 2.6 in \cite{shalev2012online}).

Furthermore, by the convexity of the $\norm{\cdot}_m$ norm, $\norm{0}_m \geq \norm{x}_m + \lan u_x, -x \ran$. Therefore,
\begin{align*}
    \lan \G \M_f^{\psi, g}(x), x \ran &= \norm{x}_m \lan u_x, x \ran \geq \norm{x}_m^2 = 2 \M_f^{\psi, g}(x).
\end{align*}
\end{proof}

\section{Federated TD-learning} \label{sec:fed_td_learning_app}

\subsection{On-policy Function Approximation}\label{sec:fed_td_learning_on_app}

\begin{proposition}\label{prop:TD_on}
On-policy TD-learning with linear function approximation Algorithm \ref{alg:TD-learning} satisfies the following:
\begin{enumerate}
    \item $\bthetait = \bvit-\bv^\pi$ \label{item:TD_learning_10}
    \item $S_t=(S_t^1,\dots,S_t^\na)$ and $A_t=(A_t^1,\dots,A_t^\na)$\label{item:TD_learning_20}
    \item $\byit=(S_t^i,A_t^i,\dots,S_{t+n-1}^i,A_{t+n-1}^i,S_{t+n}^i)$ and $\mbf \byt=(S_t,A_t,\dots,S_{t+n-1},A_{t+n-1},S_{t+n})$\label{item:TD_learning_30}
    \item $\mu^\pi:$ Stationary distribution of the policy $\pi$. \label{item:TD_learning_40}
\end{enumerate}
Furthermore, choose some arbitrary positive constant $\beta>0$. The corresponding $\Gi(\bthetait,\byit)$ and $\mathbf{b}^i (\byit)$ in Algorithm \ref{alg:TD-learning} for On-policy TD-learning with linear function approximation is as follows

\begin{enumerate}
    \item $\Gi(\bthetait,\byit)=\bthetait +\frac{1}{\beta}\phi(S_t^i)\sum_{l=t}^{t+n-1}\gamma^{l-t}\!\left(\gamma \phi(S_{l\!+\!1}^i)^\top \bthetait\!-\!\phi(S_l^i)^\top \bthetait\right)$\label{item:TD_learning_50}
    \item $\mathbf{b}^i (\byit)=\frac{1}{\beta}\phi(S_t^i)\sum_{l=t}^{t+n-1}\gamma^{l-t}\left(\mathcal R(S_l^i,A_l^i)+\gamma \phi(S_{l\!+\!1}^i)^\top \bv^\pi\!-\!\phi(S_l^i)^\top \bv^\pi\right)$\label{item:TD_learning_60}
\end{enumerate}
where $\bv^\pi$ solves the projected bellman equation $\Phi \bv^\pi = \Pi_\pi((\mathcal{T}^\pi)^n \Phi \bv^\pi)$. Furthermore, the corresponding step size $\alpha$ in Algorithm \ref{alg:fed_stoch_app} is $\alpha\times\beta$.

\end{proposition}

\begin{lemma}\label{lem:TD_learning_20}
Consider the federated on-policy TD-learning Algorithm \ref{alg:TD-learning} as a special case of FeGSAM Algorithm \ref{alg:fed_stoch_app} (see Proposition \ref{prop:TD_on}). Suppose that the trajectory $\{S_t^i\}_{t=0,1,\dots}$ converges geometrically fast to its stationary distribution as follows $d_{TV}(P(S_t^i=\cdot|S_0^i)||\mu^i(\cdot))\leq \bar{m}\bar{\rho}^t$ for all $i=1,2,\dots, \na$. The corresponding  $\bG^i(\btheta)$ in Assumption \ref{ass:mixing} for the federated TD-learning is as follows

\begin{align}\label{eq:G_bar_linear_TD}
    \bG^i(\btheta) =& \btheta + \frac{1}{\beta}\Phi^\top\bmu^\pi \lp \gamma^n (P^\pi)^n\Phi\btheta  - \Phi\btheta\rp,
\end{align}
where $\beta>0$ is an arbitrary constant introduced in Proposition \ref{prop:TD_on}. Furthermore, for $t\geq n+1$, we have $m_1=\frac{2 A_2 \Bar{m}}{\bar{\rho}^{n}}$, $m_2=2B\bar{m}$, and $m_3=m_4=0$, where $A_2$ and $B$ are specified in Lemma \ref{lem:TD_learning_40} and $\bar{\rho} = \rho$. 
\end{lemma}

\begin{lemma}\label{lem:TD_learning_30}
Consider federated on-policy TD-learning \ref{alg:TD-learning} as a special case of FeGSAM (as specified in Proposition \ref{prop:TD_on}).  Consider the $|\mathcal{S}|\times |\mathcal{S}|$ matrix $\mbf U=\Phi^\top\bmu^\pi \lp \gamma^n (P^\pi)^n  - \mbf I\rp\Phi$ with eigenvalues $\{\lambda_1,\dots,\lambda_{|\mathcal{S}|}\}$. Define $\lambda_{\max}=\max_i |\lambda_i|$ and $\delta=-\max_i \mathfrak{Re}[\lambda_i]>0$, where $\mathfrak{Re}[\cdot]$ evaluates the real part. By choosing $\beta$ large enough in the linear function \eqref{eq:G_bar_linear_TD}, there exists a weighted 2-norm $\|\btheta\|_{\Lambda} =\sqrt{\btheta^\top\Lambda \btheta}$, such that $\bG^i(\btheta)$ is a contraction with respect to this norm, that is, $\|\bG^i(\btheta_1)-\bG^i(\btheta_2)\|_{\Lambda}\leq \gamma_c \|\btheta_1-\btheta_2\|_{\Lambda}$ for $\gamma_c=1-\frac{\delta^2}{8\lambda_{\max}^2}$.
\end{lemma}

\begin{lemma}\label{lem:TD_learning_40}
Consider the federated on-policy TD-learning Algorithm \ref{alg:TD-learning} as a special case of FeGSAM (as specified in Proposition \ref{prop:TD_on}). There exist some constants $A_1$, $A_2$, and $B$ such that the properties of Assumption \ref{ass:lipschitz} are satisfied.
\end{lemma}

\begin{lemma}\label{lem:TD_learning_50}
Consider the federated on-policy TD-learning Algorithm \ref{alg:TD-learning} as a special case of FeGSAM (as specified in Proposition \ref{prop:TD_on}). Assumption \ref{ass:noise_independence} holds for this algorithm.
\end{lemma}
\subsubsection{Proofs}\hfill

\begin{proof}[Proof of Proposition \ref{prop:TD_on}]
Items \ref{item:TD_learning_10}-\ref{item:TD_learning_40} are by definition. Subtracting $\bv^\pi$ from both sides of the update of the TD-learning, we have
\begin{align*}
    \underbrace{\bv_{t+1}^i - \bv^\pi}_{\btheta^i_{t+1}} =& \underbrace{\bvit - \bv^\pi}_{\bthetait} + \alpha \phi(S_t^i) \sum_{l=t}^{t+n-1} \gamma^{l-t} \lp \mathcal{R} (S_l^i, A_l^i) \!+\!\gamma \phi (S_{l\!+\!1}^i)^\top \bvit\!-\!\phi(S_l^i)^\top \bvit \rp \\
    =& \bthetait + \alpha \phi(S_t^i) \sum_{l=t}^{t+n-1} \gamma^{l-t} \Big( \mathcal{R} (S_l^i,A_l^i)\!+\!\gamma \phi(S_{l\!+\!1}^i)^\top (\underbrace{\bvit - \bv^\pi}_{\bthetait} +\bv^\pi)\!-\!\phi(S_l^i)^\top (\underbrace{\bvit - \bv^\pi}_{\bthetait} + \bv^\pi) \Big) \\
    =& \bthetait + \alpha\beta \Bigg(\underbrace{\bthetait + \frac{1}{\beta}\phi(S_t^i) \sum_{l=t}^{t+n-1} \gamma^{l-t} \lp \!\gamma \phi(S_{l\!+\!1}^i)^\top \bthetait\!-\!\phi(S_l^i)^\top \bthetait \rp}_{\Gi(\bthetait, \byit)}\\
    &-\bthetait + \underbrace{ \frac{1}{\beta}\phi(S_t^i) \sum_{l=t}^{t+n-1} \gamma^{l-t} \lp \mathcal{R}(S_l^i,A_l^i) + \gamma \phi(S_{l\!+\!1}^i)^\top \bv^\pi - \phi(S_l^i)^\top \bv^\pi \rp}_{\mathbf{b}^i (\byit)}\Bigg).
\end{align*}
which proves items \ref{item:TD_learning_50} and \ref{item:TD_learning_60}. Furthermore, for the synchronization part of TD-learning, we have
\begin{align*}
    \bv^i_t &\leftarrow  \frac{1}{\na} \sum_{j=1}^\na \bv^j_t\\
    \implies \underbrace{\bv^i_t-\bv^\pi}_{\bthetait} &\leftarrow  \frac{1}{\na} \sum_{j=1}^\na \underbrace{(\bv^j_t-\bv^\pi)}_{\bthetajt},
\end{align*}
which is equivalent to the synchronization step in FeGSAM Algorithm \ref{alg:fed_stoch_app}. Notice that here we used the fact that all agents have the same fixed point $\bv^\pi$. 
\end{proof}

\begin{proof}[Proof of Lemma \ref{lem:TD_learning_20}]
It is easy to observe that 
\begin{align*}
    \Gi(\btheta,\byit) = \btheta + \frac{1}{\beta}\phi(S_t^i) \lp \gamma^n \phi(S_{t+n}^i)^\top \btheta - \phi(S_{t}^i)^\top \btheta \rp.
\end{align*}
Taking expectation with respect to the stationary distribution, we have
\begin{align*}
    \bG^i(\btheta) =& \E_{S_t^i\sim \mu^\pi}\left[ \btheta + \frac{1}{\beta}\phi(S_t^i) \lp \gamma^n \phi(S_{t+n}^i)^\top \btheta - \phi(S_{t}^i)^\top \btheta \rp\right]\\
    =&\E_{S_t^i\sim \mu^\pi}\left[ \E\left[\btheta + \frac{1}{\beta}\phi(S_t^i) \lp \gamma^n \phi(S_{t+n}^i)^\top \btheta - \phi(S_{t}^i)^\top \btheta \rp\big|S_t^i\right]\right]\tag{tower property of expectation}\\
    =&\E_{S_t^i\sim \mu^\pi}\left[ \btheta + \frac{1}{\beta}\phi(S_t^i) \lp \gamma^n \E[(\Phi\btheta)  (S_{t+n}^i)|S_t^i] - (\Phi\btheta)(S_{t}^i)  \rp\right]\\
    =&\E_{S_t^i\sim \mu^\pi}\left[ \btheta + \frac{1}{\beta}\phi(S_t^i) \lp \gamma^n ((P^\pi)^n\Phi\btheta)  (S_{t}^i) - (\Phi\btheta)(S_{t}^i)  \rp\right]\\
    =& \btheta + \frac{1}{\beta}\Phi^\top\bmu^\pi \lp \gamma^n (P^\pi)^n\Phi\btheta  - \Phi\btheta\rp.
\end{align*}
where $P^\pi$ is the transition probability matrix corresponding to the policy $\pi$, and $\bmu^\pi$ is a diagonal matrix with diagonal entries corresponding to elements of $\mu^\pi$. 

Moreover, we have
\begin{align*}
    \|\bG^i(\btheta)-\E[\Gi (\btheta,\by_t^i)]\|_c= &\left\|\E_{\by_t^i\sim \mu^\pi}[\Gi (\btheta,\by_t^i)]-\E[\Gi (\btheta,\by_t^i)\right\|_c\\
    =&\left\|\sum_{y_t^i}\left(\mu^\pi(y_t^i)-P(\byit=y_t^i|\by_0^i)\right)\Gi (\btheta,y_t^i)\right\|_c\\
    \leq & \sum_{y_t^i}\left|\mu^\pi(y_t^i)-P(\byit=y_t^i|\by_0^i)\right|.\left\|\Gi (\btheta,y_t^i)\right\|_c\tag{$\|ax\|_c=|a|\|x\|_c$}\\
    \leq & \sum_{y_t^i}\left|\mu^\pi(y_t^i)-P(\byit=y_t^i|\by_0^i)\right|.A_2\left\|\btheta\right\|_c.\tag{Assumption \ref{ass:lipschitz}}
\end{align*}

For brevity, we denote $P(S_t^i=s_t^i)=P(s_t^i)$. We have
    \begin{align*}
        \sum_{y_t^i}&\left|\mu^\pi(y_t^i)-P(\byit=y_t^i|\by_0^i)\right| \\
        = & \sum_{s_t^i,a_t^i,\dots,s_{t+n}^i}\left|\mu^\pi(s_t^i,a_t^i,\dots,s_{t+n}^i)-P(s_t^i,a_t^i,\dots,s_{t+n}^i|\by_0^i)\right|\\
        = & \sum_{s_t^i,a_t^i,\dots,s_{t+n}^i}\bigg|\mu^\pi(s_t^i)\pi(a_t^i|s_t^i)\mathcal{P}(s_{t+1}^i|s_t^i,a_t^i)\dots\mathcal{P}(s_{t+n}^i|s_{t+n-1}^i,a_{t+n-1}^i)\\
        &-P(s_t^i|S_n^i)\pi(a_t^i|s_t^i)\mathcal{P}(s_{t+1}^i| s_t^i,a_t^i) \dots\mathcal{P}(s_{t+n}^i|s_{t+n-1}^i,a_{t+n-1}^i)\bigg|\tag{$t\geq n+1$}\\
        = & \sum_{s_t^i,a_t^i, \dots,s_{t+n}^i}\bigg|\mu^\pi(s_t^i)-P(s_t^i|S_n^i) \bigg|\pi(a_t^i| s_t^i)\mathcal{P}(s_{t+1}^i|s_t^i, a_t^i)\dots\mathcal{P}(s_{t+n}^i| s_{t+n-1}^i, a_{t+n-1}^i)\\
        = & \sum_{s_t^i}\bigg|\mu^\pi(s_t^i)-P(s_t^i|S_n^i)\bigg|\\
        = & 2d_{TV}(P(S_t^i=\cdot|S_n^i) ||\mu^\pi(\cdot))\\
        \leq & 2\bar{m}\bar{\rho}^{t-n}\\
        = & (2\bar{m}\bar{\rho}^{-n})\bar{\rho}^{t}.
    \end{align*}

As explained in \cite{tsitsiklis1997analysis}, the projection operator $\Pi_\pi$ is a linear operator and can be written as $\Pi_\pi=\Phi (\Phi^\top \bmu\Phi)^{-1}\Phi^\top\bmu$, where $\bmu$ is a diagonal matrix with diagonal entries corresponding to the stationary distribution of the policy $\pi$. Hence, the fixed-point equation is as follows $\Phi \bv^\pi = \Phi (\Phi^\top \bmu\Phi)^{-1}\Phi^\top\bmu((\mathcal{T}^\pi)^n\Phi \bv^\pi)$. Since $\Phi$ is a full column matrix, we can eliminate it from both sides of the equality and further multiply both sides by $\Phi^\top \bmu\Phi$. We have $\Phi^\top \bmu\Phi \bv^\pi = \Phi^\top \bmu((\mathcal{T}^\pi)^n \Phi \bv^\pi)$, and therefore $ \Phi^\top\bmu(( \mathcal{T}^\pi)^n\Phi \bv^\pi-\Phi \bv^\pi)= \mbf 0$, which is equivalent to $\E_{S\sim\mu^\pi}[\phi^\top(S)((\mathcal{T}^\pi)^n\Phi \bv^\pi)(S)-(\Phi \bv^\pi)(S)]=\mbf 0$. By expanding $(\mathcal{T}^\pi)^n$, we have $\E_{S_0^i\sim\mu^\pi}[\phi^\top(S_0^i) \sum_{l=0}^{n-1}(\mathcal{R}(S_l^i,A_l^i)+ \gamma(\Phi \bv^\pi)(S_{l+1}^i) -(\Phi \bv^\pi)(S_l^i))]=\mbf 0$, which means
\begin{align}
    \E_{\by\sim\mu^\pi}\mathbf{b}^i  (\by) = \mbf 0.\label{eq:b_zero_stationary_mean_0}
\end{align}
Hence
\begin{align*}
    \left\|\E[\mathbf{b}^i  (\by_t^i)]\right\|_c =& \left\|\E[\mathbf{b}^i  (\byit)]-\E_{\byit\sim\mu^\pi}[\mathbf{b}^i  (\byit)]\right\|_c\tag{By \eqref{eq:b_zero_stationary_mean_0}}\\
    =& \left\|\sum_{y_t^i} \left(\mu^\pi(y_t^i)-P(\byit=y_t^i|\by_0^i)\right)\mathbf{b}^i  (\by_t^i) \right\|_c\\
    \leq&\sum_{y_t^i} \left|\mu^\pi(y_t^i)-P(\byit=y_t^i|\by_0^i)\right|\left\|\mathbf{b}^i  (\by_t^i)\right\|_c\\
    \leq & \sum_{y_t^i}\left|\mu^\pi(y_t^i)-P(\byit=y_t^i|\by_0^i)\right|B\tag{Assumption \ref{ass:lipschitz}}\\
    \leq & 2B\bar{m}\bar{\rho}^t
\end{align*}
\end{proof}

\begin{proof}[Proof of Lemma \ref{lem:TD_learning_30}]
Consider the $|\mathcal{S}|\times |\mathcal{S}|$ matrix $\mbf U=\Phi^\top\bmu^\pi \lp \gamma^n (P^\pi)^n  - \mbf I\rp\Phi$ with eigenvalues $\{\lambda_1,\dots,\lambda_{|\mathcal{S}|}\}$. As shown in \cite{tsitsiklis1997analysis}, since $\Phi$ is a full-rank matrix, the real part of $\lambda_i$ is strictly negative for all $i=1,\dots,|\mathcal{S}|$. Furthermore, define $\lambda_{\max}=\max_i |\lambda_i|$ and $\delta=-\max_i \mathfrak{Re}[\lambda_i]>0$, where $\mathfrak{Re}[\cdot]$ evaluates the real part. Consider the matrix $\mbf U'=\mbf I+\frac{1}{2\lambda_{\max}^2/\delta}\mbf U$. It is easy to show that the eigenvalues of $\mbf U'$ are $\{1+\frac{\lambda_1}{2\lambda_{\max}^2/\delta},\dots,1+\frac{\lambda_{|\mathcal{S}|}}{2\lambda_{\max}^2/\delta}\}$. For an arbitrary $i$, the norm of the i'th eigenvalue satisfies
\begin{align*}
\left|1+\frac{\lambda_i}{2\lambda_{\max}^2/\delta}\right|^2 =& \left(1+\frac{\mathfrak{Re}[\lambda_i]}{2\lambda_{\max}^2/\delta}\right)^2 + \left(\frac{\mathfrak{Im}[\lambda_i]}{2\lambda_{\max}^2/\delta}\right)^2\\
\leq& \left(1+\frac{\mathfrak{Re}[\lambda_i]}{2\lambda_{\max}^2/\delta}\right) + \left(\frac{\mathfrak{Im}[\lambda_i]}{2\lambda_{\max}^2/\delta}\right)^2 \tag{$\mathfrak{Re}[\lambda_i]<0$}\\
\leq &\left(1+\frac{-\delta}{2\lambda_{\max}^2/\delta}\right) + \left(\frac{\lambda_{\max}}{2\lambda_{\max}^2/\delta}\right)^2\\
= &  1-\frac{\delta^2}{4\lambda_{\max}^2}.
\end{align*}
Hence, all the eigenvalues of $\mbf U'$ are in the unit circle. By \cite[Page 46 footnote]{bertsekas1995dynamic2}, we can find a weighted 2-norm as $\|\btheta\|_{\Lambda}=\sqrt{\btheta^\top\Lambda \btheta}$ such that $\mbf U'$ is contraction with respect to this norm with some contraction factor $\gamma_c$. In particular, there exists a choice of $\Lambda$ such that we have $\gamma_c=1-\frac{\delta^2}{8\lambda_{\max}^2}$.
\end{proof}

\begin{proof}[Proof of Lemma \ref{lem:TD_learning_40}]
The existence of $A_1$ and $A_2$ immediately follows after observing that $\Gi(\bthetait,\byit)$ is a linear function of $\bthetait$. Furthermore, the result on $B$ follows due to $\bv^\pi$ being bounded as shown in \cite{NACLFA_arxiv}. 
\end{proof}

\begin{proof}[Proof of Lemma \ref{lem:TD_learning_50}]
For the sake of brevity, we write ``$S_t^i=s^i_t$'' simply as $s^i_t$, and similarly for other random variables. We have
\begin{align*}
    &\E_{t-r} [f(\by_t^i)\times g(\by_t^j)]\\
    =& \hspace{-4mm}\sum\limits_{\substack{s_t^i,a_t^i,\dots,s_{t+n-1}^i,a_{t+n-1}^i,s_{t+n}^i \\ s_t^j,a_t^j,\dots,s_{t+n-1}^j,a_{t+n-1}^j,s_{t+n}^j}}\hspace{-8mm} P(s_t^i,a_t^i,\dots,s_{t+n-1}^i,a_{t+n-1}^i,s_{t+n}^i,s_t^j,a_t^j,\dots,s_{t+n-1}^j,a_{t+n-1}^j,s_{t+n}^j|\mathcal{F}_{t-r})f(\by_t^i)\times g(\by_t^j)\\
    =& \hspace{-4mm}\sum\limits_{\substack{s_{t-r}^i,\dots,s_{t+n-1}^i,a_{t+n-1}^i,s_{t+n}^i \\ s_{t-r}^j,\dots,s_{t+n-1}^j,a_{t+n-1}^j,s_{t+n}^j}}\hspace{-8mm} P(s_{t-r}^i,\dots,s_{t+n-1}^i,a_{t+n-1}^i,s_{t+n}^i,s_{t-r}^j,\dots,s_{t+n-1}^j,a_{t+n-1}^j,s_{t+n}^j|\mathcal{F}_{t-r})f(\by_t^i)\times g(\by_t^j)\\
    =& \hspace{-4mm}\sum\limits_{\substack{s_{t-r}^i,\dots,s_{t+n-1}^i,a_{t+n-1}^i,s_{t+n}^i \\ s_{t-r}^j,\dots,s_{t+n-1}^j,a_{t+n-1}^j,s_{t+n}^j}}\hspace{-14mm} P(s_{t-r}^i,\dots,s_{t+n-1}^i,a_{t+n-1}^i,s_{t+n}^i|\mathcal{F}_{t-r})P(s_{t-r}^j,\dots,s_{t+n-1}^j,a_{t+n-1}^j,s_{t+n}^j|\mathcal{F}_{t-r})f(\by_t^i)g(\by_t^j)\\
    =&\E_{t-r} [f(\by_t^i)]\times \E_{t-r}[g(\by_t^j)].
\end{align*}
\end{proof}

\begin{proof}[Proof of Theorem \ref{thm:fed_TD_on}]
By Proposition \ref{prop:TD_on} and Lemmas \ref{lem:TD_learning_20}, \ref{lem:TD_learning_30}, \ref{lem:TD_learning_40}, and \ref{lem:TD_learning_50}, it is clear that federated TD-learning with linear function approximation Algorithm \ref{alg:TD-learning} satisfies all Assumptions \ref{ass:mixing}, \ref{ass:contraction}, \ref{ass:lipschitz},  and \ref{ass:noise_independence} on the FeGSAM Algorithm \ref{alg:fed_stoch_app}. Furthermore, by the proof of Theorem \ref{thm:main}, we have $w_t=(1-\frac{\alpha\varphiz_2}{2})^{-t}$, and the constant $c_{TDL}$ in the sampling distribution $q_T^{c_{TDL}}$ in Algorithm \ref{alg:TD-learning} is $c_{TDL}=(1-\frac{\alpha\varphiz_2}{2})^{-1}$. Furthermore, by choosing the step size $\alpha$ small enough, we can satisfy the requirements in \eqref{eq:alpha_const_1}, \eqref{eq:alpha_const_2}, \eqref{eq:alpha_const_3}, 
\eqref{eq:alpha_const_4}. By choosing $\sync$ large enough, we can satisfy $\sync>\tau_\alpha$. Hence, the result of Theorem \ref{thm:main} holds for this algorithm with some $c_{TDL}>1$. In addition, it is easy to see that $(1-\frac{\alpha\varphiz_2}{2})^{-\tau_\alpha}= \mathcal{O}(1)$, which is a constant that can be absorbed in $\mathcal{C}_1^{TD_L}$. Finally, for the sample complexity result, we simply employ Corollary \ref{lem:sample_complexity}.

Next, we derive the constant $c_{TDL}$. Since $\|\cdot\|_c=\|\cdot\|_\Lambda$, which is smooth, we choose $g(\cdot)=\frac{1}{2}\|\cdot\|_{\Lambda}^2$. Taking $\psi=1$, we have $l_{cs}=u_{cs}=1$. Therefore, we have $\varphiz_1 = 1$, and $\varphiz_2=1-\gamma_c$, and $c_{TDL}=\lp1-\frac{\alpha(1-\gamma_c)}{2}\rp^{-1}$, where $\gamma_c$ is defined in Lemma \ref{lem:TD_learning_30}.
\end{proof}

\subsection{Off-policy Tabular Setting}\label{sec:fed_td_learning_off_app}
In this subsection, we verify that the off-policy federated TD-learning Algorithm \ref{alg:TD-learning_off} satisfies the properties of the FeGSAM Algorithm \ref{alg:fed_stoch_app}. In the following, $V^\pi$ is the solution of the Bellman equation \eqref{eq:BE_tabular}.
{\small
\begin{align} \label{eq:BE_tabular}
    V^\pi(s) = \sum_a \pi(a|s) \left[\mathcal{R}(s,a)+\gamma\sum_{s'}\mathcal{P}(s'|s,a) V^\pi(s')\right]
\end{align}
}

Note that $V^\pi$ is independent of the agent sampling policy. Furthermore, we take $\|\cdot\|_c=\|\cdot\|_\infty$. 

\begin{proposition} \label{prop:on_tabular}
off-policy $n$-step federated TD-learning is equivalent to the FeGSAM Algorithm \ref{alg:fed_stoch_app} with the following parameters.
\begin{enumerate}
    \item $\bthetait = \bV_t^i - \bV^\pi$ \label{item:TD_learning_1}
    \item $S_t=(S_t^1,\dots,S_t^\na)$ and $A_t=(A_t^1,\dots,A_t^\na)$ \label{item:TD_learning_2}
    \item $\byit=(S_t^i,A_t^i,\dots,S_{t+n-1}^i,A_{t+n-1}^i,S_{t+n}^i)$ and $\mbf \byt=(S_t,A_t,\dots,S_{t+n-1},A_{t+n-1},S_{t+n})$ \label{item:TD_learning_3}
    \item $\mu^i:$ Stationary distribution of the sampling policy of the $i$-th agent. \label{item:TD_learning_4}
    \item $\Gi(\bthetait, \byit)_{s} = \btheta_t^i(s) + \mathbbm{1}_{\{s=S_t^i\}} \left( \sum_{l=t}^{t+n-1} \gamma^{l-t} \lp \Pi_{j=t}^{l} \I^{(i)}(S^i_j, A^i_j) \rp \left[ \gamma \btheta_t^i (S_{l+1}^i) - \btheta_t^i (S_{l}^i) \right] \right)$ \label{item:TD_learning_5}
    \item $\mathbf{b}^i (\byit)_s = \mathbbm{1}_{\{s=S_t^i\}} \sum_{l=t}^{t+n-1} \gamma^{l-t} \lp \Pi_{j=t}^{l} \I^{(i)}(S^i_j, A^i_j) \rp \left[ \mathcal{R} (S_l^i, A_l^i) + \gamma \bV^\pi(S_{l+1}^i) - \bV^\pi(S_l^i) \right]$ \label{item:TD_learning_6}
\end{enumerate}
\end{proposition}

\begin{lemma}\label{lem:TD_learning_2}
Consider the federated off-policy TD-learning Algorithm \ref{alg:TD-learning_off} as a special case of FeGSAM (as specified in Proposition \ref{prop:on_tabular}). Suppose that the trajectory $\{S_t^i\}_{t=0,1,\dots}$ converges geometrically fast to its stationary distribution as follows $d_{TV}(P(S_t^i=\cdot|S_0^i)||\mu^i(\cdot))\leq \bar{m}\bar{\rho}^t$ for all $i=1,2,\dots, \na$. The corresponding  $\bG^i(\btheta)$ in Assumption \ref{ass:mixing} for the federated TD-learning is as follows
\begin{align*}
    \bG^i(\btheta)_{s} =& \btheta(s) + \left[ \gamma^{n} \bmu^i( P^\pi)^{n} \btheta - \bmu^i \btheta \right](s).
\end{align*}
Furthermore, for $t \geq n+1$, we have $m_1=\frac{2 A_2 \Bar{m}}{\bar{\rho}^{n}}$, where $A_2$ is the constant specified in Assumption \ref{ass:lipschitz}, $m_2=m_3=m_4=0$, and $\bar{\rho} =\rho$. 
\end{lemma}

\begin{lemma}\label{lem:TD_learning_3}
Consider federated off-policy TD-learning \ref{alg:TD-learning_off} as a special case of FeGSAM (as specified in Proposition \ref{prop:on_tabular}). The corresponding contraction factor $\gamma_c$ in Assumption \ref{ass:contraction} for this algorithm is $\gamma_c=1-\mu_{\min}(1-\gamma^{n+1})$, where $\mu_{\min} = \min_{s,i}\mu^i(s)$.
\end{lemma}

\begin{lemma} \label{lem:TD_learning_4}
Consider federated off-policy TD-learning \ref{alg:TD-learning_off} as a special case of FeGSAM (as specified in Proposition \ref{prop:on_tabular}).
The constants $A_1$, $A_2$, and $B$ in Assumption \ref{ass:lipschitz} can be chosen as follows:
$A_1 = A_2= 1+(1+\gamma)
\begin{cases}
n & \text{ if } \ \gamma \I_{\max} = 1 \\ \frac{1-(\gamma \I_{\max})^n}{1-\gamma \I_{\max}} & \text{ o.w.}
\end{cases},
\quad$ and $B=\frac{2\I_{\max}}{1-\gamma}\begin{cases}n&\text{if}\quad \gamma\I_{\max}=1\\\frac{1-(\gamma\I_{\max})^n}{1-\gamma\I_{\max}}&o.w.\end{cases}$, where $\I_{\max} = \max_{s^i,a^i,i}\I^{(i)}(s^i,a^i)$.
\end{lemma}

\begin{lemma}\label{lem:TD_learning_5}
Consider federated off-policy TD-learning \ref{alg:TD-learning_off} as a special case of FeGSAM (as specified in Proposition \ref{prop:on_tabular}). Assumption \ref{ass:noise_independence} holds for this algorithm.
\end{lemma}

\subsubsection{Proofs}

\begin{proof}[Proof of Proposition \ref{prop:on_tabular}]
Items \ref{item:TD_learning_1}-\ref{item:TD_learning_4} are by definition. Furthermore, by the update of the TD-learning, and subtracting $\bV^\pi$ from both sides, we have
\begin{align*}
    &\underbrace{\bV_{t+1}^i - \bV^\pi(s)}_{\btheta_{t+1}^i(s)} = \underbrace{\bV_{t}^i - \bV^\pi(s)}_{\btheta_t^i(s)} \\
    & \qquad + \alpha  \mathbbm{1}_{\{s=S_t^i\}} \left( \sum_{l=t}^{t+n-1} \gamma^{l-t} \lb \Pi_{j=t}^{l} \I^{(i)}(S^i_j,A^i_j) \rb \left( \mathcal{R}(S_l^i, A_l^i)\!+\!\gamma \bV_t^i(S_{l+1}^i) -\bV_t^i(S_{l}^i) \right) \right) \\
    &= \btheta_t^i(s) \\
    & + \alpha  \mathbbm{1}_{\{s=S_t^i\}} \Bigg\{ \sum_{l=t}^{t+n-1} \gamma^{l-t} \lp \Pi_{j=t}^{l} \I^{(i)}(S^i_j, A^i_j) \rp \Bigg[ \mathcal{R} (S_l^i,A_l^i)\!+\!\gamma \Big( \underbrace{\bV_t^i(S_{l+1}^i) - \bV^\pi (S_{l+1}^i)}_{\btheta_t^i (S_{l+1}^i)} + \bV^\pi (S_{l+1}^i) \Big) \\
    & \qquad \qquad \qquad \qquad \qquad \qquad \qquad \qquad \qquad \qquad \qquad - \Big( \underbrace{\bV_t^i (S_{l}^i) - \bV^\pi (S_{l}^i)}_{\btheta_t^i (S_{l}^i)} + \bV^\pi(S_{l}^i) \Big) \Bigg] \Bigg\} \\
    &= \btheta_t^i(s)  + \alpha  \Bigg\{ \underbrace{\btheta_t^i(s) + \mathbbm{1}_{\{s=S_t^i\}} \left( \sum_{l=t}^{t+n-1} \gamma^{l-t} \lp \Pi_{j=t}^{l} \I^{(i)}(S^i_j, A^i_j) \rp \left[\gamma\btheta_t^i(S_{l+1}^i)-\btheta_t^i(S_{l}^i)\right]\right)}_{\mbf G^i(\btheta_t^i,\byit)_s} -\btheta_t^i(s)\\
    & \quad \quad + \underbrace{\mathbbm{1}_{\{s=S_t^i\}} \sum_{l=t}^{t+n-1} \gamma^{l-t} \lp \Pi_{j=t}^{l} \I^{(i)} (S^i_j, A^i_j) \rp \left[ \mathcal{R} (S_l^i, A_l^i) + \gamma \bV^\pi (S_{l+1}^i) - \bV^\pi (S_l^i) \right]}_{\mathbf{b}^i (\byit)_s} \Bigg\},
\end{align*}
which proves items \ref{item:TD_learning_5} and \ref{item:TD_learning_6}. Furthermore, for the synchronization part of TD-learning, if $t \mod \sync = 0$,
\begin{align*}
    \bV^i_t &\leftarrow  \frac{1}{\na} \sum_{j=1}^\na \bV^j_t\\
    \implies \underbrace{\bV^i_t-\bV^\pi}_{\bthetait} &\leftarrow  \frac{1}{\na} \sum_{j=1}^\na \underbrace{(\bV^j_t-\bV^\pi)}_{\bthetajt},
\end{align*}
which is equivalent to the synchronization step in FeGSAM Algorithm \ref{alg:fed_stoch_app}. Notice that here we used the fact that all agents have the same fixed point $\bV^\pi$.
\end{proof}

\begin{proof}[Proof of Lemma \ref{lem:TD_learning_2}]
By taking expectation of $\Gi(\bthetait,\byit)_{s}$, we have
\begin{align*}
    \bG^i(\btheta)_{s} = &\E_{S_t^i\sim\mu^i} \left[ \btheta(s)+\mathbbm{1}_{\{s=S_t^i\}}\left(\sum_{l=t}^{t+n-1}\gamma^{l-t} \lp \Pi_{j=t}^{l} \I^{(i)} (S^i_j, A^i_j) \rp \left[\gamma\btheta(S_{l+1}^i)-\btheta(S_{l}^i)\right]\right)\right] \\
    = &\btheta(s)+\sum_{l=t}^{t+n-1}\underbrace{\E_{S_t^i\sim\mu^i}\left[\mathbbm{1}_{\{s=S_t^i\}}\gamma^{l-t} \lp \Pi_{j=t}^{l} \I^{(i)} (S^i_j, A^i_j) \rp \left[\gamma\btheta(S_{l+1}^i)-\btheta(S_{l}^i)\right]\right]}_{T_l}.
\end{align*}
Denote $\E_{k}^i [\cdot] =\E [\cdot|\{S^i_r, A^i_r\}_{r \leq k-1}, S^i_k]$. For $T_l$, we have
\begin{align*}
    T_l =& \E_{S_t^i\sim\mu^i}\left[\E_{l}^i\left[\mathbbm{1}_{\{s=S_t^i\}}\gamma^{l-t} \lp \Pi_{j=t}^{l} \I^{(i)}(S^i_j, A^i_j) \rp \left[\gamma\btheta(S_{l+1}^i)-\btheta(S_{l}^i) \right]\right]\right]\\
    = &\E_{S_t^i\sim\mu^i}\left[\mathbbm{1}_{\{s=S_t^i\}}\gamma^{l-t} \lp \Pi_{j=t}^{l-1} \I^{(i)}(S^i_j, A^i_j) \rp \left[\gamma \E_{l}^i\left[\I^{(i)}(S^i_l, A^i_l) \btheta(S_{l+1}^i) \right] - \btheta(S_{l}^i) \right] \right].
\end{align*}
Here,
\begin{align*}
    \E_{l}^i \left[ \I^{(i)} (S^i_l, A^i_l) \btheta(S_{l+1}^i)\right] &= \sum_{s,a} P(S_{l+1}^i=s,A_{l}^i=a|S_{l}^i) \I^{(i)}(S^i_l, A_l^i=a) \btheta(s) \\
    &= \sum_{s,a} P(S_{l+1}^i=s, A_{l}^i=a|S_{l}^i) \frac{\pi(a|S_l^i)}{\pi^i(a|S_l^i)} \btheta(s) \\
    &= \sum_{s,a} \pi^i(a|S_{l}^i)\mathcal{P}(s|S_{l}^i,a) \frac{\pi(a|S_l^i)}{\pi^i(a|S_l^i)}\btheta(s) \\
    &= \sum_{s,a} \mathcal{P}(s|S_{l}^i,a) \pi(a|S_l^i)\btheta(s)\equiv[P^\pi \btheta](S_l^i),
\end{align*}
where $[P^\pi]_{s_0,s_1}=\sum_{a} \mathcal{P}(s_1|s_0,a) \pi(a|s_0)$. Hence, we have
\begin{align*}
    T_l = &\E_{S_t^i\sim\mu^i}\left[\mathbbm{1}_{\{s=S_t^i\}}\gamma^{l-t}\Pi_{j=t}^{l-2}\I^{(i)}(S^i_j, A^i_j) \E_{l-1}^i\left[\I^{(i)}(S^i_{l-1}, A^i_{l-1}) \left[\gamma(P^\pi \btheta)(S_l^i)-\btheta(S_{l}^i)\right]\right]\right]\\
    = &  \E_{S_t^i\sim\mu^i}\left[\mathbbm{1}_{\{s=S_t^i\}}\gamma^{l-t}\Pi_{j=t}^{l-2}\I^{(i)}(S^i_j, A^i_j)\left[\gamma((P^\pi)^2 \btheta)(S_{l-1}^i)-(P^\pi\btheta)(S_{l-1}^i)\right]\right]\\
    = &\dots\\
    = & \E_{S_t^i\sim\mu^i}\left[\mathbbm{1}_{\{s=S_t^i\}}\gamma^{l-t}\left[\gamma((P^\pi)^{l-t+1} \btheta)(S_{t}^i)-((P^\pi)^{l-t}\btheta)(S_{t}^i)\right]\right]\\
    = & \mu^i(s)\gamma^{l-t}\left[\gamma((P^\pi)^{l-t+1} \btheta)(s)-((P^\pi)^{l-t}\btheta)(s)\right]\\
    = &\gamma^{l-t}\left[\gamma(\bmu^i( P^\pi)^{l-t+1} \btheta)(s)-(\bmu^i( P^\pi)^{l-t}\btheta)(s)\right] \\
    = & \left[\gamma^{l-t+1}\bmu^i( P^\pi)^{l-t+1} \btheta - \gamma^{l-t}\bmu^i( P^\pi)^{l-t} \btheta\right](s),
\end{align*}
where we denote $\bmu^i$ as diagonal matrix with diagonal entries corresponding to the stationary distribution $\mu^i$. 
Hence, in total we have
\begin{align*}
    \bG^i(\btheta)_{s} = &\btheta(s)+\sum_{l=t}^{t+n-1}\left[\gamma^{l-t+1}\bmu^i( P^\pi)^{l-t+1} \btheta - \gamma^{l-t}\bmu^i( P^\pi)^{l-t} \btheta\right](s)\\
    = &\btheta(s)+\left[\gamma^{n}\bmu^i( P^\pi)^{n} \btheta - \bmu^i \btheta\right](s).
\end{align*}
Furthermore, by the same argument as in the proof of Lemma \ref{lem:TD_learning_20}, we have
\begin{align*}
    \|\bG^i(\btheta)-\E[\Gi (\btheta,\by_t^i)]\|_c
    \leq & \sum_{y_t^i}\left|\mu^i(y_t^i)-P(\byit=y_t^i|\by_0^i)\right|.A_2\left\|\btheta\right\|_c,
    \end{align*}
    and
    \begin{align*}
        \sum_{y_t^i}\left|\mu^i(y_t^i)-P(\byit=y_t^i|\by_0^i)\right| \leq & (2\bar{m}\bar{\rho}^{-n})\bar{\rho}^{t},
    \end{align*}
which proves that $m_1=2\bar{m}A_2\bar{\rho}^{-n}$ constant. In addition, we have
\begin{align*}
    & \left\|\E[\mathbf{b}^i  (\by_t^i)]\right\|_c = \left\|\E\left[\mathbbm{1}_{\{s=S_t^i\}}\sum_{l=t}^{t+n-1}\gamma^{l-t}\Pi_{j=t}^{l}\I^{(i)}(S^i_j, A^i_j)\left[\mathcal{R}(S_l^i,A_l^i)+\gamma\bV^\pi(S_{l+1}^i)-\bV^\pi(S_l^i)\right]\right]\right\|_c\\
    &= \left\|\E\left[\mathbbm{1}_{\{s=S_t^i\}}\sum_{l=t}^{t+n-1}\gamma^{l-t}\E_{l}\left[\Pi_{j=t}^{l}\I^{(i)}(S^i_j, A^i_j)\left[\mathcal{R}(S_l^i,A_l^i)+\gamma\bV^\pi(S_{l+1}^i)-\bV^\pi(S_l^i)\right]\right]\right]\right\|_c\\
    &= \left\|\E\left[\mathbbm{1}_{\{s=S_t^i\}}\sum_{l=t}^{t+n-1}\gamma^{l-t}\Pi_{j=t}^{l-1}\I^{(i)}(S^i_j, A^i_j)\underbrace{\E_{l}\left[\I^{(i)}(S^i_l, A^i_l)\left[\mathcal{R}(S_l^i,A_l^i)+\gamma\bV^\pi(S_{l+1}^i)-\bV^\pi(S_l^i)\right]\right]}_{T}\right]\right\|_c.
\end{align*}
For the term $T$, we have
\begin{align*}
    T=& \sum_a \pi^i(a|S_{l}^i).\frac{\pi(a|S_{l}^i)}{\pi^i(a|S_{l}^i)} \left[\mathcal{R}(S_l^i,a)+\gamma\sum_{s'}\mathcal{P}(s'|S_l^i,a)\bV^\pi(s')-\bV^\pi(S_l^i)\right]\\
    =&\sum_a\pi(a|S_{l}^i) \left[\mathcal{R}(S_l^i,a)+\gamma\sum_{s'}\mathcal{P}(s'|S_l^i,a)\bV^\pi(s')-\bV^\pi(S_l^i)\right] \\
    = & 0,
\end{align*}
which shows that $m_2=m_3=m_4=0$. 
\end{proof}

\begin{proof}[Proof of Lemma \ref{lem:TD_learning_3}] 
\begin{align*}
    \|\bG^i (\btheta_1) - \bG^i(\btheta_2) \|_c =&\left\|\btheta_1+\left[\gamma^{n+1}\bmu^i( P^\pi)^{n+1} \btheta_1 - \bmu^i \btheta_1\right] - \left(\btheta_2+\left[\gamma^{n+1}\bmu^i( P^\pi)^{n+1} \btheta_2 - \bmu^i \btheta_2\right]\right)\right\|_\infty\\
    =&\left\|\left(I-\bmu^i(I-\gamma^{n+1}(P^\pi)^{n+1})\right)(\btheta_1-\btheta_2)\right\|_\infty\\
    \leq &\left\|I-\bmu^i(I-\gamma^{n+1}(P^\pi)^{n+1})\right\|_\infty\left\|\btheta_1-\btheta_2\right\|_\infty.\tag{definition of matrix norm}
\end{align*}
Since the elements of the matrix $I-\bmu^i(I-\gamma^{n+1}(P^\pi)^{n+1})$ is all positive, we have $\left\|I-\bmu^i(I-\gamma^{n+1}(P^\pi)^{n+1})\right\|_\infty = \left\|(I-\bmu^i(I-\gamma^{n+1}(P^\pi)^{n+1}))\mathbf{1}\right\|_\infty = \|\mathbf{1}-\bmu^i(\mathbf{1}-\gamma^{n+1}\mathbf{1})\|_\infty=1-\mu_{\min}^i(1-\gamma^{n+1})\leq 1-\mu_{\min}(1-\gamma^{n+1})$.
\end{proof}

\begin{proof}[Proof of Lemma \ref{lem:TD_learning_4}]
\begin{align*}
    &\| \Gi (\btheta_1, \by) - \Gi (\btheta_2, \by) \|_c \\
    &= \max_s \Bigg| \btheta_1(s) - \btheta_2(s) + \mathbbm{1}_{\{s=S_t^i\}} \left( \sum_{l=t}^{t+n-1} \gamma^{l-t} \lp \Pi_{j=t}^{l} \I^{(i)}(S^i_j, A^i_j) \rp \left[ \gamma \btheta_1(S_{l+1}^i) - \btheta_1(S_{l}^i) \right] \right) \\
    & \qquad \qquad - \mathbbm{1}_{\{s=S_t^i\}} \left( \sum_{l=t}^{t+n-1} \gamma^{l-t} \lp \Pi_{j=t}^{l} \I^{(i)}(S^i_j, A^i_j) \rp \left[ \gamma \btheta_2 (S_{l+1}^i) - \btheta_2 (S_{l}^i) \right] \right) \Bigg| \tag{$\because c = \infty$} \\
    & \leq \max_s \left[ \left| \btheta_1(s) - \btheta_2(s) \right| + \left| \sum_{l=t}^{t+n-1} \gamma^{l-t} \Pi_{j=t}^{l} \I^{(i)}(S^i_j, A^i_j) \left[ \gamma \lp \btheta_1(S_{l+1}^i) - \btheta_2(S_{l+1}^i) \rp - \lp \btheta_1(S_{l}^i) - \btheta_2(S_{l}^i) \rp \right] \right| \right] \tag{triangle inequality} \\
    & \leq \max_s\left[\left|\btheta_1(s)-\btheta_2(s)\right|+\sum_{l=t}^{t+n-1}\gamma^{l-t}\Pi_{j=t}^{l} \I^{(i)}(S^i_j, A^i_j) \left[\gamma\left|\btheta_1(S_{l+1}^i)-\btheta_2(S_{l+1}^i)\right|+\left|\btheta_1(S_{l}^i)-\btheta_2(S_{l}^i)\right|\right]\right]\tag{triangle inequality}\\
    & \leq \max_s\left[\left|\btheta_1(s)-\btheta_2(s)\right|+\sum_{l=t}^{t+n-1}\gamma^{l-t}\Pi_{j=t}^{l} \I^{(i)}(S^i_j, A^i_j) \left[\gamma\left\|\btheta_1-\btheta_2\right\|_\infty+\left\|\btheta_1-\btheta_2\right\|_\infty\right]\right]\tag{definition of $\|\cdot\|_\infty$}\\
    & \leq \max_s\left[\left|\btheta_1(s)-\btheta_2(s)\right|+\sum_{l=t}^{t+n-1}\gamma^{l-t}\I_{\max}^{l-t+1}\left[\gamma\left\|\btheta_1-\btheta_2\right\|_\infty+\left\|\btheta_1-\btheta_2\right\|_\infty\right]\right]\tag{definition of $\I_{\max}$}\\
    &= \left\|\btheta_1-\btheta_2\right\|_\infty+\left[\gamma\left\|\btheta_1-\btheta_2\right\|_\infty+\left\|\btheta_1-\btheta_2\right\|_\infty\right]\I_{\max}\sum_{l=0}^{n-1}(\gamma \I_{\max})^{l}\\
    &= \left\|\btheta_1-\btheta_2\right\|_\infty+\left[\gamma\left\|\btheta_1-\btheta_2\right\|_\infty+\left\|\btheta_1-\btheta_2\right\|_\infty\right]\I_{\max}\begin{cases}n&\text{if}\quad \gamma\I_{\max}=1\\\frac{1-(\gamma\I_{\max})^n}{1-\gamma\I_{\max}}&o.w.\end{cases}
\end{align*}
Furthermore, we have
\begin{align*}
    \|\mathbf{b}^i (\byit)\|_c = & \max_{S_{t}^i,A_{t}^i\dots,S_{t+n}^i}\left|\mathbbm{1}_{\{s=S_t^i\}} \sum_{l=t}^{t+n-1} \gamma^{l-t} \lp \Pi_{j=t}^{l} \I^{(i)}(S^i_j, A^i_j) \rp \left[ \mathcal{R} (S_l^i, A_l^i) + \gamma \bV^\pi(S_{l+1}^i) - \bV^\pi(S_l^i) \right]\right|\\
    \leq & \max_{S_{t}^i,A_{t}^i\dots,S_{t+n}^i} \sum_{l=t}^{t+n-1} \gamma^{l-t} \lp \Pi_{j=t}^{l} \I^{(i)}(S^i_j, A^i_j) \rp \left| \mathcal{R} (S_l^i, A_l^i) + \gamma \bV^\pi(S_{l+1}^i) - \bV^\pi(S_l^i) \right|\tag{triangle inequality}\\
    \leq & \max_{S_{t}^i,A_{t}^i\dots,S_{t+n}^i} \sum_{l=t}^{t+n-1} \gamma^{l-t} \lp  \I_{\max}^{l-t+1} \rp \left| \mathcal{R} (S_l^i, A_l^i) + \gamma \bV^\pi(S_{l+1}^i) - \bV^\pi(S_l^i) \right|\\
    \leq & \max_{S_{t}^i,A_{t}^i\dots,S_{t+n}^i} \sum_{l=t}^{t+n-1} \gamma^{l-t} \lp  \I_{\max}^{l-t+1} \rp \left[\left| \mathcal{R} (S_l^i, A_l^i)\right| + \gamma \left|\bV^\pi(S_{l+1}^i)\right| + \left|\bV^\pi(S_l^i) \right|\right]\tag{triangle inequality}\\
    \leq & \max_{S_{t}^i,A_{t}^i\dots,S_{t+n}^i} \sum_{l=t}^{t+n-1} \gamma^{l-t} \lp  \I_{\max}^{l-t+1} \rp \left[1 +  \frac{\gamma}{1-\gamma} + \frac{1}{1-\gamma}\right]\\
    = & \frac{2\I_{\max}}{1-\gamma} \sum_{l=t}^{t+n-1}   (\gamma\I_{\max})^{l-t} \\
    = &  \frac{2\I_{\max}}{1-\gamma}\begin{cases}n&\text{if}\quad \gamma\I_{\max}=1\\\frac{1-(\gamma\I_{\max})^n}{1-\gamma\I_{\max}}&o.w.\end{cases}
\end{align*}
\end{proof}

\begin{proof}[Proof of Lemma \ref{lem:TD_learning_5}]
The proof follows similar to Lemma \ref{lem:TD_learning_50}.
\end{proof}

\begin{proof}[Proof of Theorem \ref{thm:fed_TD_off}]
By Proposition \ref{prop:on_tabular} and Lemmas \ref{lem:TD_learning_2}, \ref{lem:TD_learning_3}, \ref{lem:TD_learning_4}, and \ref{lem:TD_learning_5}, it is clear that the federated off-policy TD-learning Algorithm \ref{alg:TD-learning_off} satisfies all the Assumptions \ref{ass:mixing}, \ref{ass:contraction}, \ref{ass:lipschitz},  and \ref{ass:noise_independence} of the FeGSAM Algorithm \ref{alg:fed_stoch_app}. Furthermore, by the proof of Theorem \ref{thm:main}, we have $w_t=(1-\frac{\alpha\varphiz_2}{2})^{-t}$, and the constant $c$ in the sampling distribution $q_T^c$ in Algorithm \ref{alg:TD-learning} is $c=(1-\frac{\alpha\varphiz_2}{2})^{-1}$. In equation \eqref{eq:w_t} we evaluate the exact value of $w_t$.

Furthermore, by choosing step size $\alpha$ small enough, we can satisfy the requirements in \eqref{eq:alpha_const_1}, \eqref{eq:alpha_const_2}, \eqref{eq:alpha_const_3}, 
\eqref{eq:alpha_const_4}. By choosing $\sync$ large enough, we can satisfy $\sync>\tau$, and by choosing $T$ large enough we can satisfy $T>K+\tau$. Hence, the result of Theorem \ref{thm:main} holds for this algorithm. 

Next, we derive the constants involved in Theorem \ref{thm:main} step by step. After deriving the constants $\mathcal{C}_1$, $\mathcal{C}_2$, $\mathcal{C}_3$, $\mathcal{C}_4$, and $\mathcal{C}_5$ in Theorem \ref{thm:main}, we can directly get the constants $\mathcal{C}_i^{TD_T}$ for $i=1,2,3$. 

In this analysis we only consider the terms involving $|\mathcal{S}|$, $|\mathcal{A}|$, $\frac{1}{1-\gamma}$, $\I_{\max}$, and $\mu_{\min}$. Since $\|\cdot\|_c=\|\cdot\|_\infty$, we choose $g(\cdot)=\frac{1}{2}\|\cdot\|_p^2$, i.e. the $p$-norm with $p=2\log(|\mathcal{S}|)$. It is known that $g(\cdot)$ is $(p-1)$ smooth with respect to $\|\cdot\|_p$ norm \cite{beck2017first}, and hence $L=\Theta(\log(|\mathcal{S}|))$. Hence, we have $l_{cs}=|\mathcal{S}|^{-1/p}=\frac{1}{\sqrt{e}}=\Theta(1)$ and $u_{cs}=1$. Therefore, we have $\varphiz_1 = \frac{1+\psi u_{cs}^2}{1+\psi \ell_{cs}^2}=\frac{1+\psi}{1+\frac{\psi}{\sqrt{e}}}\leq1+\psi$. By choosing $\psi=(\frac{1+\gamma_c}{2\gamma_c})^2-1=\frac{1+2\gamma_c-3\gamma_c^2}{4\gamma_c^2} \geq (1-\gamma_c)=\mu_{\min}(1-\gamma^{n+1})=\Omega( \mu_{\min}(1-\gamma))$, which is $\psi= \mathcal{O}(1)$, we have $\varphiz_1 = \frac{1+\psi}{1+\frac{\psi}{\sqrt{e}}} =\sqrt{e} \frac{(\frac{1+\gamma_c}{2\gamma_c})^2}{\sqrt{e}+(\frac{1+\gamma_c}{2\gamma_c})^2-1}=\mathcal{O}(1)$, and
\begin{align*}
\varphiz_2=&1-\gamma_c\sqrt{\frac{1+\psi}{1+\frac{\psi}{\sqrt{e}}}}=1-\gamma_c\sqrt{\frac{(\frac{1+\gamma_c}{2\gamma_c})^2}{1+\frac{(\frac{1+\gamma_c}{2\gamma_c})^2-1}{\sqrt{e}}}}=1-\frac{0.5(1+\gamma_c)e^{1/4}}{\sqrt{\sqrt{e}-1+\lp\frac{1+\gamma_c}{2\gamma_c}\rp^2}}\\
&=1-\frac{0.5e^{1/4}(2-\mu_{\min}(1-\gamma^{n+1}))}{\sqrt{\sqrt{e}-1+\left(\frac{2-\mu_{\min}(1-\gamma^{n+1})}{2-2\mu_{\min}(1-\gamma^{n+1})}\right)^2}}\\
>& 1-\gamma_c\sqrt{1+\psi} \\
&= 1-\gamma_c\frac{1+\gamma_c}{2\gamma_c}=\frac{1-\gamma_c}{2}=0.5\mu_{\min}(1-\gamma^{n+1}) =\Omega(\mu_{\min}(1-\gamma))
\end{align*}

\[
\varphiz_3 =\frac{L (1+\psi u_{cs}^2)}{\psi\ell_{cs}^2} =\mathcal{O} \lp \frac{\log (|\mathcal{S}|)(1+\psi)}{\psi}\rp\leq\mathcal{O}\lp\frac{\log (|\mathcal{S}|)}{1-\gamma_c}\rp=\mathcal{O}\lp\frac{\log (|\mathcal{S}|)}{\mu_{\min}(1-\gamma)}\rp.
\]
Using $\varphiz_2$, we have
\begin{align}\label{eq:w_t}
w_t=\left(1-\frac{\alpha\varphiz_2}{2}\right)^{-t}=\lp 1 - \alpha/2 + \frac{0.25\alpha e^{1/4}(2-\mu_{\min}(1-\gamma^{n+1}))}{\sqrt{\sqrt{e}-1+\left(\frac{2-\mu_{\min}(1-\gamma^{n+1})}{2-2\mu_{\min}(1-\gamma^{n+1})}\right)^2}}\rp^{-t}.
\end{align}

Further, we have
\[
l_{cm} = (1+\psi l_{cs}^2)^{1/2} =\Theta(1)
\]
\[
u_{cm} = (1+\psi u_{cs}^2)^{1/2} =\Theta(1)
\]

Since TV-divergence is upper bounded with $1$, we have $\bar{m}=\mathcal{O}(1)$. By Lemma \ref{lem:TD_learning_4}, we have
\[
A_1=A_2=1+(1+\gamma)
\begin{cases}
n & \text{ if } \ \gamma \I_{\max} = 1 \\ \frac{1-(\gamma \I_{\max})^n}{1-\gamma \I_{\max}} & \text{ o.w.}
\end{cases}=\mathcal{O}(\I_{\max}^{n-1})
\]
and $A_1=A_2=\Omega(1)$,
\[
B=\frac{2\I_{\max}}{1-\gamma}\begin{cases}n&\text{if}\quad \gamma\I_{\max}=1\\\frac{1-(\gamma\I_{\max})^n}{1-\gamma\I_{\max}}&o.w.\end{cases}=\mathcal{O}\lp\frac{\I_{\max}^{n}}{1-\gamma}\rp,
\]
and $B=\Omega(1)$. Hence $m_1=\frac{2A_2\bar{m}}{\bar{\rho}^n}=\mathcal{O}(\I_{\max}^{n-1})$. Also, we have $m_2=m_3=m_4=0$.

We choose the $D$-norm in Lemma \ref{lem:after_expectation} as the $2$-norm $\|\cdot\|_2$. Hence, by primary norm equivalence, we have $l_{cD}=\frac{1}{\sqrt{|\mathcal{S}|}}$, and $u_{cD}=1$, and hence $\frac{u_{cD}}{l_{cD}}=\sqrt{|\mathcal{S}|}$.

We can evaluate the rest of the constants as follows
\[
\zetaone=\zetafive=\zeta_6=\sqrt{\varphiz_2/10}=\Omega(\sqrt{\mu_{\min}(1-\gamma)})
\]
\[
\zetatwo = \sqrt{\frac{\varphiz_2}{10} \cdot \frac{\psi l_{cs}^2}{2 L A_2}} = \Omega\left(\sqrt{\mu_{\min}(1-\gamma).\frac{\mu_{\min}(1-\gamma)}{\log(|\mathcal{S}|)}}\right) = \Omega\lp\frac{\mu_{\min}(1-\gamma)}{\sqrt{\log(|\mathcal{S}|)}}\rp,
\]
and similarly
\[
\zetathree = \sqrt{\frac{\varphiz_2}{10} \cdot \frac{\psi l_{cs}^2}{L (A_1 + 1)}}=\Omega\lp\frac{\mu_{\min}(1-\gamma)}{\sqrt{\log(|\mathcal{S}|)}}\rp.
\]
\begin{align*}
C_1 = &3\sqrt{ A^2_2 + 1}=\mathcal{O}\lp\I_{\max}^{n-1}\rp,
\end{align*}
and
\begin{align*}
    C_1=\Omega(1),
\end{align*}
\begin{align*}
    C_2=&3C_1+8=\mathcal{O}\lp\I_{\max}^{n-1}\rp,
\end{align*}
\begin{align*}
C_3=& \frac{ m_4 L}{\psi l_{cs}^2}=0,
\end{align*}
\begin{align*}
    C_4 = &\left( \frac{L^2}{2 \psi^2 l_{cs}^4} + \frac{2 \alpha L A_2}{\psi l_{cs}^2} \lp \tfrac{1}{2} + \tfrac{u_{cm}^2}{2 \zetatwo^2} \rp + \frac{\alpha L (A_1+1)}{\psi l_{cs}^2} \lp \tfrac{3 u_{cm}^2}{2 \zetathree^2} + 2 \rp +\frac{3m_1L\alpha^2}{l_{cs}^2\psi}\right) \\
    =& \mathcal{O}\left(\frac{\log^2(|\mathcal{S}|)}{\mu_{\min}^2(1-\gamma)^2}+\frac{\log(|\mathcal{S}|)\I_{\max}^{n-1}}{\mu_{\min}(1-\gamma)}.\frac{\log(|\mathcal{S}|)}{\mu_{\min}^2(1-\gamma)^2}+\frac{\log(|\mathcal{S}|)\I_{\max}^{n-1}}{\mu_{\min}(1-\gamma)}.\frac{\log(|\mathcal{S}|)}{\mu_{\min}^2(1-\gamma)^2}+\frac{\I_{\max}^{n-1}\log(|\mathcal{S}|)}{\mu_{\min}(1-\gamma)}\right)\\
    =&\mathcal{O}\lp\frac{\log^2(|\mathcal{S}|)\I_{\max}^{n-1}}{\mu_{\min}^2(1-\gamma)^2}\rp,
\end{align*}
\begin{align*}
    C_5=&\left( \lp \frac{3 u_{cm}^2}{2 \zetathree^2} + \frac{3}{2} \rp \frac{L (A_1+1)}{\psi l_{cs}^2} + \frac{L A_2}{\psi l_{cs}^2} + \frac{L^2 u_{cm}^2 }{2 l_{cs}^4 \zetafive^2  \psi^2} + \frac{3 A_1^2 L \alpha^2}{2 \psi l_{cs}^{2}} \right)\\
    =&\mathcal{O}\lp \frac{\log(|\mathcal{S}|)}{\mu_{\min}^2(1-\gamma)^2}.\frac{\log(|\mathcal{S}|)\I_{\max}^{n-1}}{\mu_{\min}(1-\gamma)}+\frac{\log(|\mathcal{S}|)\I_{\max}^{n-1}}{\mu_{\min}(1-\gamma)}+\frac{\log^2(|\mathcal{S}|)}{\mu_{\min}^3(1-\gamma)^3}+\frac{\I_{\max}^{2n-2}\log(|\mathcal{S}|)}{\mu_{\min}(1-\gamma)}\rp\\
    =&\mathcal{O}\lp\frac{\log^2(|\mathcal{S}|)\I_{\max}^{2n-2}}{\mu_{\min}^3(1-\gamma)^3}\rp,
\end{align*}
\begin{align*}
    C_6 =& \left( \lp \frac{3 u_{cm}^2}{2 \zetathree^2} + \frac{3}{2} \rp \frac{L (A_1+1)}{\psi l_{cs}^2} +\frac{m_1L\alpha}{2l_{cs}^2\psi}\right)=\mathcal{O}\lp\frac{\log(|\mathcal{S}|)}{\mu_{\min}^2(1-\gamma)^2}.\frac{\log(|\mathcal{S}|)\I_{\max}^{n-1}}{\mu_{\min}(1-\gamma)}+\frac{\I_{\max}^{n-1}\log(|\mathcal{S}|)}{\mu_{\min}(1-\gamma)}\rp\\
    =&\mathcal{O}\lp\frac{\log^2(|\mathcal{S}|)\I_{\max}^{n-1}}{\mu_{\min}^3(1-\gamma)^3}\rp,
\end{align*}
\begin{align*}
    C_7=&\frac{m^2_4 u_{cm}^2 L^2}{2 \zetaone^2 l_{cs}^4 \psi^2}\alpha=  0,
\end{align*}
\begin{align*}
    C_8 =& \left(\frac{1}{2 } + \frac{3 L}{2 \psi  l_{cs}^2} \right)=\mathcal{O}\lp\frac{\log(|\mathcal{S}|)}{\mu_{\min}(1-\gamma)}\rp,
\end{align*}
\begin{align*}
    C_9 =& \frac{8 u_{cD} B}{l_{cD}} =\mathcal{O}\lp\frac{\sqrt{|\mathcal{S}|}\I_{\max}^{n}}{1-\gamma}\rp,
\end{align*}
\begin{align*}
    C_{10}=&\frac{8 m_2 u_{cD}}{l_{cD} (1-\rho)}= 0,
\end{align*}
\begin{align*}
    C_{11}=&\frac{8 C_1^2 C_3^2 u_{cm}^2}{\zeta_{6}^2}=0,
\end{align*}
\begin{align*}
    C_{12} =& \frac{8 C_4 u_{cD}^2 B^2}{l_{cD}^2}=\mathcal{O}\lp\frac{\log^2(|\mathcal{S}|)\I_{\max}^{n-1}}{\mu_{\min}^2(1-\gamma)^2}.|\mathcal{S}|.\frac{\I_{\max}^{2n}}{(1-\gamma)^2}\rp=\mathcal{O}\lp\frac{|\mathcal{S}|\log^2(|\mathcal{S}|)\I_{\max}^{3n-1}}{\mu_{\min}^2(1-\gamma)^4}\rp,
\end{align*}
\begin{align*}
    C_{13} = &\frac{14 C_4 u_{cD}^2 m_2^2}{l_{cD}^2 (1-\rho^2)}=0.
\end{align*}
\begin{align*}
    C_{14}(\tau)=&\lp C_7 + C_{11}   + 0.5 C_3^2 C_9^2   + C_3 C_{10}  + 2 C_1 C_3 C_{10}  + 3 A_1 C_3  + C_{13} + C_8 \frac{u_{cD}^2}{l_{cD}^2} 2 m^2_2 \alpha^2\rp \tau^2=0,
\end{align*}
\begin{align*}
    C_{16} (\tau) =& \lp C_8\frac{u_{cD}^2 B^2}{l_{cD}^2} + \frac{1}{2} + C_{12}\rp  \tau^2=\mathcal{O}\lp \frac{|\mathcal{S}|\log(|\mathcal{S}|)\I_{\max}^{2n}}{(1-\gamma)^3\mu_{\min}}+1+\lp\frac{|\mathcal{S}|\log^2(|\mathcal{S}|)\I_{\max}^{3n-1}}{\mu_{\min}^2(1-\gamma)^4}\rp\rp\tau^2\\
    =&\mathcal{O}\lp\frac{|\mathcal{S}|\log^2(|\mathcal{S}|)\I_{\max}^{3n-1}}{(1-\gamma)^4\mu_{\min}^2}\rp\tau^2,
\end{align*}
\begin{align*}
    C_{17} =& (3 A_1 C_3 + 8  A_1^2 C_4  + C_5 + C_6)\\
    =&\mathcal{O}\left(0+\I_{\max}^{2n-2}.\lp\frac{\log^2(|\mathcal{S}|)\I_{\max}^{n-1}}{\mu_{\min}^2(1-\gamma)^2}\rp+\lp\frac{\log^2(|\mathcal{S}|)\I_{\max}^{2n-2}}{\mu_{\min}^3(1-\gamma)^3}\rp+\lp\frac{\log^2(|\mathcal{S}|)\I_{\max}^{n-1}}{\mu_{\min}^3(1-\gamma)^3}\rp\right)\\
    =&\mathcal{O}\left(\frac{\I_{\max}^{3n-3}\log^2(|\mathcal{S}|)}{\mu_{\min}^3(1-\gamma)^3}\right).
\end{align*}
Similar to $\|\cdot\|_D=\|\cdot\|_2$, we have $l_{c2}=\frac{1}{\sqrt{|\mathcal{S}|}}$ and $u_{c2}=1$.
\begin{align*}
    \ta_1 =& \frac{2 A_1^2 u_{c2}^2}{l_{c2}^2}= \mathcal{O}\lp \I_{\max}^{2n-2}|\mathcal{S}|\rp,
\end{align*}
\begin{align*}
    \ta_2 =& \frac{2 A_2^2 u_{c2}^2}{l_{c2}^2}= \mathcal{O}\lp \I_{\max}^{2n-2}|\mathcal{S}|\rp,
\end{align*}
\begin{align*}
    \tb = &\frac{u_{c2}^2}{l_{c2}^2} B^2 =\mathcal{O} \lp\frac{|\mathcal{S}|\I_{\max}^{2n}}{(1-\gamma)^2}\rp,
\end{align*}
\begin{align*}
    M_0=&\frac{1}{l_{cm}^2}\left(\frac{1}{C_1^2} \lp B + (A_2 + 1) \lp \|\btheta_0\|_c + \frac{B}{2 C_1} \rp \rp^2+\|\btheta_0\|_c^2\right)=\mathcal{O}\lp \lp \frac{\I_{\max}^{n}}{1-\gamma}+(\I_{\max}^{n-1}).\frac{\I_{\max}^{n}}{1-\gamma}\rp^2+1\rp\\
    =&\mathcal{O}\lp \frac{\I_{\max}^{4n-2}}{(1-\gamma)^2}\rp
\end{align*}

\begin{align*}
    \mathcal{C}_1=16u_{cm}^2 M_0(\log_\rho \frac{1}{e}+\frac{1}{\varphiz_2}) = \mathcal{O}\lp \frac{\I_{\max}^{4n-2}}{(1-\gamma)^2}.\frac{1}{\mu_{\min}(1-\gamma)}\rp = \mathcal{O}\lp \frac{\I_{\max}^{4n-2}}{(1-\gamma)^3\mu_{\min}}\rp
\end{align*}
\begin{align*}
    \mathcal{C}_2 =& \frac{8u_{cm}^2\lp C_8 + \frac{1}{2} + C_{12}\rp}{\varphiz_2} = \mathcal{O}\left(\frac{1}{\mu_{\min}(1-\gamma)}\left(\frac{\log(|\mathcal{S}|)}{\mu_{\min}(1-\gamma)} + 1+\frac{|\mathcal{S}|\log^2(|\mathcal{S}|)\I_{\max}^{3n-1}}{\mu_{\min}^2(1-\gamma)^4}\right)\right) \\
    =&\mathcal{O}\lp \frac{|\mathcal{S}|\log^2(|\mathcal{S}|)\I_{\max}^{3n-1}}{\mu_{\min}^2(1-\gamma)^4}\rp
\end{align*}
\begin{align*}
    \mathcal{C}_3 =& \frac{80\tb^2 C_{17}u_{cm}^2\lp 1 + \frac{4 (m_2+m_4) \rho}{B (1 - \rho)} \rp}{\varphiz_2} = \mathcal{O}\lp \frac{1}{\mu_{\min}(1-\gamma)}.\frac{|\mathcal{S}|^2 \I_{\max}^{4n}}{(1-\gamma)^4 }. \frac{ \I_{\max}^{3n-3} \log^2(|\mathcal{S}|)}{\mu_{\min}^3(1-\gamma)^3} \rp\\
    =&\mathcal{O}\lp\frac{\I_{\max}^{7n-3}|\mathcal{S}|^2\log^2(|\mathcal{S}|)}{\mu_{\min}^4(1-\gamma)^8}\rp,
\end{align*}
\begin{align*}
    \mathcal{C}_4 = \frac{320u_{cm}^2\tb^2 C_{17}m_3}{\varphiz_2B} = 0
\end{align*}
\begin{align*}
    \mathcal{C}_5 =& 8u_{cm}^2  \lp C_7 + C_{11}   + 0.5 C_3^2 C_9^2   + C_3 C_{10}  + 2 C_1 C_3 C_{10}  + 3 A_1 C_3  + C_{13}\rp/\varphiz_2 =0.
\end{align*}

Finally, for the sample complexity result, we simply employ Corollary \ref{lem:sample_complexity}.
\end{proof}

\subsection{Off-policy Function Approximation Setting}\label{sec:fed_td_learning_off_LFA_app}

\begin{proposition}\label{prop:TD_off_LFA}
Off-policy TD-learning with linear function approximation Algorithm \ref{alg:TD-learning_off_LFA} satisfies the following:
\begin{enumerate}
    \item $\bthetait = \bvit-\bv^{\pi}$ \label{item:TD_learning_10_LFA}
    \item $S_t=(S_t^1,\dots,S_t^\na)$ and $A_t=(A_t^1,\dots,A_t^\na)$\label{item:TD_learning_20_LFA}
    \item $\byit=(S_t^i,A_t^i,\dots,S_{t+n-1}^i,A_{t+n-1}^i,S_{t+n}^i)$ and $\mbf \byt=(S_t,A_t,\dots,S_{t+n-1},A_{t+n-1},S_{t+n})$\label{item:TD_learning_30_LFA}
    \item $\mu^{\pi^i_b}:$ Stationary distribution of the behaviour policy $\pi^i_b$. \label{item:TD_learning_40_LFA}
\end{enumerate}
Furthermore, choose some arbitrary positive constant $\beta>0$. The corresponding $\Gi(\bthetait,\byit)$ and $\mathbf{b}^i (\byit)$ in Algorithm \ref{alg:TD-learning_off_LFA} for Off-policy TD-learning with linear function approximation is as follows

\begin{enumerate}
    \item $\Gi(\bthetait,\byit)=\bthetait +\frac{1}{\beta}\phi(S_t^i)\sum_{l=t}^{t+n-1}\gamma^{l-t}\!\lb \Pi_{j=t}^{l} \I^{(i)}(S^i_j, A^i_j) \rb\left(\gamma \phi(S_{l\!+\!1}^i)^\top \bthetait\!-\!\phi(S_l^i)^\top \bthetait\right)$\label{item:TD_learning_50_LFA}
    \item $\mathbf{b}^i (\byit)=\frac{1}{\beta}\phi(S_t^i)\sum_{l=t}^{t+n-1}\gamma^{l-t}\lb \Pi_{j=t}^{l} \I^{(i)}(S^i_j, A^i_j) \rb\left(\mathcal R(S_l^i,A_l^i)+\gamma \phi(S_{l\!+\!1}^i)^\top \bv^{\pi}\!-\!\phi(S_l^i)^\top \bv^{\pi}\right)$.\label{item:TD_learning_60_LFA}
\end{enumerate}
Furthermore, the corresponding step size $\alpha$ in Algorithm \ref{alg:fed_stoch_app} is $\alpha\times\beta$. 
\end{proposition}

\begin{lemma}\label{lem:TD_learning_200}
Consider the federated off-policy TD-learning Algorithm \ref{alg:TD-learning_off_LFA} as a special case of the FeGSAM Algorithm \ref{alg:fed_stoch_app} (see Proposition \ref{prop:TD_off_LFA}). Suppose that the trajectory $\{S_t^i\}_{t=0,1,\dots}$ converges geometrically fast to its stationary distribution as follows $d_{TV}(P(S_t^i=\cdot|S_0^i)||\mu^i(\cdot))\leq \bar{m}\bar{\rho}^t$ for all $i=1,2,\dots, \na$. The corresponding  $\bG^i(\btheta)$ in Assumption \ref{ass:mixing} for the federated TD-learning is as follows
\begin{align}\label{eq:G_bar_linear_TD_off}
    \bG^i(\btheta) =& \btheta + \frac{1}{\beta}\Phi^\top\bmu^{\pi_b^i} \lp  (\gamma P^\pi)^n\Phi\btheta  - \Phi\btheta\rp,
\end{align}
where $\beta>0$ is an arbitrary constant introduced in Proposition \ref{prop:TD_off_LFA}.  Furthermore, for $t\geq n+1$, the bounds in \eqref{eq:b_ass} are satisfied for some constants $m_1,m_2, m_3,m_4 \geq 0$ and $\bar{\rho} = \rho$.
\end{lemma}

\begin{lemma}\label{lem:TD_learning_300}
Consider federated off-policy TD-learning \ref{alg:TD-learning_off_LFA} as a special case of FeGSAM (as specified in Proposition \ref{prop:TD_off_LFA}). By choosing $\beta$ large enough in the linear function \eqref{eq:G_bar_linear_TD_off}, and for large enough $n$, there exists a weighted 2-norm $\|\btheta\|_{\Lambda}=\sqrt{\btheta^\top\Lambda \btheta}$, such that $\bG^i(\btheta)$ is a contraction with respect to this norm, i.e., $\|\bG^i(\btheta_1)-\bG^i(\btheta_2)\|_{\Lambda}\leq \gamma_c \|\btheta_1-\btheta_2\|_{\Lambda}$ for $\gamma_c<1$.
\end{lemma}

\begin{lemma}\label{lem:TD_learning_400}
Consider the federated off-policy TD-learning Algorithm \ref{alg:TD-learning_off_LFA} as a special case of FeGSAM (as specified in Proposition \ref{prop:TD_off_LFA}). There exist some constants $A_1$, $A_2$, and $B$ such that the properties of Assumption \ref{ass:lipschitz} are satisfied.
\end{lemma}

\begin{lemma}\label{lem:TD_learning_500}
Consider the federated off-policy TD-learning Algorithm \ref{alg:TD-learning_off_LFA} as a special case of FeGSAM (as specified in Proposition \ref{prop:TD_off_LFA}). Assumption \ref{ass:noise_independence} holds for this algorithm.
\end{lemma}

\subsubsection{Proofs}\hfill

\begin{proof}[Proof of Proposition \ref{prop:TD_off_LFA}]
Items \ref{item:TD_learning_10_LFA}-\ref{item:TD_learning_40_LFA} are by definition. Subtracting $\bv^{\pi}$ from both sides of the update of the TD-learning, we have
\begin{align*}
    \underbrace{\bv_{t+1}^i - \bv^{\pi}}_{\btheta^i_{t+1}} =& \underbrace{\bvit - \bv^{\pi}}_{\bthetait} + \alpha \phi(S_t^i) \sum_{l=t}^{t+n-1} \gamma^{l-t} \lb \Pi_{j=t}^{l} \I^{(i)}(S^i_j, A^i_j) \rb\lp \mathcal{R} (S_l^i, A_l^i) \!+\!\gamma \phi (S_{l\!+\!1}^i)^\top \bvit\!-\!\phi(S_l^i)^\top \bvit \rp \\
    =& \bthetait + \alpha \phi(S_t^i) \sum_{l=t}^{t+n-1} \gamma^{l-t} \lb \Pi_{j=t}^{l} \I^{(i)}(S^i_j, A^i_j) \rb\Big( \mathcal{R} (S_l^i,A_l^i)\!+\!\gamma \phi(S_{l\!+\!1}^i)^\top (\underbrace{\bvit - \bv^{\pi}}_{\bthetait} +\bv^{\pi})\!\\
    &\quad\quad\quad\quad\quad\quad\quad\quad\quad\quad\quad\quad\quad\quad\quad\quad\quad\quad-\!\phi(S_l^i)^\top (\underbrace{\bvit - \bv^{\pi}}_{\bthetait} + \bv^{\pi}) \Big) \\
    =& \bthetait + \alpha\beta \Bigg(\underbrace{\bthetait + \frac{1}{\beta}\phi(S_t^i) \sum_{l=t}^{t+n-1} \gamma^{l-t}\lb \Pi_{j=t}^{l} \I^{(i)}(S^i_j, A^i_j) \rb \lp \!\gamma \phi(S_{l\!+\!1}^i)^\top \bthetait\!-\!\phi(S_l^i)^\top \bthetait \rp}_{\Gi(\bthetait, \byit)}\\
    &-\bthetait + \underbrace{ \frac{1}{\beta}\phi(S_t^i) \sum_{l=t}^{t+n-1} \gamma^{l-t}\lb \Pi_{j=t}^{l} \I^{(i)}(S^i_j, A^i_j) \rb \lp \mathcal{R}(S_l^i,A_l^i) + \gamma \phi(S_{l\!+\!1}^i)^\top \bv^{\pi} - \phi(S_l^i)^\top \bv^{\pi} \rp}_{\mathbf{b}^i (\byit)}\Bigg).
\end{align*}
which proves items \ref{item:TD_learning_10_LFA} and \ref{item:TD_learning_20_LFA}. Furthermore, for the synchronization part of TD-learning, we have
\begin{align*}
    \bv^i_t &\leftarrow  \frac{1}{\na} \sum_{j=1}^\na \bv^j_t\\
    \implies \underbrace{\bv^i_t-\bv^\pi}_{\bthetait} &\leftarrow  \frac{1}{\na} \sum_{j=1}^\na \underbrace{(\bv^j_t-\bv^\pi)}_{\bthetajt},
\end{align*}
which is equivalent to the synchronization step in FeGSAM Algorithm \ref{alg:fed_stoch_app}. Notice that here we used the fact that all agents have the same fixed point $\bv^\pi$. 
\end{proof}

\begin{proof}[Proof of Lemma \ref{lem:TD_learning_200}]
By definition, we have
\begin{align*}
    \Gi(\btheta,\byit)=\btheta + \frac{1}{\beta}\phi(S_t^i) \sum_{l=t}^{t+n-1} \gamma^{l-t}\lb \Pi_{j=t}^{l} \I^{(i)}(S^i_j, A^i_j) \rb \lp \!\gamma \phi(S_{l\!+\!1}^i)^\top \btheta\!-\!\phi(S_l^i)^\top \btheta \rp
\end{align*}

Taking expectation with respect to the stationary distribution, we have
\begin{align*}
    \bG^i(\btheta) =& \E_{S_t^i\sim \mu^{\pi_b^i}}\left[\btheta + \frac{1}{\beta}\phi(S_t^i) \sum_{l=t}^{t+n-1} \gamma^{l-t}\lb \Pi_{j=t}^{l} \I^{(i)}(S^i_j, A^i_j) \rb \lp \!\gamma \phi(S_{l\!+\!1}^i)^\top \btheta\!-\!\phi(S_l^i)^\top \btheta \rp\right]\\
    =&\E_{S_t^i\sim \mu^{\pi_b^i}}\left[\btheta + \frac{1}{\beta}\phi(S_t^i) \sum_{l=t}^{t+n-1} \gamma^{l-t}\E\left[\lb \Pi_{j=t}^{l} \I^{(i)}(S^i_j, A^i_j) \rb \lp \!\gamma \phi(S_{l\!+\!1}^i)^\top \btheta\!-\!\phi(S_l^i)^\top \btheta \rp\big|S_t^i\right]\right].
\end{align*}
Next, we look at the individual terms in the summation. We have
\begin{align*}
    &\E\left[\lb \Pi_{j=t}^{l} \I^{(i)}(S^i_j, A^i_j) \rb \lp \!\gamma \phi(S_{l\!+\!1}^i)^\top \btheta\!-\!\phi(S_l^i)^\top \btheta \rp\big|S_t^i\right]\\
    &=\E\left[\lb  \I^{(i)}(S^i_t, A^i_t)\I^{(i)}(S^i_{t+1}, A^i_{t+1})\dots \I^{(i)}(S^i_l, A^i_l) \rb \lp \!\gamma \phi(S_{l\!+\!1}^i)^\top \btheta\!-\!\phi(S_l^i)^\top \btheta \rp\big|S_t^i\right]\\
    &=\E\left[\E\left[\lb  \I^{(i)}(S^i_t, A^i_t)\I^{(i)}(S^i_{t+1}, A^i_{t+1})\dots \I^{(i)}(S^i_l, A^i_l) \rb \lp \!\gamma \phi(S_{l\!+\!1}^i)^\top \btheta\!-\!\phi(S_l^i)^\top \btheta \rp\big|S_t^i,A_t^i,\dots,S_l^i\right]|S_t^i\right]\\
    &=\E\left[\lb  \I^{(i)}(S^i_t, A^i_t)\I^{(i)}(S^i_{t+1}, A^i_{t+1})\dots  \rb \lp \!\gamma \E\left[\I^{(i)}(S^i_l, A^i_l)\phi(S_{l\!+\!1}^i)^\top \big|S_t^i,A_t^i,\dots,S_l^i\right]\btheta\!-\!\phi(S_l^i)^\top \btheta \rp|S_t^i\right]\\
    &=\E\left[\lb  \I^{(i)}(S^i_t, A^i_t)\I^{(i)}(S^i_{t+1}, A^i_{t+1})\dots  \rb \lp \!\gamma (P^{\pi}\Phi \btheta)(S_l^i)\!-\!(\Phi \btheta)(S_l^i) \rp|S_t^i\right]\\
    &= \lp \!\gamma ((P^{\pi})^{l-t+1}\Phi \btheta)(S_{t}^i)\!-\!((P^{\pi})^{l-t}\Phi \btheta)(S_{t}^i) \rp.
\end{align*}
Hence, we have
\begin{align*}
    \bG^i(\btheta) =& \E_{S_t^i\sim \mu^{\pi_b^i}}\left[\btheta + \frac{1}{\beta}\phi(S_t^i)[(\gamma P^{\pi})^{n}\Phi \btheta-\Phi\btheta](S_{t}^i)\right]\\
    =& \btheta + \frac{1}{\beta}\Phi^\top\bmu^{\pi_b^i} \lp \gamma^n (P^\pi)^n\Phi\btheta  - \Phi\btheta\rp.
\end{align*}
Furthermore, we have
\begin{align*}
    \mathbf{b}^i (\byit)=&\frac{1}{\beta}\phi(S_t^i) \sum_{l=t}^{t+n-1} \gamma^{l-t}\lb \Pi_{j=t}^{l} \I^{(i)}(S^i_j, A^i_j) \rb \lp \mathcal{R}(S_l^i,A_l^i) + \gamma \phi(S_{l\!+\!1}^i)^\top \bv^{\pi} - \phi(S_l^i)^\top \bv^{\pi} \rp\\
    =&\underbrace{\frac{1}{\beta}\phi(S_t^i) \sum_{l=t}^{t+n-1} \gamma^{l-t}\lb \Pi_{j=t}^{l} \I^{(i)}(S^i_j, A^i_j) \rb \lp \mathcal{R}(S_l^i,A_l^i) + \gamma \phi(S_{l\!+\!1}^i)^\top \bv^{i,\pi} - \phi(S_l^i)^\top \bv^{i,\pi} \rp}_{T_1(\by_t^i)}\\
    &+\underbrace{\frac{1}{\beta}\phi(S_t^i) \sum_{l=t}^{t+n-1} \gamma^{l-t}\lb \Pi_{j=t}^{l} \I^{(i)}(S^i_j, A^i_j) \rb \lp  \gamma \phi(S_{l\!+\!1}^i)^\top [\bv^{\pi}-\bv^{i,\pi}] - \phi(S_l^i)^\top [\bv^{\pi}-\bv^{i,\pi}] \rp}_{T_2},
\end{align*}
where $\bv^{i,\pi}$ is the unique solution satisfying $\Phi \bv^{i,\pi} = \Pi_\Phi^{\pi_b^i}((\mathcal{T}^\pi)^n \Phi \bv^{i,\pi})$. Hence, we have $\E_{\by\sim\mu^{\pi_b^i}}T_1  (\by) = \mbf 0$. It can be seen that $\|T_1(\by)\|_c\leq c_{b}^{1TD}$ for some constant $c_{b}^{1TD}$. In addition, we have $\|T_2\|_c\leq c_{b}^{2TD}\|\bv^{\pi}-\bv^{i,\pi}\|_c$, where $c_{b}^{2TD}$ is some constant. Furthermore, since $\bv^\pi$ is the fixed point of \eqref{eq:BEO_FA_off}, we have $\E_{\byt\sim\mu_b}\frac{1}{\na}\sumik \mathbf{b}^i (\byit) = \mbf 0$.

We have
\begin{align*}
    \|\bG^i(\btheta)-\E[\Gi (\btheta,\by_t^i)]\|_c= &\left\|\E_{\by_t^i\sim \mu^{\pi_b^i}}[\Gi (\btheta,\by_t^i)]-\E[\Gi (\btheta,\by_t^i)\right\|_c\\
    \leq & \sum_{y_t^i}\left|\mu^{\pi_b^i}(y_t^i)-P(\byit=y_t^i|\by_0^i)\right|.A_2\left\|\btheta\right\|_c\\
    \leq& (2\bar{m}\bar{\rho}^{-n})\bar{\rho}^{t} A_2\left\|\btheta\right\|_c.\tag{Similar to Lemma \ref{lem:TD_learning_20}}
\end{align*}
Moreover, we have
\begin{align*}
    \left\|\E[\mathbf{b}^i  (\by_t^i)]\right\|_c =&\left\|\E[T_1(\by_t^i)+T_2]\right\|_c\\
    =&\left\|\E[T_1(\by_t^i)+T_2]-\E_{\byit\sim\mu^{\pi_b^i}}[T_1(\byit)]\right\|_c\\
    \leq & c_b^{2TD}\|\bv^\pi-\bv^{i,\pi}\| + \left\|\E[T_1(\byit)]-\E_{\byit\sim\mu^{\pi_b^i}}[T_1(\byit)]\right\|_c\\
    \leq &c_b^{2TD}\|\bv^\pi-\bv^{i,\pi}\| + 2c_b^{1TD}\bar{m}\bar{\rho}^t.
\end{align*}
In addition, we have
\begin{align*}
    \left\|\E\left[\frac{1}{\na}\sumik\mathbf{b}^i  (\by_t^i)\right]\right\|_c = & \left\|\E\left[\frac{1}{\na}\sumik\mathbf{b}^i  (\by_t^i)\right]-\E_{\by_t\sim\mu^{\pi_b}}\left[\frac{1}{\na}\sumik\mathbf{b}^i  (\by_t^i)\right]\right\|_c\\
    \leq & m_4\rho^t
\end{align*}
for some constant $m_4$.
\end{proof}

\begin{proof}[Proof of Lemma \ref{lem:TD_learning_300}]
The proof follows from Lemmas \ref{lem:TD_learning_30} and \cite[Proposition 3.1]{NACLFA_arxiv}.
\end{proof}

\begin{proof}[Proof of Lemma \ref{lem:TD_learning_400}]
The proof follows similar to Lemma \ref{lem:TD_learning_40}. 
\end{proof}

\begin{proof}[Proof of Lemma \ref{lem:TD_learning_500}]
    The proof follows similar to Lemma \ref{lem:TD_learning_50}
\end{proof}
\section{Federated \texorpdfstring{$Q$}--learning}\label{sec:fed_Q_learning_app}
In this section, we verify that the federated $Q$-learning algorithm \ref{alg:Q-learning} satisfies the properties of the FeGSAM Algorithm \ref{alg:fed_stoch_app}. In the following, $Q^*$ is the solution to the Bellman optimality equation \eqref{eq:bOE}
{\small
\begin{align}\label{eq:bOE}
    Q^*(s,a) = \mathcal{R}(s,a)+\gamma \E_{S'\sim \mathcal{P}(\cdot|s,a)}\left[\max_{a'} Q^*(S',a')\right]. 
\end{align}
}
Note that $Q^*$ is independent of the agent's sampling policy. Furthermore, $\|\cdot\|_c=\|\cdot\|_\infty$.
 \begin{proposition}\label{prop:Q_learning}
        Federated $Q$-learning algorithm \ref{alg:Q-learning} is equivalent to the FeGSAM Algorithm \ref{alg:fed_stoch_app} with the following parameters.
        \begin{enumerate}
            \item $\bthetait = Q^i_t-Q^{*}$ \label{item:Q_learning_1}
            \item $S_t=(S_t^1,\dots,S_t^\na)$ and $A_t=(A_t^1,\dots,A_t^\na)$ \label{item:Q_learning_2}
            \item $\byit=(S_t^i,A_t^i,S_{t+1}^i,A_{t+1}^i)$ and $\mbf \byt=(S_t,A_t,S_{t+1},A_{t+1})$\label{item:Q_learning_3}
            \item $\mu^i:$ Stationary distribution of the sampling policy of the i'th agent. \label{item:Q_learning_4}
            \item $\Gi(\bthetait, \byit)_{(s,a)}=\bthetait(s,a)   $ \\ $ +\mathbbm{1}_{\{S_t^i=s, A_t^i=a\}}\times\big[\gamma \max_{a'} \left( \bthetait + Q^{*} (S_{t+1}^i,a')\right) - \bthetait(S_t^i, A_t^i) - \gamma \max_{a'}Q^{*}(S_{t+1}^i,a')\big]$ \label{item:Q_learning_5}
            \item $\mathbf{b}^i (\byit)_{(s,a)}=\mathbbm{1}_{\{S_t^i=s,A_t^i=a\}} \left[\mathcal{R}(S_t^i,A_t^i)+\gamma\max_{a'}Q^*(S_{t+1}^i,a')-Q^*(S_t^i,A_t^i)\right]$ \label{item:Q_learning_6}
        \end{enumerate}
        where $\mathbbm{1}_A$ is the indicator function corresponding to set $A$, such that $\mathbbm{1}_A = 1$ is $A$ is true, and $0$ otherwise.
    \end{proposition}
    
\begin{lemma}\label{lem:Q_learning_2}
Consider the federated $Q$-learning Algorithm \ref{alg:Q-learning} as a special case of FeGSAM (as specified in Proposition \ref{prop:Q_learning}). Suppose that the trajectory $\{S_t^i,A_t^i\}_{t=0,1,\dots}$ converges geometrically fast to its stationary distribution as follows $d_{TV}(P(S_t^i=\cdot,A_t^i=\cdot|S_0^i,A_0^i)||\mu^i(\cdot,\cdot))\leq \bar{m}\bar{\rho}^t$ for all $i=1,2,\dots, \na$. The corresponding  $\bG^i(\btheta)$ in Assumption \ref{ass:mixing} for the federated $Q$-learning is as follows

\begin{align*}
    \bG^i(\btheta)_{(s,a)} =& \btheta(s,a)  +\mu^i(s,a)\times\E_{S'\sim \mathcal{P}(\cdot|s,a)}\left[\gamma \max_{a'} \left( \btheta + Q^{*} (S',a')\right) - \btheta(s, a) - \gamma \max_{a'}Q^{*}(S',a')\right].
\end{align*}
Furthermore, we have $m_1=2A_2\bar{m}$, where $A_2$ is specified in Lemma \ref{lem:Q_learning_4}, $m_2=m_3=m_4=0$, and $\bar{\rho} =\rho $. 
\end{lemma}

\begin{lemma}\label{lem:Q_learning_3}
Consider federated $Q$-learning as a special case of FeGSAM (as specified in Proposition \ref{prop:Q_learning}). The corresponding contraction factor $\gamma_c$ in Assumption \ref{ass:contraction} for this algorithm is $\gamma_c=\left(1-(1-\gamma)\mu_{\min}\right)$, where $\mu_{\min} = \min_{s,a,i}\mu^i(s,a)$
\end{lemma}

\begin{lemma}\label{lem:Q_learning_4}
Consider federated $Q$-learning as a special case of FeGSAM (as specified in Proposition \ref{prop:Q_learning}). The constants $A_1$, $A_2$, and $B$ in Assumption \ref{ass:contraction} are as follows:
$A_1 = A_2= 2$ and $B=\frac{2}{1-\gamma}$.
\end{lemma}

\begin{lemma}\label{lem:Q_learning_5}
Consider federated $Q$-learning as a special case of FeGSAM (as specified in Proposition \ref{prop:Q_learning}). Assumption \ref{ass:noise_independence} holds for this algorithm.
\end{lemma}
\subsection{Proofs}
    
\begin{proof}[Proof of Proposition \ref{prop:Q_learning}] Items \ref{item:Q_learning_1}-\ref{item:Q_learning_4} are by definition. Furthermore, by the update of the $Q$-learning, and subtracting $Q^*$ from both sides, we have
\begin{align*}
&\underbrace{\Q^i_{t+1}(s,a)-Q^*(s,a)}_{\btheta^i_{t+1}(s,a)} = \underbrace{\Q^i_t(s,a)-Q^*(s,a)}_{\bthetait(s,a) } \\
&+ \alpha \mathbbm{1}_{\{(s,a)=(S_t^i,A_t^i)\}}\lp \mathcal R(S_t^i,A_t^i) +\gamma \max_a \Q^i_t(S_{t+1}^i,a)-\Q^i_t(S_{t}^i,A_{t}^i) \rp\\
=&\bthetait(s,a) \\
&+ \alpha \mathbbm{1}_{\{(s,a)=(S_t^i,A_t^i)\}}\bigg( \mathcal R(S_t^i,A_t^i) +\gamma \max_a\left[ \underbrace{\Q^i_t(S_{t+1}^i,a)-Q^*(S_{t+1}^i,a)}_{\btheta^i_{t}(S_{t+1}^i,a)}+Q^*(S_{t+1}^i,a)\right]\\
&-\left(\underbrace{\Q^i_t(S_{t}^i,A_{t}^i)-Q^*(S_{t}^i,A_{t}^i)}_{\btheta^i_{t}(S_{t}^i,A_{t}^i)}+Q^*(S_{t}^i,A_{t}^i)\right) \bigg)\tag{addition and subtraction}\\
=& \bthetait(s,a)  \\
&+\alpha\bigg(\underbrace{\bthetait(s,a)+\mathbbm{1}_{\{S_t^i=s, A_t^i=a\}}\times\big[\gamma \max_{a'} \left( \bthetait + Q^{*} (S_{t+1}^i,a')\right) - \bthetait(S_t^i, A_t^i) - \gamma \max_{a'}Q^{*}(S_{t+1}^i,a')\big]}_{\Gi(\bthetait, \byit)_{(s,a)}}-\bthetait(s,a)\\
&\quad\quad +\underbrace{\mathbbm{1}_{\{S_t^i=s, A_t^i=a\}}\left(\mathcal{R}(S_t^i,A_t^i)+\gamma\max_{a'}Q^*(S_{t+1}^i,a')-Q^*(S_t^i,A_t^i)\right)}_{\mathbf{b}^i (\byit)_{(s,a)}}\bigg),
\end{align*}
which proves items \ref{item:Q_learning_5} and \ref{item:Q_learning_6}. Furthermore, for the synchronization part of $Q$-learning, we have
\begin{align*}
    \Q^i_t &\leftarrow  \frac{1}{\na} \sum_{j=1}^\na \Q^j_t\\
    \implies \underbrace{\Q^i_t-Q^*}_{\bthetait} &\leftarrow  \frac{1}{\na} \sum_{j=1}^\na \underbrace{(\Q^j_t-Q^*)}_{\bthetajt},
\end{align*}
which is equivalent to the synchronization step in FeGSAM Algorithm \ref{alg:fed_stoch_app}. Notice that here we used the fact that all agents have the same fixed point $Q^*$.
\end{proof}

\begin{proof}[Proof of Lemma \ref{lem:Q_learning_2}]
$\bG(\btheta)_{(s,a)}$ can be found simply by taking the expectation of $\Gi(\bthetait, \byit)_{(s,a)}$, defined in Proposition \ref{prop:Q_learning}, with respect to the stationary distribution $\mu^i$. Furthermore, we have
\begin{align*}
    &\|\bG^i(\btheta)-\E[\Gi (\btheta,\by_t^i)]\|_c\\
    = &\left\|\E_{\by_t^i\sim \mu^i}[\Gi (\btheta,\by_t^i)]-\E[\Gi (\btheta,\by_t^i)\right\|_c\\
    =&\left\|\sum_{s,a,s',a'}\left(\mu^i(s,a,s',a')-P(S_t^i=s,A_t^i=a,S_{t+1}^i=s',A_{t+1}^i=a'|S_0^i,A_0^i)\right)\Gi (\btheta,\by_t^i)\right\|_c\\
    =&\sum_{s,a,s',a'}\left|\mu^i(s,a,s',a')-P(S_t^i=s,A_t^i=a,S_{t+1}^i=s',A_{t+1}^i=a'|S_0^i,A_0^i)\right|.\left\|\Gi (\btheta,\by_t^i)\right\|_c\tag{$\|ax\|_c=|a|\|x\|_c$}\\
    \leq &\sum_{s,a,s',a'}\left|\mu^i(s,a,s',a')-P(S_t^i=s,A_t^i=a,S_{t+1}^i=s',A_{t+1}^i=a'|S_0^i,A_0^i)\right|.A_2\|\btheta\|_c\tag{Assumption \ref{lem:Q_learning_4}}\\
    = &\sum_{s,a,s',a'}\bigg|\mu^i(s,a)\mathcal{P}(s'|s,a)\pi^i(a'|s')-P(S_t^i=s,A_t^i=a|S_0^i,A_0^i)\mathcal{P}(s'|s,a)\pi^i(a'|s')\bigg|.A_2\|\btheta\|_c\tag{definition of transition probability}\\
    = &\sum_{s,a}\bigg|\mu^i(s,a)-P(S_t^i=s,A_t^i=a|S_0^i,A_0^i)\bigg|.A_2\|\btheta\|_c\\
    = &2d_{TV}(\mu^i(\cdot,\cdot),P(S_t^i=\cdot,A_t^i=\cdot|S_0^i,A_0^i)).A_2\|\btheta\|_c\\
    \leq & 2A_2\|\btheta\|_c \bar{m}\bar{\rho}^t.
\end{align*}
In addition, we have
\begin{align*}
    \|\E[\mathbf{b}^i  (\by_t^i)]\|_c = & \max_{s,a}\left|P(S_t^i=s,A_t^i=a|S_0^i,A_0^i)\left(\mathcal{R}(s,a)+\gamma\E_{S'\sim\mathcal{P}(\cdot|s,a)}[\max_{a'}Q^*(S',a')]-Q^*(s,a)\right)\right|\\
    =& 0.\tag{Bellman optimality equation \eqref{eq:bOE}} 
\end{align*}

\end{proof}

\begin{proof}[Proof of Lemma \ref{lem:Q_learning_3}]
\begin{align*}
    &\|\bG^i (\btheta_1) - \bG^i(\btheta_2) \|_c =\max_{s,a}|\bG^i (\btheta_1)_{(s,a)} - \bG^i(\btheta_2)_{(s,a)}|\\
    = & \max_{s,a}\Bigg|\btheta_1(s,a) +\mu^i(s,a)\times\E_{S'\sim \mathcal{P}(\cdot|s,a)}\left[\gamma \max_{a'} \left( \btheta_1 + Q^{*} (S',a')\right) - \btheta_1(s, a) - \gamma \max_{a'}Q^{*}(S',a')\right]\\
    &\quad-\left(\btheta_2(s,a)  +\mu^i(s,a)\times\E_{S'\sim \mathcal{P}(\cdot|s,a)}\left[\gamma \max_{a'} \left( \btheta_2 + Q^{*} (S',a')\right) - \btheta_2(s, a) - \gamma \max_{a'}Q^{*}(S',a')\right]\right)\Bigg|\\
    =& \max_{s,a}\bigg|(1-\mu^i(s,a))(\btheta_1(s,a)-\btheta_2(s,a)) \\
    &\quad\quad+ \gamma\mu^i(s,a)\E_{S'\sim \mathcal{P}(\cdot|s,a)}\left[\max_{a'} \left( \btheta_1 + Q^{*} (S',a')\right) - \max_{a'} \left( \btheta_2 + Q^{*} (S',a')\right)\right]\bigg|\\
    \leq& \max_{s,a}\Bigg[\bigg|(1-\mu^i(s,a))(\btheta_1(s,a)-\btheta_2(s,a))\bigg| \\
    &+ \gamma\bigg|\mu^i(s,a)\E_{S'\sim \mathcal{P}(\cdot|s,a)}\left[\max_{a'} \left( \btheta_1 + Q^{*} (S',a')\right) - \max_{a'} \left( \btheta_2 + Q^{*} (S',a')\right)\right]\bigg|\Bigg]\tag{triangle inequality}\\
    \leq& \max_{s,a}\Bigg[(1-\mu^i(s,a))\left\|\btheta_1-\btheta_2\right\|_\infty \tag{definition of $\|\cdot\|_\infty$}\\
    &+ \gamma\mu^i(s,a)\left[\bigg|\E_{S'\sim \mathcal{P}(\cdot|s,a)}\left[\max_{a'} \left( \btheta_1 + Q^{*} (S',a')\right) - \max_{a'} \left( \btheta_2 + Q^{*} (S',a')\right)\right]\bigg|\right]\Bigg]\tag{$\mu^i(s,a)\geq 0$}\\
    \leq& \max_{s,a}\Bigg[(1-\mu^i(s,a))\left\|\btheta_1-\btheta_2\right\|_\infty \\
    &+ \gamma\mu^i(s,a)\left[\E_{S'\sim \mathcal{P}(\cdot|s,a)}\bigg|\max_{a'} \left( \btheta_1 + Q^{*} (S',a')\right) - \max_{a'} \left( \btheta_2 + Q^{*} (S',a')\right)\bigg|\right]\Bigg]\tag{Jensen's inequality}
\end{align*}
Next, we note that for any functions $f(\cdot)$ and $g(\cdot)$, we have
\begin{align}\label{eq:max_ineq}
    \left|(\max_x f(x)) - (\max_x g(x))\right|\leq \max_x \left|f(x)-g(x)\right|.
\end{align}
The reason is as follows. We have $\max_x f(x) = \max_x f(x)-g(x)+g(x) \leq (\max_x f(x)-g(x)) + (\max_x g(x))$. Hence, $(\max_x f(x)) - (\max_x g(x))\leq \max_x f(x)-g(x)\leq \max_x \left|f(x)-g(x)\right|$. Now suppose $\max_x f(x)\geq \max_x g(x)$. Then we can apply absolute value to the left hand side of the inequality, and we get the bound. Using a similar argument for the case $\max_x f(x)\leq \max_x g(x)$, we get the bound in \eqref{eq:max_ineq}. Hence, we have $\left|\max_{a'} \left( \btheta_1 + Q^{*} (S',a')\right) - \max_{a'} \left( \btheta_2 + Q^{*} (S',a')\right)\right|\leq \max_{a'}|\btheta_1(S',a')-\btheta_2(S',a')|\leq \|\btheta_1-\btheta_2\|_\infty$. As a result, we have
\begin{align*}
    \|\bG^i (\btheta_1) - \bG^i(\btheta_2) \|_c\leq& \max_{s,a}\left[(1-\mu^i(s,a))\left\|\btheta_1-\btheta_2\right\|_\infty + \gamma\mu^i(s,a)\left[\|\btheta_1-\btheta_2\|_\infty\right]\right]\\
    =& \max_{s,a} \left(1-(1-\gamma)\mu^i(s,a)\right)\left\|\btheta_1-\btheta_2\right\|_\infty\\
    \leq& \left(1-(1-\gamma)\mu^i_{\min}\right)\left\|\btheta_1-\btheta_2\right\|_\infty\\
    \leq& \left(1-(1-\gamma)\mu_{\min}\right)\left\|\btheta_1-\btheta_2\right\|_\infty\\
    =& \left(1-(1-\gamma)\mu_{\min}\right)\left\|\btheta_1-\btheta_2\right\|_c.
\end{align*}
\end{proof}

\begin{proof}[Proof of Lemma \ref{lem:Q_learning_4}]
First, for $A_1$, we have
\begin{align*}
    &\| \Gi (\btheta_1, \by) - \Gi (\btheta_2, \by) \|_c\\
    =& \max_{s,a}\bigg|\btheta_1(s,a)   +\mathbbm{1}_{\{S=s, A=a\}}\times\big[\gamma \max_{a'} \left( \btheta_1 + Q^{*} (S',a')\right) - \btheta_1(S, A) - \gamma \max_{a'}Q^{*}(S',a')\big]\\
    &\quad\quad\quad-\left(\btheta_2(s,a)   +\mathbbm{1}_{\{S=s, A=a\}}\times\big[\gamma \max_{a'} \left( \btheta_2 + Q^{*} (S',a')\right) - \btheta_2(S, A) - \gamma \max_{a'}Q^{*}(S',a')\big]\right)\bigg|\\
    \leq& \max_{s,a}\bigg[\left|(1-\mathbbm{1}_{\{S=s, A=a\}})(\btheta_1(s,a)-\btheta_2(s,a))\right| \\
    &\quad\quad+ \gamma\mathbbm{1}_{\{S=s, A=a\}}\left|\max_{a'} \left( \btheta_1 + Q^{*} (S',a')\right)-\max_{a'} \left( \btheta_2 + Q^{*} (S',a')\right)\right|\bigg]\tag{triangle inequality}\\
    \leq& \max_{s,a}\bigg[\left\|\btheta_1-\btheta_2\right\|_\infty + \gamma\mathbbm{1}_{\{S=s, A=a\}}\left|\max_{a'} \left( \btheta_1 + Q^{*} (S',a')\right)-\max_{a'} \left( \btheta_2 + Q^{*} (S',a')\right)\right|\bigg] \tag{definition of $\|\cdot\|_\infty$}\\
    \leq& \max_{s,a}\bigg[\left\|\btheta_1-\btheta_2\right\|_\infty + \gamma\mathbbm{1}_{\{S=s, A=a\}}\|\btheta_1-\btheta_2\|_\infty\bigg]\tag{By \eqref{eq:max_ineq}}\\
    \leq& 2 \|\btheta_1-\btheta_2\|_\infty\\
    = & 2 \|\btheta_1-\btheta_2\|_c
\end{align*}
Second, for $A_2$, we have
\begin{align*}
    \| \Gi (\btheta, \by) \|_c =& \max_{s,a}\left|\btheta(s,a)   +\mathbbm{1}_{\{S=s, A=a\}}\times\big[\gamma \max_{a'} \left( \btheta + Q^{*} (S',a')\right) - \btheta(S, A) - \gamma \max_{a'}Q^{*}(S',a')\big]\right|\\
    \leq &\max_{s,a}\left[ \left|(1-\mathbbm{1}_{\{S=s, A=a\}})\btheta(a,s)\right|+\gamma\mathbbm{1}_{\{S=s, A=a\}}\left|\max_{a'} \left( \btheta + Q^{*} (S',a')\right)  -  \max_{a'}Q^{*}(S',a')\right|\right]\tag{triangle inequality}\\
    \leq& \max_{s,a}\left[ \left\|\btheta\right\|_\infty+\gamma\mathbbm{1}_{\{S=s, A=a\}}\left|\max_{a'} \left( \btheta + Q^{*} (S',a')\right)  -  \max_{a'}Q^{*}(S',a')\right|\right]\tag{definition of $\|\cdot\|_\infty$}\\
    \leq & \max_{s,a}\left[ \left\|\btheta\right\|_\infty+\gamma\max_{a'}\left|  \btheta  (S',a')\right|\right]\tag{By \eqref{eq:max_ineq}}\\
    \leq &2\left\|\btheta\right\|_\infty\\
    = &2\left\|\btheta\right\|_c.
\end{align*}
Lastly, for $B$, we have
\begin{align*}
    \|b^i(\by^i)\|_c =& \max_{s,a}\left|\mathbbm{1}_{\{S=s,A=a\}}\left[\mathcal{R}(S,A)+\gamma\max_{a'}Q^*(S',a')-Q^*(S,A)\right]\right|\\
    \leq &\max_{s,a}\mathbbm{1}_{\{S=s,A=a\}}\left[\left|\mathcal{R}(S,A)\right|+\gamma\max_{a'}\left|Q^*(S',a')\right|+\left|Q^*(S,A)\right|\right]\tag{triangle inequality}\\
    \leq & \max_{s,a}\left[1+\gamma\frac{1}{1-\gamma}+\frac{1}{1-\gamma}\right]\\
    =&\frac{2}{1-\gamma}.
\end{align*}
\end{proof}

\begin{proof}[Proof of Lemma \ref{lem:Q_learning_5}]
The proof follows similar to Lemma \ref{lem:TD_learning_50}.
\end{proof}

\begin{proof}[Proof of Theorem \ref{thm:fed_Q}] By Proposition \ref{prop:Q_learning} and Lemmas \ref{lem:Q_learning_2}, \ref{lem:Q_learning_3}, \ref{lem:Q_learning_4}, and \ref{lem:Q_learning_5}, it is clear that the federated $Q$ Algorithm \ref{alg:Q-learning} satisfies all the Assumptions \ref{ass:mixing}, \ref{ass:contraction}, \ref{ass:lipschitz},  and \ref{ass:noise_independence} of the FeGSAM Algorithm \ref{alg:fed_stoch_app}. Furthermore, by the proof of Theorem \ref{thm:main}, we have $w_t=(1-\frac{\alpha\varphiz_2}{2})^{-t}$, and the constant $c$ in the sampling distribution $q_T^c$ in Algorithm \ref{alg:TD-learning} is $c=(1-\frac{\alpha\varphiz_2}{2})^{-1}$. In equation \eqref{eq:w_t_Q} we evaluate the exact value of $w_t$.

Furthermore, by choosing the step size $\alpha$ small enough, we can satisfy the requirements in \eqref{eq:alpha_const_1}, \eqref{eq:alpha_const_2}, \eqref{eq:alpha_const_3}, 
\eqref{eq:alpha_const_4}. By choosing $\sync$ large enough, we can satisfy $\sync>\tau$, and by choosing $T$ large enough, we can satisfy $T>K+\tau$. Hence, the result of Theorem \ref{thm:main} holds for this algorithm. 

Next, we derive the constants involved in Theorem \ref{thm:main} step by step. In this analysis, we only consider the terms that involve $|\mathcal{S}|$, $|\mathcal{A}|$, $\frac{1}{1-\gamma}$, $\I_{\max}$, and $\mu_{\min}$. Since $\|\cdot\|_c=\|\cdot\|_\infty$, we choose $g(\cdot)=\frac{1}{2}\|\cdot\|_p^2$, that is, the $p$-norm with $p=2\log(|\mathcal{S}|)$. It is known that $g(\cdot)$ is $(p-1)$ smooth with respect to the $\|\cdot\|_p$ norm \cite{beck2017first}, and hence $L=\Theta(\log(|\mathcal{S}|))$. Hence, we have $l_{cs}=|\mathcal{S}|^{-1/p}=\frac{1}{\sqrt{e}}=\Theta(1)$ and $u_{cs}=1$. Therefore, we have $\varphiz_1 = \frac{1+\psi u_{cs}^2}{1+\psi \ell_{cs}^2}=\frac{1+\psi}{1+\frac{\psi}{\sqrt{e}}}\leq1+\psi$. By choosing $\psi=(\frac{1+\gamma_c}{2\gamma_c})^2-1=\frac{1+2\gamma_c-3\gamma_c^2}{4\gamma_c^2} \geq (1-\gamma_c)=\mu_{\min}(1-\gamma)=\Omega( \mu_{\min}(1-\gamma))$, which is $\psi= \mathcal{O}(1)$, we have $\varphiz_1 = \frac{1+\psi}{1+\frac{\psi}{\sqrt{e}}} =\sqrt{e} \frac{(\frac{1+\gamma_c}{2\gamma_c})^2}{\sqrt{e}+(\frac{1+\gamma_c}{2\gamma_c})^2-1}=\mathcal{O}(1)$, and
\begin{align*}
\varphiz_2&=1-\gamma_c\sqrt{\frac{1+\psi}{1+\frac{\psi}{\sqrt{e}}}}=1-\gamma_c\sqrt{\frac{(\frac{1+\gamma_c}{2\gamma_c})^2}{1+\frac{(\frac{1+\gamma_c}{2\gamma_c})^2-1}{\sqrt{e}}}}=1-\frac{0.5(1+\gamma_c)e^{1/4}}{\sqrt{\sqrt{e}-1+\lp\frac{1+\gamma_c}{2\gamma_c}\rp^2}}\\
&=1-\frac{0.5e^{1/4}(2-\mu_{\min}(1-\gamma))}{\sqrt{\sqrt{e}-1+\left(\frac{2-\mu_{\min}(1-\gamma)}{2-2\mu_{\min}(1-\gamma)}\right)^2}}> 1-\gamma_c\sqrt{1+\psi} \\
&= 1-\gamma_c\frac{1+\gamma_c}{2\gamma_c}=\frac{1-\gamma_c}{2}=0.5\mu_{\min}(1-\gamma) =\Omega(\mu_{\min}(1-\gamma))
\end{align*}

\[
\varphiz_3 =\frac{L (1+\psi u_{cs}^2)}{\psi\ell_{cs}^2} =\mathcal{O} \lp \frac{\log (|\mathcal{S}|)(1+\psi)}{\psi}\rp\leq\mathcal{O}\lp\frac{\log (|\mathcal{S}|)}{1-\gamma_c}\rp=\mathcal{O}\lp\frac{\log (|\mathcal{S}|)}{\mu_{\min}(1-\gamma)}\rp.
\]
Using $\varphiz_2$, we have
\begin{align}\label{eq:w_t_Q}
w_t=\left(1-\frac{\alpha\varphiz_2}{2}\right)^{-t}=\lp 1 - \alpha/2 + \frac{0.25\alpha e^{1/4}(2-\mu_{\min}(1-\gamma))}{\sqrt{\sqrt{e}-1+\left(\frac{2-\mu_{\min}(1-\gamma)}{2-2\mu_{\min}(1-\gamma)}\right)^2}}\rp^{-t}.
\end{align}

Further, we have
\[
l_{cm} = (1+\psi l_{cs}^2)^{1/2} =\Theta(1)
\]
\[
u_{cm} = (1+\psi u_{cs}^2)^{1/2} =\Theta(1)
\]

Since TV-divergence is upper bounded by $1$, we have $\bar{m}=\mathcal{O}(1)$. By Lemma \ref{lem:Q_learning_4}, we have
\[
A_1=A_2=2=\mathcal{O}(1)
\]
and $A_1=A_2=\Omega(1)$,
\[
B=\frac{2}{1-\gamma}=\mathcal{O}\lp\frac{1}{1-\gamma}\rp.
\]
Hence $m_1=2A_2\bar{m}=\mathcal{O}(1)$. Also, we have $m_2=0$.

We choose the $D$-norm in Lemma \ref{lem:after_expectation} as the $2$-norm $\|\cdot\|_2$. Hence, by primary norm equivalence, we have $l_{cD}=\frac{1}{\sqrt{|\mathcal{S}|}}$, $u_{cD}=1$, and hence $\frac{u_{cD}}{l_{cD}}=\sqrt{|\mathcal{S}|}$. The rest of the proof is similar to the proof of Theorem \ref{thm:fed_TD_off} where $\I_{\max}$ is substituted with $1$. The sample complexity can also be derived using Corollary \ref{lem:sample_complexity}. 
\end{proof}
\end{appendix}

\end{document}